%% file: main.tex
\pdfoutput=1
\documentclass[10pt,twocolumn]{article}

\usepackage[letterpaper,textwidth=6.75in,textheight=9.25in,columnsep=0.25in,centering]{geometry}

\usepackage[round]{natbib}

\usepackage{amsmath,amssymb,amsthm,mathtools}
\usepackage{graphicx}
\usepackage{booktabs}
\usepackage{microtype}
\usepackage{tikz}
\usetikzlibrary{arrows.meta,calc,fit,positioning}
\usepackage{pgfplots}
\pgfplotsset{compat=1.18}
\usepgfplotslibrary{fillbetween}
\usepackage[hidelinks]{hyperref}
\hypersetup{pdftitle={An Accuracy\textendash Information Tradeoff for Loss-Difference Conditional Mutual Information},pdfauthor={Hazar Yueksel}}

\input{macros}

\begin{document}

\twocolumn[{%
\centering
{\LARGE\bfseries An Accuracy--Information Tradeoff for\\[2pt] Loss-Difference Conditional Mutual Information\par}
\vspace{1.2em}
{\large Hazar Yueksel\par}
\vspace{0.3em}
\texttt{hazar@hazaryueksel.com}\par
\vspace{1.8em}
}]
{\renewcommand{\thefootnote}{}\footnotetext{The supplementary archive referred to in the text (the Lean~4
formalization, the code and the data) is available as the ancillary files of this
arXiv submission.}}

\begin{abstract}
Loss-difference conditional mutual information (ld-CMI) uses the smallest of the standard
observations in the supersample hierarchy of generalization bounds: it measures what a learner's loss differences
reveal about which candidate of each pair it was trained on. Accuracy is known to force
information into the model; data processing does not carry such lower bounds to losses. We show,
by bounding three moments of the loss differences, that accuracy also forces ld-CMI. For linear
predictors with a smooth convex loss of nonzero slope at zero, such as the logistic loss, plus a
regularizer whose curvature and growth are both of power $r\ge2$, on product distributions over a
scaled sign cube in dimension at least linear in $n$, every proper learner with expected excess
risk at most $\eps$ on these distributions at the optimal sample size $n\asymp\eps^{-2+2/r}$ has
worst-case ld-CMI of order $n$ bits, and $\Theta(n/(1+(\tau/\eps)^2))$ bits under Gaussian noise
of standard deviation $\tau$ on the loss differences. The same holds without a regularizer, at
$n\asymp\eps^{-2}$. Consequently, range-scaled ld-CMI bounds
cannot vanish on these distributions, although every proper learner's generalization gap is
$O(n^{-1/2})$. We also show that model-level information does not determine noisy loss-difference
information, and that the growth, slope and dimension conditions are needed, the last up to a
logarithm.
\end{abstract}

\input{sec1_intro}
\input{sec2_setting}
\input{sec3_result}
\input{sec4_consequences}
\input{sec5_proof}
\input{sec6_necessity}
\input{sec7_separation}
\input{sec8_measurements}
\input{sec9_limitations}
\label{main:end}

\subsection*{AI use statement}
In this work, we used generative AI tools for helping to develop the theoretical framework and to
formulate mathematical claims, for providing ingredients for their proofs and assisting in writing
them, for proposing and refining hypotheses and giving feedback on the methodology, and for
implementing the numerical methods and interpreting their results. We have not used generative AI
tools to generate data sets; translation, data-set cleaning and qualitative or thematic data
analysis are not applicable to this work. Additionally, we used generative AI tools for literature
search, for writing code and figures, and for drafting and editing the text. We have reviewed all
AI-assisted work. We checked all AI-assisted proofs, code and descriptions of cited work, and the
main results are also machine-checked in Lean 4. We take responsibility for the final content of
this work, including text, claims, or artifacts produced with the aid of generative AI.

\input{refs}
\clearpage
\onecolumn
\appendix

\input{app_A_notation}
\input{app_B_tools}
\input{app_C_transfer}
\input{app_D_mainproof}
\input{app_E_examples}
\input{app_F_consequences}
\input{app_G_necessity}
\input{app_H_separation}
\input{app_I_measurements}
\input{app_J_related}
\input{app_K_lean}

\end{document}

%% file: macros.tex
\newcommand{\E}{\mathbb{E}}
\newcommand{\Pp}{\mathbb{P}}
\newcommand{\R}{\mathbb{R}}
\newcommand{\eps}{\varepsilon}
\newcommand{\norm}[1]{\lVert#1\rVert}
\newcommand{\opnorm}[1]{\lVert#1\rVert_{\mathrm{op}}}
\newcommand{\ip}[2]{\langle#1,#2\rangle}
\newcommand{\ind}{\mathbf{1}}
\newcommand{\Normal}{\mathcal{N}}
\DeclareMathOperator{\clip}{clip}
\DeclareMathOperator{\var}{Var}
\DeclareMathOperator{\KL}{KL}
\DeclareMathOperator{\proj}{proj}
\DeclareMathOperator{\rank}{rank}

\DeclareMathOperator{\Bin}{Bin}
\DeclareMathOperator*{\argmin}{arg\,min}

\newcommand{\Hent}{\mathrm{H}}   % Shannon entropy
\newcommand{\hdiff}{\mathrm{h}}  % differential entropy
\newcommand{\htwo}{\mathrm{h}_2} % binary entropy function

\newcommand{\thh}{\widehat\theta}
\newcommand{\Zt}{\widetilde Z}
\newcommand{\Zbar}{\bar Z}
\newcommand{\Ztr}{Z^{\mathrm{tr}}}
\newcommand{\Zgh}{Z^{\mathrm{gh}}}
\newcommand{\Btr}{B_{\mathrm{tr}}}
\newcommand{\sighat}{\hat\sigma}
\newcommand{\gWstar}{g_W^{*}}
\newcommand{\gen}{\mathrm{gen}}
\newcommand{\rec}{\mathrm{rec}}
\newcommand{\epsdp}{\eps_{\mathrm{DP}}}

\newcommand{\Ii}{\mathcal{I}^{\mathrm{ind}}}
\newcommand{\Ij}{\mathcal{I}^{\mathrm{joint}}}
\newcommand{\Is}{\mathcal{I}^{\mathrm{cand}}}
\newcommand{\Qc}{\mathcal{Q}}
\newcommand{\J}{\mathsf{J}}

\newcommand{\Vres}{V_{\mathrm{res}}}
\newcommand{\Mbar}{\bar M}
\newcommand{\Vbar}{\bar V}
\newcommand{\Vresbar}{\bar V_{\mathrm{res}}}
\newcommand{\neps}{n_{\eps}}
\newcommand{\Cdim}{C_{\mathrm{dim}}}
\newcommand{\cnet}{c_{\mathrm{net}}}
\newcommand{\Tb}{\mathfrak{T}}

\newcommand{\leanname}[1]{\nolinkurl{#1}}
\makeatletter
\g@addto@macro\UrlBreaks{\do\_}
\makeatother

\theoremstyle{plain}
\newtheorem{theorem}{Theorem}
\newtheorem{proposition}[theorem]{Proposition}
\newtheorem{corollary}[theorem]{Corollary}
\newtheorem{lemma}[theorem]{Lemma}
\theoremstyle{definition}
\newtheorem{definition}[theorem]{Definition}
\newtheorem{assumption}[theorem]{Assumption}
\theoremstyle{remark}
\newtheorem{remark}[theorem]{Remark}
\newtheorem*{remark*}{Remark}
\theoremstyle{plain}

\makeatletter
\g@addto@macro\appendix{%
  \setcounter{theorem}{0}\setcounter{table}{0}\setcounter{figure}{0}%
  \@addtoreset{theorem}{section}%
  \@addtoreset{table}{section}%
  \@addtoreset{figure}{section}%
  \renewcommand{\thetable}{\thesection.\arabic{table}}%
}

\newcommand{\restate@name}{}
\newtheorem*{restate@env}{\restate@name}
\newenvironment{restate}[3][]{%
  \renewcommand{\restate@name}{#2~\ref{#3}}%
  \def\restate@note{#1}%
  \ifx\restate@note\@empty
    \def\restate@next{\begin{restate@env}}%
  \else
    \def\restate@next{\begin{restate@env}[#1]}%
  \fi
  \restate@next}{\end{restate@env}}
\makeatother

%% file: sec1_intro.tex
\section{Introduction}\label{sec:intro}

Information-theoretic generalization bounds control the generalization gap of a learner by the
information that an observation of the learner carries about its training data. In the supersample
framework of \citet{steinke2020}, $n$ pairs of candidates $\Zt$ are drawn i.i.d., a hidden fair bit
$U_i$ per pair selects the training example, and the other candidate, the \emph{ghost}, is a fresh draw;
the bounds use the conditional mutual information (CMI) between an observation and the bits given $\Zt$.
Observations form a data-processing hierarchy: the model $\thh$; its losses on all candidates
(evaluated CMI; \citealp{steinke2020,hellstrom2022}); and the difference of the two
losses of each pair, whose information about that pair's bit is the loss-difference CMI (ld-CMI) of
\citet{wangmao2023a}; summed over pairs, it is the smallest of these. The gap is the average signed
loss difference, and bounds that use less of the learner's output are no looser \citep{hellstrom2025}.
Must
accuracy put information into the loss differences, as it does into the model? And how much
observation noise removes it?

Accuracy is known to force information into the model (Table~\ref{tab:known}).
\citet[Theorem~1]{livni2023} bounds the information $I(\thh;S_n)$ between the model and the training set
from below for every learner accurate in expectation on a quadratic instance, leaving CMI open;
\citet{attias2024} bound the CMI $I(\thh;U\mid\Zt)$ from below for every learner
accurate with high probability; and \citet[Theorem~2.7]{voitovych2025} show that every learner whose
expected excess risk is $O(n^{-1/2})$ is traceable: a \emph{tracer}, which reads the model, a point and the population mean,
flags $\Omega(n)$ training points while flagging fresh points with probability at most $\xi$. By data
processing, loss-difference information is at most model information, so these results give no lower
bound for it; none involves observation noise. The loss-level lower bounds we know of concern one
algorithm on one instance, without noise: evaluated-CMI bounds do not vanish for gradient descent although it
generalizes \citep[Theorem~17]{haghifam2023}, and the sign of a loss difference reveals its bit with
constant probability \citep[Example~1 and Appendix~F]{wangmao2023b}. \citet[Section~6]{wangmao2024}
anticipate that the model's accuracy--CMI tradeoff also holds for ld-CMI.

\begin{table}[t]
\caption{What Accuracy Is Known to Force. Each row is a lower bound for every learner of the stated
kind; only the last row allows observation noise, and it has a matching upper bound at every noise
level.}\label{tab:known}
\centering\scriptsize
\setlength{\tabcolsep}{3pt}
% Cells are set ragged right; rows therefore end with \tabularnewline.
\begin{tabular}{@{}p{50pt}p{92pt}p{74pt}@{}}
\toprule
Work & Observable and learners & Instance; dimension\tabularnewline
\midrule
\raggedright\citet{attias2024} & \raggedright $I(\thh;U\mid\Zt)=\Omega(\eps^{-2})$ or
$\Omega(\eps^{-1})$, also per sample (the first for proper learners), for every learner accurate with high probability &
\raggedright affine and strongly convex losses on the sign cube;
$d\ge\Omega(n^2\log n)$\textsuperscript{a}\tabularnewline
\raggedright\citet{voitovych2025} & \raggedright tracer recall $\Omega(n)$ from $(\thh,z,\mu)$, for
every learner with expected excess risk $O(n^{-1/2})$ & \raggedright affine loss with parameters in a cube,
sparse signs; $d\gtrsim n\log(1/\xi)$\tabularnewline
\raggedright\citet{haghifam2023}; \citet{wangmao2023b} & \raggedright non-vanishing evaluated-CMI bounds; signs
of loss differences; gradient descent & \raggedright one-hot convex Lipschitz instance;
$d=2n^2$\tabularnewline
\raggedright This paper & \raggedright $\Theta(n/(1+(\tau/\eps)^2))$ for three information
measures of $X+\tau G$ and every proper learner accurate on the family & \raggedright smooth convex links,
power-$r$ regularizers, product laws on the scaled sign cube; $n\asymp\eps^{-2+2/r}$, $d\gtrsim n$\tabularnewline
\bottomrule
\end{tabular}\\[2pt]
\parbox{\linewidth}{\scriptsize \textsuperscript{a}\,As stated in arXiv:2402.09327v2.}
\end{table}

We answer both questions for linear predictors on the unit ball with a smooth convex link of nonzero
slope at zero and a regularizer whose curvature and growth
are both of power $r\ge2$, on a family of product distributions over a scaled sign cube in dimension $d\gtrsim n$;
the class contains ridge-regularized logistic regression on a balanced-label problem. Let $X$ be the
vector of the $n$ loss differences and $G\sim\Normal(0,I_n)$ independent noise. Every \emph{proper}
learner (one with outputs in the unit ball) whose expected excess risk is at most $\eps$ on every
distribution of the family satisfies, at the optimal sample size $n\asymp\eps^{-2+2/r}$ and for every
$\tau\ge0$,
\begin{equation}\label{eq:headline}
 \sup_pI(X+\tau G;U\mid\Zt)=\Theta\Big(\frac n{1+(\tau/\eps)^2}\Big),
\end{equation}
with the supremum over the family, and likewise for the summed individual and the joint information
(Theorem~\ref{thm:main}). So at $\tau=0$ the worst-case ld-CMI is of order $n$ bits; noise at the scale
of the excess risk leaves a constant fraction of a bit per pair, and the total falls to $\Theta(1)$ bits
only at $\tau\asymp\eps^{1/r}$. The noisy observation shows what the loss differences reveal at
resolution $\tau$: guessing the bits from it is a balanced membership test on each pair, and
\citet[Theorem~6.8 of the extended version]{steinke2020} bound evaluated CMI through such noise. At the
scale of their range, it gives a generalization bound of order $\eps$, the worst-case order of the gap,
though no better than one that uses no information (\mbox{Section}~\ref{sec:consequences}).

Our contributions are as follows.
\begin{list}{\textbullet}{\setlength{\leftmargin}{1em}\setlength{\labelwidth}{0.6em}%
\setlength{\labelsep}{0.4em}\setlength{\topsep}{1pt}\setlength{\itemsep}{1pt}\setlength{\parsep}{0pt}}
\item The two-sided tradeoff (Theorem~\ref{thm:main}), which at \mbox{$\tau=0$} proves for this class what
\citet{wangmao2024} anticipated; by data processing, accuracy then also forces model
information $I(\thh;U\mid\Zt)$ of order $n$ in dimension $d\gtrsim n$ (Appendix~\ref{app:converse}).
\item Consequences (Section~\ref{sec:consequences}). For every accurate learner of the class, the
range-scaled ld-CMI bound of \citet{wangmao2023a} stays above a constant at some distribution of the
family, although every proper learner's gap is $O(n^{-1/2})$. The loss differences have range of order one but root-mean-square size of
order $\eps$, so the range normalization fails; normalized by the second moment, the bound has the
order of the gap on prior average, as does the root-mean-square loss difference, a bound that uses no
information. An accurate $(\epsdp,\delta)$-differentially private learner needs
$\epsdp=\Omega(1)$ unless $\delta\gtrsim\eps$.
\item A proof through three moments (Section~\ref{sec:proof}). By fingerprinting, accuracy forces the
signed mean of the loss differences, $n$ times the gap, to be of order $n\eps$ on prior average;
localization bounds their energy $\sum_i\E X_i^2$ on prior average, and at every distribution their squared distance
from a reference fixed by the candidates, by order $n\eps^2$; a transfer theorem that needs no
independence among the loss differences turns these bounds into \eqref{eq:headline}. For the affine
link with $g(\theta)=\norm\theta^2/2$ (Remark~\ref{rem:affine}), the signed-mean bound follows the
arguments of \citet{attias2024} and \citet{voitovych2025}, and the energy bound follows
\citet{attias2024}; other regularizers need a calibration through the matched upper growth, and a
nonlinear link adds to the signed mean a train--ghost remainder, which localization keeps one power
of $\eps$ below the signal; this is the main technical step.
\item Necessity (Section~\ref{sec:necessity}): the matched growth and the nonzero slope cannot be
dropped, and linear dimension is needed up to a logarithm.
\item A separation (Section~\ref{sec:separation}): two learners, sample-optimal for two different
losses and related by a bijection, share model information and tracing guarantees, yet at $\tau=n^{-2/3}$ one learner's loss differences
carry at most $3n^{1/3}$ bits at every distribution of the family and the other's at least $n/8$ at one.
\item Measurements for ridge and logistic ridge ERM at $d=64n$ (Section~\ref{sec:measurements}): at each
learner's own moments, the transfer theorem brackets the estimated information within factors $0.63$
and $1.51$ wherever it is at least $0.01$ bits per pair.
\end{list}
The numbered results of Sections~\ref{sec:result}--\ref{sec:necessity} and Proposition~\ref{prop:separation} are formalized in
Lean~4 \citep{moura2021lean} with Mathlib \citep{mathlib2020}, using only the standard axioms;
Appendix~\ref{app:lean} lists the declarations and their differences from the text.

\subsection{Related Work}\label{sec:related}

Mutual-information bounds \citep{russo2016,xu2017} have individual-sample, conditional, evaluated and
loss-difference forms
\citep{bu2020,steinke2020,hellstrom2020,hellstrom2022,harutyunyan2021,haghifam2021unified,wangmao2023a},
surveyed by \citet{hellstrom2025}. Further limits of information-based explanations concern the model
\citep{bassily2018,nachum2018,livni2020}; on some natural prediction problems accurate learners must
encode much of their training data in the model \citep{brown2021}, and \citet{feldman2025} prove tradeoffs between the number of
samples and the information that an accurate learner must memorize, for simple binary classification
problems. Fingerprinting underlies lower bounds for private estimation and learning
\citep{bun2014,steinke2017,kamath2022,bassily2014}, and our privacy witness resembles one-run auditing
\citep{steinke2023auditing}; Appendix~\ref{app:related} gives more detail.

%% file: sec2_setting.tex
\section{Setting}\label{sec:setting}

Information is in bits, and $\ln$ is the natural logarithm. Let $W=B_d(1)$ be the closed unit ball of
$\R^d$ and $\mathcal Z=\{\pm1/\sqrt d\}^d$ the scaled sign cube, so $\norm z=1$ on $\mathcal Z$. For
$p\in[-1,1]^d$, the product law $D_p$ on $\mathcal Z$ has independent coordinates with
$\Pp(\sqrt d\,Z_k=1)=(1+p_k)/2$ and mean $\mu_p=p/\sqrt d$; the \emph{family} is
$\{D_p:p\in[-1,1]^d\}$, and $\E_p$, $\Pp_p$ refer to data from $D_p$. The loss of $\theta\in W$ is
$f(\theta,z)=g(\theta)+\phi(\ip\theta z)$, with a regularizer $g$ and a link $\phi$;
$F_D(\theta)=\E_Df(\theta,Z)$ is its population risk under a law $D$, $F_p=F_{D_p}$, and $\theta_p^*$ is the minimizer of $F_p$ over $W$, unique under
Assumption~\ref{ass:class}.

A \emph{proper learner} maps $n$ training points $S_n$ and a seed $R$, independent of the data, to
$\thh=A_n(S_n;R)\in W$, jointly measurably. It is \emph{$\eps$-accurate on the family} if
\begin{equation}\label{eq:accuracy}
 \E F_p(\thh)-\min_WF_p\le\eps\qquad\text{for every }p\in[-1,1]^d,
\end{equation}
the expectation being over $S_n\sim D_p^n$ and $R$. Accuracy on every distribution on the unit ball
implies \eqref{eq:accuracy}, and our information bounds use accuracy only in this form.

The supersample consists of $2n$ candidates $\Zt=(Z_{j,i})_{j\in\{0,1\},i\in[n]}$, i.i.d.\ from the
data law, and independent selection bits $U$ uniform on $\{0,1\}^n$; the learner trains on
$S_n=(Z_{U_i,i})_{i\in[n]}$. Let $\sigma_i=1-2U_i$, $\Ztr_i=Z_{U_i,i}$ and $\Zgh_i=Z_{1-U_i,i}$; the
ghost $\Zgh_i$ is independent of the training set and the seed.

\begin{definition}[Loss-difference release]\label{def:release}
The \emph{loss-difference release} of a proper learner is $X\in\R^n$ with
\begin{equation}\label{eq:release}
 X_i=f(\thh,Z_{1,i})-f(\thh,Z_{0,i}),\qquad i\in[n].
\end{equation}
An observer sees $Y=X+\tau G$, with $\tau\ge0$ and $G\sim\Normal(0,I_n)$ independent of everything
else. We measure its information about the bits by $\Ii(Y)=\sum_iI(Y_i;U_i)$, $\Ij(Y)=I(Y;U)$ or
$\Is(Y)=I(Y;U\mid\Zt)$, with the law as a subscript when needed, as in $\Is_{D_p}(X+\tau G)$.
\end{definition}

The regularizer cancels: $X_i=\phi(\ip\thh{Z_{1,i}})-\phi(\ip\thh{Z_{0,i}})$. The three measures
correspond to observers who see one $Y_i$ at a time, all of $Y$, or $Y$ and all candidates. As the bits are
independent of each other and of $\Zt$, and by data processing given $\Zt$ (Appendix~\ref{app:hierarchy}), the three information measures of $Y$ satisfy
\begin{equation}\label{eq:hierarchy}
 \Ii\le\Ij\le\Is\le I(X;U\mid\Zt)\le I(\thh;U\mid\Zt).
\end{equation}
For the same reasons, the summed ld-CMI $\sum_iI(Y_i;U_i\mid\Zt)$, which at $\tau=0$ is the sum over pairs of the
ld-CMI of \citet{wangmao2023a}, lies between its unconditional form $\Ii(Y)$ and $\Is(Y)$, so
Theorem~\ref{thm:main} bounds both forms from above and below. The product $\sigma_iX_i=f(\thh,\Zgh_i)-f(\thh,\Ztr_i)$ is the
ghost-minus-training loss of pair $i$, so the expected generalization gap $\gen_D=\E[F_D(\thh)-\widehat F_{S_n}(\thh)]$ at a
law $D$, with $\widehat F_{S_n}$ the training risk, is $1/n$ times the \emph{signed mean} of the release:
for $\gen_p=\gen_{D_p}$,
\begin{equation}\label{eq:gap}
 n\,\gen_p=\textstyle\sum_i\E_p\sigma_iX_i .
\end{equation}

\begin{assumption}[Link and regularizer]\label{ass:class}
The link $\phi$ is convex and $C^2$ near $[-1,1]$, with $b=-\phi'(0)>0$, $|\phi'|\le L$ and
$0\le\phi''\le H$ on $[-1,1]$. The regularizer $g$ is convex and $C^1$ near $W$, with $g(0)=0$,
$\nabla g(0)=0$ and $\sup_W\norm{\nabla g}\le L_g$, and for a fixed $r\ge2$ and constants $a_-,a_+>0$ its
Bregman divergence $D_g(\theta',\theta)=g(\theta')-g(\theta)-\ip{\nabla g(\theta)}{\theta'-\theta}$
satisfies, for all $\theta,\theta'\in W$,
\begin{equation}\label{eq:matched}
 D_g(\theta',\theta)\ge a_-\norm{\theta'-\theta}^r,\qquad 0\le g(\theta)\le a_+\norm\theta^r .
\end{equation}
The regularizer may depend on $d$ but is fixed before $p$; $b,L,H,L_g,r,a_\pm$ do not depend on $d$.
\end{assumption}

So the curvature of $g$ is bounded below, and its growth above, with the same power $r$; the curvature
bound alone does not determine the tradeoff (Proposition~\ref{prop:growth}). Examples
(Appendix~\ref{app:examples}) are the links $\ln(1+e^{-t})$ (logistic), $e^{-t}$, $(1-t)^2/2$ and $-t$
(affine), and $g(\theta)=\norm{\Lambda\theta}^r/r$ with the singular values of $\Lambda$ in a fixed compact
subset of $(0,\infty)$; for $r=2$, any smooth strongly convex $g$ with $g(0)=0$, $\nabla g(0)=0$ and
bounded Hessian. With balanced labels $y=\pm1$ independent of $z$ and features $x=yz$, the logistic loss of
$y\ip\theta x$ is $\phi(\ip\theta z)$; links with $\phi'(0)>0$ reduce to $b>0$ through $z\mapsto-z$
(Remark~\ref{rem:reflection}).

%% file: sec3_result.tex
\section{Main result}\label{sec:result}

Let $q=r/(r-1)\in(1,2]$ be conjugate to $r$, and let $\neps=\eps^{-2/q}=\eps^{-2+2/r}$, the optimal
sample size (Proposition~\ref{prop:samples}): $\eps^{-1}$ for ridge regularization ($r=2$) and
$\eps^{-3/2}$ for quartic regularization ($r=4$). Below, $O$, $\Omega$, $\Theta$, $\asymp$ and
$\lesssim$ hide constants that depend only on $b,L,H,L_g,r,a_\pm$ of
Assumption~\ref{ass:class} and on the sample-size multipliers $c_\pm$, and $\Qc$ is any of $\Ii$,
$\Ij$, $\Is$.

\begin{theorem}[Two-sided tradeoff]\label{thm:main}
Let Assumption~\ref{ass:class} hold and let $0<c_-\le c_+$. There are $\eps_0,c_1,C_1>0$ and
$\Cdim\ge32$, depending only on the constants of Assumption~\ref{ass:class} and on $c_\pm$, with the
following property. If $0<\eps\le\eps_0$, $c_-\neps\le n\le c_+\neps$ and $d\ge\Cdim\max\{n,\neps\}$,
then every proper learner that is $\eps$-accurate on the family satisfies, for each $\Qc$ and every
$\tau\ge0$,
\begin{equation}\label{eq:main}
 c_1\frac n{1+(\tau/\eps)^2}\le\sup_{p}\Qc_{D_p}(X+\tau G)\le C_1\frac n{1+(\tau/\eps)^2},
\end{equation}
where $p$ ranges over $[-1,1]^d$. The lower bound holds on average over a beta prior on $p$ that does
not depend on the learner (Section~\ref{sec:architecture}), and the upper bound holds at every $p$. Both
bounds hold for the infimum over any nonempty class of such learners.
\end{theorem}

Here $\Cdim=32$ whenever $b\le4$, and Appendix~\ref{app:constants} gives $\Cdim$, $\eps_0$, $c_1$ and
$C_1$ explicitly; $c_1$ and $C_1$ are far apart (with the affine link and ridge regularization,
$c_1\approx4.2\cdot10^{-4}$ and $C_1\approx6.9$ at the limiting normalized moments and $c_\pm=1$), so the theorem is
a statement about orders. For $\tau\gg\eps$, the worst-case information is of order $n(\eps/\tau)^2$
bits, and $\Theta(\eps^{2/r})$ at any fixed $\tau>0$.

Accuracy at a single law cannot force information at that law: the learner that always outputs the
minimizer $\theta_D^*$ of $F_D$ over $W$ is exactly accurate at $D$, and its release, a function of the
candidates, carries no information about $U$. The lower half of Theorem~\ref{thm:main} therefore uses
accuracy across the family and bounds the worst case over the family through an average over a prior
fixed in advance. The upper bound at $p$ uses accuracy only at $D_p$, and the maximizing law may depend
on the learner and on $\tau$.

\begin{proposition}[Optimal sample size]\label{prop:samples}
Under Assumption~\ref{ass:class}, projected stochastic gradient descent in epochs (epoch SGD) is a
proper learner that is $\eps$-accurate on
every distribution on the unit ball with $O(\neps)$ samples. Conversely, if $d\ge\Cdim\neps$ and $\eps$
is small enough, every proper learner that is $\eps$-accurate on the family uses $\Omega(\neps)$
samples.
\end{proposition}

So Theorem~\ref{thm:main} concerns the optimal sample size, and learners that satisfy its hypotheses
exist at every $n$ of the range once $c_-$ is at least the constant of epoch SGD (Appendix~\ref{app:examples}). Without a regularizer,
\eqref{eq:main} holds at the optimal sample size $n\asymp\eps^{-2}$ (Appendix~\ref{app:noreg}).

\paragraph{An Example.} Ridge ERM, with $\phi(t)=-t$ and $g(\theta)=\norm\theta^2/2$, outputs the
training mean $\Zbar$ and is $\eps$-accurate on the family with $\eps=1/(2n)$, so $n=\neps/2$. Its
release is $X_i=-\ip\Zbar{\Delta_i}$, with $\Delta_i=Z_{1,i}-Z_{0,i}$. Writing
$\Zbar=z_c-(2n)^{-1}\sum_j\sigma_j\Delta_j$, with $z_c$ the mean of all $2n$ candidates, gives
\[
 X_i=\sigma_i\frac{\norm{\Delta_i}^2}{2n}-\ip{z_c}{\Delta_i}+\frac1{2n}\sum_{j\ne i}\sigma_j\ip{\Delta_j}{\Delta_i}
\]
(Lemma~\ref{lem:ridge-moments}). The first term carries the bit of pair $i$: as $\norm{\Delta_i}^2\approx2$,
it is about $\sigma_i/n=2\sigma_i\eps$. The other two terms do not involve $U_i$, and near the zero-mean
law, where the prior of Section~\ref{sec:architecture} concentrates, each has root-mean-square size of
order $1/\sqrt{nd}$, because the inner products of $\Delta_i$ with the candidates of the other pairs
and with their differences $\Delta_j$ are of order $d^{-1/2}$. So the signal-to-noise ratio of each
pair, the squared signal over the variance of the other two terms, is bounded below when
$d\gtrsim n$, and each pair then carries a constant fraction of a bit about its selection; for
$d\ll n$ the ratio is of order $d/n$. Gaussian noise
with $\tau\gg\eps$ lowers the fraction to order $(\eps/\tau)^2$. Theorem~\ref{thm:main} shows that, in the
worst case over the family, every accurate learner of the class behaves in this way.

%% file: sec4_consequences.tex
\section{What the tradeoff means for loss-difference CMI bounds and for privacy}\label{sec:consequences}

Theorem~\ref{thm:main} and its proof have two consequences, proved in Appendix~\ref{app:consequences}.
A learner is \emph{$(\epsdp,\delta)$-differentially private} if
$\Pp(A_n(S;R)\in E)\le e^{\epsdp}\,\Pp(A_n(S';R)\in E)+\delta$ for all measurable $E\subseteq W$ and all
training sets $S,S'\in\mathcal Z^n$ that differ in one point.

\begin{corollary}[Loss-difference CMI bounds and privacy]\label{cor:consequences}
Let $\thh$ be a learner of Theorem~\ref{thm:main}.
\begin{enumerate}
\item[(i)] Some law $D_p$ of the family has $\sum_iI_{D_p}(X_i;U_i)\ge c_1n/2$, and hence
$n^{-1}\sum_i\sqrt{I_{D_p}(X_i;U_i\mid\Zt)}\ge c_1/2$. In contrast, every proper learner has
$|\gen_D|\le4L/\sqrt n$ at every law $D$ on the unit ball.
\item[(ii)] If $\thh$ is also $(\epsdp,\delta)$-differentially private, then
$e^{\epsdp}-1\ge c'-2L\delta/\eps$, for a constant $c'>0$ that depends only on the constants of
Assumption~\ref{ass:class} and on $c_\pm$. So $\epsdp=\Omega(1)$ unless $\delta\gtrsim\eps$.
\end{enumerate}
\end{corollary}

For losses with values in $[0,1]$, \citet[Theorem~3.1]{wangmao2023a} prove
$|\gen_D|\le n^{-1}\sum_i\sqrt{2I_D(X_i;U_i\mid\Zt)}$, with the information in nats, and a sharper form
with the square root inside the average over $\Zt$; their Theorem~3.2 uses $I_D(X_i;U_i)$ instead.
For other bounded losses the right sides are multiplied by the length of their range. By (i), for every
accurate learner of the class, each of these range-scaled bounds is at least a positive constant at
some law of the family (Appendix~\ref{app:range}), whereas the gap they bound is at most
$4L/\sqrt n$. Part (i) follows from Theorem~\ref{thm:main} at $\tau=0$, and the gap bound from symmetrization
and contraction \citep{ledoux1991}.

For (ii), privacy caps each pair's signed loss difference at every law,
$\E_p\sigma_iX_i\le(e^{\epsdp}-1)\E_p|X_i|+2L\delta$, whereas by Proposition~\ref{prop:finite} the prior
averages of $\sum_i\sigma_iX_i$ and $\sum_iX_i^2$ are, respectively, at least of order $n\eps$ and at most
of order $n\eps^2$. For $r=2$ the conclusion of (ii) is consistent with the lower bounds of \citet{bassily2014} for private
strongly convex empirical risk minimization at $d\asymp n$, and for $r>2$ with the rates of
\citet{asi2021} for functions with growth of power $r$ around the optimum, which our curvature
condition implies and which they show to be optimal up to logarithmic factors.
\citet[Theorem~3.6]{voitovych2025} prove a bound of the same form for a statistic of the model; what
(ii) adds is the witness $\sum_i\sigma_iX_i$, which uses only the released loss differences and the
selection bits, for every link and every
$r\ge2$ of Assumption~\ref{ass:class}.

\begin{remark*}[Second-moment scaling]
For every learner and law, the fair-bit inequality (Section~\ref{sec:transfer}) gives
$(\E_p\sigma_iX_i)^2\le2\ln2\,\E_p[X_i^2]\,I_{D_p}(X_i;U_i)$, so by \eqref{eq:gap} and Cauchy--Schwarz, with
$V_p=\sum_i\E_pX_i^2$,
\begin{equation}\label{eq:second-moment}
\begin{aligned}
 |\gen_p|&\le\frac1n\sum_i\sqrt{2\ln2\,\E_p[X_i^2]\,I_{D_p}(X_i;U_i)}\\
 &\le\frac1n\sqrt{2\ln2\,V_p\,\Ii_{D_p}(X)};
\end{aligned}
\end{equation}
\citet[Section~4]{wangmao2023a} give related variance-based bounds. As $\Ii\le n$,
\eqref{eq:second-moment} is at most $\sqrt{2\ln2}$ times $\sqrt{V_p/n}$, itself a bound on $|\gen_p|$ by
Cauchy--Schwarz. By Proposition~\ref{prop:finite}, all three are of order $\eps$ on prior average,
while the range of the loss differences is of order one: what fails in (i) is the range normalization.
\end{remark*}

\begin{remark*}[Noise at the scale of the range]
Seen through noise at the scale of their range, the loss differences give a bound of the right order,
which Cauchy--Schwarz also gives without information.
Let $Y=X+\tau G$ with $\tau>0$. Then $\E_p\sigma_iY_i=\E_p\sigma_iX_i$, and as $|X_i|\le2L$, $\sigma_iY_i$
is sub-Gaussian with variance proxy $s^2=4L^2+\tau^2$ when $U_i$ is drawn independently of $Y_i$. The
decoupling estimate of \citet{russo2016} and \citet{xu2017}, Cauchy--Schwarz and \eqref{eq:hierarchy}
give, at every law,
\begin{equation}\label{eq:noisy-bound}
\begin{aligned}
 |\gen_p|&\le\frac1n\sum_i\sqrt{2\ln2\,s^2I_{D_p}(Y_i;U_i)}\\
 &\le\big(2\ln2\,s^2\,\Is_{D_p}(Y)/n\big)^{1/2}.
\end{aligned}
\end{equation}
At $\tau=2L$, \eqref{eq:upper} and Proposition~\ref{prop:finite} bound the right side by
$(2C_{\mathrm{res}})^{1/2}\eps$ at every law for every learner of Theorem~\ref{thm:main}. This is the order of
the gap on prior average, and $\sqrt2$ times the bound $C_{\mathrm{res}}^{1/2}\eps$ on $|\gen_p|$ that
Cauchy--Schwarz gives without information, as $\E_p\sigma_iX_{p,i}^*(\Zt)=0$. For ridge ERM, in our runs
at $d=32n$ and $d=64n$, the resulting bound is at most $1.44$ times the gap on prior average
(Appendix~\ref{app:noisy-bound}).
\end{remark*}

Fast-rate bounds \citep{grunwald2021,hellstrom2022} divide by $n$ an information measure of a kind
bounded here, up to constants of the loss, so at some law of the family they do not vanish either. Binary
Jensen--Shannon, stability-scaled and projected-and-quantized bounds
\citep{dong2025,wangmao2023b,sefidgaran2025} use other observations or functionals of the information,
which Theorem~\ref{thm:main} does not bound.

%% file: sec5_proof.tex
\section{Proof of \texorpdfstring{Theorem~\ref{thm:main}}{the main theorem}: three moments and one transfer theorem}\label{sec:proof}

\subsection{Architecture and the Prior}\label{sec:architecture}

Accuracy yields three moment bounds for every learner (Proposition~\ref{prop:finite}), and one transfer
theorem turns them into information bounds at every noise level (Theorem~\ref{thm:transfer});
Figure~\ref{fig:architecture} shows where each hypothesis enters. The lower half averages over the
\emph{beta prior} $\Pi_\kappa$, under which the $p_k$ are i.i.d.\ with density
$\pi_\kappa(t)=c_\kappa(1-t^2)^{\kappa/2-1}$ on $(-1,1)$, so that $v=\E\norm{\mu_p}^2=1/(\kappa+1)$. We fix
$v=(4\eps/\beta)^{2/q}$ for all learners, with $\beta=bt_0/q$ and $t_0=(b/(ra_++b))^{1/(r-1)}$, so
$\kappa\asymp\neps\asymp n$; $\E_\Pi$ averages over the prior, the data, the bits and the seed.

\begin{figure}[t]
\centering
\begin{tikzpicture}[font=\scriptsize,>={Stealth[length=3.5pt]},
  box/.style={draw,rounded corners=1.5pt,align=center,inner sep=1.5pt,minimum height=0.55cm},
  hyp/.style={box,fill=black!15,text width=1.62cm},
  every edge/.style={draw,->}]
\node[hyp] (grow) at (0.9,0) {slope $b>0$, growth\\ $g\le a_+\norm\theta^r$};
\node[hyp] (acc) at (2.98,0) {accuracy\\ on the family};
\node[hyp] (curv) at (5.06,0) {curvature\\ $D_g\ge a_-\norm{\cdot}^r$};
\node[hyp] (dim) at (7.14,0) {dimension\\ $d\ge32n$, $d\ge32/v$};
\node[box,text width=3.3cm] (s1) at (1.93,-0.88) {Step 1: calibration,\\ fingerprinting (Lemma~\ref{lem:fp})};
\node[box,text width=3.3cm] (s2) at (6.15,-0.88) {Step 2: localization at\\ every law (Lemmas~\ref{lem:matrix}, \ref{lem:weighted})};
\node[box,text width=2.35cm] (s3) at (1.25,-1.75) {Step 3: remainder;\\ signed mean $M$};
\node[box,text width=2.3cm] (p7) at (4.07,-1.75) {Prop.~\ref{prop:finite} + Thm.~\ref{thm:transfer}\\ $\Rightarrow$ Thm.~\ref{thm:main}};
\node[box,text width=2.35cm] (s4) at (6.9,-1.75) {Step 4: energy $V$,\\ residual $\Vres$};
\path (grow.south) edge (s1.north -| grow.south)
      (acc.south) edge (s1.north -| acc.south)
      (acc.south) edge (s2.north west)
      (curv.south) edge (s2.north -| curv.south)
      (dim.south) edge (s2.north -| dim.south)
      (s1.south -| s3.north) edge (s3)
      (s2.south west) edge (s3.north east)
      (s2.south -| s4.north) edge (s4)
      (s3) edge (p7) (s4) edge (p7);
\end{tikzpicture}
\caption{Structure of the Proof. Shaded boxes are hypotheses; Section~\ref{sec:necessity} shows that the
slope and growth conditions are needed, and the dimension condition up to a logarithm.}
\label{fig:architecture}
\end{figure}
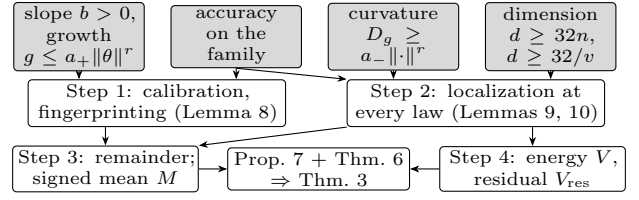

\subsection{The Transfer Theorem}\label{sec:transfer}

\begin{theorem}[Moment-to-information transfer]\label{thm:transfer}
Let $X$ be the release of a proper learner, let $p$ be drawn from a prior $\Pi$ on $[-1,1]^d$
independently of $U$, $R$ and $G$, and let $\tau\ge0$. (a)~If $\sum_i\E_\Pi\sigma_iX_i\ge M>0$ and
$\sum_i\E_\Pi X_i^2\le V$, then
\begin{equation}\label{eq:lower}
 \E_{p\sim\Pi}\Ii_{D_p}(X+\tau G)\ge\frac{M^2}{2\ln2\,(V+n\tau^2)}.
\end{equation}
(b)~If $\E_p\norm{X-x_p(\Zt)}^2\le\Vres$ at a law $D_p$, for a reference $x_p(\Zt)\in\R^n$ that is a
measurable function of $p$ and the candidates only, then $\Is_{D_p}(X)\le n$ and, for $\tau>0$,
\begin{equation}\label{eq:upper}
 \Is_{D_p}(X+\tau G)\le\min\Big\{n,\frac n{2\ln2}\ln\Big(1+\frac{\Vres}{n\tau^2}\Big)\Big\}.
\end{equation}
Neither part assumes independent coordinates of $X$.
\end{theorem}

The \emph{fair-bit inequality}, $I(Y;U)\ge(\E\sigma Y)^2/(2\ln2\,\E Y^2)$ for a fair bit $U$ with sign
$\sigma=1-2U$ and a square-integrable real $Y$ \citetext{via \citealp{topsoe2000}; \citealp[Theorem~6]{luli2015}}, applied to
each $X_i+\tau G_i$ at each $p$, with Cauchy--Schwarz over the pairs $(p,i)$, gives \eqref{eq:lower}:
noise leaves the signed mean unchanged and adds $\tau^2$ per pair to the energy. Per pair,
\eqref{eq:lower} is thus the signal-to-noise ratio $(M/n)^2/(2\ln2\,(V/n+\tau^2))$ of the
\emph{per-pair signal} $M/n$, and it holds for any independent centred noise of variance $\tau^2$, such as Laplace noise. For
\eqref{eq:upper}, which uses Gaussianity, fix $p$: given $\Zt$, $U\to X\to Y$, and subtracting
$x_p(\Zt)$, subadditivity, Gaussian maximum entropy and Jensen bound $I(Y;U\mid\Zt)$ by
$\frac n2\ln(1+\Vres/(n\tau^2))$ nats, and $I(Y;U\mid\Zt)\le\Hent(U\mid\Zt)=n$ bits for every $\tau\ge0$
(Appendix~\ref{app:transfer}).

\subsection{The Three Moment Bounds}\label{sec:moments}

\begin{proposition}[Three moment bounds]\label{prop:finite}
Under the hypotheses of Theorem~\ref{thm:main}, let $\Pi=\Pi_\kappa$ with $v=(4\eps/\beta)^{2/q}$.
There are $c_M,C_V,C_{\mathrm{res}}>0$, depending only on the constants of Assumption~\ref{ass:class}
and on $c_\pm$, such that every proper learner that is $\eps$-accurate on the family satisfies
$\sum_i\E_\Pi\sigma_iX_i\ge c_Mn\eps$ and $\sum_i\E_\Pi X_i^2\le C_Vn\eps^2$, and
$\E_p\norm{X-X_p^*(\Zt)}^2\le C_{\mathrm{res}}n\eps^2$ at every $p$, where the reference
$X_{p,i}^*(\Zt)=\phi(\ip{\theta_p^*}{Z_{1,i}})-\phi(\ip{\theta_p^*}{Z_{0,i}})$ is the release of the
population minimizer.
\end{proposition}

Appendix~\ref{app:mainproof} proves it, with finite values $M$, $V$, $\Vres$, in four steps that use
three tools (Appendix~\ref{app:tools}).

\begin{lemma}[Beta fingerprinting; \citealp{steinke2017}]\label{lem:fp}
Let $\Pi=\Pi_\kappa$ with $\kappa>0$. For every proper learner,
$\E_\Pi\ip\thh{\sum_{i=1}^n(Z_{U_i,i}-\mu_p)}=\kappa\,\E_\Pi\ip\thh{\mu_p}$.
\end{lemma}

\begin{lemma}[Operator-norm tails in linear dimension]\label{lem:matrix}
Fix $p$ and let $d\ge32n$. The centred training matrix $\Btr\in\R^{n\times d}$, with rows
$(Z_{U_i,i}-\mu_p)^\top$, satisfies $\Pp_p(\opnorm{\Btr}^2>K_0+K_1t)\le e^{-t}$ for all $t\ge0$, with
$K_0<2.031$ and $K_1=2/(\cnet d)$, where $\cnet=\frac{19}{40}(\frac{29}{30})^2$. The pair matrix
$\Delta$, with rows $\Delta_i=(Z_{1,i}-Z_{0,i})^\top$, satisfies the same bound with $2K_0$ and
$2K_1$.
\end{lemma}

\begin{lemma}[Dependent weighted tails]\label{lem:weighted}
Let $K_0,K_1\ge0$, let $N\ge0$ satisfy $\Pp(N>K_0+K_1t)\le e^{-t}$ for all $t\ge0$, and let
$0\le\omega\le\bar\omega$ with $\E\omega\le w$ for some $0<w\le\bar\omega$. Then, whatever the
dependence between $N$ and $\omega$,
$\E[N\omega]\le\Tb_{\bar\omega}(w;K_0,K_1):=w\{K_0+K_1[1+\ln(\bar\omega/w)]\}$.
\end{lemma}

Write $u(t)=\phi(t)-\phi(0)+bt$ for the curved part of the link ($0\le u(t)\le Ht^2/2$ and
$|u'(t)|\le H|t|$ on $[-1,1]$), $m=\ip\thh{\mu_p}$ and $s_i=\ip\thh{Z_{U_i,i}-\mu_p}$, so
$\sum_is_i^2=\norm{\Btr\thh}^2$.

\paragraph{Step 1: The Signal from Accuracy.} Accuracy forces $\thh$ to correlate with the mean. The
comparator $\theta_c=t_0\norm{\mu_p}^{q-2}\mu_p$ ($0$ if $\mu_p=0$) has
$g(\theta_c)-b\ip{\theta_c}{\mu_p}\le-\beta\norm{\mu_p}^q$ by the upper growth, while $g,u\ge0$ give
$F_p(\thh)\ge\phi(0)-bm$; so accuracy at every $p$ and the prior moments give $b\,\E_\Pi m\ge A_\phi$
for an explicit $A_\phi\ge\{4(32/33)^{q/2}-1\}\eps-O(\eps^2)$, and by Lemma~\ref{lem:fp}
\begin{equation}\label{eq:fp-main}
 b\,\E_\Pi\ip\thh{\textstyle\sum_i(Z_{U_i,i}-\mu_p)}=b\kappa\,\E_\Pi m\ge\kappa A_\phi,
\end{equation}
of order $n\eps\asymp\eps^{-1+2/r}$; for the affine link, \eqref{eq:fp-main} is the signed mean itself.

\paragraph{Step 2: Localization at Every Law.} The lower curvature gives
$F_p(\theta)-F_p(\theta_p^*)\ge a_-\norm{\theta-\theta_p^*}^r$ on $W$, so accuracy and Jensen give
$\E_p\norm{\thh-\theta_p^*}^2\le w=\min\{4,(\eps/a_-)^{2/r}\}$ at every $p$, and comparison with $0$
gives $\norm{\theta_p^*}\le C_0\norm{\mu_p}^{q-1}$, $C_0=(b/(2a_-))^{1/(r-1)}$. With the prior moments
and Lemmas~\ref{lem:matrix}--\ref{lem:weighted}, which let the error be largest on atypical training matrices,
\begin{equation}\label{eq:localized-main}
\begin{gathered}
 \E_\Pi\norm\thh^2=O(\eps^{2/r}),\qquad \E_\Pi[\norm\thh^2\norm{\mu_p}^2]=O(\eps^2),\\
 \E_\Pi[\norm\thh^2\opnorm{\Btr}^2]=O(\eps^{2/r}).
\end{gathered}
\end{equation}

\paragraph{Step 3: The Train--Ghost Remainder.} Write
$\sigma_iX_i=\phi(\ip\thh{\Zgh_i})-\phi(\ip\thh{\Ztr_i})$ with $\phi(t)=\phi(0)-bt+u(t)$, and condition
on the training set, the seed and $p$: the ghost is independent of $\thh$ with mean $\mu_p$, so Jensen
gives $\E u(\ip\thh{\Zgh_i})\ge u(m)$, while $u(m+s_i)\le u(m)+u'(m)s_i+Hs_i^2/2$ at the training point.
Summing and using \eqref{eq:fp-main},
\begin{equation}\label{eq:taylor-main}
\begin{aligned}
 \textstyle\sum_i\E_\Pi\sigma_iX_i\ge{}& b\kappa\,\E_\Pi m-\E_\Pi\big[u'(m)\textstyle\sum_is_i\big]\\
 &-\tfrac H2\E_\Pi\norm{\Btr\thh}^2.
\end{aligned}
\end{equation}
As $|u'(m)|\le H\norm\thh\norm{\mu_p}$ and $|\sum_is_i|\le\sqrt n\norm\thh\opnorm{\Btr}$,
Cauchy--Schwarz with one factor $\norm\thh^2$ in each expectation and \eqref{eq:localized-main} bound
the last two terms by $O(\eps^{2/r})$, one power of $\eps$ below the signal, so
$\sum_i\E_\Pi\sigma_iX_i\ge c_Mn\eps$. The middle term need not vanish, as $\thh$ depends on the
training points; $\norm\thh\le1$ alone would bound both terms by constants, as large as the
signal when $r=2$.

\paragraph{Step 4: Energy and Residual.} By the mean value theorem $|X_i|\le
L|\ip\thh{Z_{1,i}-Z_{0,i}}|$, and the centred ghost margin $\ip\thh{\Zgh_i-\mu_p}$ is uncorrelated with
$s_i$ and has second moment at most $\E\norm\thh^2/d$, so
$\sum_i\E_\Pi X_i^2\le L^2(\E_\Pi[\norm\thh^2\opnorm{\Btr}^2]+\frac nd\E_\Pi\norm\thh^2)=O(n\eps^2)$.
At a fixed $p$, moving along $\theta_p^*+t(\thh-\theta_p^*)$ with $|\phi'|\le L$ and $\phi''\le H$ gives
$\norm{X-X_p^*(\Zt)}\le(L+H)\opnorm\Delta\norm{\thh-\theta_p^*}$, and Lemma~\ref{lem:weighted} with the
tail of $\opnorm\Delta^2$ gives $\E_p\norm{X-X_p^*(\Zt)}^2\le(L+H)^2\Tb_4(w;2K_0,2K_1)=O(n\eps^2)$,
although the error depends on the candidates in $\Delta$. Localization thus makes $\norm\thh$ of order
$\eps^{1/r}$ and, as $d\gtrsim n$, the loss differences of order $\eps$, in root mean square on prior
average.

\subsection{Assembly}\label{sec:assembly}

With the bounds of Proposition~\ref{prop:finite}, $M^2/(V+n\tau^2)\asymp n/(1+(\tau/\eps)^2)$ in
\eqref{eq:lower}, which gives the lower half of \eqref{eq:main} for $\Ii$ on prior average, hence for the
supremum; and $n\min\{1,\ln(1+\Vres/(n\tau^2))\}\lesssim n/(1+(\tau/\eps)^2)$ in \eqref{eq:upper}, which
with the cap $n$ at $\tau=0$ gives the upper half for $\Is$ at every $p$. By \eqref{eq:hierarchy} both
halves extend to the other two measures; as the constants do not depend on the learner, they hold for
the infimum over any nonempty class of such learners.

Fingerprinting \citep{steinke2017}, localization by uniform convexity
\citep{juditsky2014,daspremont2024}, the fair-bit inequality and Gaussian maximum entropy are classical,
and for the affine link with $g(\theta)=\norm\theta^2/2$ the signed-mean and energy bounds are of the
kind used by \citet{attias2024}, and the signed-mean bound also by \citet{voitovych2025}
(Remark~\ref{rem:affine}). New are the steps that other links and regularizers
require: calibration of the
population risk of the actual loss through the matched upper growth, without which the lower half fails
(Proposition~\ref{prop:growth}), and localization at every law, against a training matrix on which the
error may depend, which together keep the remainder of Step~3 below the signal; and the residual bound
at every law, which gives the upper half, a bound that has a precursor for uniformly stable
algorithms \citep[Theorem~6.8 of the extended version]{steinke2020}.

%% file: sec6_necessity.tex
\section{The hypotheses are needed}\label{sec:necessity}

Without the matched growth or the nonzero slope, the lower half of Theorem~\ref{thm:main} fails, and it
needs dimension of order $n$ up to a logarithm (proofs in Appendix~\ref{app:necessity}).

\begin{proposition}[Matched growth is needed]\label{prop:growth}
Let $\phi(t)=-t$ and $g(\theta)=\norm\theta^2/2$, so that
$D_g(\theta',\theta)=\norm{\theta'-\theta}^2/2\ge\norm{\theta'-\theta}^4/8$ on $W$, whereas no bound
$g(\theta)\le a_+\norm\theta^4$ holds near $0$. Let $n\ge182$, $d\ge32n$ and $\eps=\tau=n^{-2/3}$.
Every proper learner that is $\eps$-accurate on the family satisfies
$\sup_p\Qc_{D_p}(X+\eps G)\ge0.007\,n^{1/3}$ for each $\Qc$. The mean learner $\thh=\Zbar=n^{-1}\sum_iZ_{U_i,i}$ has excess risk at most
$1/(2n)\le\eps$ at every law on the unit ball, and $\sup_p\Qc_{D_p}(X+\eps G)\le3n^{1/3}$ for each $\Qc$.
\end{proposition}

So the least worst-case information over $\eps$-accurate learners is $\Theta(n^{1/3})=o(n)$, whereas a
version of Theorem~\ref{thm:main} that used only the quartic lower curvature would predict
$\Omega(n)$, since $n=\eps^{-3/2}$. But $g$ is quadratic, so at this accuracy the problem needs only
about $\eps^{-1}=n^{2/3}$ samples, the sample size at which Theorem~\ref{thm:main} applies to it with
$r=2$. The tradeoff concerns the optimal sample size, which lower curvature alone does not determine.

\begin{remark*}[Zero slope]
If $\phi'(0)=0$, then $\theta=0$ minimizes $F_D$ at every law $D$ on the unit ball, by convexity of $\phi$
and $g\ge0=g(0)$; the learner that always outputs $0$ is exactly accurate, and its release vanishes.
\end{remark*}

\begin{proposition}[Dimension]\label{prop:dimension}
Let $\phi$ be Borel measurable and $L$-Lipschitz on $[-1,1]$ with $L>0$, let $D$ be any law on $\R^d$ with
$\norm Z\le1$ almost surely, and let $\thh$ be any measurable proper learner; no accuracy is assumed.
For every $\tau>0$ and each $\Qc\in\{\Ii,\Ij,\Is\}$,
\begin{equation}\label{eq:dimension}
 \Qc_D(X+\tau G)\le\min\Big\{n,\ d\log_2\Big(1+\frac{4L\sqrt n}\tau\Big)+\frac1{2\ln2}\Big\}.
\end{equation}
Consequently, if $n\ge c_-\neps$ and some proper learner has $\sup_p\Qc_{D_p}(X+\eps G)\ge cn$ with
$c>0$, then $d\ge cn/(6\log_2n)$ for all $n$ above an explicit threshold that depends only on $c$,
$c_-$, $q$ and $L$.
\end{proposition}

At $\tau=\eps$ the lower half of Theorem~\ref{thm:main} is this premise with $c=c_1/2$, so for large $n$
no proper learner attains the lower half in dimension below $c_1n/(12\log_2n)$: linear dimension is necessary
up to a logarithm. Without noise, rounding the output of an $\eps$-accurate learner
to a maximal $\eps/(L+L_g)$-separated subset of $W$ gives a $2\eps$-accurate proper learner whose
release carries at most $d\log_2(1+2(L+L_g)/\eps)$ bits at every law, so the lower half at $\tau=0$ also needs
$d\gtrsim n/\log n$.

%% file: sec7_separation.tex
\section{Model information does not determine the tradeoff}\label{sec:separation}

The mean learner $\Zbar$ for $f_2(\theta,z)=\norm\theta^2/2-\ip\theta z$ and the empirical risk
minimizer $\norm\Zbar^{-2/3}\Zbar$ for $f_4(\theta,z)=\norm\theta^4/4-\ip\theta z$ are sample-optimal
for their losses and related by a bijection of $W$, so at every law they have the same model
information and tracers with the same decisions. Their releases differ by the factor
$\norm\Zbar^{-2/3}$, about $n^{1/3}$ near the zero-mean law, and at noise level $n^{-2/3}$ the first
release carries at most $3n^{1/3}$ bits at every law of the family and the second at least $n/8$ bits at the
zero-mean law, for $n\ge16$ and $d\ge32n$ (Proposition~\ref{prop:separation}). So conclusions that
involve no noise level and are invariant under reparametrizations of the model, such as those of
\citet{attias2024} and \citet{voitovych2025}, cannot determine the information in noisy loss
differences, although within one problem of the class Theorem~\ref{thm:main} gives the worst-case
information of all accurate learners the same order.

%% file: sec8_measurements.tex
\section{Measurements}\label{sec:measurements}

\input{fig_measurements}

We measured the information in the loss differences of two learners of the class with $r=2$ and
$g(\theta)=\norm\theta^2/2$: \emph{ridge ERM}, with the affine link $\phi(t)=-t$, and \emph{logistic ridge ERM}, with the logistic
link. The laws are drawn from the beta prior with $v=16\eps$ and $\eps=1/(2n)$, and
$d=64n=\Cdim\max\{n,\neps\}$. For
ridge ERM this prior is that of Section~\ref{sec:architecture}; for logistic ridge ERM, $\eps$ only
fixes the prior, which is all that Theorem~\ref{thm:transfer} requires. We estimate the information
per pair $\iota$, the prior average of $\Ii(Y)/n$ in bits, from below by a cross-fitted variational
bound using an odd cubic log-odds model, with 95\% intervals that treat trials as clusters, and we evaluate both halves of
Theorem~\ref{thm:transfer} at each learner's own $M$, $V$ and $\Vres$, exact for ridge ERM and
estimated for logistic ridge ERM (details, and a learner with $r=4$, in Appendix~\ref{app:measurements}).

\paragraph{The Tradeoff.}
Without noise, the estimates of $\iota$ lie between $0.888$ and $0.915$ bits per pair in all nine
runs, for $n$ from $50$ to $3000$. For ridge ERM, Lemma~\ref{lem:ridge-moments} gives
$M^2/(nV)\approx1/(1+18n/d)$ at this prior; at $d=64n$ it is $0.78$ to $0.80$, exact for ridge ERM and estimated for logistic ridge ERM, so against $\tau/(M/n)$, which uses each learner's own signal, the lower half of
Theorem~\ref{thm:transfer} is nearly the same curve in every run, and so are the estimates
(Figure~\ref{fig:measurements}(a)). For $\tau\ge5M/n$ the estimates decay with log--log slopes between $-2.25$ and
$-1.88$, as $1/(1+(\tau/\eps)^2)$ predicts.
Theorem~\ref{thm:transfer} brackets the estimates: wherever the estimate is at least $0.01$, the lower half
is at least $0.63$ times it and the upper half at most $1.51$ times it, and 139 of the 144 measured
points lie in the bracket up to their 95\% intervals; the other five, with estimates of at most
$0.0034$, lie in the tail, where the halves nearly meet, and fall below the lower half, as lower-bound
estimates may. The runs thus show how close Theorem~\ref{thm:transfer} comes to the information at
the learners' own moments; they do not test the constants of Theorem~\ref{thm:main}, which are far
looser because Proposition~\ref{prop:finite} bounds the moments of every learner of the class. Logistic ridge ERM stays at small margins, as localization predicts: every
parameter it returned has norm at most $0.21$, so the runs probe its link only where it is nearly
affine.

\paragraph{Generalization.}
Figure~\ref{fig:measurements}(b) illustrates Corollary~\ref{cor:consequences}(i) for ridge ERM at
$\tau=0$. The prior-averaged gap is exactly $(1-8/n)/n$, and the information-free bound
$\sqrt{\E_\Pi V_p/n}$ exceeds it by only 12--13\%. The range-scaled ld-CMI bounds of
\citet{wangmao2023a} do not decay: on prior average they are at least $2L\sqrt{2\ln2}\,\iota$ with
$L=1$, whose estimate stays between $2.09$ and $2.15$, that is, $128$ to $6290$ times the gap and above the
uniform bound $4L/\sqrt n$.

%% file: fig_measurements.tex
% Figure 2, generated by experiments/code/make_paper_figure.py from experiments/data/paper/*.csv
% (inline coordinates; pgfplots only). Do not edit by hand.
\begin{figure}[t]
\centering
\begin{tikzpicture}
\begin{axis}[scale only axis,width=0.8\linewidth,xmode=log,ymode=log,tick label style={font=\tiny},label style={font=\scriptsize},xlabel style={yshift=2pt},ylabel style={yshift=-2pt},legend style={font=\tiny,draw=none,fill=none,row sep=-1pt},legend cell align=left,clip mode=individual,error bars/error bar style={line width=0.3pt},error bars/error mark options={mark size=0.8pt,line width=0.3pt},height=3.0cm,xmin=0.01,xmax=100,ymin=1e-4,ymax=1.6,
 xtick={0.01,0.1,1,10,100},ytick={1e-4,1e-3,1e-2,1e-1,1},
 xlabel={$\tau/(M/n)$},ylabel={bits per pair},
 legend columns=2,legend style={at={(0.015,0.03)},anchor=south west,/tikz/every even column/.append style={column sep=5pt}}]
\addplot[fill=black!13,draw=none,forget plot] coordinates {(0.01,1) (0.01122,1) (0.012589,1) (0.014125,1) (0.015849,1) (0.017783,1) (0.019953,1) (0.022387,1) (0.025119,1) (0.028184,1) (0.031623,1) (0.035481,1) (0.039811,1) (0.044668,1) (0.050119,1) (0.056234,1) (0.063096,1) (0.070795,1) (0.079433,1) (0.089125,1) (0.1,1) (0.1122,1) (0.12589,1) (0.14125,1) (0.15849,1) (0.17783,1) (0.19953,1) (0.22387,1) (0.25119,1) (0.28184,1) (0.31623,1) (0.35481,1) (0.39811,1) (0.44668,1) (0.50119,1) (0.56234,1) (0.63096,0.91542) (0.70795,0.8) (0.79433,0.69303) (0.89125,0.59507) (1,0.50651) (1.122,0.42749) (1.2589,0.3579) (1.4125,0.29741) (1.5849,0.24545) (1.7783,0.20134) (1.9953,0.16427) (2.2387,0.1334) (2.5119,0.10791) (2.8184,0.086996) (3.1623,0.06994) (3.5481,0.056099) (3.9811,0.044912) (4.4668,0.0359) (5.0119,0.028661) (5.6234,0.022859) (6.3096,0.018216) (7.0795,0.014507) (7.9433,0.011547) (8.9125,0.0091871) (10,0.0073071) (11.22,0.0058103) (12.589,0.0046191) (14.125,0.0036715) (15.849,0.0029179) (17.783,0.0023187) (19.953,0.0018424) (22.387,0.0014639) (25.119,0.0011631) (28.184,0.000924) (31.623,0.00073405) (35.481,0.00058314) (39.811,0.00046324) (44.668,0.00036799) (50.119,0.00029232) (56.234,0.00023221) (63.096,0.00018446) (70.795,0.00014652) (79.433,0.00011639) (89.125,9.2453e-05) (100,7.3439e-05) (100,7.2125e-05) (89.125,9.0798e-05) (79.433,0.0001143) (70.795,0.00014389) (63.096,0.00018114) (56.234,0.00022802) (50.119,0.00028703) (44.668,0.0003613) (39.811,0.00045477) (35.481,0.00057241) (31.623,0.00072043) (28.184,0.00090666) (25.119,0.0011409) (22.387,0.0014356) (19.953,0.0018061) (17.783,0.0022719) (15.849,0.0028572) (14.125,0.0035922) (12.589,0.0045149) (11.22,0.0056722) (10,0.0071223) (8.9125,0.0089372) (7.9433,0.011205) (7.0795,0.014034) (6.3096,0.017555) (5.6234,0.021923) (5.0119,0.027325) (4.4668,0.033973) (3.9811,0.042112) (3.5481,0.052009) (3.1623,0.063947) (2.8184,0.078206) (2.5119,0.095039) (2.2387,0.11464) (1.9953,0.1371) (1.7783,0.16237) (1.5849,0.19021) (1.4125,0.22021) (1.2589,0.25175) (1.122,0.28407) (1,0.31632) (0.89125,0.34768) (0.79433,0.3774) (0.70795,0.40489) (0.63096,0.42975) (0.56234,0.45179) (0.50119,0.47097) (0.44668,0.48741) (0.39811,0.50132) (0.35481,0.51294) (0.31623,0.52256) (0.28184,0.53046) (0.25119,0.53691) (0.22387,0.54215) (0.19953,0.54638) (0.17783,0.54979) (0.15849,0.55253) (0.14125,0.55473) (0.12589,0.55648) (0.1122,0.55788) (0.1,0.559) (0.089125,0.5599) (0.079433,0.56061) (0.070795,0.56117) (0.063096,0.56162) (0.056234,0.56198) (0.050119,0.56227) (0.044668,0.56249) (0.039811,0.56267) (0.035481,0.56282) (0.031623,0.56293) (0.028184,0.56302) (0.025119,0.56309) (0.022387,0.56315) (0.019953,0.56319) (0.017783,0.56323) (0.015849,0.56326) (0.014125,0.56328) (0.012589,0.5633) (0.01122,0.56331) (0.01,0.56332)} -- cycle;
\addplot+[only marks,mark=pentagon,mark size=1.3pt,color=red!80!black,line width=0.4pt,forget plot,error bars/.cd,y dir=both,y explicit] coordinates {(0.077803,0.91057) += (0,0.00209) -= (0,0.00209) (0.1556,0.89594) += (0,0.0022) -= (0,0.0022) (0.25934,0.86034) += (0,0.0025) -= (0,0.0025) (0.51868,0.70976) += (0,0.00323) -= (0,0.00323) (0.77803,0.54005) += (0,0.00356) -= (0,0.00356) (1.556,0.22995) += (0,0.00311) -= (0,0.00311) (2.5934,0.096067) += (0,0.00215) -= (0,0.00215) (5.1868,0.026504) += (0,0.00119) -= (0,0.00119) (7.7803,0.011314) += (0,0.000789) -= (0,0.000788) (15.56,0.0032846) += (0,0.000424) -= (0,0.000424) (25.934,0.0010194) += (0,0.000239) -= (0,0.000239) (51.868,0.00015766) += (0,0.000124) -= (0,0.000124)};
\addplot+[only marks,mark=o,mark size=1.1pt,color=red!80!black,line width=0.4pt,forget plot,error bars/.cd,y dir=both,y explicit] coordinates {(0.067853,0.90071) += (0,0.00219) -= (0,0.00219) (0.13571,0.88965) += (0,0.00229) -= (0,0.00229) (0.22618,0.8643) += (0,0.00247) -= (0,0.00247) (0.45236,0.74672) += (0,0.00301) -= (0,0.00301) (0.67853,0.59745) += (0,0.0035) -= (0,0.0035) (1.3571,0.27965) += (0,0.00326) -= (0,0.00326) (2.2618,0.12406) += (0,0.00238) -= (0,0.00238) (4.5236,0.034399) += (0,0.00134) -= (0,0.00134) (6.7854,0.015533) += (0,0.000935) -= (0,0.000935) (13.571,0.0033811) += (0,0.000417) -= (0,0.000417) (22.618,0.0016508) += (0,0.000315) -= (0,0.000315) (45.236,0.00025128) += (0,0.000119) -= (0,0.000119) (67.853,0.00016438) += (0,9.68e-05) -= (0,9.68e-05)};
\addplot+[only marks,mark=square,mark size=1.1pt,color=red!80!black,line width=0.4pt,forget plot,error bars/.cd,y dir=both,y explicit] coordinates {(0.062559,0.89021) += (0,0.00246) -= (0,0.00246) (0.12512,0.88077) += (0,0.0025) -= (0,0.0025) (0.20853,0.85984) += (0,0.00272) -= (0,0.00272) (0.41706,0.75941) += (0,0.00319) -= (0,0.00319) (0.62559,0.62416) += (0,0.00343) -= (0,0.00343) (1.2512,0.31231) += (0,0.00362) -= (0,0.00362) (2.0853,0.13914) += (0,0.00272) -= (0,0.00272) (4.1706,0.039842) += (0,0.00157) -= (0,0.00157) (6.2559,0.017855) += (0,0.00101) -= (0,0.00101) (12.512,0.0047808) += (0,0.000557) -= (0,0.000557) (20.853,0.0016707) += (0,0.000328) -= (0,0.000328) (41.706,0.00035179) += (0,0.000161) -= (0,0.000161) (62.559,0.00017934) += (0,0.000108) -= (0,0.000108)};
\addplot+[only marks,mark=triangle,mark size=1.1pt,color=red!80!black,line width=0.4pt,forget plot,error bars/.cd,y dir=both,y explicit] coordinates {(0.060738,0.88637) += (0,0.00222) -= (0,0.00223) (0.12148,0.87708) += (0,0.00232) -= (0,0.00232) (0.20246,0.85516) += (0,0.0025) -= (0,0.0025) (0.40492,0.76128) += (0,0.00295) -= (0,0.00295) (0.60738,0.63225) += (0,0.00333) -= (0,0.00333) (1.2148,0.32557) += (0,0.00331) -= (0,0.00331) (2.0246,0.14808) += (0,0.00259) -= (0,0.00259) (4.0492,0.041419) += (0,0.00149) -= (0,0.00149) (6.0738,0.019574) += (0,0.000984) -= (0,0.000985) (12.148,0.0048809) += (0,0.000499) -= (0,0.000499) (20.246,0.0012874) += (0,0.000282) -= (0,0.000282) (40.492,0.00040198) += (0,0.000173) -= (0,0.000173) (60.738,0.0001888) += (0,9.93e-05) -= (0,9.93e-05)};
\addplot+[only marks,mark=pentagon*,mark size=1.3pt,color=blue!70!black,line width=0.4pt,forget plot,error bars/.cd,y dir=both,y explicit] coordinates {(0.017857,0.9131) += (0,0.00205) -= (0,0.00205) (0.035714,0.91239) += (0,0.00206) -= (0,0.00206) (0.059524,0.91062) += (0,0.00208) -= (0,0.00207) (0.11905,0.90237) += (0,0.00214) -= (0,0.00214) (0.17857,0.88731) += (0,0.0023) -= (0,0.0023) (0.35714,0.81038) += (0,0.00283) -= (0,0.00283) (0.59524,0.66099) += (0,0.00338) -= (0,0.00338) (1.1905,0.33333) += (0,0.00334) -= (0,0.00334) (1.7857,0.18478) += (0,0.00283) -= (0,0.00283) (3.5714,0.054188) += (0,0.00168) -= (0,0.00168) (5.9524,0.019356) += (0,0.00101) -= (0,0.00101) (11.905,0.0053371) += (0,0.000535) -= (0,0.000535) (17.857,0.0022832) += (0,0.000364) -= (0,0.000364) (35.714,0.00067165) += (0,0.000192) -= (0,0.000192) (59.524,0.00013252) += (0,9.85e-05) -= (0,9.85e-05)};
\addplot+[only marks,mark=*,mark size=1.1pt,color=blue!70!black,line width=0.4pt,forget plot,error bars/.cd,y dir=both,y explicit] coordinates {(0.016304,0.90603) += (0,0.00212) -= (0,0.00212) (0.032609,0.90558) += (0,0.00212) -= (0,0.00212) (0.054348,0.90398) += (0,0.00214) -= (0,0.00215) (0.1087,0.89677) += (0,0.00218) -= (0,0.00218) (0.16304,0.88367) += (0,0.00233) -= (0,0.00233) (0.32609,0.82105) += (0,0.00278) -= (0,0.00278) (0.54348,0.68795) += (0,0.00327) -= (0,0.00327) (1.087,0.37177) += (0,0.00348) -= (0,0.00348) (1.6304,0.21191) += (0,0.00297) -= (0,0.00297) (3.2609,0.061967) += (0,0.00179) -= (0,0.00179) (5.4348,0.024359) += (0,0.00116) -= (0,0.00116) (10.87,0.0059909) += (0,0.00059) -= (0,0.00059) (16.304,0.0027633) += (0,0.000393) -= (0,0.000393) (32.609,0.00067351) += (0,0.000213) -= (0,0.000213) (54.348,0.00019781) += (0,0.000113) -= (0,0.000113)};
\addplot+[only marks,mark=square*,mark size=1.1pt,color=blue!70!black,line width=0.4pt,forget plot,error bars/.cd,y dir=both,y explicit] coordinates {(0.015411,0.89629) += (0,0.00302) -= (0,0.00302) (0.030822,0.89568) += (0,0.00305) -= (0,0.00305) (0.05137,0.89451) += (0,0.00303) -= (0,0.00303) (0.10274,0.88858) += (0,0.0031) -= (0,0.0031) (0.15411,0.87731) += (0,0.00335) -= (0,0.00335) (0.30822,0.81769) += (0,0.00427) -= (0,0.00427) (0.5137,0.69882) += (0,0.0047) -= (0,0.0047) (1.0274,0.40061) += (0,0.00523) -= (0,0.00523) (1.5411,0.23182) += (0,0.00421) -= (0,0.00421) (3.0822,0.071346) += (0,0.00281) -= (0,0.00281) (5.137,0.028574) += (0,0.0019) -= (0,0.0019) (10.274,0.0068946) += (0,0.00099) -= (0,0.00099) (15.411,0.0026125) += (0,0.000565) -= (0,0.000565) (30.822,0.0005245) += (0,0.000279) -= (0,0.000279) (51.37,0.00044398) += (0,0.000228) -= (0,0.000228)};
\addplot+[only marks,mark=triangle*,mark size=1.1pt,color=blue!70!black,line width=0.4pt,forget plot,error bars/.cd,y dir=both,y explicit] coordinates {(0.015121,0.89087) += (0,0.00246) -= (0,0.00246) (0.030242,0.89025) += (0,0.00247) -= (0,0.00247) (0.050403,0.88902) += (0,0.00247) -= (0,0.00247) (0.10081,0.88308) += (0,0.00251) -= (0,0.00251) (0.15121,0.87226) += (0,0.00255) -= (0,0.00254) (0.30242,0.81837) += (0,0.00267) -= (0,0.00267) (0.50403,0.70244) += (0,0.00335) -= (0,0.00335) (1.0081,0.40398) += (0,0.0034) -= (0,0.0034) (1.5121,0.2384) += (0,0.00297) -= (0,0.00297) (3.0242,0.072361) += (0,0.00197) -= (0,0.00197) (5.0403,0.027055) += (0,0.00126) -= (0,0.00126) (10.081,0.0074346) += (0,0.000675) -= (0,0.000675) (15.121,0.003118) += (0,0.000425) -= (0,0.000425) (30.242,0.00088982) += (0,0.000238) -= (0,0.000238) (50.403,0.00033531) += (0,0.000131) -= (0,0.000131)};
\addplot+[only marks,mark=diamond*,mark size=1.1pt,color=blue!70!black,line width=0.4pt,forget plot,error bars/.cd,y dir=both,y explicit] coordinates {(0.01504,0.88788) += (0,0.00296) -= (0,0.00296) (0.03008,0.88731) += (0,0.00299) -= (0,0.00299) (0.050134,0.88599) += (0,0.00295) -= (0,0.00295) (0.10027,0.87926) += (0,0.00322) -= (0,0.00322) (0.1504,0.8696) += (0,0.00297) -= (0,0.00297) (0.3008,0.81537) += (0,0.00416) -= (0,0.00416) (0.50134,0.69956) += (0,0.00326) -= (0,0.00326) (1.0027,0.41243) += (0,0.00534) -= (0,0.00534) (1.504,0.2396) += (0,0.00399) -= (0,0.00399) (3.008,0.075978) += (0,0.00288) -= (0,0.00288) (5.0134,0.026627) += (0,0.00192) -= (0,0.00192) (10.027,0.0070827) += (0,0.000822) -= (0,0.000822) (15.04,0.0030973) += (0,0.000621) -= (0,0.000621) (30.08,0.00068735) += (0,0.000275) -= (0,0.000275) (50.134,0.00015533) += (0,0.000127) -= (0,0.000127)};
\addlegendimage{only marks,mark=*,mark size=1.1pt,color=blue!70!black}\addlegendentry{ridge ERM}
\addlegendimage{only marks,mark=o,mark size=1.1pt,color=red!80!black}\addlegendentry{logistic ridge ERM}
\addlegendimage{area legend,fill=black!13,draw=none}\addlegendentry{Theorem~\ref{thm:transfer}}
\addlegendimage{only marks,mark=pentagon,mark size=1.3pt,black}\addlegendentry{$n=50$}
\addlegendimage{only marks,mark=o,mark size=1.1pt,black}\addlegendentry{$n=100$}
\addlegendimage{only marks,mark=square,mark size=1.1pt,black}\addlegendentry{$n=300$}
\addlegendimage{only marks,mark=triangle,mark size=1.1pt,black}\addlegendentry{$n=1000$}
\addlegendimage{only marks,mark=diamond,mark size=1.1pt,black}\addlegendentry{$n=3000$}
\node[anchor=north east,font=\scriptsize] at (rel axis cs:0.985,0.97) {(a)};
\end{axis}
\end{tikzpicture}\\[1pt]
\begin{tikzpicture}
\begin{axis}[scale only axis,width=0.8\linewidth,xmode=log,ymode=log,tick label style={font=\tiny},label style={font=\scriptsize},xlabel style={yshift=2pt},ylabel style={yshift=-2pt},legend style={font=\tiny,draw=none,fill=none,row sep=-1pt},legend cell align=left,clip mode=individual,error bars/error bar style={line width=0.3pt},error bars/error mark options={mark size=0.8pt,line width=0.3pt},height=2.4cm,xmin=35,xmax=4500,ymin=1e-4,ymax=5,
 xtick={100,1000},ytick={1e-4,1e-3,1e-2,1e-1,1},
 xlabel={$n$},ylabel={gap and bounds},
 legend columns=4,legend style={at={(0.5,1.03)},anchor=south,/tikz/every even column/.append style={column sep=4pt}}]
\addplot[black,thin,domain=50:3000,samples=60] {(1-8/x)/x};
\addlegendentry{gap}
\addplot[black!60,thin,densely dashed,domain=50:3000,samples=40] {4/sqrt(x)};
\addlegendentry{$4L/\sqrt n$}
\addplot[black,densely dotted,mark=o,mark size=1.1pt,mark options={solid}] coordinates {(50,0.018776) (100,0.010331) (300,0.0036613) (1000,0.0011218) (3000,0.00037618)};
\addlegendentry{$\sqrt{\E_\Pi V_p/n}$}
\addplot+[only marks,mark=*,mark size=1.2pt,color=blue!70!black,error bars/.cd,y dir=both,y explicit] coordinates {(50,2.1507) += (0,0.00482) -= (0,0.00482) (100,2.1342) += (0,0.00499) -= (0,0.00499) (300,2.1112) += (0,0.0071) -= (0,0.00711) (1000,2.0981) += (0,0.00579) -= (0,0.00578) (3000,2.0912) += (0,0.00692) -= (0,0.00693)};
\addlegendentry{$2L\sqrt{2\ln2}\,\iota$}
\node[anchor=south west,font=\scriptsize] at (rel axis cs:0.015,0.03) {(b)};
\end{axis}
\end{tikzpicture}
\caption{Measurements at $d=64n$. Values are averages over the prior; bars are 95\% intervals. (a)~Lower-bound estimates of the bits per pair $\iota$ against $\tau/(M/n)$ for ridge ERM (filled) and logistic ridge ERM (open); shaded: between the two halves of Theorem~\ref{thm:transfer} at each learner's own $M$, $V$ and $\Vres$. (b)~Ridge ERM at $\tau=0$: the gap, its bounds $4L/\sqrt n$ and $\sqrt{\E_\Pi V_p/n}$, and the lower bound $2L\sqrt{2\ln2}\,\iota$ on the range-scaled ld-CMI bounds; $L=1$.}
\label{fig:measurements}
\end{figure}
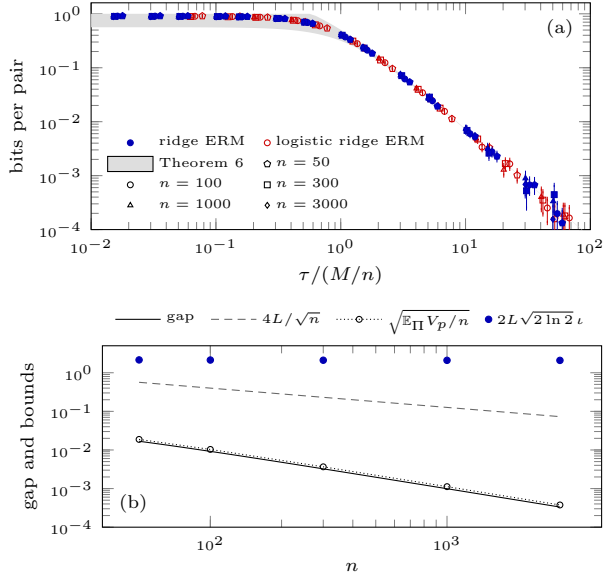

%% file: sec9_limitations.tex
\section{Limitations and outlook}\label{sec:limitations}

Our results concern proper learners, linear predictors with a smooth convex link and a fixed matched
power regularizer or none, product distributions on a scaled sign cube, dimension $d\gtrsim n$ (necessary up to
a logarithm) and the optimal sample size only. The lower half is an average over a prior that concentrates near the zero-mean law, so it bounds the worst case over the
family; bounds at a single law for learners accurate on the whole family are open. The constants are
explicit but far apart: at the limiting normalized moments and $c_\pm=1$, $C_1/c_1\approx1.6\cdot10^4$ for ridge and
$2\cdot10^5$ for the logistic link with $g(\theta)=\norm\theta^4/4$, whose explicit lower bound at $d=32n$
changes sign at $n\approx6.5\cdot10^4$ (Appendix~\ref{app:constants}). The upper half concerns an
observer who sees only the loss differences and the candidates; small
information in this sense is not a privacy guarantee. The range $n/\log n\lesssim d\lesssim n$ and
constants informative at moderate $n$ remain open.

%% file: app_A_notation.tex
\section{Notation}\label{app:notation}

This appendix collects the notation of the paper. Symbols that are used in a single proof or in a
single appendix are defined where they are used.

\paragraph{Conventions.}
Information is measured in bits; $\ln$ is the natural logarithm and $\log_2$ the binary one, and a
quantity computed in nats is divided by $\ln2$ before it is compared with bits. $\E_p$ and $\Pp_p$
are taken when the data law is $D_p$, over the candidates, the selection bits, the learner's seed $R$
and the noise $G$. $\E_\Pi$ and $\Pp_\Pi$ first draw $p$ from a prior $\Pi$, which is the beta prior
$\Pi_\kappa$ unless stated otherwise, and then take $\E_p$ and $\Pp_p$; $I_{D_p}$ denotes mutual
information when the data law is $D_p$. The norm $\norm\cdot$ is Euclidean, $\opnorm\cdot$ is the
operator norm, $\ip\cdot\cdot$ is the Euclidean inner product and ${}^\top$ denotes the transpose.
Entropies are set in roman type: $\Hent$ is Shannon entropy, $\hdiff$ is differential entropy and
$\htwo$ is the binary entropy function in bits, while the italic $H$ is always the curvature bound of
the link. The index $i\in[n]$ labels pairs, $j\in\{0,1\}$ the candidate within a pair and $k\in[d]$
coordinates; the coordinate index comes last, as in $Z_{j,i,k}$. The symbols $O$, $\Omega$, $\Theta$,
$\asymp$ and $\lesssim$ hide constants that depend only on the constants of
Assumption~\ref{ass:class} and on $c_\pm$.

\begingroup
\small
% Local helpers for the three tables of this appendix (defined inside this group only):
% \notationrow{symbol}{meaning}{where defined} typesets one row with ragged-right cells;
% \notationblock{title} typesets the title row of a block.
\newcommand{\notationrow}[3]{\raggedright#1&\raggedright#2&\raggedright#3\tabularnewline}
\newcommand{\notationblock}[1]{\toprule
  \multicolumn{2}{@{}l}{\textbf{#1}}&\emph{Defined in}\tabularnewline\midrule}

\par\medskip\noindent
\begin{tabular}{@{}p{.17\linewidth}p{.64\linewidth}p{.13\linewidth}@{}}
\notationblock{Problem, learner and accuracy}
\notationrow{$d$, $W=B_d(1)$, $\mathcal Z$}
 {dimension; parameter set, the closed unit ball of $\R^d$; data cube $\{\pm1/\sqrt d\}^d$, on
  which $\norm z=1$}
 {Section~\ref{sec:setting}}
\notationrow{$p=(p_k)$, $D_p$, $\mu_p$}
 {law parameter in $[-1,1]^d$; product law on $\mathcal Z$ with independent coordinates and
  $\Pp(\sqrt d\,Z_k=1)=(1+p_k)/2$; its mean $\mu_p=p/\sqrt d$}
 {Section~\ref{sec:setting}}
\notationrow{the family; $D$}
 {the laws $\{D_p:p\in[-1,1]^d\}$; a general law on the unit ball of $\R^d$}
 {Section~\ref{sec:setting}}
\notationrow{$\phi$, $b$, $L$, $H$}
 {convex link; slope parameter $b=-\phi'(0)>0$; bounds $|\phi'|\le L$ and $0\le\phi''\le H$ on
  $[-1,1]$}
 {Assumption~\ref{ass:class}}
\notationrow{$g$, $L_g$, $D_g$}
 {convex regularizer with $g(0)=0$ and $\nabla g(0)=0$; bound $\sup_W\norm{\nabla g}\le L_g$;
  Bregman divergence $D_g(\theta',\theta)=g(\theta')-g(\theta)-\ip{\nabla g(\theta)}{\theta'-\theta}$}
 {Assumption~\ref{ass:class}}
\notationrow{$r$, $q$, $a_-$, $a_+$}
 {common power $r\ge2$ of curvature and growth; conjugate exponent $q=r/(r-1)\in(1,2]$; constants
  with $D_g(\theta',\theta)\ge a_-\norm{\theta'-\theta}^r$ and $0\le g(\theta)\le a_+\norm\theta^r$
  on $W$}
 {Assumption~\ref{ass:class}}
\notationrow{$f$, $F_D$, $F_p$, $\theta_p^*$}
 {loss $f(\theta,z)=g(\theta)+\phi(\ip\theta z)$; population risk $F_D(\theta)=\E_Df(\theta,Z)$;
  $F_p=F_{D_p}$; the minimizer of $F_p$ over $W$ (and $\theta_D^*$ that of $F_D$)}
 {Section~\ref{sec:setting}}
\notationrow{$S_n$, $A_n$, $R$, $\thh$}
 {training set of $n$ points; proper learner; its seed, independent of the data; its output
  $\thh=A_n(S_n;R)\in W$}
 {Section~\ref{sec:setting}}
\notationrow{$\eps$}
 {accuracy on the family: $\E_pF_p(\thh)-\min_WF_p\le\eps$ for every $p$}
 {Section~\ref{sec:setting}}
\notationrow{$\widehat F_{S_n}$, $\gen_D$, $\gen_p$}
 {training risk $\widehat F_{S_n}(\theta)=n^{-1}\sum_if(\theta,Z_{U_i,i})$; expected generalization gap
  $\gen_D=\E[F_D(\thh)-\widehat F_{S_n}(\thh)]$ at a law $D$; $\gen_p=\gen_{D_p}$}
 {Section~\ref{sec:setting}}
\notationrow{$n$, $\neps$, $c_-$, $c_+$}
 {sample size, equal to the number of pairs; optimal sample-size scale
  $\neps=\eps^{-2/q}=\eps^{-2+2/r}$; multipliers with $c_-\neps\le n\le c_+\neps$}
 {Section~\ref{sec:result}}
\notationrow{$\Cdim$, $\eps_0$}
 {dimension multiplier in $d\ge\Cdim\max\{n,\neps\}$; accuracy threshold}
 {Theorem~\ref{thm:main}}
\notationrow{$e$}
 {error $\thh-\theta_p^*$ of the learner}
 {Appendix~\ref{app:localization}}
\notationrow{$\epsdp$, $\delta$}
 {privacy parameters}
 {Section~\ref{sec:consequences}}
\notationrow{$\Zbar$, $T_4$}
 {training mean $n^{-1}\sum_iZ_{U_i,i}$; the map $T_4(t)=\norm t^{-2/3}t$ for $0<\norm t\le1$,
  $T_4(0)=0$}
 {Section~\ref{sec:result} and Appendix~\ref{app:separation}}
\notationrow{$f_2$, $f_4$}
 {the losses $\norm\theta^2/2-\ip\theta z$ and $\norm\theta^4/4-\ip\theta z$ (affine link, $r=2$ and
  $r=4$); their empirical risk minimizers over $W$ are $\Zbar$ and $T_4(\Zbar)$}
 {Section~\ref{sec:separation}}
\bottomrule
\end{tabular}

\par\medskip\noindent
\begin{tabular}{@{}p{.17\linewidth}p{.64\linewidth}p{.13\linewidth}@{}}
\notationblock{Supersample, release and information}
\notationrow{$\Zt=(Z_{j,i})$}
 {supersample of $2n$ i.i.d.\ candidates, $j\in\{0,1\}$, $i\in[n]$; $\tilde z$ is a realization}
 {Section~\ref{sec:setting}}
\notationrow{$U=(U_i)$, $\sigma_i$}
 {selection bits, i.i.d.\ fair and independent of $\Zt$; signs $\sigma_i=1-2U_i$}
 {Section~\ref{sec:setting}}
\notationrow{$\Ztr_i$, $\Zgh_i$}
 {training candidate $Z_{U_i,i}$ and ghost $Z_{1-U_i,i}$ of pair $i$}
 {Section~\ref{sec:setting}}
\notationrow{$X=(X_i)$}
 {loss-difference release,
  $X_i=f(\thh,Z_{1,i})-f(\thh,Z_{0,i})=\phi(\ip\thh{Z_{1,i}})-\phi(\ip\thh{Z_{0,i}})$; $\sigma_iX_i$
  is the ghost-minus-training loss of pair $i$}
 {Definition~\ref{def:release}}
\notationrow{$G$, $\tau$, $Y$}
 {noise $G\sim\Normal(0,I_n)$, independent of everything else; its standard deviation $\tau\ge0$;
  noisy release $Y=X+\tau G$}
 {Definition~\ref{def:release}}
\notationrow{$\Ii$, $\Ij$, $\Is$; $\Qc$}
 {summed individual information $\Ii(Y)=\sum_iI(Y_i;U_i)$, joint information $\Ij(Y)=I(Y;U)$ and
  information given the candidates $\Is(Y)=I(Y;U\mid\Zt)$, with the law as a subscript, as in
  $\Is_{D_p}$; $\Qc$ is any one of the three}
 {Definition~\ref{def:release}}
\notationrow{ld-CMI}
 {$I(X_i;U_i\mid\Zt)$ for pair $i$; summed over pairs, $\sum_iI(X_i;U_i\mid\Zt)$, and
  $\sum_iI(Y_i;U_i\mid\Zt)$ with noise}
 {Sections~\ref{sec:intro} and~\ref{sec:setting}}
\notationrow{$\Phi_{\tilde z}$}
 {release map of $\tilde z$, $\Phi_{\tilde z}(\theta)_i=\phi(\ip\theta{z_{1,i}})-\phi(\ip\theta{z_{0,i}})$;
  thus $X=\Phi_{\Zt}(\thh)$}
 {Appendices~\ref{app:dp} and~\ref{app:dimension}}
\notationrow{$x_p(\Zt)$, $X_p^*(\Zt)$}
 {a reference, measurable in $p$ and $\Zt$ only; the release $\Phi_{\Zt}(\theta_p^*)$ of the
  population minimizer}
 {Theorem~\ref{thm:transfer}, Proposition~\ref{prop:finite}}
\notationrow{$M_p$}
 {signed mean $\sum_i\E_p\sigma_iX_i=n\gen_p$ at $D_p$}
 {Appendix~\ref{app:fairbit}}
\notationrow{$V_p$}
 {energy $\sum_i\E_pX_i^2$ at $D_p$}
 {Section~\ref{sec:consequences}}
\notationrow{$\Btr$, $\Delta$}
 {centred training matrix with rows $(Z_{U_i,i}-\mu_p)^\top$; pair matrix with rows
  $\Delta_i=(Z_{1,i}-Z_{0,i})^\top$}
 {Lemma~\ref{lem:matrix}}
\notationrow{$z_c$}
 {candidate centre $(2n)^{-1}\sum_{j,i}Z_{j,i}$, the mean of all $2n$ candidates}
 {Section~\ref{sec:result}}
\notationrow{$m$, $s_i$, $u$}
 {ghost-margin mean $\ip\thh{\mu_p}$; centred training margin $\ip\thh{Z_{U_i,i}-\mu_p}$; link
  remainder $u(t)=\phi(t)-\phi(0)+bt$}
 {Section~\ref{sec:moments}}
\notationrow{$\sighat_i$, $\widehat U_i$, $\lambda$}
 {output sign and decoded bit of the clipped decoder; its scale $\lambda=V/(2M)$}
 {Appendix~\ref{app:decoder}}
\notationrow{$\J$}
 {$\J(t)=1-\htwo((1+t)/2)$, the information of a fair bit seen through a binary symmetric channel
  with correlation $t$}
 {Appendix~\ref{app:fairbit}}
\notationrow{$\xi$}
 {bound on the false-positive rate of a tracer, the probability that it flags a fresh point}
 {Section~\ref{sec:intro}}
\notationrow{$\iota$}
 {information per pair $n^{-1}\E_\Pi\Ii_{D_p}(X+\tau G)$ in bits, a function of $\tau$; $\iota(0)$ is
  its value at $\tau=0$}
 {Section~\ref{sec:measurements}}
\bottomrule
\end{tabular}

\par\medskip\noindent
\begin{tabular}{@{}p{.17\linewidth}p{.64\linewidth}p{.13\linewidth}@{}}
\notationblock{Prior, moments and constants}
\notationrow{$\Pi_\kappa$, $\pi_\kappa$, $\kappa$}
 {beta prior on $p$: i.i.d.\ coordinates with density $\pi_\kappa(t)=c_\kappa(1-t^2)^{\kappa/2-1}$ on
  $(-1,1)$, where $c_\kappa$ normalizes; its parameter $\kappa>0$}
 {Section~\ref{sec:architecture}}
\notationrow{$v$}
 {prior variance $v=\E_{\Pi_\kappa}\norm{\mu_p}^2=1/(\kappa+1)$; for Theorem~\ref{thm:main},
  $v=(4\eps/\beta)^{2/q}$}
 {Section~\ref{sec:architecture}}
\notationrow{$t_0$, $\beta$, $\theta_c$}
 {calibration constants $t_0=(b/(ra_++b))^{1/(r-1)}$ and $\beta=bt_0/q$; comparator
  $\theta_c=t_0\norm{\mu_p}^{q-2}\mu_p$, and $\theta_c=0$ if $\mu_p=0$}
 {Sections~\ref{sec:architecture} and~\ref{sec:moments}}
\notationrow{$M$, $V$, $\Vres$}
 {budgets: a lower bound on the signed mean and upper bounds on the energy and on the residual, as in
  Theorem~\ref{thm:transfer}; Appendix~\ref{app:mainproof} gives their finite values for the learners
  of Theorem~\ref{thm:main}}
 {Theorem~\ref{thm:transfer}, Appendix~\ref{app:mainproof}}
\notationrow{$c_M$, $C_V$, $C_{\mathrm{res}}$}
 {order constants: the signed mean is at least $c_Mn\eps$, the energy at most $C_Vn\eps^2$ and the
  residual at most $C_{\mathrm{res}}n\eps^2$}
 {Proposition~\ref{prop:finite}}
\notationrow{$c_1$, $C_1$}
 {constants of the two-sided bound}
 {Theorem~\ref{thm:main}}
\notationrow{$\Mbar$, $\Vbar$, $\Vresbar$}
 {normalized moments $M/(n\eps)$, $V/(n\eps^2)$ and $\Vres/(n\eps^2)$}
 {Appendix~\ref{app:constants}}
\notationrow{$C_0$, $w$}
 {$\norm{\theta_p^*}\le C_0\norm{\mu_p}^{q-1}$ with $C_0=(b/(2a_-))^{1/(r-1)}$; localization bound
  $\E_p\norm{\thh-\theta_p^*}^2\le w=\min\{4,(\eps/a_-)^{2/r}\}$}
 {Section~\ref{sec:moments}}
\notationrow{$w_\theta$, $Q_\theta$, $C_\theta$}
 {bounds on $\E_\Pi\norm\thh^2$, on $\E_\Pi[\norm\thh^2\norm{\mu_p}^2]$ and on
  $\E_\Pi[\norm\thh^2\opnorm{\Btr}^2]$}
 {Appendix~\ref{app:mainproof}}
\notationrow{$\rho$, $\cnet$, $K_0$, $K_1$, $K_0^*$}
 {aspect ratio $n/d$; constants of the tail bounds; $K_0^*=K_0(1/32)<2.031$ bounds $K_0$ when
  $d\ge32n$}
 {Lemma~\ref{lem:matrix}, Appendix~\ref{app:tails}}
\notationrow{$N$, $\omega$, $\bar\omega$, $\Tb$}
 {tail variable, weight and cap; weighted tail function $\Tb_{\bar\omega}(w;K_0,K_1)$}
 {Lemma~\ref{lem:weighted}}
\bottomrule
\end{tabular}
\par
\endgroup

%% file: app_B_tools.tex
\section{Probability tools}\label{app:tools}

This appendix proves Lemmas~\ref{lem:fp}--\ref{lem:weighted} and derives the moments of the beta prior
that the proofs of Appendix~\ref{app:mainproof} use.

\subsection{Beta-Prior Moments and Fingerprinting}\label{app:prior-moments}

Fix $\kappa>0$ and let $p\sim\Pi_\kappa$, so that $p_1,\dots,p_d$ are independent with density
\[
 \pi_\kappa(t)=c_\kappa(1-t^2)^{\kappa/2-1},\qquad -1<t<1,
\]
where $c_\kappa$ is the normalizing constant; equivalently, $(1+p_k)/2$ has the
$\mathrm{Beta}(\kappa/2,\kappa/2)$ distribution. Write $v=1/(\kappa+1)$. On $(-1,1)$ the density
satisfies $\frac{d}{dt}[(1-t^2)\pi_\kappa(t)]=-\kappa t\,\pi_\kappa(t)$, and
$(1-t^2)\pi_\kappa(t)=c_\kappa(1-t^2)^{\kappa/2}$ tends to zero at both endpoints because $\kappa>0$.
Integration by parts therefore gives, for every polynomial $\Upsilon$,
\begin{equation}\label{eq:prior-ibp}
 \kappa\,\E\big[p_k\Upsilon(p_k)\big]=\E\big[(1-p_k^2)\Upsilon'(p_k)\big].
\end{equation}
With $\Upsilon(t)=1$, $t$ and $t^3$, \eqref{eq:prior-ibp} gives $\E p_k=0$, $(\kappa+1)\E p_k^2=1$
and $(\kappa+3)\E p_k^4=3\E p_k^2$. Since $\kappa+3=(1+2v)/v$ and the coordinates are independent,
\begin{equation}\label{eq:prior-moments}
 \E p_k=0,\quad \E p_k^2=v,\quad \E p_k^4=\frac{3v^2}{1+2v},\quad
 \E\norm{\mu_p}^2=v,\quad \E\norm{\mu_p}^4\le v^2(1+2/d).
\end{equation}
For the last bound, expand $\norm{\mu_p}^4=d^{-2}(\sum_kp_k^2)^2$:
$\E\norm{\mu_p}^4=d^{-2}\{d\,\E p_1^4+d(d-1)v^2\}=v^2\{1+(3/(1+2v)-1)/d\}\le v^2(1+2/d)$.

For $1\le q\le2$, these moments give
\begin{align}
 \E\norm{\mu_p}^{2(q-1)}&\le v^{q-1},\label{eq:prior-concave}\\
 \E\norm{\mu_p}^{2q}&\le v^q(1+2/d)^{q/2},\label{eq:prior-high}\\
 \E\norm{\mu_p}^{q}&\ge v^{q/2}\big(1+1/(dv)\big)^{-q/2}.\label{eq:prior-low}
\end{align}
The first two follow from Jensen's inequality for the concave maps $x\mapsto x^{q-1}$ and
$x\mapsto x^{q/2}$ on $[0,\infty)$, applied to $\norm{\mu_p}^2$ and to $\norm{\mu_p}^4$, and from the
bounds on $\E\norm{\mu_p}^2$ and $\E\norm{\mu_p}^4$ in \eqref{eq:prior-moments}. For the third,
\eqref{eq:prior-moments} gives
$\var(\sum_kp_k^2)=d(\E p_1^4-v^2)=2dv^2(1-v)/(1+2v)\le dv$, hence
$\E\norm p^4\le(dv)^2\{1+1/(dv)\}$ and, by the Cauchy--Schwarz inequality,
$\E\norm p^3\le(\E\norm p^2\,\E\norm p^4)^{1/2}\le(dv)^{3/2}\{1+1/(dv)\}^{1/2}$. The Cauchy--Schwarz
inequality applied to $\norm p^2=\norm p^{1/2}\norm p^{3/2}$ gives
$(dv)^2=(\E\norm p^2)^2\le\E\norm p\,\E\norm p^3$, so $\E\norm p\ge(dv)^{1/2}\{1+1/(dv)\}^{-1/2}$.
Jensen's inequality for the convex map $x\mapsto x^q$ and division by $d^{q/2}$, since
$\norm{\mu_p}=\norm p/\sqrt d$, prove \eqref{eq:prior-low}. These bounds are loose, but explicit and
uniform in the dimension.

\begin{restate}[Beta fingerprinting]{Lemma}{lem:fp}
Let $\Pi=\Pi_\kappa$ with $\kappa>0$. For every proper learner,
\[
 \E_\Pi\Big\langle\thh,\sum_{i=1}^n(Z_{U_i,i}-\mu_p)\Big\rangle=\kappa\,\E_\Pi\ip\thh{\mu_p}.
\]
\end{restate}

This is the beta-prior fingerprinting lemma of \citet{steinke2017};
\citet[Lemma~3.8]{voitovych2025} extend it to sparse sign data. We include the proof to fix the
normalization. It uses only that the output $\thh=A_n(S_n;R)$ is bounded and jointly measurable in
$(S_n,R)$, with the seed $R$ independent of $p$ and of the data. Recall that, given $p$, the training
set $S_n=(Z_{U_i,i})_{i\in[n]}$ consists of $n$ independent draws from $D_p$, because $U$ is
independent of the candidates.

\begin{proof}
Fix a coordinate $k$ and the other coordinates $p_{-k}$ of $p$; until the last step, all
expectations are conditional on $p_{-k}$. For $z=(z_1,\dots,z_n)\in\mathcal Z^n$,
\[
 \Pp_p(S_n=z)=\prod_{i=1}^n\prod_{k'=1}^d\frac{1+\sqrt d\,p_{k'}z_{i,k'}}2 .
\]
Let $\bar\theta_k(z)=\E[\thh_k\mid S_n=z]$ be the $k$-th output coordinate averaged over the seed; it
is bounded, and it does not depend on $p$ because $R$ is independent of $p$ and of the data. Since
$\mathcal Z^n$ is finite and each factor above is affine in $p_k$, the function
\[
 \Upsilon_k(t)=\E[\thh_k\mid p_k=t]=\sum_{z\in\mathcal Z^n}\bar\theta_k(z)\,\Pp_p(S_n=z)\big|_{p_k=t}
\]
is a polynomial of degree at most $n$. For $|t|<1$ and $z_{i,k}=\pm1/\sqrt d$,
\[
 \frac{\partial}{\partial t}\ln\frac{1+\sqrt d\,t\,z_{i,k}}2
 =\frac{\sqrt d\,z_{i,k}}{1+\sqrt d\,t\,z_{i,k}}
 =\frac{\sqrt d\,(z_{i,k}-t/\sqrt d)}{1-t^2},
\]
as one checks separately for the two signs. Since $\mu_{p,k}=p_k/\sqrt d$, differentiating the
likelihood term by term gives
\[
 (1-t^2)\,\Upsilon_k'(t)=\sqrt d\;\E\Big[\bar\theta_k(S_n)\sum_{i=1}^n(Z_{U_i,i,k}-\mu_{p,k})\Bigm|p_k=t\Big]
 =\sqrt d\;\E\Big[\thh_k\sum_{i=1}^n(Z_{U_i,i,k}-\mu_{p,k})\Bigm|p_k=t\Big].
\]
Now apply \eqref{eq:prior-ibp} with $\Upsilon=\Upsilon_k$. Its left side is
$\kappa\E[p_k\thh_k]=\kappa\sqrt d\,\E[\mu_{p,k}\thh_k]$, and its right side is
$\sqrt d\,\E[\thh_k\sum_i(Z_{U_i,i,k}-\mu_{p,k})]$. Dividing by $\sqrt d$, averaging over $p_{-k}$ and
summing over $k$ proves the lemma. The argument uses only $\kappa>0$; $\kappa$ need not be an integer.
\end{proof}

\subsection{Operator-Norm and Weighted Tails}\label{app:tails}

\begin{restate}[Operator-norm tails in linear dimension]{Lemma}{lem:matrix}
Fix $p$ and let $d\ge32n$. The centred training matrix $\Btr\in\R^{n\times d}$, with rows
$(Z_{U_i,i}-\mu_p)^\top$, satisfies $\Pp_p(\opnorm{\Btr}^2>K_0+K_1t)\le e^{-t}$ for all $t\ge0$, with
$K_0<2.031$ and $K_1=2/(\cnet d)$, where $\cnet=\frac{19}{40}(\frac{29}{30})^2$. The pair matrix
$\Delta$, with rows $\Delta_i=(Z_{1,i}-Z_{0,i})^\top$, satisfies the same bound with $2K_0$ and
$2K_1$.
\end{restate}

We prove the lemma in a form that holds for all $n$ and $d$. Put $\rho=n/d$ and
\[
 K_0=K_0(\rho)=\frac{2\rho\ln61+\ln(40/21)}{\cnet},\qquad K_1=\frac2{\cnet d}.
\]
Then, at every $p$ and for all $t\ge0$,
\[
 \E_p\opnorm{\Btr}^2\le K_0,\qquad
 \Pp_p\big(\opnorm{\Btr}^2>K_0+K_1t\big)\le e^{-t},\qquad
 \Pp_p\big(\opnorm\Delta^2>2K_0+2K_1t\big)\le e^{-t}.
\]
Since $K_0(\rho)$ increases with $\rho$, the condition $d\ge32n$ gives
$K_0\le K_0^*:=K_0(1/32)=(\ln61/16+\ln(40/21))/\cnet=2.0306\ldots<2.031$, which is
Lemma~\ref{lem:matrix}. The lemma concerns the data alone, so it holds whatever the learner does.

\begin{proof}
Let $\eta_0=1/30$ and $\lambda_0=19/80$, so that $1+2/\eta_0=61$, $1-2\lambda_0=21/40$ and
$2\lambda_0(1-\eta_0)^2=\cnet$. Given $p$, the training points are independent draws from $D_p$, so
the entries $Z_{U_i,i,k}-\mu_{p,k}$ of $\Btr$ are independent and centred, and each takes values in an
interval of length $2/\sqrt d$.

Fix $x$ in the unit sphere $\mathbb S^{n-1}$ of $\R^n$. The $k$-th coordinate of $\Btr^\top x$ is
$\sum_ix_i(Z_{U_i,i,k}-\mu_{p,k})$, and its $i$-th term takes values in an interval of length
$2|x_i|/\sqrt d$. By Hoeffding's lemma, $\E_p\exp\{t'(\Btr^\top x)_k\}\le\exp\{t'^2\sum_ix_i^2/(2d)\}
=e^{t'^2/(2d)}$ for every real $t'$. Since
$(2\pi)^{-1/2}\int_\R\exp\{\sqrt{2\lambda_0d}\,yx'-y^2/2\}\,dy=e^{\lambda_0dx'^2}$ for every real $x'$,
Fubini's theorem and this bound with $t'=\sqrt{2\lambda_0d}\,y$ give
\[
 \E_p\exp\{\lambda_0d(\Btr^\top x)_k^2\}
 \le\frac1{\sqrt{2\pi}}\int_\R e^{\lambda_0y^2-y^2/2}\,dy=(1-2\lambda_0)^{-1/2}.
\]
The coordinates of $\Btr^\top x$ involve disjoint columns of $\Btr$ and are therefore independent, so
$\E_p\exp\{\lambda_0d\norm{\Btr^\top x}^2\}\le(1-2\lambda_0)^{-d/2}$.

Let $\mathcal S\subset\mathbb S^{n-1}$ be an $\eta_0$-net of the sphere with
$|\mathcal S|\le(1+2/\eta_0)^n$, which exists by a volume argument. If $x^\star\in\mathbb S^{n-1}$
satisfies $\norm{\Btr^\top x^\star}=\opnorm\Btr$ and $x\in\mathcal S$ satisfies
$\norm{x-x^\star}\le\eta_0$, then $\opnorm\Btr\le\norm{\Btr^\top x}+\eta_0\opnorm\Btr$. Hence
$(1-\eta_0)^2\opnorm\Btr^2\le\max_{x\in\mathcal S}\norm{\Btr^\top x}^2$, and bounding the maximum by
the sum over the net,
\[
 \E_p\exp\{\lambda_0d(1-\eta_0)^2\opnorm\Btr^2\}\le(1+2/\eta_0)^n(1-2\lambda_0)^{-d/2}.
\]
By the choice of $\eta_0$ and $\lambda_0$, the logarithm of the right side equals
$n\ln61+\frac d2\ln(40/21)=K_0/K_1$, and $\lambda_0d(1-\eta_0)^2=\cnet d/2=1/K_1$. Jensen's inequality
for the exponential function gives
$\E_p\opnorm\Btr^2\le K_1\ln\E_p\exp\{\opnorm\Btr^2/K_1\}\le K_0$, and Markov's inequality gives, for
$t\ge0$,
\[
 \Pp_p\big(\opnorm\Btr^2>K_0+K_1t\big)\le\exp\{K_0/K_1-(K_0+K_1t)/K_1\}=e^{-t}.
\]

The entries $Z_{1,i,k}-Z_{0,i,k}$ of $\Delta$ are independent and centred, and each is the sum of the
two independent centred terms $Z_{1,i,k}-\mu_{p,k}$ and $-(Z_{0,i,k}-\mu_{p,k})$, with values in
intervals of length $2/\sqrt d$. Hoeffding's lemma, applied to both terms, gives
$\E_p\exp\{t'(\Delta^\top x)_k\}\le e^{t'^2/d}$, so the matrix $\Delta/\sqrt2$ satisfies the bound used
above for $\Btr$. The rest of the argument applies verbatim and gives
$\Pp_p(\opnorm\Delta^2/2>K_0+K_1t)\le e^{-t}$.
\end{proof}

In Lemma~\ref{lem:weighted} the constants $K_0,K_1\ge0$ are arbitrary; Appendix~\ref{app:mainproof}
applies it with the constants above, and with $2K_0,2K_1$ for the pair matrix.

\begin{restate}[Dependent weighted tails]{Lemma}{lem:weighted}
Let $K_0,K_1\ge0$, let $N\ge0$ satisfy $\Pp(N>K_0+K_1t)\le e^{-t}$ for all $t\ge0$, and let
$0\le\omega\le\bar\omega$ with $\E\omega\le w$ for some $0<w\le\bar\omega$. Then, whatever the
dependence between $N$ and $\omega$,
$\E[N\omega]\le\Tb_{\bar\omega}(w;K_0,K_1):=w\{K_0+K_1[1+\ln(\bar\omega/w)]\}$.
\end{restate}

\begin{proof}
If $K_1=0$, letting $t\to\infty$ in the tail bound shows $N\le K_0$ almost surely, so
$\E[N\omega]\le K_0w$. Otherwise let $t^\star=K_0+K_1\ln(\bar\omega/w)$, which is at least $K_0\ge0$
because $w\le\bar\omega$. Pointwise,
$N\omega\le t^\star\omega+\bar\omega(N-t^\star)_+$: if $N\le t^\star$ this holds because $\omega\ge0$,
and otherwise $(N-t^\star)\omega\le(N-t^\star)\bar\omega$. Integrating the tail bound,
\[
 \E(N-t^\star)_+=\int_0^\infty\Pp(N>t^\star+y)\,dy
 \le\int_0^\infty\frac w{\bar\omega}\,e^{-y/K_1}\,dy=\frac{K_1w}{\bar\omega}.
\]
Hence $\E[N\omega]\le t^\star\E\omega+K_1w\le t^\star w+K_1w=\Tb_{\bar\omega}(w;K_0,K_1)$.
\end{proof}

We extend $\Tb$ to all $w\ge0$. Put $\Tb_{\bar\omega}(0;K_0,K_1)=0$, which is the value of
$\E[N\omega]$ when $\E\omega=0$, since then $\omega=0$ almost surely. For $w>\bar\omega$ the lemma holds
with $w\wedge\bar\omega$ in place of $w$, because $\E\omega\le\bar\omega$ always; accordingly,
$\Tb_{\bar\omega}(w;K_0,K_1)$ stands for $\Tb_{\bar\omega}(\bar\omega;K_0,K_1)$ when $w>\bar\omega$. The
function $\Tb_{\bar\omega}(\cdot\,;K_0,K_1)$ is continuous and nondecreasing on $[0,\bar\omega]$, since
its derivative in $w$ is $K_0+K_1\ln(\bar\omega/w)\ge0$, and it is nondecreasing in $K_0$ and in $K_1$;
upper bounds on $w$, $K_0$ and $K_1$ may therefore be substituted.

The tail bounds of Lemma~\ref{lem:matrix} hold at every $p$, so they also hold under $\Pp_\Pi$ for
every prior $\Pi$, and Lemma~\ref{lem:weighted} applies to prior averages as well. Lemma~\ref{lem:weighted}
is what handles learners whose error is largest on atypical matrices: in
Appendix~\ref{app:mainproof} the weight is the squared norm or the squared error of the learner, which
may depend on the matrix in any way.

%% file: app_C_transfer.tex
\section{Proof of \texorpdfstring{Theorem~\ref{thm:transfer}}{the transfer theorem} and the information hierarchy}\label{app:transfer}

This appendix proves Theorem~\ref{thm:transfer} and the hierarchy~\eqref{eq:hierarchy}, analyses a
clipped decoder of the selection bits, and compares the upper half of Theorem~\ref{thm:main} with the
information of a noisy model (Remark~\ref{rem:noisy-model}). Throughout, $G\sim\Normal(0,I_n)$ is independent of
$p$, of the data, of the selection bits and of the learner. We restate the theorem.

\begin{restate}[Moment-to-information transfer]{Theorem}{thm:transfer}
Let $X$ be the release of a proper learner, let $p$ be drawn from a prior $\Pi$ on $[-1,1]^d$
independently of $U$, $R$ and $G$, and let $\tau\ge0$.
\begin{enumerate}
\item[(a)] If $\sum_i\E_\Pi\sigma_iX_i\ge M>0$ and $\sum_i\E_\Pi X_i^2\le V$, then
\begin{equation*}
 \E_{p\sim\Pi}\Ii_{D_p}(X+\tau G)\ge\frac{M^2}{2\ln2\,(V+n\tau^2)}.\tag{\ref{eq:lower}}
\end{equation*}
\item[(b)] If $\E_p\norm{X-x_p(\Zt)}^2\le\Vres$ at a law $D_p$, for a reference $x_p(\Zt)\in\R^n$ that
is a measurable function of $p$ and the candidates only, then $\Is_{D_p}(X)\le n$ and, for $\tau>0$,
\begin{equation*}
 \Is_{D_p}(X+\tau G)\le\min\Big\{n,\frac n{2\ln2}\ln\Big(1+\frac{\Vres}{n\tau^2}\Big)\Big\}.
 \tag{\ref{eq:upper}}
\end{equation*}
\end{enumerate}
Neither part assumes independent coordinates of $X$.
\end{restate}

The proof uses only the joint law of $(p,\Zt,U,X,G)$: $U$ is uniform on $\{0,1\}^n$ and independent
of $(p,\Zt)$, and $G$ is independent of $(p,\Zt,U,X)$. It therefore applies to any random vector
$X\in\R^n$ defined together with the supersample, not only to loss differences.
Appendix~\ref{app:fairbit} proves part~(a) and Appendix~\ref{app:hierarchy} part~(b).

\subsection{The Fair-Bit Inequality and the Lower Bound}\label{app:fairbit}

Let $U$ be a fair bit, $\sigma=1-2U$, and $Y$ a real observation with $\E Y^2<\infty$. Write
$\bar\sigma(y)=\E[\sigma\mid Y=y]\in[-1,1]$ for the posterior mean of $\sigma$. Then
\begin{equation}\label{eq:fairbit}
 I(Y;U)\ge\frac{\E\bar\sigma(Y)^2}{2\ln2}\ge\frac{(\E\sigma Y)^2}{2\ln2\,\var(Y)}
 \ge\frac{(\E\sigma Y)^2}{2\ln2\,\E Y^2},
\end{equation}
where a ratio with a vanishing denominator is read as zero; its numerator then vanishes as well.

Given $Y$, the bit $U$ equals $0$ with probability $(1+\bar\sigma(Y))/2$. Hence
$\Hent(U\mid Y)=\E\,\htwo((1+\bar\sigma(Y))/2)$ and
\[
 I(Y;U)=1-\Hent(U\mid Y)=\E\,\J(\bar\sigma(Y)),\qquad \J(t)=1-\htwo\Big(\frac{1+t}2\Big).
\]
For $|t|\le1$,
\[
 \J(t)=\frac{(1+t)\ln(1+t)+(1-t)\ln(1-t)}{2\ln2}
 =\frac1{2\ln2}\sum_{\ell\ge1}\frac{t^{2\ell}}{\ell(2\ell-1)},
\]
so $\J$ is even and convex, it is increasing on $[0,1]$, and $\J(t)\ge t^2/(2\ln2)$. This gives the
first inequality in \eqref{eq:fairbit}. For the second, $\E\bar\sigma(Y)=\E\sigma=0$ and
$\E\sigma Y=\E[Y\bar\sigma(Y)]=\E[(Y-\E Y)\bar\sigma(Y)]$, so the Cauchy--Schwarz inequality gives
$(\E\sigma Y)^2\le\var(Y)\,\E\bar\sigma(Y)^2$. The third holds because $\var(Y)\le\E Y^2$.

For a fair bit, $I(Y;U)$ in nats is the Jensen--Shannon divergence between the two conditional laws
of $Y$ given $U$, and $\E\bar\sigma(Y)^2$ is one half of their triangular discrimination. The first
inequality in \eqref{eq:fairbit} therefore states that this Jensen--Shannon divergence is at least one
quarter of the triangular discrimination, a classical inequality \citep{topsoe2000}; see also
\citet[Theorem~6]{luli2015}. The constant $1/(2\ln2)$ cannot be
improved: if $Y=\sigma$ with probability $(1+t)/2$ and $Y=-\sigma$ otherwise, then $\E\sigma Y=t$,
$\var(Y)=1$ and $I(Y;U)=\J(t)=t^2/(2\ln2)+O(t^4)$.

\begin{proof}[Proof of Theorem~\ref{thm:transfer}(a)]
Let $Y=X+\tau G$ and fix $p$. Under $D_p$ each $U_i$ is a fair bit, because $U$ is uniform and
independent of $p$, and $\E_pY_i^2<\infty$ for $\Pi$-almost every $p$ by the energy bound. Applying
\eqref{eq:fairbit} to $(Y_i,U_i)$ under $D_p$ and summing over $i$,
\[
 \Ii_{D_p}(Y)\ge\frac1{2\ln2}\sum_{i=1}^n\frac{(\E_p\sigma_iY_i)^2}{\E_pY_i^2},
\]
where a term with a vanishing denominator is zero. Average over $p\sim\Pi$ and apply the
Cauchy--Schwarz inequality in the form $\int x^2/y\ge(\int x)^2/\int y$, for $y>0$, to the functions
$x=\E_p\sigma_iY_i$ and $y=\E_pY_i^2$ of the pair $(p,i)$:
\[
 \E_\Pi\Ii_{D_p}(Y)\ge\frac{\big(\sum_i\E_\Pi\sigma_iY_i\big)^2}{2\ln2\,\sum_i\E_\Pi Y_i^2}.
\]
Since $G$ is independent of $(p,\Zt,U,X)$ and its coordinates are centred with unit variance,
$\E_p\sigma_iY_i=\E_p\sigma_iX_i$ and $\E_pY_i^2=\E_pX_i^2+\tau^2$. The hypotheses of part~(a) give
$\sum_i\E_\Pi\sigma_iY_i\ge M>0$ and $\sum_i\E_\Pi Y_i^2\le V+n\tau^2$, which proves \eqref{eq:lower}.
\end{proof}

Only the independence of the noise and the first two moments of its coordinates enter this proof, so
\eqref{eq:lower} holds for any noise vector that is independent of $(p,\Zt,U,X)$ and has centred
coordinates of variance $\tau^2$. With $M_p=\sum_i\E_p\sigma_iX_i$ and $V_p=\sum_i\E_pX_i^2$, the
proof applied to the point mass at a fixed $p$ gives the fixed-law form
\[
 \Ii_{D_p}(X+\tau G)\ge\frac{M_p^2}{2\ln2\,(V_p+n\tau^2)}\qquad\text{whenever }M_p>0.
\]

\subsection{The Hierarchy and the Gaussian Upper Bound}\label{app:hierarchy}

Fix $p$; all information quantities in this subsection are taken under $D_p$. Let $Y\in\R^n$ be any random
vector defined together with $(\Zt,U)$, where, as always, $U$ is uniform on $\{0,1\}^n$ and
independent of $\Zt$.

\paragraph{The Hierarchy.}
Since $U_1,\dots,U_n$ are independent, $\Hent(U)=\sum_i\Hent(U_i)$, while
$\Hent(U\mid Y)\le\sum_i\Hent(U_i\mid Y)$. Hence, by data processing in the last step,
\[
 \Ij(Y)=\Hent(U)-\Hent(U\mid Y)\ge\sum_{i=1}^nI(Y;U_i)\ge\sum_{i=1}^nI(Y_i;U_i)=\Ii(Y).
\]
Since $U$ is independent of $\Zt$, the chain rule gives
\[
 \Ij(Y)=I(Y;U)\le I(Y,\Zt;U)=I(\Zt;U)+I(Y;U\mid\Zt)=\Is(Y).
\]
The same two
arguments, applied pair by pair, place the summed per-pair conditional information between the
two ends:
\[
 \Ii(Y)\le\sum_{i=1}^nI(Y_i;U_i\mid\Zt)\le\Is(Y).
\]
Indeed, $U_i$ is independent of $\Zt$, so $I(Y_i;U_i)\le I(Y_i,\Zt;U_i)=I(Y_i;U_i\mid\Zt)$; and
$U_1,\dots,U_n$ remain independent given $\Zt$, so the first display, applied conditionally on $\Zt$,
gives $I(Y;U\mid\Zt)\ge\sum_iI(Y;U_i\mid\Zt)\ge\sum_iI(Y_i;U_i\mid\Zt)$. For $Y=X$ the middle term is
the sum over pairs of the ld-CMI of \citet{wangmao2023a}.

Finally, let $Y=X+\tau G$. Given $\Zt$, the release is a function of $\thh$ (Definition~\ref{def:release}),
and $G$ is independent of $(\Zt,U,\thh)$, so $U\to\thh\to X\to X+\tau G$ is a Markov chain given $\Zt$.
Conditional data processing gives $\Is(X+\tau G)\le I(X;U\mid\Zt)\le I(\thh;U\mid\Zt)$, which
completes the proof of \eqref{eq:hierarchy}; Lemma~\ref{lem:directions} records these two inequalities
together with their per-pair forms. By the hierarchy, the lower bound \eqref{eq:lower} also holds for
$\Ij$ and $\Is$, and the upper bound \eqref{eq:upper} also holds for $\Ii$ and $\Ij$.

\paragraph{The Gaussian Upper Bound.}
\begin{proof}[Proof of Theorem~\ref{thm:transfer}(b)]
Fix $p$ and $\tau>0$, let $Y=X+\tau G$, and measure information and differential entropy in nats until
the last step. Given $\Zt$, $U\to X\to Y$ is a Markov chain, because $G$ is independent of
$(\Zt,U,X)$. Conditional data processing gives
$I(Y;U\mid\Zt)\le I(Y;X\mid\Zt)=\hdiff(Y\mid\Zt)-\hdiff(Y\mid X,\Zt)$, and
$\hdiff(Y\mid X,\Zt)=\hdiff(\tau G)=\frac n2\ln(2\pi e\tau^2)$. Given $\Zt$, subtracting the
reference $x_p(\Zt)$ is an invertible translation, so $\hdiff(Y\mid\Zt)=\hdiff(Y-x_p(\Zt)\mid\Zt)$.
Whatever its mean, a real random variable has differential entropy at most that of the centred
Gaussian law with the same second moment. By subadditivity of differential entropy and this bound for
each coordinate given $\Zt$,
\[
 \hdiff(Y\mid\Zt)\le\sum_{i=1}^n\hdiff\big(Y_i-x_{p,i}(\Zt)\bigm|\Zt\big)
 \le\sum_{i=1}^n\E\,\tfrac12\ln\Big(2\pi e\,\E\big[(Y_i-x_{p,i}(\Zt))^2\bigm|\Zt\big]\Big),
\]
and $\E[(Y_i-x_{p,i}(\Zt))^2\mid\Zt]=\E[(X_i-x_{p,i}(\Zt))^2\mid\Zt]+\tau^2$. Jensen's inequality for
the logarithm, over the coordinates and over $\Zt$, then gives
\[
 I(Y;U\mid\Zt)\le\frac12\sum_{i=1}^n\E\ln\Big(1+\frac{\E[(X_i-x_{p,i}(\Zt))^2\mid\Zt]}{\tau^2}\Big)
 \le\frac n2\ln\Big(1+\frac{\E_p\norm{X-x_p(\Zt)}^2}{n\tau^2}\Big)
 \le\frac n2\ln\Big(1+\frac{\Vres}{n\tau^2}\Big).
\]
The reference need not be the conditional mean of $X$ given $\Zt$. Dividing by $\ln2$ gives the
second term of \eqref{eq:upper} in bits. The first term, and the bound $\Is_{D_p}(X)\le n$ at
$\tau=0$, hold because $I(Y;U\mid\Zt)\le\Hent(U\mid\Zt)=\Hent(U)=n$ bits for every $\tau\ge0$.
\end{proof}

\subsection{The Clipped Decoder}\label{app:decoder}

At $\tau=0$ the lower bound has an operational counterpart: a fixed randomized rule that reads each
loss difference separately guesses the selection bits better than chance. Let $M>0$ and $V$ be as in
Theorem~\ref{thm:transfer}(a), put $\lambda=V/(2M)$, write $\clip(x,-1,1)=\max\{-1,\min\{1,x\}\}$, and,
given the release, draw $\sighat_1,\dots,\sighat_n\in\{-1,1\}$ independently, with randomness
independent of everything else, such that
\begin{equation}\label{eq:bit}
 \Pp(\sighat_i=1\mid X)=\tfrac12\big[1+\clip(X_i/\lambda,-1,1)\big],\qquad
 \widehat U_i=(1-\sighat_i)/2.
\end{equation}
Let $\Pp_\Pi$ and $I_\Pi$ refer to the law in which $p\sim\Pi$ is drawn first. Then
\begin{equation}\label{eq:bitguarantee}
 \frac1n\sum_{i=1}^n\Pp_\Pi(\widehat U_i=U_i)\ge\frac12+\frac{M^2}{2nV},\qquad
 \E_\Pi\Ii_{D_p}(\sighat)\ge n\,\J\Big(\frac{M^2}{nV}\Big),
\end{equation}
where $\sighat=(\sighat_1,\dots,\sighat_n)$ and $\J$ is the function of Appendix~\ref{app:fairbit};
by the hierarchy, the same lower bound holds for $\E_\Pi\Ij_{D_p}(\sighat)$. For the learners of
Theorem~\ref{thm:main}, Proposition~\ref{prop:finite} gives the two moment bounds with $M=c_Mn\eps$
and $V=C_Vn\eps^2$. Then $\lambda=C_V\eps/(2c_M)$ depends only on $\eps$ and on the constants, so the
decoder uses no knowledge of $p$ or of the learner, and $M^2/(nV)=c_M^2/C_V$ is a constant.

\begin{proof}
For real $x$, $|x-\lambda\clip(x/\lambda,-1,1)|=(|x|-\lambda)_+\le x^2/(4\lambda)$; when
$|x|\ge\lambda$, the inequality is $(|x|-2\lambda)^2\ge0$. As
$\E[\sighat_i\mid X,U]=\clip(X_i/\lambda,-1,1)$,
\[
 \lambda\,\E_\Pi\sigma_i\sighat_i=\E_\Pi\big[\sigma_i\lambda\clip(X_i/\lambda,-1,1)\big]
 \ge\E_\Pi\sigma_iX_i-\frac{\E_\Pi X_i^2}{4\lambda}.
\]
Summing over $i$ and using the two moment bounds, $\sum_i\E_\Pi\sigma_i\sighat_i\ge
M/\lambda-V/(4\lambda^2)=M^2/V$ at $\lambda=V/(2M)$. Since $\sigma_i\sighat_i=1$ exactly when
$\widehat U_i=U_i$, we have $\Pp_\Pi(\widehat U_i=U_i)=(1+\E_\Pi\sigma_i\sighat_i)/2$, which proves the
first bound in \eqref{eq:bitguarantee}.

For the second, let $t_i=\E_\Pi\sigma_i\sighat_i$. Under $\Pp_\Pi$ the bit $U_i$ is fair and the guess
$\widehat U_i$ errs with probability $(1-t_i)/2$. Fano's inequality for a binary guess holds whatever
the error probability, and since $\J$ is even it gives $I_\Pi(\sighat_i;U_i)\ge1-\htwo((1-t_i)/2)=\J(t_i)$.
Since $p$ is independent of $U_i$,
$I_\Pi(\sighat_i;U_i)\le I_\Pi(\sighat_i,p;U_i)=I_\Pi(\sighat_i;U_i\mid p)=\E_\Pi I_{D_p}(\sighat_i;U_i)$.
By convexity of $\J$, $\sum_i\J(t_i)\ge n\J(\bar t)$ with $\bar t=n^{-1}\sum_it_i\ge M^2/(nV)$. The
Cauchy--Schwarz inequality gives
$M^2\le(\sum_i\E_\Pi\sigma_iX_i)^2\le n\sum_i\E_\Pi X_i^2\le nV$, so $M^2/(nV)\in(0,1]$, and since $\J$
is increasing on $[0,1]$, $n\J(\bar t)\ge n\J(M^2/(nV))$. Together these bounds prove
\eqref{eq:bitguarantee}.
\end{proof}

The success guarantee is sharp over the class of all releases with these two moments, which need not
come from accurate learners. If the $\sigma_i$ are independent fair signs and
$X_i=(V/M)\sigma_i\ind_{E_i}$ for independent events $E_i$ of probability $M^2/(nV)$, independent of
the signs, then both moment bounds hold with equality, the decoder succeeds on every pair with
probability $\frac12+M^2/(2nV)$, and no rule that sees all of $X$ does better, since off $E_i$ the bit
$U_i$ is independent of $X$.

\subsection{Model Information with Noise}\label{app:noisy-model}

The upper half of Theorem~\ref{thm:main} rests on localization, which by itself bounds the
information of a model observed through Gaussian noise.

\begin{remark}[Model information with noise]\label{rem:noisy-model}
Let $G'\sim\Normal(0,I_d)$ be independent of the data, the selection bits and the learner. Under
Assumption~\ref{ass:class}, for every $p\in[-1,1]^d$, every proper learner that is $\eps$-accurate at
$D_p$ and every $\tau>0$,
\[
 I_{D_p}(\thh+\tau G';U\mid\Zt)\le\min\Big\{n,\frac w{2\ln2\,\tau^2}\Big\},\qquad
 w=\min\{4,(\eps/a_-)^{2/r}\}.
\]
Indeed, $\E_p\norm{\thh-\theta_p^*}^2\le w$ by \eqref{eq:error}, whose proof uses accuracy at $D_p$
only, and $\theta_p^*$ is fixed once $p$ is. Given $\Zt$, $U\to\thh\to\thh+\tau G'$ is a Markov chain,
and the proof of the Gaussian upper bound in Appendix~\ref{app:hierarchy}, in dimension $d$ and with
the reference $\theta_p^*$, bounds the information by
$\frac d2\ln(1+w/(d\tau^2))\le w/(2\tau^2)$ nats, and the bound $n$ holds because $U$ takes $2^n$
values.

In the sample-size range of Theorem~\ref{thm:main}, $n\ge c_-\neps$ gives
$w\le(\eps/a_-)^{2/r}\le n\eps^2/(c_-a_-^{2/r})$, so the information of the perturbed model is at most
\[
 \min\Big\{n,\frac{n\eps^2}{2\ln2\,c_-a_-^{2/r}\tau^2}\Big\}
 \le\Big(1+\frac1{2\ln2\,c_-a_-^{2/r}}\Big)\frac n{1+(\tau/\eps)^2}.
\]
This has the order of the upper half of \eqref{eq:main}, with no condition on the dimension and with
a constant that does not contain the factor $2K_0(L+H)^2$ of $\Vres$.
\end{remark}

Remark~\ref{rem:noisy-model} does not yield the upper half of Theorem~\ref{thm:main} by data
processing, because the release is a function of $\thh$, not of $\thh+\tau G'$. The loss-level content
of that upper half is the residual bound at every law. The Lipschitz factor $(L+H)\opnorm\Delta$ of the
release map (Appendix~\ref{app:residual}) depends on the candidates, on which the learner's error
depends as well, and Lemma~\ref{lem:weighted}, with the pair-matrix tail of Lemma~\ref{lem:matrix},
bounds the residual at every law despite this dependence. Finally, the two noise scales of
Theorem~\ref{thm:main}, $\tau\asymp\eps$ and $\tau\asymp\eps^{1/r}$, need both halves of the theorem,
and its lower half is a statement about the release.

%% file: app_D_mainproof.tex
\section{Proof of \texorpdfstring{Proposition~\ref{prop:finite} and Theorem~\ref{thm:main}}{the moment bounds and the main theorem}}\label{app:mainproof}

This appendix proves Proposition~\ref{prop:finite} in a finite form, with explicit values of $M$,
$V$ and $\Vres$, and deduces Theorem~\ref{thm:main}. The formulas for $M$, $V$ and $\Vres$ and the
finite form of Proposition~\ref{prop:finite} come first. Appendix~\ref{app:calibration} compares the learner's
population risk with that of a comparator aligned with the mean $\mu_p$, which gives the alignment
term $A_\phi$; Appendix~\ref{app:localization} localizes the population minimizer and the learner's
error, which gives $w_\theta$, $Q_\theta$ and $C_\theta$; Appendices~\ref{app:signed},
\ref{app:energy} and~\ref{app:residual} bound the signed mean, the energy and the residual;
Appendix~\ref{app:orders} derives their orders and proves Theorem~\ref{thm:main}; and
Appendix~\ref{app:constants} makes the constants explicit. Besides Theorem~\ref{thm:transfer} and the
hierarchy \eqref{eq:hierarchy}, the proof uses only the prior moments and
Lemmas~\ref{lem:fp}--\ref{lem:weighted} of Appendix~\ref{app:tools}.

Throughout, $\thh=A_n(S_n;R)$ is a proper learner that is $\eps$-accurate on the family, in the sense
of \eqref{eq:accuracy}. The expectation $\E_p$ and the probability $\Pp_p$ refer to the data law $D_p$
and are taken over the candidates, the selection bits and the seed $R$; $\E_\Pi$ first draws $p$ from
the beta prior $\Pi_\kappa$ and then takes $\E_p$.

\paragraph{Finite Moment Bounds.}
Under Assumption~\ref{ass:class}, let $q=r/(r-1)$ and
\begin{align}
 t_0&=\Big(\frac{b}{ra_++b}\Big)^{1/(r-1)},&
 \beta&=\frac{bt_0}{q},&
 v&=\Big(\frac{4\eps}{\beta}\Big)^{2/q},&
 \kappa&=v^{-1}-1.\label{eq:prior-choice}
\end{align}
Assume $0<v<1$, that is, $\eps<\beta/4$, so that $\kappa>0$. The prior $\Pi_\kappa$ of
Section~\ref{sec:architecture} draws the coordinates $p_k$ independently with density
$\pi_\kappa(t)=c_\kappa(1-t^2)^{\kappa/2-1}$ on $(-1,1)$, so that $\E\norm{\mu_p}^2=1/(\kappa+1)=v$.
It depends only on $\eps$ and on the constants of Assumption~\ref{ass:class}, and it is fixed before
the learner. Set
\begin{align}
 \rho&=\frac nd,&
 \cnet&=\frac{19}{40}\Big(\frac{29}{30}\Big)^2,&
 K_0&=\frac{2\rho\ln61+\ln(40/21)}{\cnet},&
 K_1&=\frac2{\cnet d}.\label{eq:matrix-constants}
\end{align}
Recall the weighted tail function of Lemma~\ref{lem:weighted}, extended as in
Appendix~\ref{app:tails}: for $\bar\omega>0$ and $0<x\le\bar\omega$,
\begin{equation}\label{eq:T}
 \Tb_{\bar\omega}(x;K_0,K_1)=x\big\{K_0+K_1[1+\ln(\bar\omega/x)]\big\},\qquad
 \Tb_{\bar\omega}(0;K_0,K_1)=0,
\end{equation}
and $\Tb_{\bar\omega}(x;K_0,K_1)=\Tb_{\bar\omega}(\bar\omega;K_0,K_1)$ for $x>\bar\omega$. The
following quantities depend only on the constants of Assumption~\ref{ass:class} and on $\eps$, $n$
and $d$; they do not depend on the learner:
\begin{align}
 C_0&=\Big(\frac b{2a_-}\Big)^{1/(r-1)},\qquad
 w=\min\big\{4,(\eps/a_-)^{2/r}\big\},\qquad
 w_*=C_0^2v^{q-1},\nonumber\\
 w_\theta&=\min\big\{1,(\sqrt w+\sqrt{w_*})^2\big\},\qquad
 C_\theta=\Tb_1(w_\theta;K_0,K_1),\label{eq:finite-small}\\
 Q_\theta&=\big\{\sqrt{wv}+C_0v^{q/2}(1+2/d)^{q/4}\big\}^2,\nonumber\\
 P_\phi&=\frac{Ht_0^2}2\big\{v^q(1+2/d)^{q/2}+v^{q-1}/d\big\},\label{eq:finite-pop}\\
 A_\phi&=\beta v^{q/2}\big(1+1/(dv)\big)^{-q/2}-\eps-P_\phi,\label{eq:finite-align}\\
 M&=\kappa A_\phi-H\sqrt{nQ_\theta C_\theta}-\frac H2C_\theta,\label{eq:finite-M}\\
 V&=L^2(C_\theta+\rho w_\theta),\qquad
 \Vres=(L+H)^2\,\Tb_4(w;2K_0,2K_1).\label{eq:finite-VD}
\end{align}
In Appendix~\ref{app:calibration}, $A_\phi$ bounds $b\,\E_\Pi\ip\thh{\mu_p}$ from below. In
Appendix~\ref{app:localization}, $w$ bounds the mean squared error $\E_p\norm{\thh-\theta_p^*}^2$ of the
learner at every $p$, $w_*$ bounds $\E_\Pi\norm{\theta_p^*}^2$, and $w_\theta$, $Q_\theta$ and
$C_\theta$ bound $\E_\Pi\norm\thh^2$, $\E_\Pi[\norm\thh^2\norm{\mu_p}^2]$ and
$\E_\Pi[\norm\thh^2\opnorm\Btr^2]$, the quantities of \eqref{eq:localized-main}.
The reference for the upper half is the release of the population minimizer $\theta_p^*$,
\[
 X_{p,i}^*(\Zt)=\phi(\ip{\theta_p^*}{Z_{1,i}})-\phi(\ip{\theta_p^*}{Z_{0,i}}),\qquad i\in[n],
\]
which depends only on $p$ and on the candidates.

\begin{restate}[finite form]{Proposition}{prop:finite}
Let $\thh$ be a proper learner that is $\eps$-accurate on the family, and assume $0<v<1$. Then, for
every $n$ and $d$,
\begin{equation}\label{eq:finite-budgets}
 \sum_i\E_\Pi\,\sigma_iX_i\ge M,\qquad
 \sum_i\E_\Pi X_i^2\le V,\qquad
 \E_p\norm{X-X_p^*(\Zt)}^2\le\Vres\quad\text{for every }p\in[-1,1]^d.
\end{equation}
The bound at $p$ uses accuracy only at $D_p$.
\end{restate}
No condition on the dimension is needed for \eqref{eq:finite-budgets}. The dimension condition of
Theorem~\ref{thm:main}, together with a small $\eps$, makes $M$ positive and gives the orders of
Appendix~\ref{app:orders}. The information bounds use $M$ only when $M>0$.

\paragraph{The Link and the Margins.}
The link remainder $u(t)=\phi(t)-\phi(0)+bt$ of Section~\ref{sec:moments} satisfies $u(0)=u'(0)=0$
and $0\le u''=\phi''\le H$ on $[-1,1]$. So $u$ is convex on $[-1,1]$, and
\[
 0\le u(t)\le\frac{Ht^2}2\qquad\text{and}\qquad|u'(t)|\le H|t|\qquad(|t|\le1).
\]
Every margin $\ip\theta z$ below has $\theta\in W$ and either $z\in\mathcal Z$ or $z=\mu_p$. Since
$\norm z=1$ on $\mathcal Z$ and $\norm{\mu_p}\le1$, all margins, and all segments between two
margins, lie in $[-1,1]$, where the bounds on $\phi'$, $\phi''$ and $u$ apply. Finally, under $D_p$
the coordinates of $Z$ are independent with variances $(1-p_k^2)/d\le1/d$, so for every fixed
$\theta\in\R^d$,
\begin{equation}\label{eq:margin-variance}
 \E_p\ip\theta{Z-\mu_p}^2=\sum_k\theta_k^2\,\frac{1-p_k^2}d\le\frac{\norm\theta^2}d,
 \qquad\text{hence}\qquad
 \E_p\ip\theta Z^2\le\ip\theta{\mu_p}^2+\frac{\norm\theta^2}d.
\end{equation}

\subsection{Calibration}\label{app:calibration}

This step turns accuracy into alignment of the learner with the mean. Fix $p$ and write
$\mu=\mu_p$. Let $\theta_c=t_0\norm\mu^{q-2}\mu$ if $\mu\ne0$ and $\theta_c=0$ if $\mu=0$. Then
$\norm{\theta_c}=t_0\norm\mu^{q-1}\le1$, because $t_0<1$ and $\norm\mu\le1$, so $\theta_c\in W$. Since
$(q-1)r=q$, we have $\norm{\theta_c}^r=t_0^r\norm\mu^q$ and $\ip{\theta_c}\mu=t_0\norm\mu^q$. The
choice of $t_0$ gives $a_+t_0^{r-1}=a_+b/(ra_++b)\le b/r$, so the upper growth
$g(\theta_c)\le a_+\norm{\theta_c}^r$ in \eqref{eq:matched} yields
\[
 g(\theta_c)-b\ip{\theta_c}\mu
 \le-t_0\big(b-a_+t_0^{r-1}\big)\norm\mu^q
 \le-bt_0\Big(1-\frac1r\Big)\norm\mu^q
 =-\beta\norm\mu^q .
\]
Write $\phi(t)=\phi(0)-bt+u(t)$. As $\E_p\ip{\theta_c}Z=\ip{\theta_c}\mu$, the bound
$u(t)\le Ht^2/2$ and \eqref{eq:margin-variance} give
\[
 F_p(\theta_c)=\phi(0)+g(\theta_c)-b\ip{\theta_c}\mu+\E_p\,u(\ip{\theta_c}Z)
 \le\phi(0)-\beta\norm\mu^q+\frac{Ht_0^2}2\Big\{\norm\mu^{2q}+\frac{\norm\mu^{2(q-1)}}d\Big\}.
\]
On the other hand, $g\ge0$ and $u\ge0$ give $F_p(\theta)\ge\phi(0)-b\ip\theta\mu$ for every
$\theta\in W$. Accuracy at $D_p$ and $\theta_c\in W$ therefore give
\[
 \phi(0)-b\,\E_p\ip\thh\mu
 \le\E_pF_p(\thh)
 \le\min_{\theta\in W}F_p(\theta)+\eps
 \le F_p(\theta_c)+\eps .
\]
Combining the last two displays and averaging over $p\sim\Pi_\kappa$,
\[
 b\,\E_\Pi\ip\thh{\mu_p}
 \ge\beta\,\E\norm{\mu_p}^q-\eps
 -\frac{Ht_0^2}2\Big\{\E\norm{\mu_p}^{2q}+\frac{\E\norm{\mu_p}^{2(q-1)}}d\Big\}.
\]
Since $1<q\le2$, the prior moment bounds \eqref{eq:prior-concave}--\eqref{eq:prior-low} give
$\E\norm{\mu_p}^q\ge v^{q/2}(1+1/(dv))^{-q/2}$, $\E\norm{\mu_p}^{2q}\le v^q(1+2/d)^{q/2}$ and
$\E\norm{\mu_p}^{2(q-1)}\le v^{q-1}$. Therefore
\begin{equation}\label{eq:alignment}
 b\,\E_\Pi\ip\thh{\mu_p}\ge\beta\,\E\norm{\mu_p}^q-\eps-P_\phi\ge A_\phi .
\end{equation}
This calibrates the population risk of the actual loss, not the risk of a linearized loss, and it
uses no assumption on the sample size. It is the only step that uses the upper growth
$g\le a_+\norm\cdot^r$, and Proposition~\ref{prop:growth} shows that this condition cannot be dropped.

\subsection{Localization}\label{app:localization}

Uniform convexity localizes the population minimizer and the learner's error at every law. Fix $p$.
The risk $F_p(\theta)=g(\theta)+\E_p\phi(\ip\theta Z)$ is continuously differentiable near $W$,
with $\nabla F_p(\theta)=\nabla g(\theta)+\E_p[\phi'(\ip\theta Z)Z]$. Its second term is convex on
$W$, so its Bregman divergence is nonnegative there, and \eqref{eq:matched} gives
\[
 D_{F_p}(\theta',\theta)\ge D_g(\theta',\theta)\ge a_-\norm{\theta'-\theta}^r\qquad(\theta,\theta'\in W).
\]
In particular $F_p$ is strictly convex on the compact set $W$ and has a unique minimizer
$\theta_p^*$ there. The first-order condition for this minimizer reads
$\ip{\nabla F_p(\theta_p^*)}{\theta-\theta_p^*}\ge0$ for every $\theta\in W$, so
\[
 F_p(\theta)-F_p(\theta_p^*)
 =D_{F_p}(\theta,\theta_p^*)+\ip{\nabla F_p(\theta_p^*)}{\theta-\theta_p^*}
 \ge a_-\norm{\theta-\theta_p^*}^r\qquad(\theta\in W).
\]
Let $e=\thh-\theta_p^*$. Accuracy at $D_p$ gives $a_-\E_p\norm e^r\le\eps$; Jensen's inequality, with
$r\ge2$, gives $\E_p\norm e^2\le(\E_p\norm e^r)^{2/r}$; and properness gives $\norm e\le2$. Hence, at
every $p$,
\begin{equation}\label{eq:error}
 \E_p\norm e^r\le\eps/a_-,\qquad \E_p\norm e^2\le w,\qquad \norm e\le2 .
\end{equation}

Next, $\nabla F_p(0)=\nabla g(0)+\phi'(0)\E_pZ=-b\mu_p$. Adding the two Bregman lower bounds between
$\theta_p^*$ and $0$,
\[
 2a_-\norm{\theta_p^*}^r
 \le D_{F_p}(\theta_p^*,0)+D_{F_p}(0,\theta_p^*)
 =\ip{\nabla F_p(\theta_p^*)-\nabla F_p(0)}{\theta_p^*}
 \le b\ip{\mu_p}{\theta_p^*}
 \le b\norm{\mu_p}\norm{\theta_p^*},
\]
where the second inequality is the first-order condition tested against $0\in W$, which gives
$\ip{\nabla F_p(\theta_p^*)}{\theta_p^*}\le0$. Dividing by $\norm{\theta_p^*}$ when $\theta_p^*\ne0$,
and using $1/(r-1)=q-1$,
\begin{equation}\label{eq:optimizer}
 \norm{\theta_p^*}\le C_0\norm{\mu_p}^{q-1}.
\end{equation}
Only the first-order condition was used, and it holds also when $\theta_p^*$ lies on the boundary
of $W$.

Now average over the prior. By \eqref{eq:optimizer} and \eqref{eq:prior-concave},
$\E_\Pi\norm{\theta_p^*}^2\le C_0^2\E\norm{\mu_p}^{2(q-1)}\le w_*$, and by \eqref{eq:error},
$\E_\Pi\norm e^2\le w$. Minkowski's inequality applied to $\norm\thh\le\norm e+\norm{\theta_p^*}$
gives $(\E_\Pi\norm\thh^2)^{1/2}\le\sqrt w+\sqrt{w_*}$, and together with $\norm\thh\le1$ this gives
$\E_\Pi\norm\thh^2\le w_\theta$. In the same way,
$\norm\thh\norm{\mu_p}\le\norm e\norm{\mu_p}+C_0\norm{\mu_p}^q$, and conditioning on $p$ and using
\eqref{eq:error} at each $p$ gives
$\E_\Pi[\norm e^2\norm{\mu_p}^2]=\E[\norm{\mu_p}^2\,\E_p\norm e^2]\le wv$, so \eqref{eq:prior-high}
gives
\begin{equation}\label{eq:mixed}
 \E_\Pi\big[\norm\thh^2\norm{\mu_p}^2\big]
 \le\big\{\sqrt{wv}+C_0\big(\E\norm{\mu_p}^{2q}\big)^{1/2}\big\}^2
 \le Q_\theta .
\end{equation}

Finally, Lemma~\ref{lem:matrix}, in the form proved in Appendix~\ref{app:tails} for all $n$ and $d$,
gives $\Pp_p(\opnorm\Btr^2>K_0+K_1t)\le e^{-t}$ for $t\ge0$, with the constants
\eqref{eq:matrix-constants}, at every $p$; averaging over $p$, the same tail bound holds under the
prior. Lemma~\ref{lem:weighted}, applied with $N=\opnorm\Btr^2$, the weight
$\omega=\norm\thh^2\in[0,1]$, the cap $\bar\omega=1$ and $\E_\Pi\omega\le w_\theta\le1$, gives
\begin{equation}\label{eq:Ctheta}
 \E_\Pi\big[\norm\thh^2\opnorm\Btr^2\big]\le\Tb_1(w_\theta;K_0,K_1)=C_\theta .
\end{equation}
Lemma~\ref{lem:weighted} needs no independence between $N$ and $\omega$. This matters, because
$\thh$ is computed from the training points, which are the rows of $\Btr$ shifted by $\mu_p$, so the
weight $\norm\thh^2$ may be largest on atypical matrices. The same argument at a fixed $p$, with the
pair matrix and the weight $\norm e^2$, bounds the residual in Appendix~\ref{app:residual}.

\subsection{The Signed Mean}\label{app:signed}

This subsection proves the first bound of \eqref{eq:finite-budgets}. Fix $p$ and condition on the
selection bits $U$, the training set $S_n$ and the seed $R$; then $\thh$ is fixed. Write
$\Ztr_i=Z_{U_i,i}$ and $\Zgh_i=Z_{1-U_i,i}$, and let
\[
 m=\ip\thh{\mu_p},\qquad s_i=\ip\thh{\Ztr_i-\mu_p},
 \qquad\text{so that}\qquad \ip\thh{\Ztr_i}=m+s_i
 \quad\text{and}\quad \sum_is_i^2=\norm{\Btr\thh}^2 .
\]
Given $(U,S_n,R)$, the ghosts are independent draws from $D_p$, independent of $\thh$. By
Definition~\ref{def:release}, $\sigma_iX_i=\phi(\ip\thh{\Zgh_i})-\phi(\ip\thh{\Ztr_i})$ is the
ghost-minus-training loss of pair $i$, and $\phi(t)=\phi(0)-bt+u(t)$ gives
\[
 \sigma_iX_i=b\ip\thh{\Ztr_i-\Zgh_i}+u(\ip\thh{\Zgh_i})-u(m+s_i).
\]
The ghost has conditional mean $\mu_p$, so the linear term has conditional mean $bs_i$, and Jensen's
inequality for the convex $u$ gives $\E[u(\ip\thh{\Zgh_i})\mid U,S_n,R]\ge u(m)$. At the training
point, $u''\le H$ gives $u(m+s_i)\le u(m)+u'(m)s_i+Hs_i^2/2$. Hence
\[
 \E[\sigma_iX_i\mid U,S_n,R]\ge bs_i-u'(m)s_i-\frac H2s_i^2 .
\]
Sum over $i$ and average over the data, the seed and $p\sim\Pi_\kappa$. By Lemma~\ref{lem:fp}, the
linear term averages to $b\,\E_\Pi\ip\thh{\sum_i(\Ztr_i-\mu_p)}=b\kappa\,\E_\Pi m$. This gives the
inequality \eqref{eq:taylor-main} of Section~\ref{sec:moments}:
\[
 \sum_i\E_\Pi\,\sigma_iX_i
 \ge b\kappa\,\E_\Pi m-\E_\Pi\Big[u'(m)\sum_is_i\Big]-\frac H2\,\E_\Pi\norm{\Btr\thh}^2 .
\]
By \eqref{eq:alignment} and $\kappa>0$, the first term is at least $\kappa A_\phi$. The middle term is
the train--ghost remainder, which vanishes for the affine link ($u\equiv0$). Pointwise,
$|u'(m)|\le H|m|\le H\norm\thh\norm{\mu_p}$ and, by the Cauchy--Schwarz inequality over $i$,
$|\sum_is_i|\le\sqrt n\norm{\Btr\thh}\le\sqrt n\norm\thh\opnorm\Btr$, so
\[
 \Big|u'(m)\sum_is_i\Big|\le H\sqrt n\,\big(\norm\thh\norm{\mu_p}\big)\big(\norm\thh\opnorm\Btr\big).
\]
The Cauchy--Schwarz inequality for $\E_\Pi$, with one factor $\norm\thh$ in each term, together with
\eqref{eq:mixed} and \eqref{eq:Ctheta}, gives
\[
 \E_\Pi\Big|u'(m)\sum_is_i\Big|
 \le H\sqrt{n\,\E_\Pi\big[\norm\thh^2\norm{\mu_p}^2\big]\,\E_\Pi\big[\norm\thh^2\opnorm\Btr^2\big]}
 \le H\sqrt{nQ_\theta C_\theta},
\]
and $\norm{\Btr\thh}^2\le\norm\thh^2\opnorm\Btr^2$ gives $\E_\Pi\norm{\Btr\thh}^2\le C_\theta$.
Therefore
\[
 \sum_i\E_\Pi\,\sigma_iX_i\ge\kappa A_\phi-H\sqrt{nQ_\theta C_\theta}-\frac H2C_\theta=M ,
\]
which is the first bound of \eqref{eq:finite-budgets}. With only $\norm\thh\le1$ in place of
\eqref{eq:mixed} and \eqref{eq:Ctheta}, the same argument would bound the two penalty terms only by
constants at the optimal sample size, where for $r=2$ the signal $\kappa A_\phi$ is itself of constant
order. The localized bounds keep them one power of $\eps$ below the signal (Appendix~\ref{app:orders}).

\subsection{Energy}\label{app:energy}

This subsection proves the second bound of \eqref{eq:finite-budgets}. Fix $p$, and let $s_i$ be as in
Appendix~\ref{app:signed} and $s_i^{\mathrm{gh}}=\ip\thh{\Zgh_i-\mu_p}$. Since $|\phi'|\le L$ on
$[-1,1]$, the mean value theorem gives
$|X_i|\le L|\ip\thh{Z_{1,i}-Z_{0,i}}|=L|s_i-s_i^{\mathrm{gh}}|$. Given $(U,S_n,R)$, the ghost
$\Zgh_i$ is independent of $\thh$ with mean $\mu_p$. So $\E_p[s_is_i^{\mathrm{gh}}]=0$, and
\eqref{eq:margin-variance} gives $\E[(s_i^{\mathrm{gh}})^2\mid U,S_n,R]\le\norm\thh^2/d$. Hence, at
every $p$,
\begin{equation}\label{eq:energy}
 \sum_i\E_pX_i^2
 \le L^2\Big\{\E_p\norm{\Btr\thh}^2+\sum_i\E_p(s_i^{\mathrm{gh}})^2\Big\}
 \le L^2\Big\{\E_p\big[\norm\thh^2\opnorm\Btr^2\big]+\rho\,\E_p\norm\thh^2\Big\}.
\end{equation}
Averaging over the prior and using \eqref{eq:Ctheta} and $\E_\Pi\norm\thh^2\le w_\theta$ gives
$\sum_i\E_\Pi X_i^2\le L^2(C_\theta+\rho w_\theta)=V$. The bound $V$ concerns the prior average; the
residual bound of the next subsection holds at every law.

\subsection{The Residual at Every Law}\label{app:residual}

This subsection proves the third bound of \eqref{eq:finite-budgets}. Fix $p$, let
$e=\thh-\theta_p^*$, and let $\Delta$ be the pair matrix, with rows
$\Delta_i=(Z_{1,i}-Z_{0,i})^\top$. For $t\in[0,1]$, the point $\theta_p^*+te=(1-t)\theta_p^*+t\thh$
lies in $W$. The fundamental theorem of calculus along this segment gives
\[
 X_i-X_{p,i}^*(\Zt)
 =\int_0^1\Big\{\phi'\big(\ip{\theta_p^*+te}{Z_{1,i}}\big)\ip e{Z_{1,i}}
 -\phi'\big(\ip{\theta_p^*+te}{Z_{0,i}}\big)\ip e{Z_{0,i}}\Big\}\,dt .
\]
Write the integrand as
$\phi'(\ip{\theta_p^*+te}{Z_{1,i}})\ip e{\Delta_i}
+\{\phi'(\ip{\theta_p^*+te}{Z_{1,i}})-\phi'(\ip{\theta_p^*+te}{Z_{0,i}})\}\ip e{Z_{0,i}}$. Since
$|\phi'|\le L$, $0\le\phi''\le H$ and $\norm{Z_{0,i}}=1$,
\[
 \big|X_i-X_{p,i}^*(\Zt)\big|
 \le L|\ip e{\Delta_i}|+H\norm e\int_0^1\big|\ip{\theta_p^*+te}{\Delta_i}\big|\,dt .
\]
Take the Euclidean norm over $i$. The first term gives $L\norm{\Delta e}\le L\opnorm\Delta\norm e$.
For the second, Minkowski's integral inequality and $\norm{\theta_p^*+te}\le1$ give
\[
 \Big(\sum_i\Big(\int_0^1\big|\ip{\theta_p^*+te}{\Delta_i}\big|\,dt\Big)^2\Big)^{1/2}
 \le\int_0^1\norm{\Delta(\theta_p^*+te)}\,dt\le\opnorm\Delta .
\]
Hence
\[
 \norm{X-X_p^*(\Zt)}\le(L+H)\opnorm\Delta\norm e .
\]
By Lemma~\ref{lem:matrix}, in the form proved in Appendix~\ref{app:tails},
$\Pp_p(\opnorm\Delta^2>2K_0+2K_1t)\le e^{-t}$ for $t\ge0$, and by
\eqref{eq:error} the weight $\norm e^2\in[0,4]$ has $\E_p\norm e^2\le w\le4$. Lemma~\ref{lem:weighted}
with the cap $\bar\omega=4$ therefore gives
\[
 \E_p\norm{X-X_p^*(\Zt)}^2
 \le(L+H)^2\,\E_p\big[\opnorm\Delta^2\norm e^2\big]
 \le(L+H)^2\,\Tb_4(w;2K_0,2K_1)=\Vres .
\]
The error $e$ depends on the training candidates, which enter the rows of $\Delta$;
Lemma~\ref{lem:weighted} allows this dependence. The bound holds at every $p$ and uses accuracy only
at $D_p$, through \eqref{eq:error}. It concerns the loss differences themselves, with no
linearization of the link. This completes the proof of Proposition~\ref{prop:finite}.

\subsection{Orders and Assembly}\label{app:orders}

This subsection derives the orders of the finite bounds in the regime of Theorem~\ref{thm:main} and
proves the theorem. Let $0<c_-\le c_+$, $\neps=\eps^{-2/q}$, $c_-\neps\le n\le c_+\neps$ and
$d\ge\Cdim\max\{n,\neps\}$, with $\Cdim=32\max\{1,(\beta/4)^{2/q}\}$. Since $\beta<b$, $\Cdim=32$
whenever $b\le4$. As $v^{-1}=(\beta/4)^{2/q}\neps$, the dimension condition implies
\begin{equation}\label{eq:dim-condition}
 d\ge32\max\{n,v^{-1}\},
\end{equation}
and this is the only way in which the dimension enters below. Its two parts play different roles. The
first, $d\ge32n$, keeps $\rho\le1/32$ and so bounds the operator norms of Lemma~\ref{lem:matrix},
which control the energy, the residual and the penalty terms. The second, $dv\ge32$, makes the
dimension terms of the prior moments and of the margin variance \eqref{eq:margin-variance} small
against powers of $v$: under the prior, $\norm{\mu_p}^2$ has relative variance at most $1/(dv)$
(Appendix~\ref{app:prior-moments}), and $1/d\le v/32$. This gives the factor $(32/33)^{q/2}$ in the
bound on $A_\phi$ below and keeps $P_\phi$ of order $\eps^2$. Write $K_0(\rho)$ for $K_0$ of
\eqref{eq:matrix-constants} at the ratio $\rho$; it increases with $\rho$. By \eqref{eq:dim-condition},
$\rho\le1/32$, so $K_0\le K_0^*$ with
\[
 K_0^*=K_0(1/32)=\frac{\ln61/16+\ln(40/21)}{\cnet}<2.031,
\]
and $d\ge32n\ge32c_-\neps$ gives $K_1=2/(\cnet d)\le\eps^{2/q}/(16\cnet c_-)$.

Below, $O$, $\Omega$ and $\Theta$ hide constants that depend only on the constants of
Assumption~\ref{ass:class} and on $c_\pm$, and the bounds hold for all sufficiently small $\eps$,
uniformly in $n$ and $d$ with $c_-\neps\le n\le c_+\neps$ and \eqref{eq:dim-condition}. For small $\eps$ the caps in \eqref{eq:finite-small} are
inactive: $w=(\eps/a_-)^{2/r}$; since $(q-1)/q=1/r$, $w_*=C_0^2(4\eps/\beta)^{2/r}$; and
$w_\theta=(\sqrt w+\sqrt{w_*})^2$. The constant $K_0$ lies between $\ln(40/21)/\cnet$ and $K_0^*$,
and $K_1\ln(1/w_\theta)=O(\eps^{2/q}\ln(1/\eps))=o(1)$, so the weighted tail costs no logarithmic
factor in the limit and $C_\theta=\Theta(w_\theta)$. Next, $\sqrt{wv}=\Theta(\eps^{1/r+1/q})=\Theta(\eps)$
and $v^{q/2}=4\eps/\beta$, so $Q_\theta=O(\eps^2)$; and $v^q=(4\eps/\beta)^2$ and, by
\eqref{eq:dim-condition}, $v^{q-1}/d\le v^q/32$, so $P_\phi=O(\eps^2)$. Finally
$n\eps^2=\Theta(\eps^{2-2/q})=\Theta(\eps^{2/r})$. In summary,
\begin{equation}\label{eq:orders}
 w,\ w_*,\ w_\theta,\ C_\theta=\Theta(\eps^{2/r}),\qquad
 Q_\theta,\ P_\phi=O(\eps^2),\qquad
 V,\ \Vres=\Theta(\eps^{2/r})=\Theta(n\eps^2).
\end{equation}

For the signal, $\beta v^{q/2}=4\eps$ and $dv\ge32$ give
\[
 A_\phi\ge\big\{4(32/33)^{q/2}-1\big\}\eps-P_\phi\ge\tfrac{95}{33}\,\eps-O(\eps^2),
\]
because $(32/33)^{q/2}\ge32/33$ for $q\le2$; also $A_\phi\le3\eps$. As
$\kappa=v^{-1}-1=\Theta(\eps^{-2/q})=\Theta(n)$, the signal is
$\kappa A_\phi=\Theta(n\eps)=\Theta(\eps^{-1+2/r})$. Since $-2/q+2+2/r=4/r$, the two penalty terms of
\eqref{eq:finite-M} satisfy
\[
 H\sqrt{nQ_\theta C_\theta}+\frac H2C_\theta
 =O\big(\sqrt{\eps^{-2/q}\,\eps^2\,\eps^{2/r}}+\eps^{2/r}\big)
 =O(\eps^{2/r})=O(\eps\cdot n\eps),
\]
one power of $\eps$ below the signal. Hence $M=\Theta(n\eps)$: there are $\eps_0>0$, $c_M>0$, $C_V$
and $C_{\mathrm{res}}$, depending only on the constants of Assumption~\ref{ass:class} and on $c_\pm$,
such that for $0<\eps\le\eps_0$
\[
 M\ge c_Mn\eps,\qquad V\le C_Vn\eps^2,\qquad \Vres\le C_{\mathrm{res}}n\eps^2 .
\]
Together with \eqref{eq:finite-budgets}, this is Proposition~\ref{prop:finite}.
Appendix~\ref{app:constants} gives explicit values of $\eps_0$, $c_M$, $C_V$ and $C_{\mathrm{res}}$.

\paragraph{Proof of Theorem~\ref{thm:main}.}
Let $0<\eps\le\eps_0$, let $n$ and $d$ be as above, and let $\thh$ be a proper learner that is
$\eps$-accurate on the family. The prior $\Pi_\kappa$ is fixed before the learner, $p$ is drawn
independently of the selection bits, the seed and the noise, and the reference $X_p^*(\Zt)$ depends
only on $p$ and the candidates. By
Proposition~\ref{prop:finite}, Theorem~\ref{thm:transfer} applies with the reference
$x_p=X_p^*$ and with $M\ge c_Mn\eps>0$, $V\le C_Vn\eps^2$ and $\Vres\le C_{\mathrm{res}}n\eps^2$. Its
lower bound \eqref{eq:lower} gives, for every $\tau\ge0$,
\[
 \E_\Pi\Ii_{D_p}(X+\tau G)
 \ge\frac{M^2}{2\ln2\,(V+n\tau^2)}
 \ge\frac{c_M^2}{2\ln2}\,\frac n{C_V+(\tau/\eps)^2}
 \ge c_1\frac n{1+(\tau/\eps)^2},\qquad
 c_1=\frac{c_M^2}{2\ln2\,\max\{C_V,1\}},
\]
since $C_V+x\le\max\{C_V,1\}(1+x)$ for $x\ge0$. Its upper bound \eqref{eq:upper} gives, at every $p$
and for $\tau>0$, with $a=C_{\mathrm{res}}/(2\ln2)$ and $x=(\tau/\eps)^2$,
\[
 \Is_{D_p}(X+\tau G)
 \le n\min\Big\{1,\frac{\ln\big(1+C_{\mathrm{res}}(\eps/\tau)^2\big)}{2\ln2}\Big\}
 \le n\min\Big\{1,\frac ax\Big\}
 \le(1+a)\frac n{1+x},
\]
by $\ln(1+y)\le y$ and by the inequality $\min\{1,a/x\}\le(1+a)/(1+x)$, which holds for $a\ge0$ and
$x>0$ (consider the cases $x\le a$ and $x>a$). At $\tau=0$, Theorem~\ref{thm:transfer} gives
$\Is_{D_p}(X)\le\Hent(U\mid\Zt)=n$. So $\Is_{D_p}(X+\tau G)\le C_1n/(1+(\tau/\eps)^2)$ for every
$\tau\ge0$ and every $p$, with $C_1=1+C_{\mathrm{res}}/(2\ln2)$.

By \eqref{eq:hierarchy}, $\Ii\le\Ij\le\Is$ at every $p$. So the lower bound holds for all three
information measures on average over the prior, and the upper bound holds for all three at every $p$. A prior
average is at most the supremum over $p$, which gives both bounds of \eqref{eq:main} for
$\sup_p\Qc_{D_p}(X+\tau G)$, for each $\Qc\in\{\Ii,\Ij,\Is\}$ and every $\tau\ge0$. The constants
$\eps_0$, $c_1$, $C_1$ and $\Cdim$ depend only on the constants of Assumption~\ref{ass:class} and on
$c_\pm$, not on the learner, so both bounds also hold for the infimum over any nonempty class of such
learners. This proves Theorem~\ref{thm:main}. The maximizing law may depend on the learner, on $\Qc$
and on $\tau$.

\subsection{Explicit Constants}\label{app:constants}

This subsection gives explicit values of $\eps_0$, $c_M$, $C_V$, $C_{\mathrm{res}}$, $c_1$ and $C_1$;
the dimension multiplier is $\Cdim=32\max\{1,(\beta/4)^{2/q}\}$, as in Appendix~\ref{app:orders}.
Throughout, $c_-\neps\le n\le c_+\neps$ and $d$ satisfies the finite condition
\eqref{eq:dim-condition}, which is all that is used of the dimension. We work with the normalized
moments
\[
 \Mbar=\frac M{n\eps},\qquad \Vbar=\frac V{n\eps^2},\qquad \Vresbar=\frac{\Vres}{n\eps^2},
\]
so that the bounds $M\ge c_Mn\eps$, $V\le C_Vn\eps^2$ and $\Vres\le C_{\mathrm{res}}n\eps^2$ of
Appendix~\ref{app:orders} read $\Mbar\ge c_M$, $\Vbar\le C_V$ and $\Vresbar\le C_{\mathrm{res}}$. With
$K_0^*$ as in Appendix~\ref{app:orders}, write
\[
 c_v=\Big(\frac4\beta\Big)^{2/q},\qquad
 c_w=\Big(a_-^{-1/r}+C_0\Big(\frac4\beta\Big)^{1/r}\Big)^2,\qquad
 c_Q=a_-^{-1/r}\Big(\frac4\beta\Big)^{1/q}+C_0\,\frac4\beta\,\Big(1+\frac v{16}\Big)^{q/4},
\]
and let $\eps<\min\{\beta/4,\,2^ra_-,\,c_w^{-r/2}\}$, so that $0<v<1$ and the caps in
\eqref{eq:finite-small} are inactive. Then $v=c_v\eps^{2/q}$, $w=(\eps/a_-)^{2/r}$ and
$w_\theta=c_w\eps^{2/r}$.

\paragraph{A Lower Bound on $M$, Uniform in $d$.} Under these conditions,
\begin{align}
 M&\ge\eps^{1-2/q}\,\Upsilon(\eps),\label{eq:explicit-threshold}\\
 \Upsilon(\eps)&=\big(c_v^{-1}-\eps^{2/q}\big)\Big(4\Big(\frac{32}{33}\Big)^{q/2}-1-\bar P_\phi(\eps)\Big)
 -H\eps\Big(c_Q\sqrt{c_+c_w\bar K}+\frac{c_w\bar K}2\Big),\nonumber
\end{align}
% \bar P_\phi(\eps): normalized population term, P_\phi\le\eps\bar P_\phi(\eps).
where
\[
 \bar P_\phi(\eps)=\frac{Ht_0^2}2\Big(\frac4\beta\Big)^2\eps
 \Big[\Big(1+\frac v{16}\Big)^{q/2}+\frac1{32}\Big],\qquad
 \bar K=K_0^*+\frac{\eps^{2/q}}{16\cnet c_-}\Big(1+\ln\frac1{w_\theta}\Big).
\]
To prove \eqref{eq:explicit-threshold}, write
$\kappa=v^{-1}-1=\eps^{-2/q}(c_v^{-1}-\eps^{2/q})>0$. In \eqref{eq:finite-align},
$\beta v^{q/2}=4\eps$ and $dv\ge32$ give $\beta v^{q/2}(1+1/(dv))^{-q/2}\ge4(32/33)^{q/2}\eps$, while
$2/d\le v/16$, $1/d\le v/32$ and $v^q=(4/\beta)^2\eps^2$ give
$P_\phi\le(Ht_0^2/2)v^q[(1+v/16)^{q/2}+1/32]=\eps\bar P_\phi(\eps)$. Hence
$\kappa A_\phi\ge\eps^{1-2/q}(c_v^{-1}-\eps^{2/q})(4(32/33)^{q/2}-1-\bar P_\phi(\eps))$. For the
penalty terms, $K_0\le K_0^*$ and $K_1\le\eps^{2/q}/(16\cnet c_-)$ give
$C_\theta\le w_\theta\bar K=c_w\bar K\eps^{2/r}$; and $\sqrt{wv}=a_-^{-1/r}(4/\beta)^{1/q}\eps$ and
$C_0v^{q/2}(1+2/d)^{q/4}\le C_0(4/\beta)(1+v/16)^{q/4}\eps$ give $Q_\theta\le c_Q^2\eps^2$. As
$n\le c_+\neps$ and $\neps\eps^2\eps^{2/r}=\eps^{4/r}$,
\[
 H\sqrt{nQ_\theta C_\theta}+\frac H2C_\theta
 \le H\eps^{2/r}\Big(c_Q\sqrt{c_+c_w\bar K}+\frac{c_w\bar K}2\Big)
 =\eps^{1-2/q}\cdot H\eps\Big(c_Q\sqrt{c_+c_w\bar K}+\frac{c_w\bar K}2\Big),
\]
because $\eps^{2/r}=\eps^{2-2/q}$. Inserting both bounds into \eqref{eq:finite-M} gives
\eqref{eq:explicit-threshold}.

\paragraph{The Threshold and $c_M$.} As $\eps\downarrow0$,
$\Upsilon(\eps)\to c_v^{-1}(4(32/33)^{q/2}-1)>0$. Moreover, $\Upsilon(\eps')\ge\Upsilon(\eps)$ whenever
$0<\eps'\le\eps$ and $\Upsilon(\eps)>0$. Indeed, the first factor $c_v^{-1}-\eps^{2/q}$ is positive and
decreasing in $\eps$, because $\eps<\beta/4$; and $\bar P_\phi$, $c_Q$ and $\bar K$ are nondecreasing
in $\eps$. For $\bar K$, the derivative of $\eps^{2/q}(1+\ln(1/w_\theta))$, with
$w_\theta=c_w\eps^{2/r}$, is $\eps^{2/q-1}[(2/q)(1+\ln(1/w_\theta))-2/r]\ge0$, since $q\le2\le r$ and
$w_\theta\le1$, the latter because $\eps<c_w^{-r/2}$. So if $\Upsilon(\eps)>0$, its second factor is positive as well, both factors can only
grow when $\eps$ decreases, and the subtracted term can only shrink. Hence, if $\Upsilon(\eps_1)>0$ for
some $\eps_1<\min\{\beta/4,2^ra_-,c_w^{-r/2}\}$, then for every $\eps\in(0,\eps_1]$ in the regime,
\[
 M\ge\eps^{1-2/q}\,\Upsilon(\eps_1)>0,\qquad\text{so}\qquad
 \Mbar\ge\frac{\Upsilon(\eps_1)}{c_+}=:c_M .
\]
Such an $\eps_1$ exists, and we take $\eps_0=\eps_1$.

\paragraph{Energy, Residual, $C_V$ and $C_{\mathrm{res}}$.} Under the same conditions,
$C_\theta\le w_\theta\bar K$, $\rho\le1/32$, $w_\theta=c_w\eps^{2/r}$, $w=(\eps/a_-)^{2/r}$ and
$n\eps^2\ge c_-\eps^{2/r}$ give
\[
 \Vbar\le\frac{L^2c_w}{c_-}\Big(\bar K+\frac1{32}\Big),\qquad
 \Vresbar\le\frac{2(L+H)^2}{c_-a_-^{2/r}}\Big\{K_0^*+\frac{\eps^{2/q}}{16\cnet c_-}
 \Big(1+\ln\frac4w\Big)\Big\}.
\]
Both right sides are nondecreasing in $\eps$, by the derivative computed above with $4/w\ge1$ (as
$\eps<2^ra_-$) in place of $1/w_\theta$. So their values at $\eps_1$ are admissible values of $C_V$ and $C_{\mathrm{res}}$ for
every $\eps\in(0,\eps_1]$.

\paragraph{The Two-Sided Bound in Normalized Form.} At given $\eps$, $n$ and $d$ with $M>0$,
Theorem~\ref{thm:transfer} and \eqref{eq:hierarchy} give, for every $\tau\ge0$ and each
$\Qc\in\{\Ii,\Ij,\Is\}$,
\begin{align}
 \frac{\Mbar^2}{2\ln2}\,\frac n{\Vbar+(\tau/\eps)^2}
 &\le\sup_p\Qc_{D_p}(X+\tau G)\nonumber\\
 &\le n\min\Big\{1,\frac1{2\ln2}\ln\big(1+\Vresbar(\eps/\tau)^2\big)\Big\}
 \le\Big(1+\frac{\Vresbar}{2\ln2}\Big)\frac n{1+(\tau/\eps)^2},\label{eq:explicit-bracket}
\end{align}
where the middle term is read as $n$ at $\tau=0$. Since $\Vbar+x\le\max\{\Vbar,1\}(1+x)$ for $x\ge0$,
\eqref{eq:explicit-bracket} has the form \eqref{eq:main} with $c_1=\Mbar^2/(2\ln2\max\{\Vbar,1\})$ and
$C_1=1+\Vresbar/(2\ln2)$, once $\Mbar$ is replaced by a lower bound and $\Vbar$, $\Vresbar$ by upper
bounds valid over the regime. With the bounds above, Theorem~\ref{thm:main} therefore holds with
\[
 \eps_0=\eps_1,\qquad c_1=\frac{c_M^2}{2\ln2\,\max\{C_V,1\}},\qquad
 C_1=1+\frac{C_{\mathrm{res}}}{2\ln2},
\]
where $c_M=\Upsilon(\eps_1)/c_+$ and $C_V$, $C_{\mathrm{res}}$ are the values at $\eps_1$ given above.

\paragraph{Limits.} The normalized moments depend on the multipliers of $n$ and $d$. Along
$n=c\,\neps$ and $d=n/\rho$, with fixed $c\in[c_-,c_+]$ and $\rho>0$ such that $cc_v\ge32\rho$ (so
that $dv=cc_v/\rho\ge32$),
\begin{equation}\label{eq:explicit-limits}
 \Mbar\to\frac1{cc_v}\Big[4\Big(1+\frac\rho{cc_v}\Big)^{-q/2}-1\Big],\qquad
 \Vbar\to\frac{L^2\big(K_0(\rho)+\rho\big)c_w}c,\qquad
 \Vresbar\to\frac{2K_0(\rho)(L+H)^2}{c\,a_-^{2/r}}\qquad(\eps\downarrow0).
\end{equation}
Indeed, along this schedule $\kappa/n\to1/(cc_v)$, $A_\phi/\eps\to4(1+\rho/(cc_v))^{-q/2}-1$, the
penalty terms are $O(\eps\cdot n\eps)$, $K_1\ln(1/w_\theta)\to0$, $w_\theta/(n\eps^2)=c_w/c$ and
$w/(n\eps^2)=a_-^{-2/r}/c$. Over the whole regime, that is, uniformly in $n$ and $d$ with
$c_-\neps\le n\le c_+\neps$ and \eqref{eq:dim-condition}, the bounds above give
\[
 \liminf_{\eps\downarrow0}\Mbar\ge\frac{4(32/33)^{q/2}-1}{c_+c_v},\qquad
 \limsup_{\eps\downarrow0}\Vbar\le\frac{L^2(K_0^*+1/32)\,c_w}{c_-},\qquad
 \limsup_{\eps\downarrow0}\Vresbar\le\frac{2K_0^*(L+H)^2}{c_-a_-^{2/r}}.
\]
In these limits the two penalty terms of \eqref{eq:finite-M}, which come from Step~3 of
Section~\ref{sec:moments}, vanish after normalization, and the weighted tails of
Lemma~\ref{lem:weighted} cost no logarithmic factor. So at the limiting normalized moments the ratio
$C_1/c_1=(2\ln2+\Vresbar)\max\{\Vbar,1\}/\Mbar^2$ is set by the signal $\kappa A_\phi$ of Step~1 and
by the energy and residual bounds of Step~4, which rest on the localization of Step~2 and on
Lemma~\ref{lem:matrix}.

\paragraph{Monotonicity in $d$.} Every quantity in
\eqref{eq:matrix-constants}--\eqref{eq:finite-VD} is monotone in $d$ in the favourable direction. As $d$ grows, $K_0$, $K_1$, $1/d$ and the factors $(1+2/d)$ decrease and
$(1+1/(dv))^{-q/2}$ increases, while $\Tb_{\bar\omega}(x;K_0,K_1)$ is nondecreasing in $x$, $K_0$ and
$K_1$ (its derivative in $x$ is $K_0+K_1\ln(\bar\omega/x)\ge0$). Hence $C_\theta$, $Q_\theta$ and
$P_\phi$ decrease and $A_\phi$ increases, so $M$ increases while $V$ and $\Vres$ decrease. Values of
$M$, $V$ and $\Vres$ computed at one dimension therefore remain valid at every larger dimension; for
instance, those computed at $d=32n$ hold for every $d\ge32n$.

\paragraph{Quartic Logistic Learning.} For $\phi(t)=\ln(1+e^{-t})$ and $g(\theta)=\norm\theta^4/4$,
the constants of Assumption~\ref{ass:class} are $r=4$, $b=1/2$, $L=(1+e^{-1})^{-1}$, $H=1/4$,
$a_-=1/16$ and $a_+=1/4$ (Appendices~\ref{app:links} and~\ref{app:power}), so $q=4/3$,
$\neps=\eps^{-3/2}$ and $\Cdim=32$. At $c=1$ and $\rho=1/32$, the limits \eqref{eq:explicit-limits}
are $(\Mbar,\Vbar,\Vresbar)\to(0.050,29.2,15.6)$. With $c_-=c_+=1$, $\Upsilon$ is positive at
$\eps=6.0\cdot10^{-4}$ and negative at $\eps=6.1\cdot10^{-4}$, so $\eps_1=6.0\cdot10^{-4}$ is an
admissible threshold. Evaluated directly along $n\eps^{3/2}=1$ at $d=32n$, rather than through
\eqref{eq:explicit-threshold}, $M$ changes sign at $n\approx6.5\cdot10^4$ and stays positive at
larger $n$; by the monotonicity in $d$, wherever $M$ is positive at $d=32n$ it is
positive at every $d\ge32n$. At the limiting normalized moments, the lower bound of \eqref{eq:explicit-bracket}
and its first upper bound differ by a factor of at most about $1.9\cdot10^4$ over $\tau$ (about
$1.6\cdot10^4$ at $\tau=0$, and largest near $(\tau/\eps)^2\approx5$), and
$C_1/c_1\approx2.0\cdot10^5$, mainly through $1/\Mbar^2\approx400$ and $\Vbar\approx29$. At
$\eps_0=\eps_1=6.0\cdot10^{-4}$ itself, where $\Upsilon$ is close to $0$, the constants of
Theorem~\ref{thm:main} are much further apart, $C_1/c_1\approx10^{10}$; a smaller admissible $\eps_1$
brings the ratio down towards $2.1\cdot10^5$.

\paragraph{Ridge and Logistic Ridge.} With $g(\theta)=\norm\theta^2/2$ we have $r=2$ and
$a_-=a_+=1/2$; the affine link has $b=L=1$ and $H=0$, and the logistic link the values above. At $c=1$
and $\rho=1/32$ the limits \eqref{eq:explicit-limits} are $(0.187,60.4,8.12)$ for the affine link and
$(0.062,26.2,7.82)$ for the logistic link. At these limiting normalized moments,
$c_1\approx4.2\cdot10^{-4}$, $C_1\approx6.9$ and $C_1/c_1\approx1.6\cdot10^4$ for the affine link, and
$c_1\approx1.1\cdot10^{-4}$, $C_1\approx6.6$ and $C_1/c_1\approx6.2\cdot10^4$ for the logistic link. Along $n\eps=1$ at $d=32n$, $M$ is positive for every $n>16$ with the affine link, and
from $n\approx1.1\cdot10^3$ on with the logistic link. At an admissible threshold the constants of
Theorem~\ref{thm:main} are further apart, as in the quartic case: with $c_-=c_+=1$, $\eps_1=0.034$ is
admissible for the affine link and gives $C_1/c_1\approx8.6\cdot10^4$, and $\eps_1=8.7\cdot10^{-4}$ is
admissible for the logistic link and gives $C_1/c_1\approx10^9$. As $\eps_1$ decreases, the two ratios
decrease towards $1.8\cdot10^4$ and $6.7\cdot10^4$.

The script \texttt{experiments/quartic/explicit\_constants.py} in the supplementary archive evaluates
the formulas of this appendix and prints every value quoted in this subsection.

%% file: app_E_examples.tex
\section{\texorpdfstring{Proposition~\ref{prop:samples}}{The optimal sample size} and examples}\label{app:examples}

This appendix proves Proposition~\ref{prop:samples} (Appendices~\ref{app:sgd}
and~\ref{app:sample-lower}), checks the links and regularizers named in
Section~\ref{sec:setting} (Appendices~\ref{app:links} and~\ref{app:power}), treats links with
$\phi'(0)>0$ and a regularization strength in a compact range (Appendices~\ref{app:reflection}
and~\ref{app:strength}), and proves the two-sided bound without a regularizer
(Appendix~\ref{app:noreg}).

We prove Proposition~\ref{prop:samples} in the following form, with explicit constants. Recall
$q=r/(r-1)$ and $\neps=\eps^{-2/q}=\eps^{-2+2/r}$, and the calibration constants
$t_0=(b/(ra_++b))^{1/(r-1)}$ and $\beta=bt_0/q$ of Section~\ref{sec:architecture}.

\begin{restate}[explicit form]{Proposition}{prop:samples}
Let Assumption~\ref{ass:class} hold, and put $L_f=L_g+L$ and $c_v=(4/\beta)^{2/q}$.
\begin{enumerate}
\item[(i)] For $0<\eps<2L_f$, projected epoch SGD, whose step sizes and epoch lengths depend only on
$\eps$, $L_f$, $a_-$ and $r$, uses at most $N(\eps)=8L_f^2a_-^{-2/r}\neps+\log_2(4L_f/\eps)=O(\neps)$
samples, and it is a proper learner that is $\eps$-accurate on every distribution on the unit ball.
\item[(ii)] If $c_v\eps^{2/q}\le1/2$, $16Ht_0^2\eps\le\beta^2$ and $d\ge\Cdim\neps$, then every proper
learner that is $\eps$-accurate on the family uses $n\ge a_-^{2/r}(2bc_v)^{-2}\,\neps$ samples.
\end{enumerate}
\end{restate}

\subsection{Epoch SGD}\label{app:sgd}
This subsection proves part~(i). Let $D$ be a law of $Z$ with $\norm Z\le1$ almost surely, and write
$F=F_D$, $\theta^*=\theta_D^*$ and $F^*=F(\theta^*)=\min_WF$. For $\theta\in W$ and $\norm z\le1$ the
margin $\ip\theta z$ lies in $[-1,1]$, so the stochastic gradient
$\nabla_\theta f(\theta,z)=\nabla g(\theta)+\phi'(\ip\theta z)z$ has norm at most $L_f=L_g+L$. The
function $F-g$ is convex and differentiable on $W$, and $\theta^*$ satisfies the first-order condition
$\ip{\nabla F(\theta^*)}{\theta-\theta^*}\ge0$ for $\theta\in W$. Hence, by \eqref{eq:matched},
\[
 F(\theta)-F^*\ \ge\ F(\theta)-F^*-\ip{\nabla F(\theta^*)}{\theta-\theta^*}\ \ge\ D_g(\theta,\theta^*)
 \ \ge\ a_-\norm{\theta-\theta^*}^r\qquad(\theta\in W).
\]

\paragraph{One Epoch.} An epoch starts at a random point $\theta_1\in W$, reads $n'$ training points
$Z^{(1)},\dots,Z^{(n')}$ that no earlier epoch has used, sets
\[
 \theta_{t+1}=\proj_W\big(\theta_t-\eta\nabla_\theta f(\theta_t,Z^{(t)})\big)\qquad(t=1,\dots,n'),
\]
and returns the average $\bar\theta=\frac1{n'}\sum_{t=1}^{n'}\theta_t$ of the iterates before each
update. The projection does not increase the distance to $\theta^*\in W$, so
\[
 \norm{\theta_{t+1}-\theta^*}^2\le\norm{\theta_t-\theta^*}^2
 -2\eta\ip{\nabla_\theta f(\theta_t,Z^{(t)})}{\theta_t-\theta^*}+\eta^2L_f^2 .
\]
The point $Z^{(t)}$ is independent of $\theta_t$, which depends only on $\theta_1$ and on earlier
points, so the conditional mean of the stochastic gradient given $\theta_t$ is $\nabla F(\theta_t)$,
and convexity gives $\ip{\nabla F(\theta_t)}{\theta_t-\theta^*}\ge F(\theta_t)-F^*$. Summing over $t$,
taking expectations and using the convexity of $F$ at $\bar\theta$,
\begin{equation}\label{eq:sgd-epoch}
 \E F(\bar\theta)-F^*\le\frac{\E\norm{\theta_1-\theta^*}^2}{2\eta n'}+\frac{\eta L_f^2}2 .
\end{equation}
Suppose that $\E F(\theta_1)-F^*\le\eps'$. The growth bound above and Jensen's inequality ($r\ge2$) give
$\E\norm{\theta_1-\theta^*}^2\le(\E\norm{\theta_1-\theta^*}^r)^{2/r}\le(\eps'/a_-)^{2/r}$. With
\[
 \eta=\frac{\eps'}{2L_f^2},\qquad n'=\big\lceil4L_f^2a_-^{-2/r}(\eps')^{-2+2/r}\big\rceil,
\]
each term of \eqref{eq:sgd-epoch} is at most $\eps'/4$, so the epoch returns a point whose expected
excess risk is at most $\eps'/2$.

\paragraph{The Schedule.} The learner starts at $\theta_1=0$, where $F(0)-F^*\le2L_f$ because $F$ is
$L_f$-Lipschitz on $W$, whose diameter is $2$. If $2L_f\le\eps$ it returns $0$ and uses no sample.
Otherwise it runs an epoch with target $\eps'=2L_f$, then an epoch with target $L_f$ started at the
output of the first, and so on, halving the target while the target exceeds $\eps$. By induction every
epoch starts within its target, and the last epoch, whose target satisfies $\eps<\eps'\le2\eps$,
returns a point with expected excess risk at most $\eps'/2\le\eps$. The targets form a geometric
sequence with ratio $1/2$ whose smallest term exceeds $\eps$, and $2-2/r\ge1$, so the sum of
$(\eps')^{-2+2/r}$ over the targets is at most $(1+2^{-1}+2^{-2}+\cdots)\eps^{-2+2/r}=2\eps^{-2+2/r}$.
There are at most $\log_2(4L_f/\eps)$ epochs, and rounding up adds at most one sample to each, so the
total number of samples is at most $N(\eps)$. Every iterate and every average lies in the convex set $W$,
so the learner is proper, and it is measurable. Its schedule depends only on $\eps$, $L_f$, $a_-$ and
$r$, not on $D$, and with more than $N(\eps)$ samples it ignores the rest. This proves part~(i); since
the guarantee holds at every law on the unit ball, the learner is in particular $\eps$-accurate on the
family.

Maximizing $x^{2-2/r}\ln(1/x)$ over $x=\eps/(4L_f)\in(0,1/2)$ shows that
$\eps^{2-2/r}\log_2(4L_f/\eps)\le(4L_f)^{2-2/r}/(e\ln2)$, so
\[
 N(\eps)\le\Big(8L_f^2a_-^{-2/r}+\frac{(4L_f)^{2-2/r}}{e\ln2}\Big)\,\eps^{-2+2/r}
 \qquad(0<\eps<2L_f).
\]
Hence the class of learners of Theorem~\ref{thm:main} is nonempty whenever $c_-$ is at least the
factor in parentheses. Multistage stochastic methods for uniformly convex objectives are
classical \citep{juditsky2014,daspremont2024}; the argument above records the guarantee in the form
used here.

\subsection{The Sample Size Is Necessary}\label{app:sample-lower}
This subsection proves part~(ii), by the calibration and fingerprinting step of the proof of
Theorem~\ref{thm:main}; here the localization bounds the fingerprinting correlation from above.
Let $\thh$ be a proper learner that is $\eps$-accurate on the family. Take the beta prior $\Pi_\kappa$
of \eqref{eq:prior-choice}, with $v=(4\eps/\beta)^{2/q}=c_v\eps^{2/q}$ and $\kappa=v^{-1}-1$, and
assume $v\le1/2$, $16Ht_0^2\eps\le\beta^2$ and $d\ge\Cdim\neps$. Since
$\Cdim=32\max\{1,(\beta/4)^{2/q}\}$ (Appendix~\ref{app:orders}) and $v^{-1}=(\beta/4)^{2/q}\neps$,
the last condition gives $dv\ge32$.

\paragraph{The Signal Is at Least of Order $\eps^{1-2/q}$.} The calibration \eqref{eq:alignment},
$b\,\E_\Pi\ip{\thh}{\mu_p}\ge A_\phi$, uses accuracy at every $p$ and the prior moments, but neither
the sample size nor a lower bound on $d$. By \eqref{eq:finite-pop} and \eqref{eq:finite-align},
\[
 A_\phi=\beta v^{q/2}\Big(1+\frac1{dv}\Big)^{-q/2}-\eps-P_\phi,\qquad
 P_\phi=\frac{Ht_0^2}2\Big\{v^q\Big(1+\frac2d\Big)^{q/2}+\frac{v^{q-1}}d\Big\},
\]
where $\beta v^{q/2}=4\eps$ and $v^q=(4\eps/\beta)^2$ by the choice of $v$. As $dv\ge32$, $v\le1/2$ and $q\le2$, we have
$(1+1/(dv))^{-q/2}\ge32/33$, $(1+2/d)^{q/2}\le1+v/16\le33/32$ and $v^{q-1}/d\le v^q/32$. Hence
$P_\phi\le Ht_0^2v^q=16Ht_0^2\eps^2/\beta^2\le\eps$ and $A_\phi\ge(128/33-2)\eps\ge\eps$. Since $v\le1/2$,
$\kappa=(1-v)/v\ge1/(2v)$, and therefore
\[
 b\kappa\,\E_\Pi\ip{\thh}{\mu_p}\ \ge\ \kappa A_\phi\ \ge\ \frac\eps{2v}
 \ =\ \frac{\eps^{1-2/q}}{2c_v}.
\]

\paragraph{The Correlation Is at Most of Order $\sqrt{n}\,\eps^{1/r}$.} By Lemma~\ref{lem:fp},
$\kappa\,\E_\Pi\ip{\thh}{\mu_p}=\E_\Pi\ip{\thh}{\sum_i(\Ztr_i-\mu_p)}$. Given $p$, the minimizer
$\theta_p^*$ is fixed and the training points $\Ztr_i$ are independent with mean $\mu_p$ and norm $1$,
so $\E_p\ip{\theta_p^*}{\sum_i(\Ztr_i-\mu_p)}=0$ and $\E_p\norm{\sum_i(\Ztr_i-\mu_p)}^2=n(1-\norm{\mu_p}^2)\le n$.
Subtracting the first identity and applying the Cauchy--Schwarz inequality,
\[
 \kappa\,\E_\Pi\ip{\thh}{\mu_p}=\E_\Pi\Big\langle\thh-\theta_p^*,\sum_i(\Ztr_i-\mu_p)\Big\rangle
 \le\sqrt{\E_\Pi\norm{\thh-\theta_p^*}^2\;\E_\Pi\Big\|\sum_i(\Ztr_i-\mu_p)\Big\|^2}\le\sqrt{wn},
\]
because $\E_p\norm{\thh-\theta_p^*}^2\le w\le(\eps/a_-)^{2/r}$ at every $p$ by \eqref{eq:error}.

\paragraph{Conclusion.} The two bounds give $\eps^{1-2/q}/(2c_v)\le b\sqrt n\,(\eps/a_-)^{1/r}$. As
$1/q=1-1/r$, the exponent $2-4/q-2/r$ equals $-2+2/r$, and squaring gives
\[
 n\ \ge\ \frac{a_-^{2/r}}{4b^2c_v^2}\,\eps^{-2+2/r}.
\]
The conditions $v\le1/2$ and $16Ht_0^2\eps\le\beta^2$ hold for all $\eps$ below a threshold that
depends only on the constants of Assumption~\ref{ass:class}. This proves part~(ii). Together, the two
parts show that $\neps=\eps^{-2+2/r}$ is the optimal sample size for the problems of
Theorem~\ref{thm:main}: $\eps^{-1}$ for ridge regularization ($r=2$) and $\eps^{-3/2}$ for quartic
regularization ($r=4$).

\subsection{Links}\label{app:links}
Table~\ref{tab:links} lists admissible constants of Assumption~\ref{ass:class} for four links. Each
link is convex and $C^2$ on $\R$, and its constants follow by differentiating on $[-1,1]$.

These links are losses of a margin, and the class contains classification problems through the
following construction with balanced labels. Draw $z$ from $D_p$ and an independent fair label
$y\in\{-1,1\}$, and observe the labelled point $(x,y)$ with $x=yz$. Since $y^2=1$, the loss of the
prediction $\ip\theta x$ for the label $y$ is $\phi(y\ip\theta x)=\phi(\ip\theta z)$, so the
regularized risk of this classification problem is $F_p$; with the logistic link and
$g(\theta)=\norm\theta^2/2$ it is ridge-regularized logistic regression. A learner trained on labelled
points $(x_i,y_i)$ is a learner on the points $z_i=y_ix_i$ whose seed also contains the training
labels, which are independent of the points, and the loss differences of labelled candidates are the
$X_i$ of Definition~\ref{def:release}. So the labels act as additional independent randomness of the
learner, and Theorem~\ref{thm:main} applies to learners of this classification problem. In this
problem the label is a fair sign and, given the label $y$, the features have law $D_{yp}$, with mean
$y\mu_p$; under the beta prior of the lower bound, $\E\norm{\mu_p}^2=v=c_v\eps^{2/q}$, which is of order
$1/n$ in the sample-size range $c_-\neps\le n\le c_+\neps$ of Theorem~\ref{thm:main}.

\begin{table}[h!]
\centering
\caption{Links and Their Constants. Admissible constants of Assumption~\ref{ass:class}: $b=-\phi'(0)$,
$L\ge\max_{[-1,1]}|\phi'|$ and $H\ge\max_{[-1,1]}\phi''$.}\label{tab:links}
\begin{tabular}{llccc}
\toprule
Link & $\phi(t)$ & $b$ & $L$ & $H$\\
\midrule
logistic & $\ln(1+e^{-t})$ & $1/2$ & $(1+e^{-1})^{-1}$ & $1/4$\\
exponential & $e^{-t}$ & $1$ & $e$ & $e$\\
squared & $(1-t)^2/2$ & $1$ & $2$ & $1$\\
affine & $-t$ & $1$ & $1$ & $0$\\
\bottomrule
\end{tabular}
\end{table}

\subsection{Power Regularizers and Anisotropy}\label{app:power}
Let $g(\theta)=\norm{\Lambda\theta}^r/r$, where $r\ge2$ and $\Lambda$ is a $d\times d$ matrix whose
singular values lie in a fixed range $[\nu_-,\nu_+]$ with $0<\nu_-\le\nu_+<\infty$. Then $g$ is convex
and $C^1$ on $\R^d$, $g(0)=0$, and $\nabla g(\theta)=\norm{\Lambda\theta}^{r-2}\Lambda^\top\Lambda\theta$,
so $\nabla g(0)=0$ and $\sup_W\norm{\nabla g}\le\nu_+^r$. For the curvature, write $x=\Lambda\theta$,
$y=\Lambda\theta'$ and $x_t=x+t(y-x)$. The power map $x\mapsto\norm x^{r-2}x$, the gradient of
$\norm x^r/r$, satisfies the monotonicity inequality
\[
 \ip{\norm y^{r-2}y-\norm x^{r-2}x}{y-x}\ \ge\ 2^{2-r}\norm{y-x}^r\qquad(x,y\in\R^d)
\]
\citep[Chapter~12]{lindqvist2019}. Applied to $x_t$ and $x$ for $0<t\le1$, since $x_t-x=t(y-x)$, it
gives
$\ip{\norm{x_t}^{r-2}x_t-\norm x^{r-2}x}{y-x}\ge2^{2-r}t^{r-1}\norm{y-x}^r$. The Bregman divergence of
$g$ is
\[
 D_g(\theta',\theta)=\frac{\norm y^r}r-\frac{\norm x^r}r-\ip{\norm x^{r-2}x}{y-x}
 =\int_0^1\ip{\norm{x_t}^{r-2}x_t-\norm x^{r-2}x}{y-x}\,dt ,
\]
and integrating in $t$ yields
\[
 D_g(\theta',\theta)\ \ge\ \frac{2^{2-r}}r\norm{\Lambda(\theta'-\theta)}^r
 \ \ge\ \frac{2^{2-r}\nu_-^r}r\norm{\theta'-\theta}^r,\qquad
 0\le g(\theta)\le\frac{\nu_+^r}r\norm\theta^r .
\]
So Assumption~\ref{ass:class} holds with $a_-=2^{2-r}\nu_-^r/r$, $a_+=\nu_+^r/r$ and $L_g=\nu_+^r$. For
$\Lambda=I$ this covers ridge regularization $\norm\theta^2/2$ ($r=2$, $a_-=a_+=1/2$) and quartic
regularization $\norm\theta^4/4$ ($r=4$, $a_-=1/16$, $a_+=1/4$).

Similarly, a regularizer $g$ that is $C^2$ near $W$, with $g(0)=0$, $\nabla g(0)=0$ and
$2a_-I\preceq\nabla^2g\preceq2a_+I$ on a convex neighbourhood of $W$, satisfies
Assumption~\ref{ass:class} with $r=2$ and $L_g=2a_+$. Indeed, $g$ is convex there, and Taylor's
formula with integral remainder gives
$D_g(\theta',\theta)=\int_0^1(1-t)(\theta'-\theta)^\top\nabla^2g(\theta+t(\theta'-\theta))(\theta'-\theta)\,dt$,
which lies between $a_-\norm{\theta'-\theta}^2$ and $a_+\norm{\theta'-\theta}^2$. So
$g(\theta)=D_g(\theta,0)$ lies between $0$ and $a_+\norm\theta^2$, and
$\norm{\nabla g(\theta)}=\norm{\nabla g(\theta)-\nabla g(0)}\le2a_+\norm\theta\le2a_+$ on $W$.

\subsection{Links with \texorpdfstring{$\phi'(0)>0$}{a positive slope at zero}}\label{app:reflection}
Assumption~\ref{ass:class} fixes the sign of the slope, $b=-\phi'(0)>0$. Negating the data removes
this restriction.

\begin{remark}[Reflection]\label{rem:reflection}
Let $\phi$ satisfy the link conditions of Assumption~\ref{ass:class} ($\phi$ is convex and $C^2$ on a
neighbourhood of $[-1,1]$, $|\phi'|\le L$ and $0\le\phi''\le H$ on $[-1,1]$) except that
$\phi'(0)>0$, and put $b=\phi'(0)$. Write $\phi^-(t)=\phi(-t)$, and for a proper learner $A_n$ let
$A_n^-(S_n;R)=A_n(-S_n;R)$, where $-S_n$ negates every training point. Write $X^{A,\phi}$ for the
release of Definition~\ref{def:release} of the learner $A_n$ with the link $\phi$. Then:
\begin{enumerate}
\item[(i)] $\phi^-$ satisfies the link conditions of Assumption~\ref{ass:class} with the same $b$, $L$
and $H$;
\item[(ii)] if $A_n$ is $\eps$-accurate on the family for $f=g+\phi(\ip\theta z)$, then $A_n^-$ is a
proper learner that is $\eps$-accurate on the family for $f^-=g+\phi^-(\ip\theta z)$;
\item[(iii)] for every $p\in[-1,1]^d$, every $\tau\ge0$ and each $\Qc\in\{\Ii,\Ij,\Is\}$,
\[
 \Qc_{D_p}\big(X^{A^-,\phi^-}+\tau G\big)=\Qc_{D_{-p}}\big(X^{A,\phi}+\tau G\big),
\]
so the suprema of the two sides over $p$ agree.
\end{enumerate}
Consequently, the two-sided bound \eqref{eq:main} of Theorem~\ref{thm:main}, and its form for the
infimum over any nonempty class of accurate learners, hold for $\phi$ with the same threshold, the same
conditions on $n$ and $d$ and the same constants, evaluated at $b=\phi'(0)$. So does the prior-averaged
lower bound in its finite form
$M^2\le2\ln2\,(V+n\tau^2)\,\E_\Pi\Qc_{D_p}(X^{A,\phi}+\tau G)$, with the budgets $M,V$ of
Proposition~\ref{prop:finite} (when $M>0$) and the same prior $\Pi_\kappa$. Together with the
zero-slope example of Section~\ref{sec:necessity}, where accuracy forces no information, this shows
that the slope condition of Assumption~\ref{ass:class} may be read as $b=|\phi'(0)|>0$.
\end{remark}

\begin{proof}
(i) On the symmetric neighbourhood of $[-1,1]$, $\phi^-$ is convex and $C^2$, as the composition of
$\phi$ with $t\mapsto-t$. Moreover $(\phi^-)'(t)=-\phi'(-t)$ and $(\phi^-)''(t)=\phi''(-t)$, so
$-(\phi^-)'(0)=\phi'(0)=b$, $|(\phi^-)'|\le L$ and $0\le(\phi^-)''\le H$ on $[-1,1]$.

(ii) Negation maps $\mathcal Z$ onto itself and $D_p$ to $D_{-p}$, because
$\Pp(-\sqrt dZ_k=1)=(1-p_k)/2$ and the coordinates stay independent. Write
$F^\phi_D(\theta)=\E_D[g(\theta)+\phi(\ip\theta Z)]$, and $F^{\phi^-}_D$ likewise. Since
$\phi^-(\ip\theta z)=\phi(\ip\theta{-z})$, we have $F^{\phi^-}_{D_p}=F^\phi_{D_{-p}}$ on $W$. Negating
every point of an i.i.d.\ sample from $D_p$ gives an i.i.d.\ sample from $D_{-p}$, independent of $R$.
So, for every $\theta\in W$,
\[
 \E_{S_n\sim D_p^n,R}\,F^{\phi^-}_{D_p}\big(A_n^-(S_n;R)\big)
 =\E_{S_n\sim D_{-p}^n,R}\,F^\phi_{D_{-p}}\big(A_n(S_n;R)\big)
 \le F^\phi_{D_{-p}}(\theta)+\eps=F^{\phi^-}_{D_p}(\theta)+\eps,
\]
by the accuracy of $A_n$ at $D_{-p}$, as $-p\in[-1,1]^d$. Properness and measurability are inherited
from $A_n$.

(iii) On the supersample, consider the map that negates each of the $2n$ candidates and keeps the
selection bits $U$, the seed $R$ and the noise $G$. It is a measurable involution, and it carries the
joint law of $(\Zt,U,R,G)$ under $D_p$ to that under $D_{-p}$. It negates the training set, so the
output of $A_n^-$ at an outcome is the output $\thh$ of $A_n$ at the image outcome, and
$X^{A^-,\phi^-}_i=\phi(\ip{\thh}{-Z_{1,i}})-\phi(\ip{\thh}{-Z_{0,i}})$ is the value of $X^{A,\phi}_i$ at
the image outcome, whose candidates are the $-Z_{j,i}$. Hence $(X^{A^-,\phi^-}+\tau G,U)$ under $D_p$
has the law of $(X^{A,\phi}+\tau G,U)$ under $D_{-p}$, which proves (iii) for $\Ii$ and $\Ij$; and
$(X^{A^-,\phi^-}+\tau G,U,\Zt)$ under $D_p$ has the law of $(X^{A,\phi}+\tau G,U,-\Zt)$ under $D_{-p}$.
Mutual information does not change when one argument is replaced by its image under a measurable
bijection with measurable inverse, and $I(Y;U\mid\Zt)=I(U;(Y,\Zt))-I(U;\Zt)$, so (iii) holds for
$\Is$ as well. Since $p\mapsto-p$ is a bijection of $[-1,1]^d$, the suprema agree.

It remains to prove the consequence. The prior $\Pi_\kappa$ has independent coordinates with the even
density $\pi_\kappa$, so it is invariant under $p\mapsto-p$, and averaging (iii) over $\Pi_\kappa$
shows that the prior averages for $(A^-,\phi^-)$ and $(A,\phi)$ coincide. Apply Theorem~\ref{thm:main} and the
finite prior-averaged bound to the link $\phi^-$, which is admissible by (i), and to the learner
$A_n^-$, which is $\eps$-accurate on the family by (ii); for the infimum over a nonempty class of
learners that are accurate for $\phi$, apply them to the class of the reflected learners $A_n^-$.
\end{proof}

\subsection{Regularization Strength in a Compact Range}\label{app:strength}
\begin{remark}[Regularization strength]\label{rem:strength}
Fix $0<\alpha_-\le\alpha_+<\infty$ and a regularizer $g$ that satisfies Assumption~\ref{ass:class}
with constants $a_-$, $a_+$ and $L_g$. For every $\alpha\in[\alpha_-,\alpha_+]$, the regularizer
$\alpha g$ satisfies Assumption~\ref{ass:class} with the constants $\alpha_-a_-$, $\alpha_+a_+$ and
$\alpha_+L_g$, which do not depend on $\alpha$. Consequently, for the losses
$f_\alpha(\theta,z)=\alpha g(\theta)+\phi(\ip\theta z)$, Theorem~\ref{thm:main} holds with one
threshold $\eps_0$, one dimension multiplier $\Cdim$ and one pair of constants $c_1,C_1$ for the whole
range: for every $\alpha\in[\alpha_-,\alpha_+]$, every $0<\eps\le\eps_0$, all $n$ and $d$ allowed by
Theorem~\ref{thm:main} and every proper learner that is $\eps$-accurate on the family for $f_\alpha$,
the bound \eqref{eq:main} holds with these constants.
\end{remark}

\begin{proof}
For $\alpha>0$, $\alpha g$ is convex and $C^1$ near $W$, $\alpha g(0)=0$ and
$\nabla(\alpha g)=\alpha\nabla g$, so $\nabla(\alpha g)(0)=0$ and
$\sup_W\norm{\nabla(\alpha g)}\le\alpha_+L_g$. The Bregman divergence scales,
$D_{\alpha g}=\alpha D_g\ge\alpha_-a_-\norm{\theta'-\theta}^r$, and
$0\le\alpha g(\theta)\le\alpha_+a_+\norm\theta^r$ on $W$. So $\alpha g$ satisfies \eqref{eq:matched}
with the constants $\alpha_-a_-$ and $\alpha_+a_+$; decreasing the lower constant and increasing the
upper one only weakens \eqref{eq:matched}. The threshold, the dimension multiplier and the constants
of Theorem~\ref{thm:main} depend only on the constants of Assumption~\ref{ass:class} and on $c_\pm$,
so applying Theorem~\ref{thm:main} with the constants $\alpha_-a_-$, $\alpha_+a_+$ and $\alpha_+L_g$
gives the claim.
\end{proof}

\subsection{Without a Regularizer}\label{app:noreg}

Assumption~\ref{ass:class} requires $a_->0$, which excludes $g=0$. Without a regularizer, the loss
$f(\theta,z)=\phi(\ip\theta z)$ is convex and Lipschitz but need not be uniformly convex, accuracy no
longer localizes the learner's output, and the optimal sample size is of order $\eps^{-2}$ (see the end
of this subsection).
The two-sided bound still holds at that sample size. The signal forced by accuracy is then of order
$\eps^{-1}$, so properness alone, in place of the localization of Appendix~\ref{app:localization},
keeps the energy and the two penalty terms of the signed mean of order one.

\begin{proposition}[Without a regularizer]\label{prop:noreg}
Let $g=0$, so that $f(\theta,z)=\phi(\ip\theta z)$, where $\phi$ satisfies the link conditions of
Assumption~\ref{ass:class}, and let $c_+>0$. Put
\[
 c=\frac{b^4}{2048\ln2\,\max\{3L^2,c_+\}} .
\]
If $0<\eps\le\min\{b/4,\,b^3/(64(H+1)(b+\sqrt{c_+}))\}$, $n\le c_+\eps^{-2}$ and
$d\ge32\max\{n,b^2\eps^{-2}\}$, then every proper learner that is $\eps$-accurate on the family
satisfies, for each $\Qc\in\{\Ii,\Ij,\Is\}$ and every $\tau\ge0$,
\[
 \frac c{\eps^2+\tau^2}\ \le\ \sup_p\Qc_{D_p}(X+\tau G)\ \le\ \frac{c_++3L^2/(2\ln2)}{\eps^2+\tau^2},
\]
where $p$ ranges over $[-1,1]^d$.
\end{proposition}

\begin{proof}
Let $\thh$ be a proper learner that is $\eps$-accurate on the family. Here
$F_p(\theta)=\E_p\phi(\ip\theta Z)$, which is continuous on the compact set $W$, so its minimum over
$W$ exists. In this proof,
\[
 v=\frac{4\eps^2}{b^2},\qquad \kappa=v^{-1}-1=\frac{b^2}{4\eps^2}-1,
\]
and $\Pi=\Pi_\kappa$ is the beta prior of Appendix~\ref{app:prior-moments} with this $\kappa$, so that
$\E\norm{\mu_p}^2=v$. As $\eps\le b/4$, we have $v\le1/4$, so
$\kappa=(1-v)/v\ge3/(4v)\ge b^2/(8\eps^2)>0$. The conditions on $d$ give $dv=4d\eps^2/b^2\ge128$,
$1/d\le\eps^2/(32b^2)$ and $\rho=n/d\le1/32$; so the constant $K_0=K_0(\rho)$ of
Appendix~\ref{app:tails} satisfies $K_0\le K_0^*<2.031$. Finally,
$\eps\le b^3/(64(H+1)(b+\sqrt{c_+}))\le b^2/(64(H+1))$.

\emph{Calibration.} Fix $p$ and write $\mu=\mu_p$. Let $\theta_c=\mu/\norm\mu$ if $\mu\ne0$ and
$\theta_c=0$ if $\mu=0$, so that $\theta_c\in W$ and $\ip{\theta_c}\mu=\norm\mu$. The link remainder
$u(t)=\phi(t)-\phi(0)+bt$ satisfies $0\le u(t)\le Ht^2/2$ on $[-1,1]$ (see the opening of
Appendix~\ref{app:mainproof}). Together with \eqref{eq:margin-variance} and $\norm{\theta_c}\le1$, this
gives
\[
 F_p(\theta_c)=\phi(0)-b\norm\mu+\E_p\,u(\ip{\theta_c}Z)
 \le\phi(0)-b\norm\mu+\frac H2\Big(\norm\mu^2+\frac1d\Big),
\]
while $u\ge0$ gives $F_p(\theta)\ge\phi(0)-b\ip\theta\mu$ for every $\theta\in W$. Accuracy at $D_p$
therefore gives $\phi(0)-b\,\E_p\ip\thh\mu\le\E_pF_p(\thh)\le F_p(\theta_c)+\eps$, that is,
\[
 b\,\E_p\ip\thh\mu\ge b\norm\mu-\eps-\frac H2\Big(\norm\mu^2+\frac1d\Big).
\]
Unlike in Appendix~\ref{app:calibration}, the comparator is the unit vector along $\mu_p$, and no
growth condition is used. Now average over $p\sim\Pi$. By \eqref{eq:prior-low} with $q=1$,
$\E\norm{\mu_p}\ge\sqrt v\,(1+1/(dv))^{-1/2}$, and $\E\norm{\mu_p}^2=v$ by \eqref{eq:prior-moments}, so
\[
 b\,\E_\Pi\ip\thh{\mu_p}\ \ge\ \widetilde A:=b\sqrt v\Big(1+\frac1{dv}\Big)^{-1/2}-\eps
 -\frac H2\Big(v+\frac1d\Big).
\]
Here $b\sqrt v=2\eps$ and $(1+1/(dv))^{1/2}\le(129/128)^{1/2}\le8/7$, so the first term is at least
$7\eps/4$; and $v+1/d\le5\eps^2/b^2$ and $\eps\le b^2/(64(H+1))$ give
$(H/2)(v+1/d)\le5H\eps^2/(2b^2)\le\eps/4$. Hence $\widetilde A\ge\eps/2$.

\emph{Signed mean.} The derivation of \eqref{eq:taylor-main} in Appendix~\ref{app:signed} uses only
the link conditions, properness and Lemma~\ref{lem:fp}, so it applies here: with $m=\ip\thh{\mu_p}$ and
$s_i=\ip\thh{\Ztr_i-\mu_p}$,
\[
 \sum_i\E_\Pi\sigma_iX_i\ge b\kappa\,\E_\Pi m-\E_\Pi\Big[u'(m)\sum_is_i\Big]
 -\frac H2\,\E_\Pi\norm{\Btr\thh}^2 .
\]
The first term is at least $\kappa\widetilde A$. Without localization, properness bounds the other two.
Pointwise, $|u'(m)|\le H|m|\le H\norm{\mu_p}$ and
$|\sum_is_i|=|\ip\thh{\sum_i(\Ztr_i-\mu_p)}|\le\norm{\sum_i(\Ztr_i-\mu_p)}$, and given $p$ the centred
training sum has second moment $n(1-\norm{\mu_p}^2)\le n$ (Appendix~\ref{app:sample-lower}). The
Cauchy--Schwarz inequality therefore gives
\[
 \E_\Pi\Big|u'(m)\sum_is_i\Big|
 \le H\Big(\E\norm{\mu_p}^2\;\E_\Pi\Big\|\sum_i(\Ztr_i-\mu_p)\Big\|^2\Big)^{1/2}\le H\sqrt{nv}.
\]
Also $\norm{\Btr\thh}^2\le\opnorm\Btr^2$, and $\E_p\opnorm\Btr^2\le K_0$ at every $p$
(Appendix~\ref{app:tails}). Therefore
\[
 \sum_i\E_\Pi\sigma_iX_i\ \ge\ M:=\kappa\widetilde A-H\sqrt{nv}-\frac H2K_0 .
\]
With $\kappa\ge b^2/(8\eps^2)$, $\widetilde A\ge\eps/2$, $\sqrt{nv}=2\sqrt{n\eps^2}/b\le2\sqrt{c_+}/b$ and
$K_0<4$,
\[
 M\ \ge\ \frac{b^2}{16\eps}-H\Big(\frac{2\sqrt{c_+}}b+2\Big)
 \ \ge\ \frac{b^2}{16\eps}-\frac{b^2}{32\eps}\ =\ \frac{b^2}{32\eps},
\]
because $\eps\le b^3/(64(H+1)(b+\sqrt{c_+}))$ is equivalent to $(H+1)(2\sqrt{c_+}/b+2)\le b^2/(32\eps)$.

\emph{Energy.} The derivation of \eqref{eq:energy} in Appendix~\ref{app:energy} uses only the link
conditions and properness. With $\norm\thh\le1$ and $\E_p\opnorm\Btr^2\le K_0$, it gives, at every $p$,
\[
 \E_p\norm X^2=\sum_i\E_pX_i^2\ \le\ V:=L^2(K_0+\rho)\ \le\ 3L^2,
\]
since $K_0+\rho<2.031+1/32<3$.

\emph{The two bounds.} The prior $\Pi$ depends only on $b$ and $\eps$, so it is fixed before the
learner, and $p$ is drawn independently of the selection bits, the seed and the noise. The energy
bound at every $p$ gives $\sum_i\E_\Pi X_i^2\le V$, and $M>0$. So the lower bound \eqref{eq:lower} of
Theorem~\ref{thm:transfer} gives, for every $\tau\ge0$,
\[
 \E_\Pi\Ii_{D_p}(X+\tau G)\ \ge\ \frac{M^2}{2\ln2\,(V+n\tau^2)}
 \ \ge\ \frac{b^4/(1024\eps^2)}{2\ln2\,(3L^2+c_+\tau^2/\eps^2)}
 \ =\ \frac{b^4}{2048\ln2\,(3L^2\eps^2+c_+\tau^2)}\ \ge\ \frac c{\eps^2+\tau^2},
\]
by $n\le c_+\eps^{-2}$ and $3L^2\eps^2+c_+\tau^2\le\max\{3L^2,c_+\}(\eps^2+\tau^2)$. By
\eqref{eq:hierarchy} the same bound holds for $\Ij$ and $\Is$, and a prior average is at most the
supremum over $p$.

For the upper bound, fix $p$ and apply Theorem~\ref{thm:transfer}(b) with the reference $x_p(\Zt)=0$
and $\Vres=V$. It gives $\Is_{D_p}(X+\tau G)\le n$ for every $\tau\ge0$ and, for $\tau>0$, by
$\ln(1+y)\le y$,
\[
 \Is_{D_p}(X+\tau G)\le\frac n{2\ln2}\ln\Big(1+\frac V{n\tau^2}\Big)\le\frac V{2\ln2\,\tau^2}
 \le\frac{3L^2}{2\ln2\,\tau^2}.
\]
Hence $(\eps^2+\tau^2)\,\Is_{D_p}(X+\tau G)\le n\eps^2+3L^2/(2\ln2)\le c_++3L^2/(2\ln2)$, where the
second term is not needed at $\tau=0$. By \eqref{eq:hierarchy} the same bound holds for $\Ii$ and
$\Ij$, at every $p$, and therefore for the supremum over $p$.
\end{proof}

The upper bound needs neither accuracy nor the threshold on $\eps$: it holds at every $p$ for every
proper learner, when $d\ge32n$ and $n\le c_+\eps^{-2}$.

\paragraph{The Sample Size.}
Each of the three information measures is at most $n$ bits, so at $\tau=0$ the lower bound forces
$n\ge c\,\eps^{-2}$. Hence $c\,\eps^{-2}\le n\le c_+\eps^{-2}$, and the bounds of
Proposition~\ref{prop:noreg} imply
\[
 \frac c{c_+}\cdot\frac n{1+(\tau/\eps)^2}\ \le\ \sup_p\Qc_{D_p}(X+\tau G)
 \ \le\ \frac{c_++3L^2/(2\ln2)}c\cdot\frac n{1+(\tau/\eps)^2},
\]
which has the form of \eqref{eq:main}, at sample sizes of order $\eps^{-2}$ instead of $\neps$.

Without a regularizer, $\eps^{-2}$ is the optimal sample size. First, let
$0<\eps\le\min\{b/4,\,b^2/(66(H+1))\}$ and $d\ge32b^2\eps^{-2}$, and suppose that a proper learner that
is $\eps$-accurate on the family uses $n<b^2/(1024\eps^2)$ samples. Then the hypotheses of
Proposition~\ref{prop:noreg} hold with $c_+=b^2/1024$, for which
$b^3/(64(H+1)(b+\sqrt{c_+}))=b^2/(66(H+1))$, and the proof above gives
$b\kappa\,\E_\Pi\ip\thh{\mu_p}\ge\kappa\widetilde A\ge M\ge b^2/(32\eps)$. By
Lemma~\ref{lem:fp} and properness, as in Appendix~\ref{app:sample-lower},
$\kappa\,\E_\Pi\ip\thh{\mu_p}=\E_\Pi\ip\thh{\sum_i(\Ztr_i-\mu_p)}\le\E_\Pi\norm{\sum_i(\Ztr_i-\mu_p)}
\le\sqrt n$. Hence $n\ge b^2/(1024\eps^2)$, a contradiction; so every such learner uses at least
$b^2/(1024\eps^2)$ samples. Second, one epoch of projected SGD from $\theta_1=0$ with $n$ samples has
expected excess risk at most $L/\sqrt n$ at every law on the unit ball. Indeed, the derivation of
\eqref{eq:sgd-epoch} in Appendix~\ref{app:sgd} uses only that $F_D$ is convex and
differentiable on $W$, with a minimizer $\theta^*$ there, and that the stochastic gradients are
bounded, here by $L$; with $\norm{\theta^*}\le1$, $n'=n$ and $\eta=1/(L\sqrt n)$, its right side is at
most $L/(2\sqrt n)+L/(2\sqrt n)$. So $n=\lceil L^2\eps^{-2}\rceil$ samples suffice for $\eps$-accuracy,
and then $n\eps^2\le L^2+\eps^2\le2L^2$, because $\eps\le b/4\le L/4$. With $c_+=2L^2$,
Proposition~\ref{prop:noreg} therefore applies to this learner whenever $\eps$ and $d$ satisfy its
conditions.

For the affine link, the instance of \citet[Section~6.1.1]{attias2024}, the lower bound at $\tau=0$
has the order $\eps^{-2}$ of their Theorem~4.1. That theorem bounds $I(\thh;U\mid\Zt)$ from below for
every learner whose excess risk, at every law on the scaled sign cube, is at most $\eps$ with
probability at least $1-\delta'$, with $\delta'\le\eps$, at every sample size at least its sample
complexity and in dimension $d\ge\Omega(n^2\log n)$. Proposition~\ref{prop:noreg} bounds the smaller
information of the loss differences, for every link that satisfies the link conditions of
Assumption~\ref{ass:class} and every noise level. It requires only accuracy in expectation on the
family and dimension linear in $n$, but it covers only proper learners and sample sizes of order
$\eps^{-2}$.

%% file: app_F_consequences.tex
\section{Proofs for \texorpdfstring{Section~\ref{sec:consequences}}{the consequences}}\label{app:consequences}

This appendix proves the results of Section~\ref{sec:consequences}: the bound $4L/\sqrt n$ on the
generalization gap (Appendix~\ref{app:uniform-gap}), Corollary~\ref{cor:consequences}
(Appendices~\ref{app:range} and~\ref{app:dp}), the second-moment bound \eqref{eq:second-moment}
(Appendix~\ref{app:second-moment}) and the bound \eqref{eq:noisy-bound} for noise at the scale of the
range (Appendix~\ref{app:noisy-bound}). For reference, we restate Corollary~\ref{cor:consequences}.

\begin{restate}[Loss-difference CMI bounds and privacy]{Corollary}{cor:consequences}
Let $\thh$ be a learner of Theorem~\ref{thm:main}.
\begin{enumerate}
\item[(i)] Some law $D_p$ of the family has $\sum_iI_{D_p}(X_i;U_i)\ge c_1n/2$, and hence
$n^{-1}\sum_i\sqrt{I_{D_p}(X_i;U_i\mid\Zt)}\ge c_1/2$. In contrast, every proper learner has
$|\gen_D|\le4L/\sqrt n$ at every law $D$ on the unit ball.
\item[(ii)] If $\thh$ is also $(\epsdp,\delta)$-differentially private, then
$e^{\epsdp}-1\ge c'-2L\delta/\eps$, for a constant $c'>0$ that depends only on the constants of
Assumption~\ref{ass:class} and on $c_\pm$. So $\epsdp=\Omega(1)$ unless $\delta\gtrsim\eps$.
\end{enumerate}
\end{restate}

Here a learner of Theorem~\ref{thm:main} is a proper learner that is $\eps$-accurate on the family,
with $\eps$, $n$ and $d$ in the range of Theorem~\ref{thm:main}, and a law on the unit ball is a law
of $Z$ with $\norm Z\le1$ almost surely.

\subsection{The Bound \texorpdfstring{$4L/\sqrt n$}{4L/sqrt(n)}}\label{app:uniform-gap}
Let $D$ be any law of $Z$ with $\norm Z\le1$ almost surely, and let $\thh$ be any proper learner; no
accuracy is assumed. On the supersample under $D$, the training points $\Ztr_i$ and the ghosts
$\Zgh_i$ are $2n$ independent draws from $D$, and $\widehat F_{S_n}(\theta)=n^{-1}\sum_if(\theta,\Ztr_i)$
is the training risk. We show that
\begin{equation}\label{eq:uniform-gap}
 |\gen_D|\ \le\ \E\sup_{\theta\in W}\big|F_D(\theta)-\widehat F_{S_n}(\theta)\big|\ \le\ \frac{4L}{\sqrt n}.
\end{equation}
The first inequality holds because $\thh\in W$. (Every process below is continuous in $\theta$, so its
supremum may be taken over a countable dense subset of $W$.) The regularizer and the constant
$\phi(0)$ cancel from $F_D(\theta)-\widehat F_{S_n}(\theta)$, and $F_D(\theta)$ is the conditional mean of
the ghost risk $n^{-1}\sum_if(\theta,\Zgh_i)$ given the training points. Let $\zeta_1,\dots,\zeta_n$ be
independent Rademacher signs, independent of everything else. Jensen's inequality moves the
conditional mean outside the supremum. Exchanging $\Ztr_i$ and $\Zgh_i$ in the pairs with $\zeta_i=-1$
does not change the joint law, and after adding and subtracting $\phi(0)$ the training and the ghost
terms have the same law. Hence
\begin{align*}
 \E\sup_{\theta\in W}\big|F_D(\theta)-\widehat F_{S_n}(\theta)\big|
 &\le\frac1n\,\E\sup_{\theta\in W}\Big|\sum_i\zeta_i\big[\phi(\ip\theta{\Zgh_i})-\phi(\ip\theta{\Ztr_i})\big]\Big|\\
 &\le\frac2n\,\E\sup_{\theta\in W}\Big|\sum_i\zeta_i\big[\phi(\ip\theta{\Ztr_i})-\phi(0)\big]\Big| .
\end{align*}
The margins $\ip\theta{\Ztr_i}$ lie in $[-1,1]$. The map $t\mapsto(\phi(t)-\phi(0))/L$ is $1$-Lipschitz
on $[-1,1]$ and vanishes at $0$; extended to $\R$ by constants outside $[-1,1]$, it is a contraction of
$\R$ that vanishes at $0$. Given the training points, the contraction principle
\citep[Theorem~4.12]{ledoux1991}, applied to the bounded set
$\{(\ip\theta{\Ztr_i})_{i\le n}:\theta\in W\}\subset\R^n$, gives
\begin{align*}
 \E_\zeta\sup_{\theta\in W}\Big|\sum_i\zeta_i\big[\phi(\ip\theta{\Ztr_i})-\phi(0)\big]\Big|
 &\le2L\,\E_\zeta\sup_{\theta\in W}\Big|\Big\langle\theta,\sum_i\zeta_i\Ztr_i\Big\rangle\Big|
 =2L\,\E_\zeta\Big\|\sum_i\zeta_i\Ztr_i\Big\|\\
 &\le2L\Big(\sum_i\norm{\Ztr_i}^2\Big)^{1/2}\le2L\sqrt n,
\end{align*}
where $\E_\zeta$ averages over the signs only. The two displays prove \eqref{eq:uniform-gap}.

\subsection{Proof of Corollary~\ref{cor:consequences}(i)}\label{app:range}
Let $\thh$ be as in Corollary~\ref{cor:consequences}. At $\tau=0$ the individual information is
$\Ii_{D_p}(X)=\sum_iI_{D_p}(X_i;U_i)$, and the lower half of \eqref{eq:main} with $\Qc=\Ii$ gives
$\sup_p\sum_iI_{D_p}(X_i;U_i)\ge c_1n>c_1n/2$. So some $p\in[-1,1]^d$ has
$\sum_iI_{D_p}(X_i;U_i)\ge c_1n/2$; this witness law is taken from the supremum and may depend on the
learner. Fix it. Since $U_i$ is independent of $\Zt$, the chain rule gives
\[
 I(X_i;U_i\mid\Zt)=I(U_i;(X_i,\Zt))-I(U_i;\Zt)=I(U_i;(X_i,\Zt))\ \ge\ I(X_i;U_i).
\]
Each $I(X_i;U_i)$ is at most $\Hent(U_i)=1$ bit, so $\sqrt{I(X_i;U_i)}\ge I(X_i;U_i)$, and
\[
 \frac1n\sum_i\sqrt{I_{D_p}(X_i;U_i\mid\Zt)}\ \ge\ \frac1n\sum_i\sqrt{I_{D_p}(X_i;U_i)}
 \ \ge\ \frac1n\sum_iI_{D_p}(X_i;U_i)\ \ge\ \frac{c_1}2 .
\]
The same lower bound holds when the square root is taken before averaging over the candidates: given
$\Zt=\tilde z$, each term $I(X_i;U_i\mid\Zt=\tilde z)$ is at most one bit, so the average over $\Zt$
of its square root is at least $I(X_i;U_i\mid\Zt)$.

For a loss with values in $[0,1]$, Theorem~3.1 of \citet{wangmao2023a} bounds the absolute expected
generalization gap by the average over pairs of $\sqrt{2I(X_i;U_i\mid\Zt)}$, where $X_i$ is the loss
difference of pair~$i$, and also by the same average with the square root taken before averaging over
the candidates; their Theorem~3.2 bounds it by the average of $\sqrt{2I(X_i;U_i)}$. The loss $f$ is bounded on $W\times\mathcal Z$, because $g$ is bounded on $W$
and $\phi$ on $[-1,1]$. To apply the theorem, shift $f$ and divide it by the length of an interval that
contains its values, a fixed positive constant of the class; this multiplies the gap and every loss
difference by the same constant and leaves every information unchanged. Measured in bits, as here, the
resulting range-scaled bounds at $D_p$ are therefore fixed positive multiples of the averages bounded
above, hence at least a positive constant, whereas
$|\gen_{D_p}|\le4L/\sqrt n$ by \eqref{eq:uniform-gap}.

\subsection{The Second-Moment Bound}\label{app:second-moment}
Fix a law $D_p$ and any proper learner, and let $V_p=\sum_i\E_pX_i^2$. For each pair, the fair-bit
inequality \eqref{eq:fairbit}, applied to $Y=X_i$ and the fair bit $U_i$, gives
$(\E_p\sigma_iX_i)^2\le2\ln2\,\E_p[X_i^2]\,I_{D_p}(X_i;U_i)$. By \eqref{eq:gap} and the
Cauchy--Schwarz inequality for sums,
\[
 |\gen_p|\le\frac1n\sum_i|\E_p\sigma_iX_i|
 \le\frac1n\sum_i\sqrt{2\ln2\,\E_p[X_i^2]\,I_{D_p}(X_i;U_i)}
 \le\frac1n\sqrt{2\ln2\,V_p\,\Ii_{D_p}(X)} ,\tag{\ref{eq:second-moment}}
\]
which is \eqref{eq:second-moment}. The argument uses no property of $D_p$ beyond the fair selection
bits, so the bound holds at every law. Exchanging the two candidates of pair $i$ and flipping $U_i$
preserves the joint law of the candidates and the bits, leaves the training set unchanged and negates
$X_i$; so $X_i$ is symmetric and $\E_pX_i^2$ is its variance. The bound thus replaces the loss range
by the root-mean-square loss difference. \citet[Section~4.3]{wangmao2023a} give related variance-based
bounds for the zero--one loss.

\paragraph{Its Order on Prior Average.} In the regime of Theorem~\ref{thm:main},
Proposition~\ref{prop:finite} gives $\E_\Pi\,n\,\gen_p=\sum_i\E_\Pi\sigma_iX_i\ge c_Mn\eps$ and
$\E_\Pi V_p\le C_Vn\eps^2$. By the Cauchy--Schwarz inequality,
$n\,\gen_p\le\sum_i(\E_pX_i^2)^{1/2}\le\sqrt{nV_p}$, and $\Ii_{D_p}(X)\le n$ bits. Hence
\begin{align*}
 c_M\eps\ \le\ \E_\Pi\,\gen_p\ &\le\ \E_\Pi\sqrt{V_p/n}\ \le\ \sqrt{\E_\Pi V_p/n}\ \le\ \sqrt{C_V}\,\eps,\\
 c_M\eps\ \le\ \E_\Pi\,\gen_p\ &\le\ \E_\Pi\Big[\frac1n\sqrt{2\ln2\,V_p\,\Ii_{D_p}(X)}\Big]
 \ \le\ \sqrt{2\ln2}\;\E_\Pi\sqrt{V_p/n}\ \le\ \sqrt{2\ln2\,C_V}\,\eps .
\end{align*}
So on prior average the gap, the second-moment bound \eqref{eq:second-moment} and the
information-free bound $\sqrt{V_p/n}$ all lie between $c_M\eps$ and $\sqrt{2\ln2\,C_V}\,\eps$. At the
optimal sample size, where the information is of order $n$ bits (Theorem~\ref{thm:main}), the
second-moment bound thus equals the information-free bound up to constant factors on prior average. The loss differences
have range of order one but root-mean-square size of order $\eps$, so in Corollary~\ref{cor:consequences}(i)
it is the range normalization that fails, not the information.

\subsection{Noise at the Scale of the Range}\label{app:noisy-bound}
Fix a law $D_p$ and a proper learner, let $\tau>0$, $Y=X+\tau G$ and $s^2=4L^2+\tau^2$. Every margin
$\ip\thh z$ lies in $[-1,1]$, where $|\phi'|\le L$, so $|X_i|\le L|\ip\thh{Z_{1,i}-Z_{0,i}}|\le2L$.
Let $\bar\sigma_i$ be a uniform sign independent of $Y_i$. Conditioning on $X_i$ and using
$\E e^{\lambda\tau G_i}=e^{\lambda^2\tau^2/2}$ and $\cosh x\le e^{x^2/2}$, for every $\lambda\in\R$,
\[
 \E e^{\lambda\bar\sigma_iY_i}=\E\cosh(\lambda Y_i)=\E[\cosh(\lambda X_i)]\,e^{\lambda^2\tau^2/2}
 \le e^{\lambda^2s^2/2},
\]
and $\E\bar\sigma_iY_i=0$. So $\bar\sigma_iY_i$ is sub-Gaussian with variance proxy $s^2$ under the
product of the marginals of $Y_i$ and $U_i$, and the Donsker--Varadhan formula, as in
\citet{russo2016} and \citet{xu2017}, gives $|\E_p\sigma_iY_i|\le\sqrt{2s^2I_{D_p}(Y_i;U_i)}$ with the
information in nats, that is, $\sqrt{2\ln2\,s^2I_{D_p}(Y_i;U_i)}$ in bits. As $G$ is centred and
independent of $(X,U)$, $\E_p\sigma_iY_i=\E_p\sigma_iX_i$, so \eqref{eq:gap}, the Cauchy--Schwarz
inequality for sums and $\Ii\le\Is$ \eqref{eq:hierarchy} give
\[
 |\gen_p|\le\frac1n\sum_i\sqrt{2\ln2\,s^2I_{D_p}(Y_i;U_i)}
 \le\Big(\frac{2\ln2\,s^2}n\,\Ii_{D_p}(Y)\Big)^{1/2}\le\Big(\frac{2\ln2\,s^2}n\,\Is_{D_p}(Y)\Big)^{1/2},
\]
which is \eqref{eq:noisy-bound}. If $\E_p\norm{X-x_p(\Zt)}^2\le V_{\mathrm{res},p}$ for a reference
$x_p(\Zt)$ as in Theorem~\ref{thm:transfer}(b), then \eqref{eq:upper} and $\ln(1+x)\le x$ give
$\Is_{D_p}(Y)\le V_{\mathrm{res},p}/(2\ln2\,\tau^2)$, and hence
\[
 |\gen_p|\le\Big(\frac{s^2}{\tau^2}\cdot\frac{V_{\mathrm{res},p}}n\Big)^{1/2}
 =\Big(\frac{2V_{\mathrm{res},p}}n\Big)^{1/2}\qquad\text{at }\tau=2L.
\]
For a learner of Theorem~\ref{thm:main}, Proposition~\ref{prop:finite} gives
$V_{\mathrm{res},p}\le C_{\mathrm{res}}n\eps^2$ at every $p$ for the reference $X_p^*(\Zt)$, so the bound
is at most $(2C_{\mathrm{res}})^{1/2}\eps$, while the gap is at least $c_M\eps$ on prior average. For
ridge ERM, take the reference $\E_p[X\mid\Zt]$ of Lemma~\ref{lem:ridge-moments}, and let $M$ and $\Vres$
be the prior averages of $n\gen_p$ and of $V_{\mathrm{res},p}$ (Appendix~\ref{app:estimators}), so that
$M/n$ is the prior-averaged gap. Jensen's inequality over the prior gives
$\E_\Pi|\gen_p|\le(2\Vres/n)^{1/2}=(2n\Vres/M^2)^{1/2}\,M/n$. By Lemma~\ref{lem:ridge-moments},
$n\Vres/M^2\le1.017$ in the runs at $d=64n$ and $n\Vres/M^2\le1.032$ in the runs at $d=32n$, so the factor
is at most $1.43$ and $1.44$, respectively; it grows as $d/n$ decreases, through the term
$d^{-1}(n-1)(1-\norm{\mu_p}^2-\Gamma_p)$ of Lemma~\ref{lem:ridge-moments}(c). As $\tau$ grows, the bound
$(s^2V_{\mathrm{res},p}/(\tau^2n))^{1/2}$ decreases to $(V_{\mathrm{res},p}/n)^{1/2}$, which the
Cauchy--Schwarz inequality gives without any information, since $\E_p\sigma_i x_{p,i}(\Zt)=0$; at
$\tau=2L$ the information-based bound is $\sqrt2$ times this limit.

\subsection{Differential Privacy, Proof of Corollary~\ref{cor:consequences}(ii)}\label{app:dp}
This subsection derives Corollary~\ref{cor:consequences}(ii) from the two moment bounds behind the
lower half of Theorem~\ref{thm:main}. The witness is the audit statistic $\sum_i\sigma_iX_i$, which
needs only the released loss differences and the selection bits, never the model. Differential privacy
caps the mean of each term $\sigma_iX_i$ at every law, hence also on average over any prior, while
accuracy in the regime of Theorem~\ref{thm:main} forces the prior-averaged mean of the statistic to be
at least the signed budget $M$ of Proposition~\ref{prop:finite}, which is of order $n\eps$.

\paragraph{Definitions.} Recall from Section~\ref{sec:consequences} that a learner
$\thh=A_n(S_n;R)$ is $(\epsdp,\delta)$-differentially private if for all training sets
$S,S'\in\mathcal Z^n$ that differ in one point and every measurable $E\subseteq W$,
\[
 \Pp\big(A_n(S;R)\in E\big)\le e^{\epsdp}\,\Pp\big(A_n(S';R)\in E\big)+\delta ;
\]
this is privacy for the replacement of one training point. Write
$K_\phi=\sup_{s,t\in[-1,1]}|\phi(s)-\phi(t)|$ for the oscillation of the link on the margin interval.
Under Assumption~\ref{ass:class}, the mean value theorem gives $K_\phi\le2L$. Since $\norm{\thh}\le1$
and $\norm Z\le1$, every loss difference satisfies $|X_i|\le K_\phi$.

\begin{lemma}[Two-point forms of differential privacy]\label{lem:dp-forms}
Let $\theta$ and $\theta'$ be random elements of $W$ with
$\Pp(\theta\in E)\le e^{\epsdp}\,\Pp(\theta'\in E)+\delta$ for every measurable $E\subseteq W$, where
$\epsdp\ge0$, and let $c>0$.
\begin{enumerate}
\item[(a)] Every measurable $\Psi:W\to[0,c]$ satisfies $\E\Psi(\theta)\le e^{\epsdp}\,\E\Psi(\theta')+\delta c$.
\item[(b)] If the hypothesis also holds with $\theta$ and $\theta'$ exchanged, every measurable
$\Xi:W\to[-c,c]$ satisfies
$\E\Xi(\theta)-\E\Xi(\theta')\le(e^{\epsdp}-1)\big(\E|\Xi(\theta)|+\E|\Xi(\theta')|\big)+2\delta c$.
\item[(c)] Under the same two-sided hypothesis, every measurable test $\psi:W\to[0,1]$ satisfies
$\E\psi(\theta)-\E\psi(\theta')\le(e^{\epsdp}-1+2\delta)/(e^{\epsdp}+1)$.
\end{enumerate}
\end{lemma}

\begin{proof}
(a) By the layer-cake formula and the hypothesis applied to the measurable sets
$E_t=\{\Psi\ge t\}$,
\[
 \E\Psi(\theta)=\int_0^c\Pp\big(\Psi(\theta)\ge t\big)\,dt
 \le\int_0^c\Big(e^{\epsdp}\,\Pp\big(\Psi(\theta')\ge t\big)+\delta\Big)\,dt
 =e^{\epsdp}\,\E\Psi(\theta')+\delta c .
\]
(b) Write $\Xi=\Xi^+-\Xi^-$ with $0\le\Xi^\pm\le|\Xi|\le c$. Part~(a) for $\Xi^+$, from $\theta$ to
$\theta'$, and for $\Xi^-$, from $\theta'$ to $\theta$, gives
$\E\Xi^+(\theta)-\E\Xi^+(\theta')\le(e^{\epsdp}-1)\,\E\Xi^+(\theta')+\delta c$ and
$\E\Xi^-(\theta')-\E\Xi^-(\theta)\le(e^{\epsdp}-1)\,\E\Xi^-(\theta)+\delta c$. Add the two, and use
$\Xi^+(\theta')\le|\Xi(\theta')|$, $\Xi^-(\theta)\le|\Xi(\theta)|$ and $e^{\epsdp}-1\ge0$.
(c) With $x=\E\psi(\theta)$ and $y=\E\psi(\theta')$, part~(a) for $\psi$, and in the reverse direction
for $1-\psi$, gives $x\le e^{\epsdp}y+\delta$ and $1-y\le e^{\epsdp}(1-x)+\delta$. Adding them gives
$(1+e^{\epsdp})(x-y)\le e^{\epsdp}-1+2\delta$.
\end{proof}

Part~(c) is the hypothesis-testing form of differential privacy \citep[see][]{dongrothsu2022}. The
proof of Corollary~\ref{cor:consequences}(ii) uses part~(b) instead: there the factor $e^{\epsdp}-1$
multiplies the mean absolute values $\E|\Xi|$, whereas part~(c), applied to $(\Xi+c)/(2c)$, would
multiply it by the bound $c$.

\begin{proposition}[Privacy caps the signed loss difference of every pair]\label{prop:dp-pair}
Let $D$ be a law of $Z$ carried by a set $\mathcal S$ of points with norm at most $1$ (for the family,
$\mathcal S=\mathcal Z$), let $\phi$ be measurable on $[-1,1]$ with $K_\phi<\infty$, and let $A_n$ be a
measurable proper learner that satisfies the privacy inequality, with $\epsdp\ge0$, for all training
sets in $\mathcal S^n$ that differ in one point. On the supersample under $D$, every pair $i$
satisfies
\[
 \E\,\sigma_iX_i\le(e^{\epsdp}-1)\,\E|X_i|+K_\phi\delta .
\]
No accuracy is assumed.
\end{proposition}

\begin{proof}
Fix $\Zt=\tilde z$ with all candidates in $\mathcal S$, which is an event of probability one, and fix
the other selection bits $U_{-i}$. The two values of $U_i$ select training sets $S^0,S^1\in\mathcal S^n$
that differ only in position $i$, so the laws of $A_n(S^0;R)$ and $A_n(S^1;R)$ satisfy the privacy
inequality in both directions. Given $\tilde z$, $X_i=\Phi_{\tilde z}(\thh)_i$, where
$\Phi_{\tilde z}(\theta)_i=\phi(\ip\theta{z_{1,i}})-\phi(\ip\theta{z_{0,i}})$ is the release map, and
$|\Phi_{\tilde z}(\cdot)_i|\le K_\phi$ on $W$. The bit $U_i$ is fair and independent of
$(\Zt,U_{-i},R)$, and $\sigma_i=1$ exactly when $U_i=0$. So Lemma~\ref{lem:dp-forms}(b) with
$c=K_\phi$ gives
\[
 \E[\sigma_iX_i\mid\Zt,U_{-i}]
 =\tfrac12\Big(\E\,\Phi_{\tilde z}\big(A_n(S^0;R)\big)_i-\E\,\Phi_{\tilde z}\big(A_n(S^1;R)\big)_i\Big)
 \le(e^{\epsdp}-1)\,\E\big[|X_i|\bigm|\Zt,U_{-i}\big]+K_\phi\delta .
\]
Integrating over $(\Zt,U_{-i})$ proves the claim.
\end{proof}

\begin{corollary}[Privacy caps the audit statistic]\label{cor:dp-audit}
Let $\phi$ satisfy the link conditions of Assumption~\ref{ass:class}, let $A_n$ be a measurable proper
$(\epsdp,\delta)$-differentially private learner with $\epsdp\ge0$, and let $\kappa>0$, $M>0$ and $V$
satisfy $\E_\Pi\sum_i\E_p\sigma_iX_i\ge M$ and $\E_\Pi\sum_i\E_pX_i^2\le V$ for the beta prior
$\Pi_\kappa$, as in the finite form of Proposition~\ref{prop:finite}. Then, for every $t>0$,
\[
 M\le(e^{\epsdp}-1)\,\frac{V/t+nt}2+nK_\phi\delta ,
\]
and hence, for every $\eps>0$ and with $\Mbar=M/(n\eps)$ and $\Vbar=V/(n\eps^2)>0$,
\[
 \Mbar\le(e^{\epsdp}-1)\sqrt{\Vbar}+\frac{K_\phi\delta}\eps .
\]
\end{corollary}

\begin{proof}
For every $p\in[-1,1]^d$ the law $D_p$ is carried by $\mathcal Z$, whose points have norm $1$. So
Proposition~\ref{prop:dp-pair} and $|x|\le x^2/(2t)+t/2$ give
$\sum_i\E_p\sigma_iX_i\le(e^{\epsdp}-1)\big(\sum_i\E_pX_i^2/t+nt\big)/2+nK_\phi\delta$. Average over the
prior, which is carried by $[-1,1]^d$, and use the two budget bounds and $e^{\epsdp}-1\ge0$. For the
second form take $t=\eps\sqrt{\Vbar}$, so that $V/t=nt=n\eps\sqrt{\Vbar}$, and divide by $n\eps$.
\end{proof}

\begin{proof}[Proof of Corollary~\ref{cor:consequences}(ii)]
Let $\thh=A_n(S_n;R)$ be $\eps$-accurate on the family and $(\epsdp,\delta)$-differentially private
with $\epsdp\ge0$, in the regime of Theorem~\ref{thm:main}. Let $M$ and $V$ be the budgets of the
finite form of Proposition~\ref{prop:finite} at the actual $n$, $d$ and $\eps$, and $\Mbar=M/(n\eps)$,
$\Vbar=V/(n\eps^2)$. By Appendix~\ref{app:orders}, $M\ge c_Mn\eps>0$ and $V\le C_Vn\eps^2$, and
$V>0$ because $L\ge b>0$ and $V=L^2(C_\theta+\rho w_\theta)$ with $\rho w_\theta>0$.
Corollary~\ref{cor:dp-audit} gives the finite bound
\[
 e^{\epsdp}-1\ \ge\ \frac{\Mbar-K_\phi\delta/\eps}{\sqrt{\Vbar}} ,
\]
in which $\Mbar\ge c_M$ and $\Vbar\le C_V$. For the order form, take $t=\eps\sqrt{\max\{C_V,1\}}$ in
Corollary~\ref{cor:dp-audit}: then $V/t\le n\eps\sqrt{\max\{C_V,1\}}=nt$, so
\[
 c_Mn\eps\le M\le(e^{\epsdp}-1)\,n\eps\sqrt{\max\{C_V,1\}}+nK_\phi\delta .
\]
Dividing by $n\eps\sqrt{\max\{C_V,1\}}\ge n\eps$ and using $\delta\ge0$ and $K_\phi\le2L$,
\[
 e^{\epsdp}-1\ \ge\ c'-\frac{K_\phi\delta}\eps\ \ge\ c'-\frac{2L\delta}\eps,\qquad
 c'=\frac{c_M}{\sqrt{\max\{C_V,1\}}}=\sqrt{2\ln2\,c_1},
\]
where $c_1=c_M^2/(2\ln2\max\{C_V,1\})$ is the lower constant of Theorem~\ref{thm:main}. This proves
Corollary~\ref{cor:consequences}(ii). In particular, if $4L\delta\le c'\eps$ then
$e^{\epsdp}-1\ge c'/2$, that is, $\epsdp\ge\ln(1+c'/2)$: in the regime of Theorem~\ref{thm:main},
accuracy forces $\epsdp=\Omega(1)$ unless $\delta=\Omega(\eps)$. Since $n\ge c_-\eps^{-2/q}$ and
$q\le2$ give $\eps\ge(c_-/n)^{q/2}\ge\min\{1,c_-\}/n$, the usual requirement $\delta=o(1/n)$ makes the
$\delta$-term negligible.
\end{proof}

\paragraph{Relation to Known Lower Bounds.} The mechanism is the fingerprinting argument behind lower
bounds for private estimation and selection \citep{bun2014,steinke2017,kamath2022}. For $r=2$, the
conclusion that sample-optimal accuracy in dimension $d\asymp n$ is incompatible with $\epsdp$ below a
constant is consistent with the lower bounds of \citet{bassily2014} for private strongly convex
empirical risk minimization at $d\asymp n$. Corollary~\ref{cor:consequences}(ii) also has the form of
Theorem~3.6 of \citet{voitovych2025}, which bounds $e^{\epsdp}-1$ from below by a universal constant
times the difference between the trace value of the problem and $2\delta$ times a bound on the scores;
the trace value is the infimum over accurate learners of the supremum, over data laws and over
tracers that read the model and whose scores are bounded and form a 1-subgaussian process in the
parameter, of the mean score of the training points. Here the statistic
$\sigma_iX_i$ reads only the released loss differences and the selection bits of one training run, never
the model, in the spirit of one-run privacy auditing \citep{steinke2023auditing}, and the conclusion holds
for every link and every power $r\ge2$ of Assumption~\ref{ass:class}. The argument works with
correlations throughout, so it needs no conversion between mutual information and $(\epsdp,\delta)$.

%% file: app_G_necessity.tex
\section{Proofs for \texorpdfstring{Section~\ref{sec:necessity}}{the necessity results}}\label{app:necessity}

This appendix proves the results of Section~\ref{sec:necessity}: Proposition~\ref{prop:growth} in
Appendix~\ref{app:growth}, the zero-slope example in Appendix~\ref{app:zero-slope} and
Proposition~\ref{prop:dimension} in Appendix~\ref{app:dimension}. Appendix~\ref{app:powerlinear}
proves the lower bound on which Proposition~\ref{prop:growth} rests.

\subsection{Every-Learner Lower Bound for Power-Linear Losses}\label{app:powerlinear}

Fix $r\ge2$ and $q=r/(r-1)$. The \emph{power-linear loss} is
$f(\theta,z)=\norm\theta^r/r-\ip\theta z$ on $W$. It combines the affine link $\phi(t)=-t$, for which
$b=L=1$ and $H=0$, with the regularizer $g(\theta)=\norm\theta^r/r$, which satisfies
\eqref{eq:matched} with $a_-=2^{2-r}/r$ and $a_+=1/r$ (Appendix~\ref{app:power}). The next lemma is a
finite lower bound for every learner that is accurate on the family for this loss. It uses the
beta prior $\Pi_\kappa$ of Section~\ref{sec:architecture}, with a prior variance
$v=\E\norm{\mu_p}^2=1/(\kappa+1)$ chosen for this loss, and it places no upper bound on the sample
size.

\begin{lemma}[Every-learner lower bound for power-linear losses]\label{lem:powerlinear}
Let $f(\theta,z)=\norm\theta^r/r-\ip\theta z$ and $\eps>0$, and put
\[
 a_-=2^{2-r}/r,\qquad v=(4q\eps)^{2/q},\qquad \kappa=v^{-1}-1,
\]
where $0<v<1$ is assumed. With the tail constants $K_0=K_0(n/d)$ and $K_1=2/(\cnet d)$ of
Appendix~\ref{app:tails} and the function $\Tb$ of Lemma~\ref{lem:weighted}, let
\begin{equation}\label{eq:powerlinear-budgets}
\begin{gathered}
 w=\min\{4,(\eps/a_-)^{2/r}\},\qquad w_\theta=\min\{1,(\sqrt w+v^{(q-1)/2})^2\},\\
 \widetilde A=\frac1q\,v^{q/2}\Big(1+\frac1{dv}\Big)^{-q/2}-\eps,\qquad
 \widetilde M=\kappa\widetilde A,\qquad
 \widetilde V=\Tb_1(w_\theta;K_0,K_1)+\frac nd\,w_\theta .
\end{gathered}
\end{equation}
If $\widetilde M>0$, then every proper learner that is $\eps$-accurate on the family for this loss
satisfies, for every $\tau\ge0$ and each $\Qc\in\{\Ii,\Ij,\Is\}$,
\begin{equation}\label{eq:powerlinear-information}
 \E_{p\sim\Pi_\kappa}\Qc_{D_p}(X+\tau G)\ \ge\ \frac{\widetilde M^2}{2\ln2\,(\widetilde V+n\tau^2)} .
\end{equation}
\end{lemma}

\begin{proof}
The argument is the one of Section~\ref{sec:moments}, and it is shorter because the link is affine:
the train--ghost remainder of Step~3 vanishes.

\emph{Population minimizer.} For $p\in[-1,1]^d$, the population risk is
$F_p(\theta)=\norm\theta^r/r-\ip\theta{\mu_p}$ on $W$. Young's inequality,
$\ip\theta{\mu_p}\le\norm\theta\norm{\mu_p}\le\norm\theta^r/r+\norm{\mu_p}^q/q$, gives
$F_p\ge-\norm{\mu_p}^q/q$, with equality at $\theta_p^*=\norm{\mu_p}^{q-2}\mu_p$ ($\theta_p^*=0$ when
$\mu_p=0$). This point lies in $W$ because $\norm{\theta_p^*}=\norm{\mu_p}^{q-1}\le1$. Hence
\[
 \min_WF_p=-\norm{\mu_p}^q/q,\qquad \norm{\theta_p^*}=\norm{\mu_p}^{q-1}.
\]

\emph{Calibration.} As $\norm\thh^r/r\ge0$, we have $F_p(\thh)\ge-\ip\thh{\mu_p}$, so accuracy at $D_p$
gives $\E_p\ip\thh{\mu_p}\ge-\E_pF_p(\thh)\ge\norm{\mu_p}^q/q-\eps$. Averaging over the prior and
using \eqref{eq:prior-low},
\[
 \E_\Pi\ip\thh{\mu_p}\ \ge\ \frac1q\,\E\norm{\mu_p}^q-\eps\ \ge\ \widetilde A .
\]

\emph{Signed mean.} For the affine link, $\sigma_iX_i=\ip\thh{\Ztr_i-\Zgh_i}$. Conditionally on $p$,
the training set and the seed, each ghost is a fresh draw from $D_p$ that is independent of $\thh$,
so $\E_\Pi\ip\thh{\Zgh_i-\mu_p}=0$. Lemma~\ref{lem:fp} and the calibration therefore give, since
$\kappa>0$,
\[
 \sum_i\E_\Pi\sigma_iX_i=\E_\Pi\Big\langle\thh,\sum_i(\Ztr_i-\mu_p)\Big\rangle
 =\kappa\,\E_\Pi\ip\thh{\mu_p}\ \ge\ \kappa\widetilde A=\widetilde M .
\]

\emph{Localization.} Let $e=\thh-\theta_p^*$. Since $F_p-g$ is linear,
$F_p(\theta)-F_p(\theta_p^*)-\ip{\nabla F_p(\theta_p^*)}{\theta-\theta_p^*}=D_g(\theta,\theta_p^*)$, and
the first-order condition for the minimizer on $W$ makes the inner product nonnegative for
$\theta\in W$. So $F_p(\theta)-\min_WF_p\ge D_g(\theta,\theta_p^*)\ge a_-\norm{\theta-\theta_p^*}^r$ on
$W$. Accuracy and Jensen's inequality ($r\ge2$) give
$\E_p\norm e^2\le(\E_p\norm e^r)^{2/r}\le(\eps/a_-)^{2/r}$, and $\norm e\le2$ gives $\E_p\norm e^2\le4$;
so $\E_p\norm e^2\le w$ at every $p$. Minkowski's inequality, the formula for
$\norm{\theta_p^*}$ and \eqref{eq:prior-concave} give
\[
 \big(\E_\Pi\norm\thh^2\big)^{1/2}\le\big(\E_\Pi\norm e^2\big)^{1/2}
 +\big(\E\norm{\mu_p}^{2(q-1)}\big)^{1/2}\le\sqrt w+v^{(q-1)/2},
\]
and together with $\norm\thh\le1$ this is $\E_\Pi\norm\thh^2\le w_\theta$.

\emph{Energy.} Write $s_i=\ip\thh{\Ztr_i-\mu_p}$ and $s_i^{\mathrm{gh}}=\ip\thh{\Zgh_i-\mu_p}$, so that
$X_i^2=(s_i-s_i^{\mathrm{gh}})^2$. Conditionally on $p$, the training set and the seed,
$s_i^{\mathrm{gh}}$ has mean zero, and its second moment is at most $\norm\thh^2/d$, because the
covariance of $D_p$ is diagonal with entries $(1-p_k^2)/d\le1/d$. The cross terms therefore vanish,
and
\[
 \sum_i\E_pX_i^2=\E_p\norm{\Btr\thh}^2+\sum_i\E_p(s_i^{\mathrm{gh}})^2
 \le\E_p\big[\opnorm{\Btr}^2\norm\thh^2\big]+\frac nd\,\E_p\norm\thh^2 .
\]
By Lemma~\ref{lem:matrix}, in the form proved in Appendix~\ref{app:tails} for all $n$ and $d$,
$\Pp_p(\opnorm{\Btr}^2>K_0+K_1t)\le e^{-t}$ for $t\ge0$ at every $p$, hence also under the prior.
Lemma~\ref{lem:weighted}, applied with the weight $\norm\thh^2\le1$ and
$\E_\Pi\norm\thh^2\le w_\theta\le1$, gives
$\E_\Pi[\opnorm{\Btr}^2\norm\thh^2]\le\Tb_1(w_\theta;K_0,K_1)$. Hence
$\sum_i\E_\Pi X_i^2\le\widetilde V$.

\emph{Transfer.} The prior $\Pi_\kappa$ depends only on $r$ and $\eps$, so it is fixed before the
learner, and the selection bits are independent of $p$. The lower bound \eqref{eq:lower} of
Theorem~\ref{thm:transfer}, whose proof in Appendix~\ref{app:fairbit} uses only the signed mean and
the energy, applies with $M=\widetilde M>0$ and $V=\widetilde V$. This is
\eqref{eq:powerlinear-information} for $\Ii$. The chain \eqref{eq:hierarchy} at every $p$ gives it
for $\Ij$ and $\Is$.
\end{proof}

When $n\ge\neps$ and $d\ge32n$, the budgets of the lemma satisfy $\widetilde M\asymp\neps\eps$ and
$\widetilde V=O(\neps\eps^2)$ as $\eps\to0$ with $r$ fixed, uniformly in $n$, so the right side of
\eqref{eq:powerlinear-information} at $\tau=0$ is at least of order $\neps$, however large $n$ is.
Indeed, $dv\ge32\neps v=32(4q)^{2/q}\ge256$ and $v^{q/2}/q=4\eps$, so $2.98\,\eps<\widetilde A\le3\eps$,
while $\kappa=(4q)^{-2/q}\neps-1$. Moreover, $K_0\le K_0^*<2.031$,
$K_1\le1/(16\cnet n)\le\eps^{2/q}/(16\cnet)$ and
$(4q\eps)^{2/r}=v^{q-1}\le w_\theta\le((\eps/a_-)^{1/r}+(4q\eps)^{1/r})^2$, so $K_1\ln(1/w_\theta)$
tends to zero and $\widetilde V=O(\eps^{2/r})$, which is $O(\neps\eps^2)$. Proposition~\ref{prop:growth}
uses the lemma with $r=2$ and $n=\eps^{-3/2}$, far above $\neps=\eps^{-1}$, and evaluates the budgets
explicitly.

\subsection{Proof of \texorpdfstring{Proposition~\ref{prop:growth}}{Proposition 11}}\label{app:growth}

With the affine link, $g(\theta)=\norm\theta^2/2$ gives the power-linear loss
$f_2(\theta,z)=\norm\theta^2/2-\ip\theta z$ with $r=2$. The \emph{mean learner} outputs the
training mean, $\thh^{(2)}=\Zbar$, and $X^{(2)}$ denotes its release.

\begin{restate}[Matched growth is needed]{Proposition}{prop:growth}
Let $\phi(t)=-t$ and $g(\theta)=\norm\theta^2/2$, so that
$D_g(\theta',\theta)=\norm{\theta'-\theta}^2/2\ge\norm{\theta'-\theta}^4/8$ on $W$, whereas no bound
$g(\theta)\le a_+\norm\theta^4$ holds near $0$. Let $n\ge182$, $d\ge32n$ and $\eps=\tau=n^{-2/3}$.
Every proper learner that is $\eps$-accurate on the family satisfies
$\sup_p\Qc_{D_p}(X+\eps G)\ge0.007\,n^{1/3}$ for each $\Qc$. The mean learner $\thh=\Zbar=n^{-1}\sum_iZ_{U_i,i}$ has excess risk at most
$1/(2n)\le\eps$ at every law on the unit ball, and $\sup_p\Qc_{D_p}(X+\eps G)\le3n^{1/3}$ for each
$\Qc$.
\end{restate}

\begin{proof}
\emph{The regularizer.} On $W$, $D_g(\theta',\theta)=\norm{\theta'-\theta}^2/2$, and
$\norm{\theta'-\theta}\le2$ gives $\norm{\theta'-\theta}^2/2\ge\norm{\theta'-\theta}^4/8$. The bound
$\norm\theta^2/2\le a_+\norm\theta^4$ fails whenever $0<\norm\theta^2<1/(2a_+)$. So $g$ has lower
curvature of power $4$, but growth of power $2$ only; $f_2$ satisfies Assumption~\ref{ass:class} with
$r=2$ and $a_-=a_+=1/2$, and not with $r=4$.

\emph{Accuracy of the mean learner.} Let $D$ be a law on the unit ball, with mean $\mu_D$. For $f_2$,
$F_D(\theta)=\norm{\theta-\mu_D}^2/2-\norm{\mu_D}^2/2$, which is minimized on $W$ at $\mu_D$. The mean
$\Zbar$ lies in $W$, and its expected excess risk is
$\E\norm{\Zbar-\mu_D}^2/2=(\E\norm Z^2-\norm{\mu_D}^2)/(2n)\le1/(2n)$. As $1/(2n)\le n^{-2/3}$ for
$n\ge1$, the mean learner is itself $\eps$-accurate on the family.

\emph{Upper bound for the mean learner.} We prove more than the proposition needs: for $n\ge16$,
$d\ge32n$, every $p$ and every $\tau>0$,
\begin{equation*}
 \Is_{D_p}(X^{(2)}+\tau G)\le\frac{2K_0+2K_1[1+\ln(4n)]}{2\ln2\;n\tau^2}\le\frac3{n\tau^2}.
\end{equation*}
Proposition~\ref{prop:separation}(iii) states the same bound. Fix $p$. For $f_2$ the population minimizer is
$\theta_p^*=\mu_p$, so the reference of Proposition~\ref{prop:finite} is $X_p^*(\Zt)=-\Delta\mu_p$,
which depends only on $p$ and the candidates, and $X^{(2)}-X_p^*(\Zt)=-\Delta(\Zbar-\mu_p)$. At
$D_p$, $\E_p\norm{\Zbar-\mu_p}^2=(1-\norm{\mu_p}^2)/n\le1/n$ and $\norm{\Zbar-\mu_p}^2\le4$. By
Lemma~\ref{lem:matrix}, $\Pp_p(\opnorm\Delta^2>2K_0+2K_1t)\le e^{-t}$ for $t\ge0$. The mean $\Zbar$
depends on the candidates in $\Delta$, and Lemma~\ref{lem:weighted} allows this dependence: with the
weight $\norm{\Zbar-\mu_p}^2\le4$ it gives
\[
 \E_p\norm{X^{(2)}-X_p^*(\Zt)}^2\le\E_p\big[\opnorm\Delta^2\norm{\Zbar-\mu_p}^2\big]
 \le\Tb_4(1/n;2K_0,2K_1)=\frac1n\big\{2K_0+2K_1[1+\ln(4n)]\big\}.
\]
The upper bound \eqref{eq:upper} of Theorem~\ref{thm:transfer}, which at a given law uses only the
residual bound at that law (Appendix~\ref{app:hierarchy}), and $\ln(1+x)\le x$ give the first
inequality of the display. For $d\ge32n$, $2K_0\le2K_0^*<4.062$ and $2K_1=4/(\cnet d)\le0.282/n$;
since $(1+\ln4n)/n$ decreases in $n$, the numerator is at most $4.062+0.282\cdot(1+\ln64)/16<4.153$
for $n\ge16$, and $4.153/(2\ln2)<3$. At $\tau=\eps=n^{-2/3}$ we have $n\tau^2=n^{-1/3}$, so the bound
is $3n^{1/3}$ at every $p$; the chain \eqref{eq:hierarchy} gives the same bound for $\Ii$ and $\Ij$.

\emph{Lower bound for every accurate learner.} Let $n\ge182$, $d\ge32n$ and $\eps=n^{-2/3}$. Since
$32^{3/2}<182$, we have $\eps\le1/32$. We apply Lemma~\ref{lem:powerlinear} with $r=q=2$: then $a_-=1/2$, $v=8\eps\le1/4$,
$\kappa=1/(8\eps)-1$, $w=2\eps$ and $w_\theta=(\sqrt{2\eps}+\sqrt{8\eps})^2=18\eps\le1$. As
$d\eps\ge32n^{1/3}\ge32$, we have $dv\ge256$, and
\[
 \widetilde A=\frac{4\eps}{1+1/(dv)}-\eps\ge\Big(\frac{1024}{257}-1\Big)\eps>2.98\,\eps,\qquad
 \widetilde M=\Big(\frac1{8\eps}-1\Big)\widetilde A>2.98\Big(\frac18-\frac1{32}\Big)>0.279 .
\]
Next, $n/d\le1/32$ gives $K_0\le K_0^*<2.031$, and $K_1=2/(\cnet d)\le0.141/n$. Since
$\ln(1/w_\theta)=\ln(n^{2/3}/18)$ and $(1+\ln(n^{2/3}/18))/n$ decreases for $n\ge182$,
$K_1[1+\ln(1/w_\theta)]\le0.141\cdot1.579/182<0.0013$. Hence
\[
 \widetilde V\le18\eps\,\big(2.031+0.0013+\tfrac1{32}\big)<37.2\,n^{-2/3}.
\]
At $\tau=\eps$, $n\tau^2=n^{-1/3}$, and Lemma~\ref{lem:powerlinear} gives, for each $\Qc$,
\[
 \sup_p\Qc_{D_p}(X+\eps G)\ \ge\ \E_{p\sim\Pi_\kappa}\Qc_{D_p}(X+\eps G)
 \ \ge\ \frac{0.279^2\,n^{1/3}}{2\ln2\,(1+37.2\,n^{-1/3})}\ \ge\ 0.0074\,n^{1/3},
\]
where the last step uses $n^{-1/3}\le182^{-1/3}<0.177$.

Since the mean learner is itself $\eps$-accurate, the least worst-case information over proper
$\eps$-accurate learners therefore lies between $0.007\,n^{1/3}$ and $3n^{1/3}$ for each $\Qc$.
\end{proof}

The formal statement of this lower bound (Appendix~\ref{app:lean}) has the smaller constant $c_2/2$
in place of $0.007$, where
\[
 c_2=\frac{1}{512\ln2\,\big(\frac{129}{32}(4+\sqrt2)^2+\frac{4+\sqrt2}{8\cnet}\big)}>2\cdot10^{-5}.
\]

\subsection{Zero Slope}\label{app:zero-slope}

Let $\phi$ be convex on a neighbourhood of $[-1,1]$ with $\phi'(0)=0$, and let $g$ satisfy the
conditions of Assumption~\ref{ass:class}. Then the learner that always outputs $\thh=0$ is proper,
its output minimizes the population risk at every law on the unit ball, and its release vanishes.

Indeed, let $D$ be a law with $\norm Z\le1$ and let $\theta\in W$. Every margin $\ip\theta Z$ lies in
$[-1,1]$, where convexity and $\phi'(0)=0$ give $\phi(t)\ge\phi(0)$. The regularizer is convex with
$g(0)=0$ and $\nabla g(0)=0$, so $g\ge0$. Hence
$F_D(\theta)=g(\theta)+\E_D\phi(\ip\theta Z)\ge\phi(0)=F_D(0)$, and the zero learner has excess risk
zero at every law; it is $\eps$-accurate for every $\eps>0$. Its release is
$X_i=\phi(0)-\phi(0)=0$, so $X+\tau G=\tau G$ is independent of $(U,\Zt)$, and all three
information measures vanish at every law and every $\tau\ge0$. Accuracy thus forces no information when
$\phi'(0)=0$. A link with $\phi'(0)>0$ reduces to one with $b>0$ by the reflection of
Remark~\ref{rem:reflection}, so what Theorem~\ref{thm:main} needs is a nonzero slope at zero.

\subsection{Proof of \texorpdfstring{Proposition~\ref{prop:dimension}}{Proposition 12}}\label{app:dimension}

\begin{restate}[Dimension]{Proposition}{prop:dimension}
Let $\phi$ be Borel measurable and $L$-Lipschitz on $[-1,1]$ with $L>0$, let $D$ be any law on $\R^d$ with
$\norm Z\le1$ almost surely, and let $\thh$ be any measurable proper learner; no accuracy is assumed.
For every $\tau>0$ and each $\Qc\in\{\Ii,\Ij,\Is\}$,
\begin{equation*}\tag{\ref{eq:dimension}}
 \Qc_D(X+\tau G)\le\min\Big\{n,\ d\log_2\Big(1+\frac{4L\sqrt n}\tau\Big)+\frac1{2\ln2}\Big\}.
\end{equation*}
Consequently, if $n\ge c_-\neps$ and some proper learner has $\sup_p\Qc_{D_p}(X+\eps G)\ge cn$ with
$c>0$, then $d\ge cn/(6\log_2n)$ for all $n$ above an explicit threshold that depends only on $c$,
$c_-$, $q$ and $L$.
\end{restate}

Every link of Assumption~\ref{ass:class} qualifies, since $|\phi'|\le L$ on $[-1,1]$ and
$L\ge|\phi'(0)|=b>0$. We prove the consequence in the following explicit form. If some measurable proper learner satisfies
$\sup_{p\in[-1,1]^d}\Qc_{D_p}(X+\eps G)\ge cn$ for some $c>0$ and $\eps>0$, then
\begin{equation}\label{eq:dimension-needed}
 d\ \ge\ \frac{cn-1/(2\ln2)}{\log_2\big(1+4L\sqrt n/\eps\big)} .
\end{equation}
If moreover $n\ge c_-\neps=c_-\eps^{-2/q}$ with $c_->0$ and $q=r/(r-1)\le2$, then
$d\ge cn/(6\log_2n)$ for every $n\ge\max\{2,2^{2\ell_0/3},1/(c\ln2)\}$, where
$\ell_0=\log_2(1+4Lc_-^{-q/2})$.

\begin{proof}
By the chain \eqref{eq:hierarchy} it suffices to treat $\Qc=\Is$, and $\Is_D(Y)\le\Hent(U\mid\Zt)=n$
gives the first entry of the minimum. Write $Y=X+\tau G$, and work in nats until the last step.

\emph{Lipschitz release map.} For a supersample $\tilde z=(z_{j,i})$ with every $\norm{z_{j,i}}\le1$,
an event of probability one, and for $\theta\in W$, let
$\Phi_{\tilde z}(\theta)_i=\phi(\ip\theta{z_{1,i}})-\phi(\ip\theta{z_{0,i}})$, so that
$X=\Phi_{\Zt}(\thh)$. For $\theta,\theta'\in W$ every margin lies in $[-1,1]$, so
$|\Phi_{\tilde z}(\theta)_i-\Phi_{\tilde z}(\theta')_i|\le L|\ip{\theta-\theta'}{z_{1,i}}|
+L|\ip{\theta-\theta'}{z_{0,i}}|\le2L\norm{\theta-\theta'}$ for each $i$. Summing the squares over the
$n$ coordinates,
\begin{equation}\label{eq:dimension-lip}
 \norm{\Phi_{\tilde z}(\theta)-\Phi_{\tilde z}(\theta')}\le2L\sqrt n\,\norm{\theta-\theta'}
 \qquad(\theta,\theta'\in W).
\end{equation}

\emph{Net.} Put $h=\tau/(2L\sqrt n)$. If the points of a subset of $W$ are at pairwise distance at
least $h$, the open balls of radius $h/2$ around them are disjoint and lie in the ball of radius
$1+h/2$; comparing volumes, such a set has at most $(1+h/2)^d/(h/2)^d=(1+2/h)^d$ points. Let
$\mathcal T_h\subset W$ be such a set of maximal cardinality. By maximality, every $\theta\in W$ lies
within distance $h$ of some point of $\mathcal T_h$, and
\[
 |\mathcal T_h|\le(1+2/h)^d=\big(1+4L\sqrt n/\tau\big)^d .
\]
The net depends only on $d$ and $h$, not on the data or the learner.

\emph{Reference and variational bound.} Let
$\bar{\Normal}_{\tilde z}=|\mathcal T_h|^{-1}\sum_{\theta'\in\mathcal T_h}\Normal(\Phi_{\tilde z}(\theta'),\tau^2I_n)$,
a Markov kernel from supersamples to $\R^n$. For every such kernel, with $\KL$ the Kullback--Leibler
divergence,
\[
 \E\,\KL\big(P_{Y\mid U,\Zt}\,\big\|\,\bar{\Normal}_{\Zt}\big)
 =I(Y;U\mid\Zt)+\E\,\KL\big(P_{Y\mid\Zt}\,\big\|\,\bar{\Normal}_{\Zt}\big)\ \ge\ I(Y;U\mid\Zt).
\]
Given $(U,\Zt,X)$ the observation is $Y\sim\Normal(X,\tau^2I_n)$, because $G$ is independent of the
data and of the learner. So $P_{Y\mid U,\Zt}$ is the conditional average of $\Normal(X,\tau^2I_n)$
given $(U,\Zt)$, and convexity of the divergence in its first argument gives
\begin{equation}\label{eq:dimension-var}
 I(Y;U\mid\Zt)\le\E\,\KL\big(\Normal(X,\tau^2I_n)\,\big\|\,\bar{\Normal}_{\Zt}\big).
\end{equation}

\emph{Pointwise bound.} Fix an outcome in the probability-one event above, and choose
$\theta'\in\mathcal T_h$ with $\norm{\thh-\theta'}\le h$. As measures,
$\bar{\Normal}_{\tilde z}\ge|\mathcal T_h|^{-1}\Normal(\Phi_{\tilde z}(\theta'),\tau^2I_n)$, so every
probability measure $P$ satisfies
$\KL(P\,\|\,\bar{\Normal}_{\tilde z})\le\KL\big(P\,\|\,\Normal(\Phi_{\tilde z}(\theta'),\tau^2I_n)\big)
+\ln|\mathcal T_h|$. Gaussians with a common covariance satisfy
$\KL(\Normal(x,\tau^2I_n)\,\|\,\Normal(y,\tau^2I_n))=\norm{x-y}^2/(2\tau^2)$. With
$x=X=\Phi_{\Zt}(\thh)$ and $y=\Phi_{\Zt}(\theta')$, \eqref{eq:dimension-lip} gives
$\norm{x-y}\le2L\sqrt n\,h=\tau$, hence
\[
 \KL\big(\Normal(X,\tau^2I_n)\,\big\|\,\bar{\Normal}_{\Zt}\big)\le\ln|\mathcal T_h|+\frac12
 \le d\ln\Big(1+\frac{4L\sqrt n}{\tau}\Big)+\frac12 .
\]
Inserting this into \eqref{eq:dimension-var} and dividing by $\ln2$ gives
$\Is_D(Y)\le d\log_2(1+4L\sqrt n/\tau)+1/(2\ln2)$ bits. Together with the cap $n$, this is
\eqref{eq:dimension}.

\emph{Consequence.} At $\tau=\eps$, \eqref{eq:dimension} bounds $\Qc_{D_p}(X+\eps G)$ at every $p$,
since every $D_p$ satisfies $\norm Z=1$. So $cn\le d\log_2(1+4L\sqrt n/\eps)+1/(2\ln2)$, and the
logarithm is positive, which gives \eqref{eq:dimension-needed}. If $n\ge c_-\eps^{-2/q}$, then
$1/\eps\le(n/c_-)^{q/2}$ and $4L\sqrt n/\eps\le4Lc_-^{-q/2}n^{(1+q)/2}$. Since $n^{(1+q)/2}\ge1$,
$1+4Lc_-^{-q/2}n^{(1+q)/2}\le(1+4Lc_-^{-q/2})\,n^{(1+q)/2}$, and therefore
\[
 \log_2\big(1+4L\sqrt n/\eps\big)\le\tfrac{1+q}2\log_2n+\ell_0\le\tfrac32\log_2n+\ell_0\le3\log_2n
 \qquad(n\ge\max\{2,2^{2\ell_0/3}\}).
\]
For $n\ge1/(c\ln2)$ the numerator of \eqref{eq:dimension-needed} is at least $cn/2$, so
$d\ge(cn/2)/(3\log_2n)=cn/(6\log_2n)$.
\end{proof}

Theorem~\ref{thm:main} gives the premise of the second part for every learner of its class, with
$c=c_1/2$ at $\tau=\eps$, but only under its hypothesis $d\ge\Cdim\max\{n,\neps\}$. By the
proposition, when $n\ge c_-\neps$, $n\ge\max\{2,2^{2\ell_0/3},1/(c\ln2)\}$ and $d<cn/(6\log_2n)$, no
proper learner has $\sup_p\Qc_{D_p}(X+\eps G)\ge cn$. So a lower bound of order $n$ at $\tau=\eps$, such
as the lower half of \eqref{eq:main}, needs dimension of order $n$ up to a logarithmic factor. In fixed dimension every proper learner releases $O(d\log(n/\tau))$ bits as
$n/\tau\to\infty$, in particular $O(d\log n)$ bits at $\tau\asymp\eps$ when $n\asymp\neps$. The proof uses
only $\norm Z\le1$, $\norm\thh\le1$ and the Lipschitz link.

\paragraph{Without Noise.} At $\tau=0$, \eqref{eq:dimension} gives only the cap $n$, since a learner
may write $U$ into the digits of $\thh$. Rounding removes this freedom. Let $\thh$ be a proper learner that is $\eps$-accurate on
the family, let $h'=\eps/(L+L_g)$, let $\mathcal T_{h'}\subset W$ be a maximal $h'$-separated subset
of $W$, which has at most $(1+2/h')^d$ points as shown above, and let $\thh'$ be a point
of $\mathcal T_{h'}$ nearest to $\thh$, with ties broken by a fixed order. On $W$,
$\norm{\nabla F_p}\le L_g+L$, because $\norm{\nabla g}\le L_g$, $|\phi'|\le L$ on $[-1,1]$ and
$\norm Z=1$; as $\norm{\thh'-\thh}\le h'$ by maximality, $F_p(\thh')\le F_p(\thh)+\eps$. As
$\mathcal T_{h'}$ is finite, the nearest-point map with this tie-breaking is Borel, so $\thh'$ is a
proper learner, and it is $2\eps$-accurate on the family. Its release $X'$ is a function of $\thh'$
and the candidates, so by \eqref{eq:hierarchy}, for each $\Qc$ and every $p$,
\[
 \Qc_{D_p}(X')\le I(\thh';U\mid\Zt)\le\log_2|\mathcal T_{h'}|\le d\log_2\Big(1+\frac{2(L+L_g)}\eps\Big).
\]
Hence, if some proper learner is $\eps$-accurate on the family, a lower bound $cn$ at $\tau=0$ for
every proper learner that is $2\eps$-accurate on the family forces $d\ge cn/\log_2(1+2(L+L_g)/\eps)$, which is of order $n/\log n$ when $n\ge c_-\neps$.

%% file: app_H_separation.tex
\section{Model information and the loss-difference release}\label{app:separation}

This appendix states and proves Proposition~\ref{prop:separation} (Appendix~\ref{app:separation-proof}) and a
model-level converse of Theorem~\ref{thm:main} (Appendix~\ref{app:converse});
Appendices~\ref{app:directions}--\ref{app:full-rank} prove the four lemmas they use. Throughout, $X$
is the release of Definition~\ref{def:release}, $G\sim\Normal(0,I_n)$ is independent of the data, the
selection bits and the learner, and information is in bits.

\subsection{Data Processing and Tracers}\label{app:directions}

A \emph{tracer} is a measurable test $\psi(\theta,z)\in\{0,1\}$, which may depend on the law $D$, for
instance through its mean $\mu_D$; it flags $z$ as a training point when $\psi(\thh,z)=1$. At a law
$D$, its \emph{false-positive rate} is the probability that it flags a fresh point $Z'\sim D$, drawn
independently of the training set and the seed, and its \emph{expected training recall} is
$\rec=\sum_i\Pp(\psi(\thh,\Ztr_i)=1)$. Its decisions on the candidates are
$\psi_{j,i}=\psi(\thh,Z_{j,i})$.

\begin{lemma}[Data processing and tracers]\label{lem:directions}
For every link and regularizer, every proper learner, every law $D$ and every $\tau\ge0$,
\begin{align}
 \Ii(X+\tau G)\le\Ij(X+\tau G)\le\Is(X+\tau G)&\le I(X;U\mid\Zt)\le I(\thh;U\mid\Zt),
 \tag{\ref{eq:hierarchy}}\\
 I(X_i+\tau G_i;U_i\mid\Zt)\le I(X_i;U_i\mid\Zt)&\le I(\thh;U_i\mid\Zt)\qquad(i\in[n]).\label{eq:directions}
\end{align}
If a tracer has false-positive rate at most $\xi$ and expected training recall at least $\rec$ at
$D$, then
\[
 \sum_iI\big((\psi_{0,i},\psi_{1,i});U_i\big)\ \ge\ \frac{(\rec-n\xi)_+^2}{2\ln2\,n}.
\]
\end{lemma}

\begin{proof}
The first three inequalities are the hierarchy of Appendix~\ref{app:hierarchy}. Conditionally on
$\Zt$, the release is a function of $\thh$, and $G$ is independent of everything else, so
$U\to\thh\to X\to X+\tau G$ is a Markov chain given $\Zt$. Conditional data processing gives the last
two inequalities of the first line. The same argument for the chain
$U_i\to\thh\to X_i\to X_i+\tau G_i$ given $\Zt$ gives \eqref{eq:directions}; at $\tau=0$ this is the
chain recorded by \citet{wangmao2023a}.

For the tracer, note that $\sigma_i(\psi_{0,i}-\psi_{1,i})=\psi(\thh,\Ztr_i)-\psi(\thh,\Zgh_i)$, since
$\sigma_i=1$ exactly when $Z_{0,i}$ is the training candidate. The ghost is a fresh point, independent
of $\thh$, so $\E\sigma_i(\psi_{0,i}-\psi_{1,i})\ge\Pp(\psi(\thh,\Ztr_i)=1)-\xi$, and summing over $i$,
$\sum_i\E\sigma_i(\psi_{0,i}-\psi_{1,i})\ge\rec-n\xi$; moreover $\sum_i\E(\psi_{0,i}-\psi_{1,i})^2\le n$.
If $\rec\le n\xi$ there is nothing to prove. Otherwise the fair-bit inequality \eqref{eq:fairbit}
for each pair, with the value zero for a pair whose difference vanishes almost surely, and
Cauchy--Schwarz over $i$ give
\[
 \sum_iI(\psi_{0,i}-\psi_{1,i};U_i)
 \ge\frac1{2\ln2}\sum_i\frac{(\E\sigma_i(\psi_{0,i}-\psi_{1,i}))^2}{\E(\psi_{0,i}-\psi_{1,i})^2}
 \ge\frac{(\rec-n\xi)^2}{2\ln2\,n}.
\]
As $\psi_{0,i}-\psi_{1,i}$ is a function of $(\psi_{0,i},\psi_{1,i})$, data processing gives the
claim.
\end{proof}

\subsection{What Transfers at Zero Noise}\label{app:zero-noise}

\begin{lemma}[What transfers at zero noise]\label{lem:sign}
Let $D$ be a law on the unit ball and let $\phi$ be strictly decreasing on $[-1,1]$. Consider a
tracer that flags $z$ when $\chi(\ip\thh z;\thh)\ge0$, for a function $\chi(t;\theta)$ that is
strictly increasing in $t$ for every $\theta$. Centred linear scores with a threshold $c$,
$\chi(t;\theta)=t-\ip\theta{\mu_D}-c$, and negative losses with a threshold $c$,
$\chi(t;\theta)=-g(\theta)-\phi(t)-c$, are of this form. If the tracer has false-positive rate at most
$\xi$ and expected training recall at least $\rec$ at $D$, then the decoder
$\widehat U_i=\ind\{X_i<0\}$ satisfies $n^{-1}\sum_i\Pp(\widehat U_i=U_i)\ge\rec/n-\xi$, and whenever
$\rec/n-\xi>1/2$, with the function $\J$ of Appendix~\ref{app:fairbit},
\begin{equation}\label{eq:sign-transfer}
 \Ii_D(X)\ \ge\ n\,\J\big(2(\rec/n-\xi)-1\big).
\end{equation}
\end{lemma}

\begin{proof}
Let $E_i$ be the event that the tracer flags the training candidate $\Ztr_i$ and does not flag the
ghost $\Zgh_i$. On $E_i$, $\chi(\ip\thh{\Ztr_i};\thh)\ge0>\chi(\ip\thh{\Zgh_i};\thh)$, hence
$\ip\thh{\Ztr_i}>\ip\thh{\Zgh_i}$. Both margins lie in $[-1,1]$, the regularizer cancels and $\phi$
is strictly decreasing there, so $f(\thh,\Ztr_i)<f(\thh,\Zgh_i)$, that is, $\sigma_iX_i>0$. Thus on
$E_i$ we have $X_i>0$ if $U_i=0$ and $X_i<0$ if $U_i=1$, so $\widehat U_i=U_i$. The ghost is a fresh
point independent of $\thh$, so $\Pp(E_i)\ge\Pp(\psi(\thh,\Ztr_i)=1)-\xi$, and summing over $i$ gives
the bound on the average success. The estimate $\widehat U_i$ is a function of $X_i$, and Fano's
inequality for the fair bit $U_i$ gives
$I(X_i;U_i)\ge I(\widehat U_i;U_i)\ge1-\htwo(\Pp(\widehat U_i=U_i))=\J(2\Pp(\widehat U_i=U_i)-1)$.
The function $\J$ is even and convex, and increasing on $[0,1]$, so Jensen's inequality and the bound
on the average success prove \eqref{eq:sign-transfer}.
\end{proof}

The pair $(\xi,\rec)$ says nothing about the size of the margins, so it implies nothing about
$X+\tau G$ for $\tau>0$. Proposition~\ref{prop:separation} shows that two learners with identical
tracing guarantees can differ in order once noise is added.

\subsection{Accuracy of \texorpdfstring{$T_r(\Zbar)$}{the power-linear empirical risk minimizer}}\label{app:power-accuracy}

Let $r\ge2$, $q=r/(r-1)$ and $c_q=2^{2-q}/q$, and let $f_r(\theta,z)=\norm\theta^r/r-\ip\theta z$ be
the power-linear loss of Appendix~\ref{app:powerlinear}. The map
\[
 T_r(t)=\argmin_{\theta\in W}\big\{\norm\theta^r/r-\ip\theta t\big\}
 =\min\{1,\norm t^{q-1}\}\,\frac t{\norm t}\qquad(t\ne0),\qquad T_r(0)=0,
\]
is the gradient of the convex conjugate $\gWstar$ of $g_W$, where $g_W=\norm\cdot^r/r$ on $W$ and
$g_W=+\infty$ outside $W$:
\[
 \gWstar(t)=\sup_{\theta\in W}\big\{\ip\theta t-\norm\theta^r/r\big\}
 =\begin{cases}\norm t^q/q,&\norm t\le1,\\ \norm t-1/r,&\norm t\ge1.\end{cases}
\]
Both pieces are continuously differentiable with gradient $T_r$, and they agree, together with their
gradients, on the unit sphere; so $\gWstar$ is continuously differentiable and $\nabla\gWstar=T_r$.
Since the training risk of $f_r$ is $\norm\theta^r/r-\ip\theta\Zbar$, the empirical risk minimizer of
$f_r$ over $W$ is $T_r(\Zbar)$. For $r=2$, $T_2$ is the identity on $W$ and $T_2(\Zbar)=\Zbar$ is the
mean learner; for $r=4$, $T_4(t)=\norm t^{-2/3}t$ for $0<\norm t\le1$.

\begin{lemma}[Accuracy of $T_r(\Zbar)$]\label{lem:power-accuracy}
Let $r\ge2$.
\begin{enumerate}
\item[(a)] For all $x,x'\in\R^d$,
$\ip{x-x'}{T_r(x)-T_r(x')}\ge2^{2-r}\norm{T_r(x)-T_r(x')}^r$, and hence
$\norm{T_r(x)-T_r(x')}\le2^{2-q}\norm{x-x'}^{q-1}$. Moreover, $T_r$ is injective on $W$.
\item[(b)] For every law $D$ on the unit ball, with mean $\mu_D$, and every $t\in\R^d$, the population
risk of $f_r$ satisfies
$F_D(T_r(t))-\min_WF_D=D_{\gWstar}(\mu_D,t)\le c_q\norm{\mu_D-t}^q$.
\item[(c)] The learner $T_r(\Zbar)$ is proper, and on every law $D$ on the unit ball
\begin{equation}\label{eq:power-accuracy}
 \E F_D(T_r(\Zbar))-\min_WF_D\ \le\ c_q\,n^{-q/2}.
\end{equation}
For $r=2$ the excess risk equals $\E\norm{\Zbar-\mu_D}^2/2=(\E\norm Z^2-\norm{\mu_D}^2)/(2n)\le1/(2n)$.
\end{enumerate}
\end{lemma}

\begin{proof}
(a) Let $\theta_x=T_r(x)$ and $\theta_{x'}=T_r(x')$, and write $\nabla g(\theta)=\norm\theta^{r-2}\theta$.
The first-order conditions for the two minimizations on $W$ are
$\ip{\nabla g(\theta_x)-x}{\theta-\theta_x}\ge0$ and $\ip{\nabla g(\theta_{x'})-x'}{\theta-\theta_{x'}}\ge0$
for all $\theta\in W$. Taking $\theta=\theta_{x'}$ in the first and $\theta=\theta_x$ in the second and
adding,
\[
 \ip{x-x'}{\theta_x-\theta_{x'}}\ \ge\ \ip{\nabla g(\theta_x)-\nabla g(\theta_{x'})}{\theta_x-\theta_{x'}}
 \ \ge\ 2^{2-r}\norm{\theta_x-\theta_{x'}}^r,
\]
where the last step is the monotonicity inequality of the power map (Appendix~\ref{app:power};
\citealp[Chapter~12]{lindqvist2019}). The left side is at most $\norm{x-x'}\norm{\theta_x-\theta_{x'}}$,
so $\norm{\theta_x-\theta_{x'}}^{r-1}\le2^{r-2}\norm{x-x'}$; as $1/(r-1)=q-1$ and $(r-2)/(r-1)=2-q$, this is
the second inequality. On $W$, $T_r(t)=\norm t^{q-2}t$ for $t\ne0$, so $\norm{T_r(t)}=\norm t^{q-1}$
and $T_r(t)$ has the direction of $t$; hence $T_r(t)$ determines $t\in W$.

(b) For $\theta\in W$, $F_D(\theta)=\norm\theta^r/r-\ip\theta{\mu_D}$, so
$\min_WF_D=-\gWstar(\mu_D)$. Since $T_r(t)$ attains the supremum that defines $\gWstar(t)$, Fenchel's
equality $\norm{T_r(t)}^r/r+\gWstar(t)=\ip{T_r(t)}t$ holds, and therefore
\[
 F_D(T_r(t))-\min_WF_D=\ip{T_r(t)}t-\gWstar(t)-\ip{T_r(t)}{\mu_D}+\gWstar(\mu_D)=D_{\gWstar}(\mu_D,t),
\]
because $\nabla\gWstar(t)=T_r(t)$. Along the segment from $t$ to $\mu_D$, part~(a) gives
\[
 D_{\gWstar}(\mu_D,t)=\int_0^1\ip{T_r(t+y(\mu_D-t))-T_r(t)}{\mu_D-t}\,dy
 \le\int_0^12^{2-q}y^{q-1}\norm{\mu_D-t}^q\,dy=c_q\norm{\mu_D-t}^q .
\]

(c) As $\norm{T_r(t)}\le1$ for every $t$, the learner is proper. Part~(b) with $t=\Zbar$, Jensen's
inequality ($q\le2$) and $\E\norm{\Zbar-\mu_D}^2=(\E\norm Z^2-\norm{\mu_D}^2)/n\le1/n$ give
\[
 \E F_D(T_r(\Zbar))-\min_WF_D\le c_q\E\norm{\Zbar-\mu_D}^q\le c_q\big(\E\norm{\Zbar-\mu_D}^2\big)^{q/2}
 \le c_qn^{-q/2}.
\]
For $r=2$, $\gWstar(t)=\norm t^2/2$ on $W$, and both $\mu_D$ and $\Zbar$ lie in $W$, so
$D_{\gWstar}(\mu_D,\Zbar)=\norm{\mu_D-\Zbar}^2/2$ exactly.
\end{proof}

\subsection{The Pair Matrix Has Full Rank}\label{app:full-rank}

Fix a supersample $\tilde z=(z_{j,i})$, with candidate centre $z_c=(2n)^{-1}\sum_{j,i}z_{j,i}$ and pair
matrix $\Delta$ with rows $\Delta_i=(z_{1,i}-z_{0,i})^\top$. For a sign vector $\sigma\in\{-1,1\}^n$,
let $\Zbar(\sigma)$ be the training mean when $\sigma_i=1-2U_i$ for every $i$. The training candidate
of pair $i$ is $(z_{0,i}+z_{1,i})/2-\sigma_i\Delta_i^\top/2$, so
\[
 \Zbar(\sigma)=z_c-\Delta^\top\sigma/(2n).
\]
For the affine link, $X_i=-\ip\thh{Z_{1,i}-Z_{0,i}}$, so the release of a learner $\thh$ is
$X=-\Delta\thh$.

\begin{lemma}[The pair matrix has full rank]\label{lem:full-rank}
Let $r\ge2$, and let $X=-\Delta T_r(\Zbar)$ be the noiseless release of $T_r(\Zbar)$ for $f_r$.
\begin{enumerate}
\item[(a)] On every supersample, two sign vectors give the same release if and only if they give the
same training mean.
\item[(b)] At the zero-mean law $D_0$ with $d\ge32n$, $\Pp_0(\rank\Delta<n)\le e^{-225n/32}$. When
$\rank\Delta=n$, the release determines $U$; hence $I_{D_0}(X;U\mid\Zt)\ge n(1-e^{-225n/32})$.
\end{enumerate}
\end{lemma}

\begin{proof}
(a) Let $\sigma,\sigma'$ be sign vectors, and put $t=\Zbar(\sigma)$ and $t'=\Zbar(\sigma')$; both lie
in $W$. If the two releases are equal, then $\Delta(T_r(t)-T_r(t'))=0$, so $T_r(t)-T_r(t')$ lies in
the kernel of $\Delta$, which is orthogonal to the range of $\Delta^\top$, while
$t-t'=-\Delta^\top(\sigma-\sigma')/(2n)$ lies in that range. Hence $\ip{t-t'}{T_r(t)-T_r(t')}=0$, and
Lemma~\ref{lem:power-accuracy}(a) gives $T_r(t)=T_r(t')$, so $t=t'$ by injectivity on $W$.
Conversely, equal training means give equal releases.

(b) At $D_0$ the entries of $\Delta$ are independent, and each equals $0$, $2/\sqrt d$ or $-2/\sqrt d$
with probabilities $1/2$, $1/4$ and $1/4$. A column of $\Delta$ lies in a given proper linear subspace
of $\R^n$ with probability at most $1/2$: choose a nonzero vector $y$ orthogonal to the subspace and an
index $i$ with $y_i\ne0$; membership forces the column to be orthogonal to $y$, which, given the
other entries of the column, fixes its $i$-th entry, and that entry takes each value with probability
at most $1/2$. The $d$ columns are independent,
so as long as the first $k$ columns span less than $\R^n$, column $k+1$ raises the rank with
conditional probability at least $1/2$. Call step $k+1$ a success if it raises the rank or the rank is
already $n$. Every step succeeds with conditional probability at least $1/2$, so the number of
successes stochastically dominates a $\Bin(d,1/2)$ variable; and if $\rank\Delta<n$, the successes
are exactly the rank increases, which number fewer than $n$. Hence
\[
 \Pp_0(\rank\Delta<n)\le\Pp\big(\Bin(d,1/2)<n\big)\le\Pp\big(\Bin(32n,1/2)<n\big)\le e^{-225n/32},
\]
where the middle step holds because $\Bin(d,1/2)$ is stochastically increasing in $d$, and the last is
the multiplicative Chernoff bound for the lower tail: $\Bin(32n,1/2)$ has mean $16n$, and
$n=(1-\tfrac{15}{16})\,16n$, so the probability is at most
$\exp\{-(\tfrac{15}{16})^2\cdot16n/2\}=e^{-225n/32}$. When $\rank\Delta=n$, $\Delta^\top$ is injective, so distinct sign vectors give
distinct training means and, by~(a), distinct releases; given $\Zt$, the release then determines
$\sigma$ and hence $U$. Conditionally on such a supersample the information is $\Hent(U)=n$ bits, and
averaging over the supersample gives the bound.
\end{proof}

\subsection{Model Information Does Not Determine the Tradeoff}\label{app:separation-proof}

Let $f_2(\theta,z)=\norm\theta^2/2-\ip\theta z$ and $f_4(\theta,z)=\norm\theta^4/4-\ip\theta z$ be the
power-linear losses with $r=2$ and $r=4$; both belong to the class of Assumption~\ref{ass:class}, with
the affine link. The first is the strongly convex instance of \citet[Section~7.1.1]{attias2024},
restricted to the unit ball. The mean learner $\thh^{(2)}=\Zbar$ and the learner
$\thh^{(4)}=T_4(\Zbar)=\norm\Zbar^{-2/3}\Zbar$ ($\thh^{(4)}=0$ when $\Zbar=0$) are the empirical risk
minimizers of $f_2$ and $f_4$ over $W$ (Appendix~\ref{app:power-accuracy}). Their releases are
$X^{(2)}=-\Delta\Zbar$ and $X^{(4)}=-\Delta T_4(\Zbar)$, so that
\[
 X^{(4)}=\norm\Zbar^{-2/3}X^{(2)}\qquad(\Zbar\ne0),
\]
and both vanish when $\Zbar=0$. The two releases thus differ by a scalar factor that depends only on
the training mean. At the zero-mean law $D_0$, $\E_0\norm\Zbar^2=1/n$, so the factor is about
$n^{1/3}$. The two learners are accurate at different scales, $1/(2n)$ for $f_2$ and $1.2\,n^{-2/3}$
for $f_4$ (part~(i) below). At the noise level $\tau=n^{-2/3}$ of part~(v), the profile
$n/(1+(\tau/\eps)^2)$ of Theorem~\ref{thm:main} is of order $n^{1/3}$ at the first accuracy and of
order $n$ at the second, the orders of parts~(iii) and~(iv). Within one problem,
Theorem~\ref{thm:main} gives all accurate learners in its sample-size range the same order of worst-case
information, so a separation of this kind must compare two losses. The proposition shows that model
information and tracing guarantees, which the bijection $T_4$ preserves, cannot distinguish these two
scales, whereas the information of the noisy release does.

\begin{proposition}[Model information does not determine the tradeoff]\label{prop:separation}
Let $n\ge16$ and $d\ge32n$.
\begin{enumerate}
\item[(i)] Both learners are proper. At every law on the unit ball, $\thh^{(2)}$ has expected excess
risk at most $1/(2n)$ for $f_2$, and $\thh^{(4)}$ at most $1.2\,n^{-2/3}$ for $f_4$; both rates are
sample-optimal (Proposition~\ref{prop:samples}).
\item[(ii)] $T_4$ is a bijection of $W$. Hence, at every law, the model information quantities
$I(\thh;U\mid\Zt)$, $I(\thh;U_i\mid\Zt)$ and $I(\thh;U)$ of the two learners coincide, and every tracer
for one learner has a tracer with the same decisions for the other. The noiseless information quantities
$I(X^{(2)};U\mid\Zt)$ and $I(X^{(4)};U\mid\Zt)$ also coincide, and at $D_0$ they are at least
$n(1-e^{-225n/32})$.
\item[(iii)] At every $p$ and for every $\tau>0$,
\begin{equation}\label{eq:separation-upper}
 \Is_{D_p}(X^{(2)}+\tau G)\ \le\ \frac3{n\tau^2}.
\end{equation}
\item[(iv)] At $D_0$, for every $\tau\ge0$,
\begin{equation}\label{eq:separation-lower}
 \Ii_{D_0}(X^{(4)}+\tau G)\ \ge\ \frac{n^{2/3}(1+2/d)^{-4/3}}{2\ln2\,\big(4.2\,n^{-1/3}+n\tau^2\big)}.
\end{equation}
\item[(v)] So at $\tau=n^{-2/3}$, for each $\Qc$, $\Qc_{D_p}(X^{(2)}+\tau G)\le3n^{1/3}$ at every $p$,
whereas $\Qc_{D_0}(X^{(4)}+\tau G)\ge n/8$.
\end{enumerate}
\end{proposition}

\begin{proof}
(i) The learners are proper because $\norm\Zbar\le1$ and $\norm{T_4(\Zbar)}=\norm\Zbar^{1/3}\le1$.
Lemma~\ref{lem:power-accuracy}(c) gives the excess risk $1/(2n)$ for $r=2$ and $c_{4/3}n^{-2/3}$ for
$r=4$, where $q=4/3$ and $c_{4/3}=2^{2/3}\cdot3/4<1.2$. So these learners reach accuracy $\eps$
with $O(\neps)$ samples, $\neps=\eps^{-2/q}$, and by Proposition~\ref{prop:samples} every proper
learner that is $\eps$-accurate on the family needs $\Omega(\neps)$ samples when $d\ge\Cdim\neps$ and
$\eps$ is small.

(ii) The map $T_4$ is a bijection of $W$ onto itself, with inverse $\theta\mapsto\norm\theta^2\theta$, and
both maps are continuous. Since $\thh^{(4)}=T_4(\thh^{(2)})$ pointwise, each output is a measurable
bijective image of the other, so every mutual information between the model and the selection bits,
conditional on $\Zt$ or not, is the same for both learners. (The unconditional one is zero for every
learner, since the law of the training set does not depend on $U$.) If $\psi$ is a tracer for
$\thh^{(2)}$, then $\psi'(\theta,z)=\psi(T_4^{-1}(\theta),z)$ makes the same decisions for $\thh^{(4)}$,
and conversely with $T_4$ in place of $T_4^{-1}$.

For the noiseless releases, Lemma~\ref{lem:full-rank}(a), with $r=2$ (where $T_2(\Zbar)=\Zbar$) and
with $r=4$, shows that on every supersample $X^{(2)}$ and $X^{(4)}$ induce the same partition of the
sign vectors as the training mean does, including on supersamples where $\Delta$ has rank below $n$.
Given $\Zt$, both releases are functions of $U$, and the information between $U$ and a function of
$U$ depends only on the partition that the function induces; hence
$I(X^{(2)};U\mid\Zt)=I(X^{(4)};U\mid\Zt)=I(\Zbar;U\mid\Zt)$ exactly. At $D_0$,
Lemma~\ref{lem:full-rank}(b) bounds this common value from below by $n(1-e^{-225n/32})$, and by
Lemma~\ref{lem:directions} the common model information is at least as large.

(iii) This is the bound for the mean learner proved in Appendix~\ref{app:growth}.

(iv) At $D_0$, $\sigma_iX^{(4)}_i=\ip{\thh^{(4)}}{\Ztr_i-\Zgh_i}$, and the ghost is independent of
$(\thh^{(4)},\Ztr_i)$ with mean zero. So the signed mean is
\[
 M^{(4)}:=\sum_i\E_0\sigma_iX^{(4)}_i=\E_0\ip{T_4(\Zbar)}{n\Zbar}=n\,\E_0\norm\Zbar^{4/3}.
\]
Expand $\norm{n\Zbar}^2=n+2\sum_{i<j}\ip{\Ztr_i}{\Ztr_j}$. The cross terms are uncorrelated, with
mean zero and $\E_0\ip{\Ztr_i}{\Ztr_j}^2=1/d$, so $\E_0\norm\Zbar^2=1/n$ and
$\E_0\norm\Zbar^4=n^{-4}\big(n^2+2n(n-1)/d\big)\le n^{-2}(1+2/d)$. Cauchy--Schwarz,
$\E_0\norm\Zbar^2=\E_0[\norm\Zbar^{2/3}\norm\Zbar^{4/3}]\le(\E_0\norm\Zbar^{4/3}\,\E_0\norm\Zbar^{8/3})^{1/2}$,
and Jensen's inequality, $\E_0\norm\Zbar^{8/3}\le(\E_0\norm\Zbar^4)^{2/3}$, give
$\E_0\norm\Zbar^{4/3}\ge n^{-2/3}(1+2/d)^{-2/3}$, hence $M^{(4)}\ge n^{1/3}(1+2/d)^{-2/3}$.

For the energy, $V^{(4)}:=\sum_i\E_0(X^{(4)}_i)^2=\E_0\norm{\Delta T_4(\Zbar)}^2
\le\E_0\big[\opnorm\Delta^2\norm\Zbar^{2/3}\big]$. The weight $\norm\Zbar^{2/3}$ is at most one, and
$\E_0\norm\Zbar^{2/3}\le(\E_0\norm\Zbar^2)^{1/3}=n^{-1/3}$ by Jensen's inequality.
Lemma~\ref{lem:weighted} with the pair-matrix tail of Lemma~\ref{lem:matrix} gives
\[
 V^{(4)}\le\Tb_1(n^{-1/3};2K_0,2K_1)=n^{-1/3}\big\{2K_0+2K_1\big[1+\tfrac13\ln n\big]\big\}<4.2\,n^{-1/3},
\]
because for $n\ge16$ and $d\ge32n$, $2K_0<4.062$ and $2K_1\le0.282/n$, and
$4.062+0.282\,(1+\tfrac13\ln16)/16<4.1$.

The fair-bit inequality \eqref{eq:fairbit} for each coordinate $X^{(4)}_i+\tau G_i$, whose signed mean
is $\E_0\sigma_iX^{(4)}_i$ and whose second moment is $\E_0(X^{(4)}_i)^2+\tau^2$, followed by
Cauchy--Schwarz over $i$ as in Appendix~\ref{app:fairbit}, gives
$\Ii_{D_0}(X^{(4)}+\tau G)\ge(M^{(4)})^2/(2\ln2\,(V^{(4)}+n\tau^2))$. This is the lower bound of
Theorem~\ref{thm:transfer} at the single law $D_0$, and with the two estimates above it is
\eqref{eq:separation-lower}.

(v) At $\tau=n^{-2/3}$, $n\tau^2=n^{-1/3}$. By~(iii) the first release has $\Is_{D_p}\le3n^{1/3}$ at
every $p$. By~(iv) the second has $\Ii_{D_0}\ge n(1+2/d)^{-4/3}/(2\ln2\cdot5.2)\ge n/8$, since
$d\ge32n\ge512$. The chain \eqref{eq:hierarchy} transfers the upper bound to $\Ii$ and $\Ij$, and the
lower bound to $\Ij$ and $\Is$.
\end{proof}

\paragraph{Both Learners Are Traceable at the Zero-Mean Law.}
Part~(ii) concerns every tracer. At $D_0$ both learners are in fact traceable, by tracers with the
same decisions. Let $\xi\in(0,1)$ and $d\ge32n\ln(2/\xi)$. For
$\thh^{(2)}$, take the tracer that flags $z$ when $\ip\theta z\ge1/(2n)$, a centred linear score since
$\mu_0=0$; for $\thh^{(4)}$, the tracer that flags $z$ when $\norm\theta^2\ip\theta z\ge1/(2n)$. As
$\norm{T_4(t)}^2\ip{T_4(t)}z=\ip tz$, the two tracers make the same decisions. We show below that at
$D_0$ their false-positive rate is at most $\xi$ and their expected training recall is at least
$n(1-\xi)$. For $\xi<1/4$, Lemma~\ref{lem:sign} therefore applies to both learners. For $\thh^{(2)}$ take
$\chi(t;\theta)=t-1/(2n)$. For $\thh^{(4)}$ take $\chi(t;\theta)=\norm\theta^2t-1/(2n)$ for
$\theta\ne0$ and $\chi(t;0)=t-1/(2n)$; this is strictly increasing in $t$ and changes no decision,
since at $\theta=0$ the margin is zero and the point is not flagged under either rule. The lemma gives
$\Ii_{D_0}(X^{(r)})\ge n\J(1-4\xi)$ for $r=2$ and $r=4$.

To prove the two bounds, note that at $D_0$ the coordinates $\sqrt d\,Z_{U_i,i,k}$ are independent fair
signs. For each $k$, $\sqrt d\,(n\Zbar)_k$ is a sum of $n$ of them, and these sums are independent
across $k$. Hoeffding's lemma gives $\E_0\exp\{y\sqrt d\,(n\Zbar)_k\}\le e^{ny^2/2}$ for every real
$y$. Taking $y=t\sqrt{3/(4n)}$ and averaging over a standard Gaussian $t$ gives
$\E_0\exp\{3(\sqrt d\,(n\Zbar)_k)^2/(8n)\}\le\E e^{3t^2/8}=2$, and Markov's inequality for
$\exp\{\tfrac3{8n}\sum_k(\sqrt d\,(n\Zbar)_k)^2\}$ yields
\[
 \Pp_0\big(\norm\Zbar^2\ge4/n\big)=\Pp_0\Big(\sum_k\big(\sqrt d\,(n\Zbar)_k\big)^2\ge4nd\Big)
 \le e^{-3d/2}\,2^d\le e^{-4d/5},
\]
and in the same way, with $n-1$ signs per coordinate,
$\Pp_0(\norm{n\Zbar-\Ztr_i}^2\ge4(n-1))\le e^{-4d/5}$. A fresh point $Z'\sim D_0$ is independent of
$\Zbar$, and given $\Zbar$ the score $\ip\Zbar{Z'}$ is a sum of independent fair signs with
coefficients $\Zbar_k/\sqrt d$, whose squares sum to $\norm\Zbar^2/d$. Hoeffding's inequality gives
$\Pp_0(\ip\Zbar{Z'}\ge1/(2n)\mid\Zbar)\le\exp\{-d/(8n^2\norm\Zbar^2)\}$, which is at most
$e^{-d/(32n)}$ when $\norm\Zbar^2<4/n$. So the false-positive rate is at most
$e^{-d/(32n)}+e^{-4d/5}\le\xi$, because $d\ge32n\ln(2/\xi)$ makes each term at most $\xi/2$. For a
training point, $\ip\Zbar{\Ztr_i}=\big(1+\ip{n\Zbar-\Ztr_i}{\Ztr_i}\big)/n$, and $n\Zbar-\Ztr_i$ is
independent of $\Ztr_i$. Hoeffding's inequality, conditionally on $n\Zbar-\Ztr_i$, gives
\[
 \Pp_0\big(\ip{n\Zbar-\Ztr_i}{\Ztr_i}<-\tfrac12\big)\le\E_0\exp\{-d/(8\norm{n\Zbar-\Ztr_i}^2)\}
 \le e^{-d/(32(n-1))}+e^{-4d/5}\le\xi .
\]
So each training point is flagged with probability at least $1-\xi$, and the expected training recall
is at least $n(1-\xi)$.

\subsection{The Model-Level Converse}\label{app:converse}

Theorem~\ref{thm:main} also bounds the model information, through the release. For every learner of
Theorem~\ref{thm:main}, in the dimension $d\ge\Cdim\max\{n,\neps\}$ of that theorem,
\[
 c_1n\ \le\ \sup_{p\in[-1,1]^d}I_{D_p}(\thh;U\mid\Zt)\ \le\ n .
\]
Indeed, Lemma~\ref{lem:directions} at $\tau=0$ gives $\Is_{D_p}(X)\le I_{D_p}(\thh;U\mid\Zt)$ at every
$p$, and the lower half of \eqref{eq:main} at $\tau=0$ with $\Qc=\Is$ gives
$\sup_p\Is_{D_p}(X)\ge c_1n$; the upper bound is $I(\thh;U\mid\Zt)\le\Hent(U\mid\Zt)=n$. As in
Theorem~\ref{thm:main}, the lower bound holds on average over the beta prior of
Section~\ref{sec:architecture}. So accuracy on the family forces model information of order $n$, the
order of \eqref{eq:headline} at $\tau=0$. The individual-sample form of the model information has the
same order: $X_i$ is a function of $\thh$ and $(Z_{0,i},Z_{1,i})$, and $U_i$ is independent of these two
candidates, so $I_{D_p}(X_i;U_i)\le I_{D_p}(\thh;U_i\mid Z_{0,i},Z_{1,i})$ for each $i$, and the lower
half of \eqref{eq:main} at $\tau=0$ with $\Qc=\Ii$ gives
$\sup_p\sum_iI_{D_p}(\thh;U_i\mid Z_{0,i},Z_{1,i})\ge c_1n$. The implication runs from the release to
the model only: by Proposition~\ref{prop:separation}, model information does not determine the
information in noisy loss differences.

\paragraph{Comparison with \citet{attias2024}.}
For $r=2$, where $n\asymp\neps=\eps^{-1}$, these lower bounds have the order $\Omega(\eps^{-1})$ that
\citet[Theorem~4.2 and Corollary~5.7]{attias2024} prove under other conditions. They consider the loss
$f_2(\theta,z)=\norm\theta^2/2-\ip\theta z$ on $\R^d$ and learners whose excess risk, at every law on
the scaled sign cube, is at most $\eps$ with probability at least $1-\delta'$, where
$\delta'=O(n^{-2})$; their bounds hold at every sample size at least the sample complexity of the
learner, in dimension $d\ge\Omega(n^2\log n)$. The bounds above need accuracy only in expectation and
only at the product laws, and dimension $d\ge\Cdim\max\{n,\neps\}$, but they hold only for
$n\asymp\neps$ and for proper learners. For $f_2$ the restriction to proper learners costs nothing.
The minimizer $\mu_p$ of $F_p$ over $\R^d$ lies in $W$ and $F_p(\theta)-F_p(\mu_p)=\norm{\theta-\mu_p}^2/2$,
so projecting the output onto $W$ does not increase its excess risk, and by data processing the
projection carries no more information about $U$ than the output. As the excess risk of $f_2$ on $W$
is at most $2$, the projection of a learner that is $\eps$-accurate with probability at least
$1-\delta'$ is $(\eps+2\delta')$-accurate in expectation, and a lower bound for the projection holds
for the learner itself. For $r>2$ the order is $\eps^{-2+2/r}$, between the orders $\eps^{-1}$ and
$\eps^{-2}$ of the strongly convex and Lipschitz cases of \citet{attias2024}; their Lipschitz case,
Theorem~4.1, also holds for improper learners \citep[Section~6.2]{attias2024}, whereas the projection
argument above covers improper learners only for $f_2$. Dimension of order
$n/\log n$ is needed for a lower bound of this kind that holds for every accurate proper learner: the
rounded learner of Appendix~\ref{app:dimension} is $2\eps$-accurate on the family and has
$I(\thh';U\mid\Zt)\le d\log_2(1+2(L+L_g)/\eps)$ at every law.

%% file: app_I_measurements.tex
\section{Measurements}\label{app:measurements}

This appendix describes the measurements of Section~\ref{sec:measurements} and proves the exact moments
of ridge ERM that they use (Lemma~\ref{lem:ridge-moments}). It also measures a learner with $r=4$ on the
same supersamples (Appendix~\ref{app:quartic}). The supplementary archive contains the
code, the data and the commands that reproduce every number below; the code uses only the Python
standard library.

\subsection{Learners}\label{app:learners}

Both learners of Section~\ref{sec:measurements} use $g(\theta)=\norm\theta^2/2$, so $r=2$ and
$a_-=a_+=1/2$ in Assumption~\ref{ass:class}.

\paragraph{Ridge ERM.}
With the affine link $\phi(t)=-t$ ($b=L=1$, $H=0$), the empirical risk of
$f(\theta,z)=\norm\theta^2/2-\ip\theta z$ is minimized over $W$ by the training mean
$\thh=\Zbar=n^{-1}\sum_iZ_{U_i,i}$, at a cost of $O(nd)$ operations. This is also one-pass stochastic
gradient descent with steps $1/t$ from any start in $W$: the step
$\theta_t=\theta_{t-1}-t^{-1}(\theta_{t-1}-z_t)$ sets $\theta_1=z_1$ and keeps $\theta_t$ equal to the
mean of $z_1,\dots,z_t$, which lies in $W$, so the projection never acts. As
$F_p(\theta)=\norm{\theta-\mu_p}^2/2-\norm{\mu_p}^2/2$, the mean excess risk at $D_p$ is
$\E_p\norm{\Zbar-\mu_p}^2/2=(1-\norm{\mu_p}^2)/(2n)$, so $\eps=1/(2n)$ is the learner's exact
worst-case excess risk on the family.

\begin{lemma}[Exact moments of ridge ERM]\label{lem:ridge-moments}
Let $\phi(t)=-t$, $g(\theta)=\norm\theta^2/2$ and $\thh=\Zbar$. Let $z_c=(2n)^{-1}\sum_{j,i}Z_{j,i}$ be the
candidate centre, let $x(\Zt)\in\R^n$ have coordinates $x_i(\Zt)=-\ip{z_c}{\Delta_i}$, and let
$\Gamma_p=d^{-1}\sum_kp_k^2(1-p_k^2)$. For all $n,d\ge1$, every $p\in[-1,1]^d$ and every $i\in[n]$:
\begin{enumerate}
\item[(a)] $X_i=x_i(\Zt)+(2n)^{-1}\sum_j\sigma_j\ip{\Delta_j}{\Delta_i}$, so $x(\Zt)=\E_p[X\mid\Zt]$;
\item[(b)] $\E_p\sigma_iX_i=(1-\norm{\mu_p}^2)/n$, so $M_p=n\gen_p=1-\norm{\mu_p}^2$;
\item[(c)] $V_{\mathrm{res},p}=\E_p\norm{X-x(\Zt)}^2$ satisfies
 $nV_{\mathrm{res},p}=(1-\norm{\mu_p}^2)^2+d^{-1}(1-\norm{\mu_p}^2+\Gamma_p)
 +d^{-1}(n-1)(1-\norm{\mu_p}^2-\Gamma_p)$;
\item[(d)] $\E_pX_i^2=V_p/n$, and
 $n(V_p-V_{\mathrm{res},p})=d^{-1}(n-1)(1-\norm{\mu_p}^2)+d^{-1}(n-1)(2n-3)\Gamma_p$.
\end{enumerate}
\end{lemma}

\begin{proof}
Under $D_p$ the $2nd$ coordinates $Z_{j,i,k}$ are independent, with $Z_{j,i,k}^2=1/d$ and mean
$p_k/\sqrt d$. So $\Delta_1,\dots,\Delta_n$ are independent, with independent centred coordinates,
$\E_p\Delta_{i,k}^2=2(1-p_k^2)/d$ and $\E_p\norm{\Delta_i}^2=2(1-\norm{\mu_p}^2)$; the signs $\sigma_j$
are fair, independent of each other and of $\Zt$.
(a)~$X_i=-\ip{\Zbar}{\Delta_i}$, and $\Zbar=z_c-(2n)^{-1}\sum_j\sigma_j\Delta_j$ since
$Z_{U_j,j}=(Z_{0,j}+Z_{1,j}-\sigma_j\Delta_j)/2$.
(b)~By (a), $\E_p\sigma_iX_i=(2n)^{-1}\E_p\norm{\Delta_i}^2$; sum over $i$ and use \eqref{eq:gap}.
(c)~By (a), $\E_p(X_i-x_i(\Zt))^2=(4n^2)^{-1}\sum_j\E_p\ip{\Delta_j}{\Delta_i}^2$. For $j\ne i$ the term
is $\sum_k(\E_p\Delta_{i,k}^2)^2=4d^{-2}\sum_k(1-p_k^2)^2=4(1-\norm{\mu_p}^2-\Gamma_p)/d$. For $j=i$,
$\Delta_{i,k}^2=4B_k/d$, where the indicators $B_k$ that the two candidates differ in coordinate $k$
are independent with mean $(1-p_k^2)/2$ and variance $(1-p_k^4)/4$, so
$\E_p\norm{\Delta_i}^4=16d^{-2}[(\sum_k(1-p_k^2)/2)^2+\sum_k(1-p_k^4)/4]
=4(1-\norm{\mu_p}^2)^2+4(1-\norm{\mu_p}^2+\Gamma_p)/d$. Sum over $j$ and $i$.
(d)~By (a), $\E_pX_i^2=\E_p(X_i-x_i(\Zt))^2+\E_px_i(\Zt)^2$. As $\norm{Z_{0,i}}=\norm{Z_{1,i}}$,
$\ip{Z_{0,i}+Z_{1,i}}{\Delta_i}=0$, so $x_i(\Zt)=-(2n)^{-1}\ip{R_i}{\Delta_i}$, where $R_i$, the sum of
the $N=2(n-1)$ candidates outside pair $i$, is independent of $\Delta_i$. With
$\E_pR_{i,k}^2=N/d+N(N-1)p_k^2/d$ this gives
$\E_px_i(\Zt)^2=(n-1)[1-\norm{\mu_p}^2+(2n-3)\Gamma_p]/(n^2d)$. Neither term depends on $i$; summing
over $i$ gives (d).
\end{proof}

\paragraph{Budgets of Ridge ERM.}
For a prior with i.i.d.\ coordinates and $\E p_k^2=v$, we have $\E\Gamma_p=v-\E p_k^4$ and
$\E(1-\norm{\mu_p}^2)^2=1-2v+(1-1/d)v^2+\E p_k^4/d$, with $\E p_k^4=3v^2/(1+2v)$ for $\Pi_\kappa$
\eqref{eq:prior-moments}. Averaging Lemma~\ref{lem:ridge-moments} thus gives the budgets of
Theorem~\ref{thm:transfer} in closed form: $M=1-v$, $V=\E_\Pi V_p$ and $\Vres=\E_\Pi V_{\mathrm{res},p}$,
whose reference $\E_p[X\mid\Zt]$ minimizes the residual at every $p$. In the runs, 36 of 37
comparisons of a Monte Carlo average (signed mean, energy, residual, $\norm\thh^2$ or excess risk) with
its exact value fall inside the 95\% interval; the other, an excess risk, deviates by 0.13\%.

\paragraph{Logistic Ridge ERM.}
The link $\phi(t)=\ln(1+e^{-t})$ has $b=1/2$, $L=1/(1+e^{-1})$ and $H=1/4$. The learner minimizes the
$1$-strongly convex training risk $\widehat F_{S_n}(\theta)=\norm\theta^2/2+n^{-1}\sum_j\phi(\ip\theta{z_j})$
over $\R^d$, with $z_1,\dots,z_n$ the training points. Its minimizer is $\thh=n^{-1}\sum_jc_jz_j$ with
$c_j=1/(1+e^{\ip\thh{z_j}})\in(0,1)$, so $\norm\thh<1$ and $\thh$ is also the minimizer over $W$. The
solver iterates $c_j\leftarrow1/(1+\exp((Kc)_j/n))$ from $c_j=1/2$, with the Gram matrix
$K_{jk}=\ip{z_j}{z_k}$ computed exactly and stored in double precision; since $|K_{jk}|\le1$, this map
is a contraction with constant at most $1/4$ in the maximum norm. It stops when the largest change $\eta$ of a coefficient
is at most $10^{-13}$. The gradient of $\widehat F_{S_n}$ at $n^{-1}\sum_jc_jz_j$ is then
$n^{-1}\sum_j(c_j-c'_j)z_j$, with $c'$ the next iterate, of norm at most $\eta$; by strong convexity every
margin is within $\eta$ of its exact value. The reference $x(\Zt)$ for the residual is the release of the
minimizer of $\norm\theta^2/2$ plus the average of $\phi(\ip\theta z)$ over all $2n$ candidates,
computed in the same way; it depends only on the candidates. Over all trials the largest final change $\eta$ is $1.4\cdot10^{-15}$, after at most 10
iterations. A trial costs $O(n^2d)$ operations for the Gram matrices and $O(n^2)$ per iteration.
Independent solvers, gradient descent with explicit vectors at small sizes and Newton's method on
three trials of the run with $n=50$, reproduce the simulator's loss differences to within rounding
error.

\subsection{Setup}\label{app:setup}

\paragraph{Prior, Trials and Noise.}
Every run takes $\eps=1/(2n)$ and the beta prior $\Pi_\kappa$ with $v=16\eps=8/n$, so $\kappa=n/8-1$.
For ridge ERM, $t_0=1/2$ and $\beta=1/4$, so this is the prior of Section~\ref{sec:architecture} at the
learner's exact worst-case excess risk. For logistic ridge ERM, $\eps$ is only the prior's parameter:
Theorem~\ref{thm:transfer} needs only a prior fixed in advance, and the learner's accuracy is not used.
Each trial draws $B_k\sim\mathrm{Beta}(\kappa/2,\kappa/2)$ independently and sets $(1+p_k)/2$ to $B_k$
rounded to the nearest multiple of $2^{-12}$ in $[2^{-12},1-2^{-12}]$. The data law is exactly $D_p$ at
the rounded $p$, so the prior of the runs, $\Pi$ below, is the image of $\Pi_\kappa$ under this
rounding. The rounding changes $\E p_k^2$ and $\E p_k^4$ by a relative amount of at most
$1.5\cdot10^{-5}$; the closed-form budgets use the moments of $\Pi_\kappa$. A trial then draws the $2n$
candidates and $U$, trains the learner and records the release. Coordinate $k$ of a candidate is
$+1/\sqrt d$ exactly when an independent uniform $12$-bit integer is below $2^{12}(1+p_k)/2$. For ridge
ERM the inner products are computed in exact integer arithmetic, which a direct computation at small
sizes confirms. The noise $\tau G$ is drawn in the analysis, at $\tau=t\eps$ for the 16 values
$t\in\{0,0.03,0.06,0.1,0.2,0.3,0.6,1,2,3,6,10,20,30,60,100\}$.

\paragraph{Dimension.}
Theorem~\ref{thm:main} assumes $d\ge\Cdim\max\{n,\neps\}$, with $\Cdim=32$ since $b\le4$ for both links,
and $\neps=\eps^{-1}=2n$ since $r=2$; so it assumes $d\ge64n$, which the runs at $d=64n$ meet. Its proof
uses the dimension only through $d\ge32\max\{n,1/v\}$ \eqref{eq:dim-condition}, which for ridge ERM,
with $1/v=n/8$, reads $d\ge32n$: the runs at $d=32n$ meet this condition but not the one stated in
Theorem~\ref{thm:main}. Theorem~\ref{thm:transfer}, on which the comparisons rest, holds in every
dimension.

\paragraph{Runs, Seeds and Runtimes.}
Table~\ref{tab:setup} lists the runs. The trial numbered $t=0,1,\dots$ of a run has the seed $1000003$
times the run's seed plus $t$, and four Mersenne Twister generators of Python's \texttt{random}
module, seeded with four times the trial's seed plus $0$, $1$, $2$ and $3$, draw the law, the
candidates, the bits and the noise of a membership test that some runs also simulate; the analysis
draws $G$ from a generator with a fixed seed. Every run, or chunk of a run, is one single-threaded
process under Python~3.12.13, timed by wall clock. All runs were made on an Apple M1 with 8 cores and
8~GB of memory. Up to four processes ran at a time for the runs of ridge ERM with $n\in\{1000,3000\}$ and
of logistic ridge ERM with $n\in\{300,1000\}$ at $d=64n$, and one at a time for the other runs.

\begin{table}[t]
\centering
\caption{The Runs. Runtimes marked $\dagger$ include a membership test that this paper does not use;
it draws 2000 fresh points per trial after the release is recorded. The runtime marked $\ddagger$ is
the sum of four chunks of 50 trials, run in parallel in 741.2 to 742.5~s each.}\label{tab:setup}
\small
\begin{tabular}{@{}lrrrrrr@{}}
\toprule
Learner, dimension & $n$ & $d$ & Trials & Pairs & Seed & Runtime (s)\\
\midrule
ridge ERM, $d=32n$ & 100 & 3200 & 2000 & 200000 & 10100 & $104.7^\dagger$\\
 & 300 & 9600 & 600 & 180000 & 10300 & $85.8^\dagger$\\
 & 1000 & 32000 & 200 & 200000 & 11000 & $123.3^\dagger$\\
 & 3000 & 96000 & 40 & 120000 & 13000 & $139.6^\dagger$\\
\addlinespace[2pt]
ridge ERM, $d=64n$ & 50 & 3200 & 4000 & 200000 & 17050 & $198.1^\dagger$\\
 & 100 & 6400 & 2000 & 200000 & 17100 & $188.7^\dagger$\\
 & 300 & 19200 & 300 & 90000 & 10364 & $76.7^\dagger$\\
 & 1000 & 64000 & 200 & 200000 & 471000 & $202.8$\\
 & 3000 & 192000 & 40 & 120000 & 473000 & $295.8$\\
\addlinespace[2pt]
logistic ridge ERM, $d=64n$ & 50 & 3200 & 4000 & 200000 & 27050 & $296.1^\dagger$\\
 & 100 & 6400 & 2000 & 200000 & 27100 & $373.6^\dagger$\\
 & 300 & 19200 & 600 & 180000 & 37300 & $331.4$\\
 & 1000 & 64000 & 200 & 200000 & 371000 & $2968.1^\ddagger$\\
\addlinespace[2pt]
ridge ERM, low dimension & 1000 & 8 & 400 & 400000 & 30008 & $15.1^\dagger$\\
 & 1000 & 32 & 400 & 400000 & 30032 & $17.7^\dagger$\\
\bottomrule
\end{tabular}
\end{table}

\subsection{Estimators}\label{app:estimators}

Write $Y_i=X_i+\tau G_i$. The information per pair, in bits, is
$\iota=n^{-1}\E_\Pi\Ii_{D_p}(X+\tau G)=n^{-1}\sum_i\E_\Pi I_{D_p}(Y_i;U_i)$, a function of $\tau$ with
value $\iota(0)$ at $\tau=0$.

\paragraph{Variational Lower Bound on $\iota$.}
As $U_i$ is a fair bit independent of $p$, $\E_\Pi I_{D_p}(Y_i;U_i)=I(Y_i;U_i\mid p)\ge
1+\E_\Pi\log_2\widehat P(U_i\mid Y_i,p)$ for every conditional probability $\widehat P$. The left side
is the prior average of the per-law information, which Theorem~\ref{thm:transfer}(a) bounds. Because
$\widehat P$ may depend on $p$, the estimate bounds this average rather than the smaller information of
the prior mixture, $I(Y_i;U_i)\le I(Y_i,p;U_i)=I(Y_i;U_i\mid p)$. We take $\widehat P$ with log-odds of
$U_i=0$ equal to $a_1y+a_3y^3$, an odd cubic in $y=Y_i/c_p$, where $c_p^2=V_p/n+\tau^2$ for ridge ERM and
$c_p^2=V_p/(16n)+\tau^2$ for logistic ridge ERM, with $V_p$ ridge ERM's exact energy at the drawn law.
The two coefficients are fitted by
maximum likelihood (Newton's method with step halving) on the even-numbered trials and the bound is evaluated on the odd-numbered trials,
and vice versa; each fitted model is independent of the data on which it is evaluated, so the estimate
has expectation at most $\iota$, and it can be negative where $\iota$ is small. The pairs of one trial
share the law and the trained model, so every nominal 95\% interval treats trials as clusters: over $T$ trials
it is the mean of the per-trial means plus or minus their standard error times the $0.975$ quantile of
Student's $t$ distribution with the largest tabulated number of degrees of freedom not above $T-1$,
or $1.96$ when $T-1>120$. A binned plug-in estimate, which is not a bound, differs from the variational one by at
most 0.92\% wherever $\iota\ge0.1$ at $d=64n$.

\paragraph{Exact Second-Moment Bound.}
For ridge ERM, Lemma~\ref{lem:ridge-moments} gives $\E_p\sigma_iY_i=(1-\norm{\mu_p}^2)/n$ and
$\E_pY_i^2=(V_p+n\tau^2)/n$ at each drawn law, so by the fair-bit inequality \eqref{eq:fairbit}
$I_{D_p}(Y_i;U_i)\ge(1-\norm{\mu_p}^2)^2/(2\ln2\,n(V_p+n\tau^2))$; the average over the drawn laws
estimates this proved lower bound on $\iota$ without any fitting.

\paragraph{Candidate-Centred Lower Bound on $\Is$.}
With the learner's reference, let $\widetilde Y_i=(Y_i-x_i(\Zt))/(\norm{\Delta_i}^2/(2n))$, or $0$ if
$\Delta_i=0$; for ridge ERM, $\widetilde Y_i=\sigma_i+(\sum_{j\ne i}\sigma_j\ip{\Delta_j}{\Delta_i}
+2n\tau G_i)/\norm{\Delta_i}^2$ by Lemma~\ref{lem:ridge-moments}(a). The same cross-fitted bound, with a
model in $\widetilde Y_i$ alone, estimates a lower bound on $I(\widetilde Y_i;U_i)$. As $\widetilde Y_i$
is a function of $(Y_i,\Zt)$ and $U_i$ is independent of $(\Zt,p)$,
$I(\widetilde Y_i;U_i)\le I(Y_i,\Zt,p;U_i)=\E_\Pi I_{D_p}(Y_i;U_i\mid\Zt)$, whose sum over $i$ is at
most $\E_\Pi\Is_{D_p}(Y)$ (Appendix~\ref{app:hierarchy}).

\paragraph{Theorem~\ref{thm:transfer} at the Learner's Own Budgets.}
Let $M=\sum_i\E_\Pi\sigma_iX_i$, $V=\sum_i\E_\Pi X_i^2$ and $\Vres=\E_\Pi\norm{X-x(\Zt)}^2$. For $\tau>0$,
\[
 \frac{M^2}{2\ln2\,n(V+n\tau^2)}\ \le\ \iota\ \le\ \frac1n\E_\Pi\Is_{D_p}(X+\tau G)
 \ \le\ \iota^+(\Vres),\qquad \iota^+(y)=\min\Big\{1,\frac{\ln(1+y/(n\tau^2))}{2\ln2}\Big\},
\]
and at $\tau=0$ the same holds with $1$ on the right. The first inequality is \eqref{eq:lower} divided by
$n$, the second is \eqref{eq:hierarchy}, and for the third, \eqref{eq:upper} at each $p$ bounds
$n^{-1}\Is_{D_p}(X+\tau G)$ by $\iota^+(V_{\mathrm{res},p})$, where
$V_{\mathrm{res},p}=\E_p\norm{X-x(\Zt)}^2$; as $\iota^+$ is concave and nondecreasing, Jensen's
inequality gives $\E_\Pi\iota^+(V_{\mathrm{res},p})\le\iota^+(\Vres)$. For ridge ERM the budgets are
the closed forms above, and $M/n=2\eps(1-16\eps)$. For logistic ridge ERM they are averages over the
trials, with 95\% intervals of relative half-width at most 0.47\%, so its bracket is itself an estimate.
At a given run and noise level, $\iota$ lies in the bracket up to its 95\% intervals when the upper end
of its interval is at least the lower half and the lower ends of the intervals of $\iota$ and of the
candidate-centred estimate are at most the upper half.

\paragraph{Shape.}
The fit of $c_{\mathrm{fit}}/(1+(\tau/(a_{\mathrm{fit}}M/n))^2)$, with $a_{\mathrm{fit}}$ on the grid
$10^{k/100}$, minimizes the squared error in $\log\iota$ over the levels $\tau>0$ whose 95\% interval
lies above $0$. The tail slope is the least-squares slope of $\log\iota$ against $\log\tau$ over the
levels with $\tau\ge5M/n$ and $\iota\ge10^{-3}$.

\subsection{Results}\label{app:results}

\begin{table}[t]
\centering
\caption{Results per Run. $\iota(0)$: bits per pair at $\tau=0$, with its 95\% interval. Lower$/\iota$ and
upper$/\iota$: smallest ratio of the lower half and largest ratio of the upper half of
Theorem~\ref{thm:transfer}, at the learner's own budgets, to $\iota$, over the noise levels with
$\iota\ge0.01$. In: noise levels, of 16, at which $\iota$ lies in the bracket up to its 95\% intervals.
Slope, $c_{\mathrm{fit}}$, $a_{\mathrm{fit}}$: tail slope and fitted constants
(Appendix~\ref{app:estimators}). The runs at $d=32n$ do not meet the dimension condition $d\ge64n$ of
Theorem~\ref{thm:main} for these learners; Theorem~\ref{thm:transfer} holds in every dimension
(Appendix~\ref{app:setup}).}\label{tab:bracket}
\small
\begin{tabular}{@{}lrcccrrcc@{}}
\toprule
Learner, dimension & $n$ & $\iota(0)$ [95\% interval] & Lower$/\iota$ & Upper$/\iota$ & In & Slope
 & $c_{\mathrm{fit}}$ & $a_{\mathrm{fit}}$\\
\midrule
ridge ERM, $d=32n$ & 100 & 0.7049 [0.7017, 0.7080] & 0.672 & 1.822 & 16 & $-1.991$ & 0.713 & 1.000\\
 & 300 & 0.6924 [0.6890, 0.6959] & 0.673 & 1.797 & 16 & $-1.978$ & 0.697 & 1.023\\
 & 1000 & 0.6865 [0.6831, 0.6899] & 0.675 & 1.801 & 16 & $-1.875$ & 0.696 & 1.000\\
 & 3000 & 0.6849 [0.6806, 0.6891] & 0.675 & 1.802 & 16 & $-1.985$ & 0.690 & 1.000\\
\addlinespace[2pt]
ridge ERM, $d=64n$ & 50 & 0.9133 [0.9113, 0.9154] & 0.632 & 1.476 & 16 & $-1.936$ & 0.929 & 0.871\\
 & 100 & 0.9063 [0.9042, 0.9084] & 0.631 & 1.454 & 16 & $-1.986$ & 0.924 & 0.871\\
 & 300 & 0.8965 [0.8935, 0.8996] & 0.632 & 1.431 & 16 & $-2.164$ & 0.914 & 0.891\\
 & 1000 & 0.8910 [0.8885, 0.8935] & 0.633 & 1.424 & 16 & $-1.956$ & 0.905 & 0.912\\
 & 3000 & 0.8881 [0.8851, 0.8910] & 0.634 & 1.429 & 15 & $-1.953$ & 0.921 & 0.832\\
\addlinespace[2pt]
logistic ridge ERM, $d=64n$ & 50 & 0.9151 [0.9131, 0.9172] & 0.631 & 1.409 & 15 & $-1.989$ & 0.963 & 0.832\\
 & 100 & 0.9047 [0.9025, 0.9069] & 0.633 & 1.400 & 15 & $-1.881$ & 0.926 & 0.871\\
 & 300 & 0.8933 [0.8909, 0.8957] & 0.635 & 1.478 & 16 & $-1.964$ & 0.926 & 0.871\\
 & 1000 & 0.8893 [0.8871, 0.8915] & 0.634 & 1.509 & 14 & $-2.246$ & 0.905 & 0.871\\
\bottomrule
\end{tabular}
\end{table}

\paragraph{The Tradeoff at $d=64n$.}
In the nine runs at $d=64n$, $\iota(0)=0.888$ to $0.915$ bits per pair, with 95\% half-widths of at
most $0.0031$, and it decreases slowly with $n$ for both learners (Table~\ref{tab:bracket}). The
per-pair signal $M/n$ is $1.68\eps$ to $1.99\eps$ for ridge ERM and $0.39\eps$ to $0.49\eps$ for logistic
ridge ERM, whose loss differences are $4.02$ to $4.36$ times smaller. Against $\tau/(M/n)$ the curves
nearly coincide (Figure~\ref{fig:measurements}(a)): the fits leave root-mean-square residuals of
$0.05$ to $0.16$ in $\log\iota$, and the tail slopes, $-2.25$ to $-1.88$, are close to the $-2$ of
$n(\eps/\tau)^2$. Figure~\ref{fig:measurements}(a) omits the nine points whose 95\% interval reaches $0$,
all of logistic ridge ERM at $\tau\ge30\eps$.

\paragraph{The Bracket.}
At $\tau=0$ the lower half is $0.631$ to $0.635$ times $\iota(0)$, and the upper half is $1$. Over the
100 points at $d=64n$ with $\iota\ge0.01$, the lower half is at least $0.631$ times and the upper half at
most $1.509$ times $\iota$, and $\iota$ lies in the bracket up to its 95\% intervals at 139 of all 144
points. The five exceptions lie in the tail, at $\tau\ge13.5\,M/n$, all below the lower half and with
$\iota\le0.0034$: ridge ERM at $n=3000$, $\tau=100\eps$, and logistic ridge ERM at $n=50$,
$\tau=30\eps$, at $n=100$, $\tau=6\eps$, and at $n=1000$, $\tau\in\{10\eps,100\eps\}$. In the tail
the two halves nearly meet, since their ratio tends to $n\Vres/M^2=1.008$ to $1.018$, so a small
shortfall of a lower-bound estimate leaves the bracket.

\paragraph{The Tradeoff at $d=32n$.}
The four ridge runs at $d=32n$ have the same shape at a lower level, with $\iota(0)=0.685$ to $0.705$
(half-widths at most $0.0043$), and all their 64 points lie in the bracket (Table~\ref{tab:bracket}).
The plateau and the lower half at $\tau=0$ are lower than at $d=64n$, as is $M^2/(nV)$: $0.641$ to
$0.657$, against $0.781$ to $0.801$.

\paragraph{Generalization (Figure~\ref{fig:measurements}(b)).}
For ridge ERM, Lemma~\ref{lem:ridge-moments}(b) gives $\E_\Pi\gen_p=(1-v)/n=(1-8/n)/n$ exactly. The
uniform bound $4L/\sqrt n$ of Corollary~\ref{cor:consequences}(i) falls from $0.566$ at $n=50$ to
$0.0730$ at $n=3000$. Cauchy--Schwarz, $n|\gen_p|\le(nV_p)^{1/2}$, and Jensen's inequality give the
information-free bound $\E_\Pi|\gen_p|\le(V/n)^{1/2}$, which is $1.12$ to $1.13$ times the gap at $d=64n$
and $1.23$ to $1.25$ times at $d=32n$. The plotted $2L\sqrt{2\ln2}\,\iota(0)$ estimates a lower bound on
the prior average of every form of the range-scaled ld-CMI bound. On $W\times\mathcal Z$ the loss
$f(\theta,z)=\norm{\theta-z}^2/2-1/2$ takes values in $[-1/2,3/2]$, an interval of length $2L=2$.
Applied to the loss $(f+1/2)/(2L)$, which takes values in $[0,1]$ and whose loss differences
$X_i/(2L)$ carry the same information as $X_i$, Theorem~3.1 of \citet{wangmao2023a} bounds
$|\gen_p|/(2L)$ by
\[
 \frac1n\sum_i\E_{\tilde z}\sqrt{2I_{D_p}(X_i;U_i\mid\Zt=\tilde z)}
 \ \le\ \frac1n\sum_i\sqrt{2I_{D_p}(X_i;U_i\mid\Zt)},
\]
with $\E_{\tilde z}$ the average over $\tilde z$ drawn as $\Zt$, and their Theorem~3.2 bounds it by
$n^{-1}\sum_i\sqrt{2I_{D_p}(X_i;U_i)}$, with information in nats; the range-scaled bounds are $2L$ times
these averages. In bits, $2I$ in nats is $2\ln2$ times $I$.
Each per-pair information is at most one bit, also given $\Zt=\tilde z$, so its square root is at
least itself, and $\E_{\tilde z}I_{D_p}(X_i;U_i\mid\Zt=\tilde z)=I_{D_p}(X_i;U_i\mid\Zt)\ge
I_{D_p}(X_i;U_i)$ as $U_i$ is independent of $\Zt$. So each form is at least
$2L\sqrt{2\ln2}\,n^{-1}\Ii_{D_p}(X)$ at every $p$, and at least $2L\sqrt{2\ln2}\,\iota(0)=2.3548\,\iota(0)$
on prior average. With the estimate of $\iota(0)$ this is $2.09$ to $2.15$ at $d=64n$, that is, $128$ to
$6290$ times the gap for $n=50$ to $3000$, and $1.61$ to $1.66$ at $d=32n$, $180$ to $4851$ times the gap
for $n=100$ to $3000$.

\paragraph{Totals (Figure~\ref{fig:totals}).}
The least-squares slopes of the totals $n\iota$ against $n$ are $0.993$, $1.009$ and $-0.017$ for ridge
ERM and $0.991$, $0.997$ and $-0.031$ for logistic ridge ERM, at $\tau=0$, $\tau=M/n$ and
$\tau=\sqrt n\,M/(2n)$ respectively. At the last level the lower half is $2.62$ to $2.88$ bits in total,
and the totals are $2.45$ to $2.94$ bits; only for logistic ridge ERM at $n=1000$ does the whole
nominal 95\% interval of the total, $[2.05,2.85]$, lie below the lower half, $2.87$, as an estimate with
expectation at most the information can (Appendix~\ref{app:estimators}). These are the orders $n$, $n$ and $1$ of
Theorem~\ref{thm:main} at $\tau=0$, $\tau\asymp\eps$ and $\tau\asymp\eps\sqrt n$, over $n=50$ to $3000$.

\input{fig_totals}

\subsection{A Learner with Quartic Regularization}\label{app:quartic}

\paragraph{The Learner.}
The loss $f_4(\theta,z)=\norm\theta^4/4-\ip\theta z$ has the affine link ($b=L=1$, $H=0$) and the
regularizer $g(\theta)=\norm\theta^4/4$, which satisfies Assumption~\ref{ass:class} with $r=4$,
$a_-=1/16$ and $a_+=1/4$ (Appendix~\ref{app:power}). Its empirical risk minimizer over $W$ is
$\thh^{(4)}=T_4(\Zbar)=\norm\Zbar^{-2/3}\Zbar$ (Appendix~\ref{app:power-accuracy}), one of the two
learners of Proposition~\ref{prop:separation}. By part~(i) of that proposition its expected excess risk
for $f_4$ is at most $\eps_4=1.2\,n^{-2/3}$ at every law on the unit ball. So it is $\eps_4$-accurate on
the family with $n=1.2^{3/2}n_{\eps_4}\approx1.31\,n_{\eps_4}$ samples, where $n_{\eps_4}=\eps_4^{-3/2}$ is
the optimal sample size for $r=4$, and the runs at $d=64n$ satisfy $d\ge\Cdim\max\{n,n_{\eps_4}\}$, with
$\Cdim=32$ as $b\le4$.

\paragraph{Its Release on the Ridge Supersamples.}
On every supersample $X^{(4)}=\norm\Zbar^{-2/3}X^{(2)}$, where $X^{(2)}$ is the release of ridge ERM
(Appendix~\ref{app:separation-proof}), and $\norm\Zbar^2=n^{-1}\sum_i\ip\Zbar{Z_{U_i,i}}$ is a sum of
inner products that the ridge runs record exactly. The five ridge runs at $d=64n$ thus give the release of
$\thh^{(4)}$ on their laws, candidates and bits, without further simulation. Their prior has
$v=8/n$; Theorem~\ref{thm:transfer} needs only a prior fixed in advance, so its two halves apply at this
learner's own budgets, which are averages over the trials, as for logistic ridge ERM. The reference for
the residual is the release of $T_4(z_c)$, the minimizer over $W$ of the empirical risk of $f_4$ on all
$2n$ candidates, $\norm\theta^4/4-\ip\theta{z_c}$; it is $x^{(4)}(\Zt)=\norm{z_c}^{-2/3}x(\Zt)$, with
$x(\Zt)$ as in Lemma~\ref{lem:ridge-moments}, and it depends only on the candidates. As $X^{(4)}$ is
not linear in the bits, this reference is not $\E_p[X^{(4)}\mid\Zt]$, and its residual may exceed the
smallest one, which the conditional mean attains.

\paragraph{Estimators.}
The estimators are those of Appendix~\ref{app:estimators}, except for the scale of the model. At the 16
noise levels of Appendix~\ref{app:setup} they use the same draws of $G$ as for ridge ERM. These levels
reach $\tau/(M/n)=33.4$ at $n=50$ but only $7.2$ at $n=3000$, so eight further levels,
$\tau=s\,n^{-2/3}$ with $s\in\{2.5,4,6,9,13,20,30,45\}$ and their own draws of $G$, extend them to
$\tau/(M/n)=93.7$ to $111$ at every $n$. For the scale, as $\norm\Zbar^2$ concentrates at
$w_p=\E_p\norm\Zbar^2=\norm{\mu_p}^2+(1-\norm{\mu_p}^2)/n$, $X^{(4)}_i$ is close to
$w_p^{-1/3}X^{(2)}_i$, and we take $c_p^2=w_p^{-2/3}V_p/n+\tau^2$, with $V_p$ ridge ERM's exact energy;
it is again a function of $p$ alone, as the variational bound requires. At $\tau=0$ the mean of
$(X^{(4)}_i/c_p)^2$ over the pairs of a run lies between $1.00007$ and $1.0041$. The candidate-centred
statistic divides $Y_i-x^{(4)}_i(\Zt)$ by $\norm{z_c}^{-2/3}\norm{\Delta_i}^2/(2n)$, a function of the
candidates.

\begin{table}[t]
\centering
\caption{The Quartic Learner on the Supersamples of the Ridge ERM Runs at $d=64n$. $n^{2/3}M/n$: the
per-pair signal in units of $n^{-2/3}$. The other columns are those of Table~\ref{tab:bracket}, at the
learner's own budgets and over its 24 noise levels (In: of 24); the slope uses five to seven levels.
The column $\iota(0)$ nearly coincides with that of ridge ERM in Table~\ref{tab:bracket}: at $\tau=0$
the release is the ridge release times $\norm\Zbar^{-2/3}$, which is close to a function of the
law.}\label{tab:quartic}
\small
\begin{tabular}{@{}rccccrrcc@{}}
\toprule
$n$ & $n^{2/3}M/n$ & $\iota(0)$ [95\% interval] & Lower$/\iota$ & Upper$/\iota$ & In & Slope
 & $c_{\mathrm{fit}}$ & $a_{\mathrm{fit}}$\\
\midrule
50 & 0.406 & 0.9133 [0.9113, 0.9154] & 0.632 & 1.410 & 24 & $-1.967$ & 0.928 & 0.871\\
100 & 0.444 & 0.9063 [0.9042, 0.9084] & 0.632 & 1.383 & 22 & $-2.081$ & 0.920 & 0.851\\
300 & 0.470 & 0.8965 [0.8935, 0.8996] & 0.633 & 1.388 & 22 & $-1.952$ & 0.905 & 0.851\\
1000 & 0.478 & 0.8910 [0.8885, 0.8935] & 0.633 & 1.488 & 24 & $-2.048$ & 0.892 & 0.891\\
3000 & 0.480 & 0.8881 [0.8851, 0.8910] & 0.635 & 1.400 & 24 & $-1.945$ & 0.899 & 0.891\\
\bottomrule
\end{tabular}
\end{table}

\paragraph{Results.}
Table~\ref{tab:quartic} and Figure~\ref{fig:quartic}(a) show the measurements. The per-pair signal
$M/n$ rises from $0.406\,n^{-2/3}$ at $n=50$ to $0.480\,n^{-2/3}$ at $n=3000$, that is, from $0.34$ to
$0.40$ times $\eps_4$. It is $1.78$ to $6.94$ times that of ridge ERM on the same pairs: as
$\E\norm{\mu_p}^2=v=8/n$ under the prior, $\E_\Pi\norm\Zbar^2=(9-8/n)/n$, and the factor
$\norm\Zbar^{-2/3}$ is close to $(n/9)^{1/3}$. The
signal-to-noise ratio $M^2/(nV)$ is $0.781$ to $0.801$ and $n\Vres/M^2$ is $1.010$ to $1.016$, as for ridge
ERM. The estimate of $\iota(0)$, $0.888$ to $0.913$, agrees with that of ridge ERM on the same pairs to
within $2\cdot10^{-5}$. At $\tau=0$ the quartic release is the ridge release times $\norm\Zbar^{-2/3}$,
which is close to $w_p^{-1/3}$ because $\norm\Zbar^2$ concentrates at $w_p$; a factor that depends only
on the law leaves the information at that law unchanged, and the scale $c_p$ removes it from the
statistic of the estimator. Against $\tau/(M/n)$ the whole curve has the shape found for ridge ERM
(Table~\ref{tab:bracket}): the fits
give $a_{\mathrm{fit}}=0.85$ to $0.89$ and $c_{\mathrm{fit}}=0.89$ to $0.93$, with root-mean-square
residuals of $0.02$ to $0.20$ in $\log\iota$, and the tail slopes are $-2.081$ to $-1.945$. Of the 120
points, 116 lie in the bracket of Theorem~\ref{thm:transfer} up to their 95\% intervals; the other four
lie in the tail, at $\tau/(M/n)=20.3$ to $45$, below the lower half, with estimates of at most
$0.0014$. Over the 79 points with $\iota\ge0.01$, the lower half is at least $0.631$ times and the
upper half at most $1.489$ times $\iota$. Figure~\ref{fig:quartic}(a) omits the six points whose 95\%
interval reaches $0$, all at $\tau/(M/n)\ge63.8$.

\paragraph{Both Releases at \texorpdfstring{$\tau=n^{-2/3}$}{the Noise Level of the Separation}.}
The two learners have the same model information at every law
(Proposition~\ref{prop:separation}(ii)). At $\tau=n^{-2/3}$, which is $\tau/(M/n)=2.08$ to $2.46$ for
$\thh^{(4)}$ and $4.39$ to $14.5$ for ridge ERM, the estimated totals $n\iota$ on the same pairs are
$5.25$, $12.4$, $40.7$, $141$ and $428$ bits for $X^{(4)}$, against $1.76$, $2.77$, $4.39$, $6.98$ and $10.9$
bits for $X^{(2)}$, for $n=50$, $100$, $300$, $1000$ and $3000$ (Figure~\ref{fig:quartic}(b)); all ten lie
in the bracket of Theorem~\ref{thm:transfer} at their learner's budgets, up to their 95\% intervals. The
totals of ridge ERM are $0.16$ to $0.25$ times the bound $3n^{1/3}$ that Proposition~\ref{prop:separation}
proves at every law, and those of $\thh^{(4)}$ are $0.84$ to $1.14$ times $n/8$, which it proves at the
zero-mean law. These are averages over the prior of the runs, not values at the zero-mean law, so they
illustrate the separation rather than test the proposition.

\input{fig_quartic}

\subsection{Low Dimension}\label{app:lowdim}

Two ridge runs at $n=1000$ use $d=8$ and $d=32$, with the same $\eps$ and prior as at $d=32n$. At
$\tau=0$ the exact second-moment bound, averaged over the drawn laws, gives totals of $0.396$ and
$1.36$ bits (variational estimates
$0.323$ and $1.29$), against $686$ bits at $d=32n$. For $\tau>0$, Proposition~\ref{prop:dimension} caps
every information of every learner at every law by $d\log_2(1+4L\sqrt n/\tau)+1/(2\ln2)$, which at
$\tau=\eps$ is $144$ bits for $d=8$ and $575$ bits for $d=32$; the candidate-centred lower bounds on
$\E_\Pi\Is_{D_p}(X+\eps G)$ are $5.0$ and $21.5$ bits, below the caps by factors of $29$ and $27$. At
$\tau=0$ the cap is only $n$, so the totals at $\tau=0$ are compared with the total at $d=32n$ and not
with these caps.

\subsection{What the Measurements Do Not Show}\label{app:meas-limits}

Every value is a prior average; it bounds the worst case over the family from below and says nothing
about a single law. The information values are lower-bound estimates, and at 70 of the 240 points of
Appendices~\ref{app:results} and~\ref{app:lowdim} the estimate lies below the lower half of
Theorem~\ref{thm:transfer} or, for ridge ERM, below the exact second-moment bound. Where the estimate is
at least $0.01$ bits per pair this happens at 5 points, by less than 3.2\%, and the largest such estimate
is $0.027$ bits per pair; 6 points fall short by more than their 95\% intervals, all with estimates of at
most $0.0034$ bits per pair, and 3 point estimates are negative. At $d=8$ and $d=32$ the variational
estimate lies below the second-moment bound at 31 of the 32 points. For the quartic learner of
Appendix~\ref{app:quartic}, the estimate lies below the lower half at 27 of its 120 points; the four of
these with estimates of at least $0.01$ bits per pair fall short by less than $6.2\%$, and four points,
all with estimates of at most $0.0014$ bits per pair, fall short by more than their 95\% intervals. For logistic ridge ERM and the quartic learner the budgets, and hence the
bracket, are Monte Carlo estimates.

The margins are small. Over all trials of logistic ridge ERM, $\norm\thh\le0.2072$, and its root mean
square falls from $0.2013$ at $n=50$ to $0.0473$ at $n=1000$, roughly like $\eps^{1/2}$, as localization
predicts. So every margin $t$ has $|t|\le0.2072$ and $|\phi'(t)-\phi'(0)|=\tanh(|t|/2)/2\le0.0517$, about
a tenth of $b$: the link is nearly affine where it is evaluated, and the measurements say nothing about
large margins. The fitted constants describe these learners, not $c_1$ and $C_1$ of
Theorem~\ref{thm:main}, and $n$ from 50 to 3000 illustrates the orders of the tradeoff without testing them.
Ridge and logistic ridge ERM have $r=2$. For $r=4$ only the affine link was measured, through the quartic
learner, whose release is computed from the ridge supersamples and so shares their laws, candidates and
bits; no learner with $r=4$ and a nonlinear link, such as quartic logistic learning, was measured.

%% file: fig_totals.tex
% Figure I.1, generated by experiments/code/make_paper_figure.py from experiments/data/paper/totals.csv
% (inline coordinates; pgfplots only). Do not edit by hand.
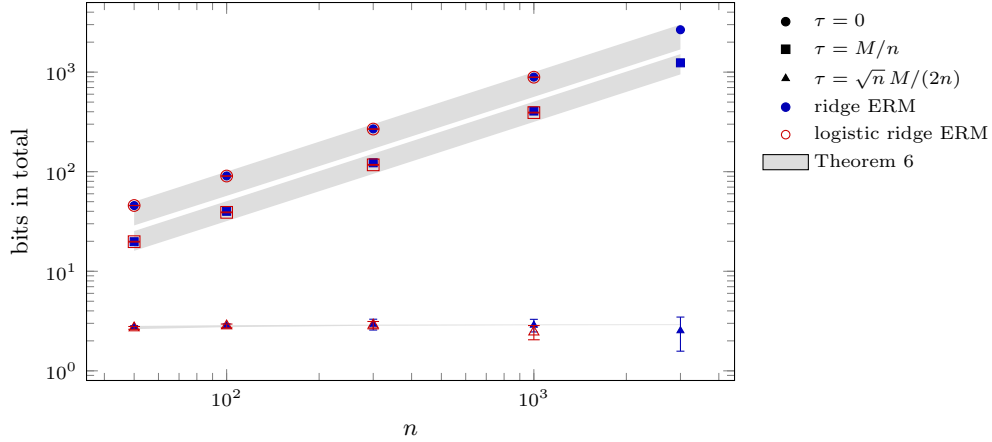
\begin{figure}[t]
\centering
\begin{tikzpicture}
\begin{axis}[scale only axis,width=0.5\linewidth,height=5cm,xmode=log,ymode=log,xmin=35,xmax=4500,ymin=0.8,ymax=5000,tick label style={font=\scriptsize},label style={font=\small},xlabel={$n$},ylabel={bits in total},legend style={font=\scriptsize,draw=none,fill=none,at={(1.03,1)},anchor=north west},legend cell align=left,clip mode=individual]
\addplot[fill=black!13,draw=none,forget plot] coordinates {(50,50) (100,100) (300,300) (1000,1000) (3000,3000) (3000,1690.1) (1000,564.09) (300,169.93) (100,57.2) (50,28.877)} -- cycle;
\addplot[fill=black!13,draw=none,forget plot] coordinates {(50,25.326) (100,50.569) (300,151.69) (1000,505.61) (3000,1516.8) (3000,948.97) (1000,316.55) (300,95.187) (100,31.903) (50,16.037)} -- cycle;
\addplot[fill=black!13,draw=none,forget plot] coordinates {(50,2.8242) (100,2.8731) (300,2.9109) (1000,2.9245) (3000,2.9285) (3000,2.8805) (1000,2.8707) (300,2.8372) (100,2.7468) (50,2.6233)} -- cycle;
\addplot+[only marks,mark=*,mark size=1.5pt,color=blue!70!black,forget plot,error bars/.cd,y dir=both,y explicit] coordinates {(50,45.667) += (0,0.102) -= (0,0.102) (100,90.631) += (0,0.212) -= (0,0.212) (300,268.96) += (0,0.905) -= (0,0.905) (1000,891) += (0,2.46) -= (0,2.46) (3000,2664.2) += (0,8.82) -= (0,8.82)};
\addplot+[only marks,mark=o,mark size=2.2pt,color=red!80!black,forget plot,error bars/.cd,y dir=both,y explicit] coordinates {(50,45.755) += (0,0.102) -= (0,0.102) (100,90.471) += (0,0.217) -= (0,0.217) (300,267.99) += (0,0.724) -= (0,0.723) (1000,889.28) += (0,2.2) -= (0,2.2)};
\addplot+[only marks,mark=square*,mark size=1.5pt,color=blue!70!black,forget plot,error bars/.cd,y dir=both,y explicit] coordinates {(50,19.798) += (0,0.174) -= (0,0.174) (100,40.034) += (0,0.352) -= (0,0.352) (300,122.82) += (0,1.57) -= (0,1.57) (1000,406.58) += (0,3.4) -= (0,3.4) (3000,1239.8) += (0,16) -= (0,16)};
\addplot+[only marks,mark=square,mark size=2.2pt,color=red!80!black,forget plot,error bars/.cd,y dir=both,y explicit] coordinates {(50,19.821) += (0,0.18) -= (0,0.181) (100,39.069) += (0,0.355) -= (0,0.356) (300,117.2) += (0,1.13) -= (0,1.13) (1000,392.23) += (0,3.45) -= (0,3.45)};
\addplot+[only marks,mark=triangle*,mark size=1.5pt,color=blue!70!black,forget plot,error bars/.cd,y dir=both,y explicit] coordinates {(50,2.7583) += (0,0.0848) -= (0,0.0848) (100,2.8369) += (0,0.126) -= (0,0.126) (300,2.9366) += (0,0.364) -= (0,0.368) (1000,2.8761) += (0,0.415) -= (0,0.418) (3000,2.5282) += (0,0.94) -= (0,0.953)};
\addplot+[only marks,mark=triangle,mark size=2.2pt,color=red!80!black,forget plot,error bars/.cd,y dir=both,y explicit] coordinates {(50,2.7009) += (0,0.0877) -= (0,0.0881) (100,2.8267) += (0,0.125) -= (0,0.125) (300,2.8866) += (0,0.244) -= (0,0.246) (1000,2.4537) += (0,0.395) -= (0,0.403)};
\addlegendimage{only marks,mark=*,mark size=1.6pt,black}\addlegendentry{$\tau=0$}
\addlegendimage{only marks,mark=square*,mark size=1.6pt,black}\addlegendentry{$\tau=M/n$}
\addlegendimage{only marks,mark=triangle*,mark size=1.6pt,black}\addlegendentry{$\tau=\sqrt n\,M/(2n)$}
\addlegendimage{only marks,mark=*,mark size=1.6pt,color=blue!70!black}\addlegendentry{ridge ERM}
\addlegendimage{only marks,mark=o,mark size=1.6pt,color=red!80!black}\addlegendentry{logistic ridge ERM}
\addlegendimage{area legend,fill=black!13,draw=none}\addlegendentry{Theorem~\ref{thm:transfer}}
\end{axis}
\end{tikzpicture}
\caption{Total Information against $n$. Totals $n\iota$ in bits (prior averages of lower-bound estimates, with 95\% intervals) against $n$ at $d=64n$ for ridge ERM (filled) and logistic ridge ERM (open), at $\tau=0$, $\tau=M/n$ and $\tau=\sqrt n\,M/(2n)$ with each learner's own $M$. Shaded: between the lower and upper halves of Theorem~\ref{thm:transfer} at the learners' own budgets. Values at $\tau>0$ are interpolated linearly in log--log coordinates between neighbouring noise levels of the grid.}
\label{fig:totals}
\end{figure}

%% file: fig_quartic.tex
% Figure of Appendix I, generated by experiments/code/summarize_quartic.py (inline coordinates; pgfplots only).
% Do not edit by hand.
\begin{figure}[t]
\centering
\begin{tikzpicture}
\begin{axis}[scale only axis,width=0.36\linewidth,height=3.3cm,xmode=log,ymode=log,tick label style={font=\scriptsize},label style={font=\small},legend style={font=\scriptsize,draw=none,fill=none,row sep=-1pt},legend cell align=left,clip mode=individual,error bars/error bar style={line width=0.3pt},error bars/error mark options={mark size=0.8pt,line width=0.3pt},xmin=0.001,xmax=316,ymin=1e-4,ymax=1.6,xtick={0.001,0.01,0.1,1,10,100},ytick={1e-4,1e-3,1e-2,1e-1,1},xlabel={$\tau/(M/n)$},ylabel={bits per pair},legend style={at={(0.03,0.03)},anchor=south west}]
\addplot[fill=black!13,draw=none,forget plot] coordinates {(0.001,1) (0.001122,1) (0.0012589,1) (0.0014125,1) (0.0015849,1) (0.0017783,1) (0.0019953,1) (0.0022387,1) (0.0025119,1) (0.0028184,1) (0.0031623,1) (0.0035481,1) (0.0039811,1) (0.0044668,1) (0.0050119,1) (0.0056234,1) (0.0063096,1) (0.0070795,1) (0.0079433,1) (0.0089125,1) (0.01,1) (0.01122,1) (0.012589,1) (0.014125,1) (0.015849,1) (0.017783,1) (0.019953,1) (0.022387,1) (0.025119,1) (0.028184,1) (0.031623,1) (0.035481,1) (0.039811,1) (0.044668,1) (0.050119,1) (0.056234,1) (0.063096,1) (0.070795,1) (0.079433,1) (0.089125,1) (0.1,1) (0.1122,1) (0.12589,1) (0.14125,1) (0.15849,1) (0.17783,1) (0.19953,1) (0.22387,1) (0.25119,1) (0.28184,1) (0.31623,1) (0.35481,1) (0.39811,1) (0.44668,1) (0.50119,1) (0.56234,1) (0.63096,0.91411) (0.70795,0.79879) (0.79433,0.69191) (0.89125,0.59405) (1,0.5056) (1.122,0.42668) (1.2589,0.35719) (1.4125,0.29679) (1.5849,0.24493) (1.7783,0.2009) (1.9953,0.1639) (2.2387,0.1331) (2.5119,0.10766) (2.8184,0.08679) (3.1623,0.069772) (3.5481,0.055963) (3.9811,0.044802) (4.4668,0.035812) (5.0119,0.02859) (5.6234,0.022802) (6.3096,0.018171) (7.0795,0.014471) (7.9433,0.011518) (8.9125,0.0091642) (10,0.0072889) (11.22,0.0057958) (12.589,0.0046075) (14.125,0.0036623) (15.849,0.0029106) (17.783,0.0023129) (19.953,0.0018378) (22.387,0.0014602) (25.119,0.0011601) (28.184,0.00092168) (31.623,0.00073221) (35.481,0.00058168) (39.811,0.00046208) (44.668,0.00036707) (50.119,0.00029159) (56.234,0.00023163) (63.096,0.00018399) (70.795,0.00014615) (79.433,0.0001161) (89.125,9.2221e-05) (100,7.3254e-05) (112.2,5.8189e-05) (125.89,4.6221e-05) (141.25,3.6715e-05) (158.49,2.9164e-05) (177.83,2.3166e-05) (199.53,1.8401e-05) (223.87,1.4617e-05) (251.19,1.1611e-05) (281.84,9.2226e-06) (316.23,7.3258e-06) (316.23,7.2134e-06) (281.84,9.0811e-06) (251.19,1.1432e-05) (223.87,1.4392e-05) (199.53,1.8119e-05) (177.83,2.281e-05) (158.49,2.8716e-05) (141.25,3.6151e-05) (125.89,4.551e-05) (112.2,5.7293e-05) (100,7.2126e-05) (89.125,9.0798e-05) (79.433,0.0001143) (70.795,0.00014389) (63.096,0.00018114) (56.234,0.00022802) (50.119,0.00028703) (44.668,0.0003613) (39.811,0.00045477) (35.481,0.0005724) (31.623,0.00072043) (28.184,0.00090666) (25.119,0.0011409) (22.387,0.0014356) (19.953,0.0018061) (17.783,0.0022719) (15.849,0.0028572) (14.125,0.0035923) (12.589,0.0045149) (11.22,0.0056722) (10,0.0071223) (8.9125,0.0089372) (7.9433,0.011205) (7.0795,0.014034) (6.3096,0.017555) (5.6234,0.021924) (5.0119,0.027325) (4.4668,0.033974) (3.9811,0.042113) (3.5481,0.052011) (3.1623,0.06395) (2.8184,0.07821) (2.5119,0.095046) (2.2387,0.11465) (1.9953,0.13711) (1.7783,0.16239) (1.5849,0.19024) (1.4125,0.22025) (1.2589,0.2518) (1.122,0.28412) (1,0.31639) (0.89125,0.34776) (0.79433,0.3775) (0.70795,0.405) (0.63096,0.42988) (0.56234,0.45193) (0.50119,0.47113) (0.44668,0.48758) (0.39811,0.50149) (0.35481,0.51312) (0.31623,0.52275) (0.28184,0.53066) (0.25119,0.53711) (0.22387,0.54235) (0.19953,0.54659) (0.17783,0.55) (0.15849,0.55274) (0.14125,0.55494) (0.12589,0.5567) (0.1122,0.5581) (0.1,0.55922) (0.089125,0.56011) (0.079433,0.56082) (0.070795,0.56139) (0.063096,0.56184) (0.056234,0.5622) (0.050119,0.56249) (0.044668,0.56271) (0.039811,0.56289) (0.035481,0.56304) (0.031623,0.56315) (0.028184,0.56324) (0.025119,0.56331) (0.022387,0.56337) (0.019953,0.56341) (0.017783,0.56345) (0.015849,0.56348) (0.014125,0.5635) (0.012589,0.56352) (0.01122,0.56353) (0.01,0.56355) (0.0089125,0.56355) (0.0079433,0.56356) (0.0070795,0.56357) (0.0063096,0.56357) (0.0056234,0.56358) (0.0050119,0.56358) (0.0044668,0.56358) (0.0039811,0.56358) (0.0035481,0.56358) (0.0031623,0.56359) (0.0028184,0.56359) (0.0025119,0.56359) (0.0022387,0.56359) (0.0019953,0.56359) (0.0017783,0.56359) (0.0015849,0.56359) (0.0014125,0.56359) (0.0012589,0.56359) (0.001122,0.56359) (0.001,0.56359)} -- cycle;
\addplot[only marks,mark=o,mark size=1.0pt,color=blue!70!black!55,line width=0.3pt] coordinates {(0.017857,0.9131) (0.035714,0.91239) (0.059524,0.91062) (0.11905,0.90237) (0.17857,0.88731) (0.35714,0.81038) (0.59524,0.66099) (1.1905,0.33333) (1.7857,0.18478) (3.5714,0.054188) (5.9524,0.019356) (11.905,0.0053371) (17.857,0.0022832) (35.714,0.00067165) (59.524,0.00013252) (0.016304,0.90603) (0.032609,0.90558) (0.054348,0.90398) (0.1087,0.89677) (0.16304,0.88367) (0.32609,0.82105) (0.54348,0.68795) (1.087,0.37177) (1.6304,0.21191) (3.2609,0.061967) (5.4348,0.024359) (10.87,0.0059909) (16.304,0.0027633) (32.609,0.00067351) (54.348,0.00019781) (0.015411,0.89629) (0.030822,0.89568) (0.05137,0.89451) (0.10274,0.88858) (0.15411,0.87731) (0.30822,0.81769) (0.5137,0.69882) (1.0274,0.40061) (1.5411,0.23182) (3.0822,0.071346) (5.137,0.028574) (10.274,0.0068946) (15.411,0.0026125) (30.822,0.0005245) (51.37,0.00044398) (0.015121,0.89087) (0.030242,0.89025) (0.050403,0.88902) (0.10081,0.88308) (0.15121,0.87226) (0.30242,0.81837) (0.50403,0.70244) (1.0081,0.40398) (1.5121,0.2384) (3.0242,0.072361) (5.0403,0.027055) (10.081,0.0074346) (15.121,0.003118) (30.242,0.00088982) (50.403,0.00033531) (0.01504,0.88788) (0.03008,0.88731) (0.050134,0.88599) (0.10027,0.87926) (0.1504,0.8696) (0.3008,0.81537) (0.50134,0.69956) (1.0027,0.41243) (1.504,0.2396) (3.008,0.075978) (5.0134,0.026627) (10.027,0.0070827) (15.04,0.0030973) (30.08,0.00068735) (50.134,0.00015533)};
\addlegendentry{ridge ERM}
\addplot+[only marks,mark=*,mark size=1.0pt,color=green!45!black,line width=0.4pt,error bars/.cd,y dir=both,y explicit] coordinates {(0.010021,0.91325) += (0,0.00205) -= (0,0.00205) (0.020041,0.91304) += (0,0.00205) -= (0,0.00205) (0.033402,0.91248) += (0,0.00206) -= (0,0.00206) (0.066804,0.90996) += (0,0.00208) -= (0,0.00208) (0.10021,0.90513) += (0,0.00214) -= (0,0.00214) (0.20041,0.8809) += (0,0.00235) -= (0,0.00235) (0.33402,0.82616) += (0,0.00272) -= (0,0.00272) (0.66804,0.60687) += (0,0.0034) -= (0,0.0034) (1.0021,0.41645) += (0,0.00353) -= (0,0.00353) (2.0041,0.15254) += (0,0.00263) -= (0,0.00263) (3.3402,0.059521) += (0,0.00172) -= (0,0.00172) (6.1527,0.018256) += (0,0.000984) -= (0,0.000984) (6.6804,0.01638) += (0,0.000929) -= (0,0.000929) (9.8444,0.0079953) += (0,0.000674) -= (0,0.000674) (10.021,0.0071767) += (0,0.000641) -= (0,0.000641) (14.767,0.0033577) += (0,0.000432) -= (0,0.000432) (20.041,0.0019806) += (0,0.00033) -= (0,0.00033) (22.15,0.001435) += (0,0.000283) -= (0,0.000283) (31.994,0.0008348) += (0,0.000215) -= (0,0.000215) (33.402,0.00053763) += (0,0.000182) -= (0,0.000182) (49.222,0.00022992) += (0,0.000113) -= (0,0.000113) (73.833,0.00010192) += (0,8.22e-05) -= (0,8.22e-05) (0.0072761,0.90623) += (0,0.00212) -= (0,0.00212) (0.014552,0.90618) += (0,0.00212) -= (0,0.00212) (0.024254,0.90586) += (0,0.00213) -= (0,0.00213) (0.048507,0.90442) += (0,0.00212) -= (0,0.00212) (0.072761,0.90147) += (0,0.00217) -= (0,0.00217) (0.14552,0.88966) += (0,0.00228) -= (0,0.00228) (0.24254,0.85854) += (0,0.00248) -= (0,0.00248) (0.48507,0.72282) += (0,0.00319) -= (0,0.00319) (0.72761,0.56466) += (0,0.00353) -= (0,0.00353) (1.4552,0.24891) += (0,0.00311) -= (0,0.00311) (2.4254,0.1095) += (0,0.00231) -= (0,0.00231) (4.8507,0.029592) += (0,0.00128) -= (0,0.00128) (5.6288,0.021867) += (0,0.00106) -= (0,0.00106) (7.2761,0.013512) += (0,0.000865) -= (0,0.000865) (9.0061,0.0090376) += (0,0.000697) -= (0,0.000697) (13.509,0.0040877) += (0,0.000482) -= (0,0.000482) (14.552,0.003542) += (0,0.000444) -= (0,0.000444) (20.264,0.0013404) += (0,0.000267) -= (0,0.000267) (24.254,0.0011596) += (0,0.000254) -= (0,0.000254) (29.27,0.00075157) += (0,0.000207) -= (0,0.000207) (45.03,0.00019523) += (0,0.000123) -= (0,0.000123) (67.545,0.0001429) += (0,9.07e-05) -= (0,9.07e-05) (0.004771,0.8965) += (0,0.00302) -= (0,0.00302) (0.0095421,0.89644) += (0,0.00302) -= (0,0.00302) (0.015903,0.89636) += (0,0.00302) -= (0,0.00302) (0.031807,0.89592) += (0,0.00302) -= (0,0.00302) (0.04771,0.89473) += (0,0.00307) -= (0,0.00307) (0.095421,0.88826) += (0,0.00323) -= (0,0.00323) (0.15903,0.87529) += (0,0.00334) -= (0,0.00334) (0.31807,0.81554) += (0,0.00395) -= (0,0.00395) (0.4771,0.72028) += (0,0.0045) -= (0,0.0045) (0.95421,0.43531) += (0,0.00517) -= (0,0.00517) (1.5903,0.22431) += (0,0.00466) -= (0,0.00466) (3.1807,0.066768) += (0,0.0029) -= (0,0.0029) (4.771,0.029273) += (0,0.00186) -= (0,0.00186) (5.3232,0.026509) += (0,0.00173) -= (0,0.00173) (8.517,0.010314) += (0,0.00107) -= (0,0.00107) (9.5421,0.0071929) += (0,0.000967) -= (0,0.000967) (12.776,0.0045634) += (0,0.000796) -= (0,0.000796) (15.903,0.0033392) += (0,0.00061) -= (0,0.00061) (19.163,0.0020083) += (0,0.000518) -= (0,0.000518) (27.68,0.00045468) += (0,0.000289) -= (0,0.000289) (42.585,0.0002089) += (0,0.000177) -= (0,0.000177) (0.0031378,0.891) += (0,0.00246) -= (0,0.00246) (0.0062756,0.89096) += (0,0.00246) -= (0,0.00246) (0.010459,0.89093) += (0,0.00246) -= (0,0.00246) (0.020919,0.89071) += (0,0.00246) -= (0,0.00246) (0.031378,0.89016) += (0,0.00245) -= (0,0.00245) (0.062756,0.88793) += (0,0.00244) -= (0,0.00244) (0.10459,0.88237) += (0,0.00248) -= (0,0.00248) (0.20919,0.85441) += (0,0.00266) -= (0,0.00266) (0.31378,0.81293) += (0,0.00282) -= (0,0.00282) (0.62756,0.6169) += (0,0.00358) -= (0,0.00358) (1.0459,0.38747) += (0,0.00373) -= (0,0.00373) (2.0919,0.14156) += (0,0.00268) -= (0,0.00268) (3.1378,0.067667) += (0,0.00193) -= (0,0.00193) (5.2297,0.025875) += (0,0.00123) -= (0,0.00123) (6.2756,0.018518) += (0,0.00106) -= (0,0.00106) (8.3675,0.010227) += (0,0.000777) -= (0,0.000777) (10.459,0.0067871) += (0,0.000578) -= (0,0.000578) (12.551,0.0047226) += (0,0.000529) -= (0,0.000529) (18.827,0.0018348) += (0,0.000326) -= (0,0.000326) (27.194,0.00088888) += (0,0.000231) -= (0,0.000231) (41.838,0.00045625) += (0,0.000173) -= (0,0.000173) (62.756,0.00013078) += (0,9.43e-05) -= (0,9.43e-05) (0.0021657,0.88805) += (0,0.00294) -= (0,0.00294) (0.0043314,0.88804) += (0,0.00295) -= (0,0.00295) (0.0072189,0.88801) += (0,0.00294) -= (0,0.00294) (0.014438,0.88778) += (0,0.00297) -= (0,0.00297) (0.021657,0.88761) += (0,0.00292) -= (0,0.00292) (0.043314,0.88633) += (0,0.00296) -= (0,0.00296) (0.072189,0.88362) += (0,0.00285) -= (0,0.00285) (0.14438,0.871) += (0,0.00303) -= (0,0.00303) (0.21657,0.84959) += (0,0.00281) -= (0,0.00281) (0.43314,0.74417) += (0,0.00377) -= (0,0.00377) (0.72189,0.55643) += (0,0.00578) -= (0,0.00578) (1.4438,0.25254) += (0,0.00423) -= (0,0.00423) (2.1657,0.13159) += (0,0.00353) -= (0,0.00353) (4.3314,0.036168) += (0,0.00196) -= (0,0.00196) (5.2058,0.02622) += (0,0.00126) -= (0,0.00126) (7.2189,0.012683) += (0,0.00107) -= (0,0.00107) (8.3292,0.009888) += (0,0.00112) -= (0,0.00112) (12.494,0.0047991) += (0,0.000786) -= (0,0.000786) (18.741,0.0020722) += (0,0.000489) -= (0,0.000489) (27.07,0.00099673) += (0,0.000369) -= (0,0.000369) (41.646,0.0004025) += (0,0.000207) -= (0,0.000207) (62.469,0.00018764) += (0,0.000152) -= (0,0.000152)};
\addlegendentry{quartic learner}
\addlegendimage{area legend,fill=black!13,draw=none}\addlegendentry{Theorem~\ref{thm:transfer}}
\node[anchor=north east,font=\scriptsize] at (rel axis cs:0.97,0.97) {(a)};
\end{axis}
\end{tikzpicture}\hfill
\begin{tikzpicture}
\begin{axis}[scale only axis,width=0.36\linewidth,height=3.3cm,xmode=log,ymode=log,tick label style={font=\scriptsize},label style={font=\small},legend style={font=\scriptsize,draw=none,fill=none,row sep=-1pt},legend cell align=left,clip mode=individual,error bars/error bar style={line width=0.3pt},error bars/error mark options={mark size=0.8pt,line width=0.3pt},xmin=35,xmax=4500,ymin=1,ymax=1000,xtick={100,1000},ytick={1,10,100,1000},xlabel={$n$},ylabel={bits in total at $\tau=n^{-2/3}$},legend style={at={(0.03,0.97)},anchor=north west}]
\addplot+[only marks,mark=*,mark size=1.3pt,color=green!45!black,line width=0.4pt,error bars/.cd,y dir=both,y explicit] coordinates {(50,5.255) += (0,0.112) -= (0,0.112) (100,12.443) += (0,0.245) -= (0,0.245) (300,40.709) += (0,1.07) -= (0,1.07) (1000,140.7) += (0,2.52) -= (0,2.52) (3000,428.06) += (0,8.96) -= (0,8.96)};
\addlegendentry{quartic learner}
\addplot+[only marks,mark=o,mark size=1.3pt,color=blue!70!black,line width=0.4pt,error bars/.cd,y dir=both,y explicit] coordinates {(50,1.7645) += (0,0.0683) -= (0,0.0683) (100,2.7695) += (0,0.124) -= (0,0.124) (300,4.3869) += (0,0.381) -= (0,0.381) (1000,6.976) += (0,0.629) -= (0,0.629) (3000,10.872) += (0,1.58) -= (0,1.58)};
\addlegendentry{ridge ERM}
\addplot[black,thin,densely dotted,domain=50:3000,samples=40] {x/8};
\addlegendentry{$n/8$}
\addplot[black,thin,densely dashed,domain=50:3000,samples=40] {3*x^(1/3)};
\addlegendentry{$3n^{1/3}$}
\node[anchor=south east,font=\scriptsize] at (rel axis cs:0.97,0.03) {(b)};
\end{axis}
\end{tikzpicture}
\caption{The Quartic Learner on the Ridge Supersamples. The learner $\thh^{(4)}=\norm\Zbar^{-2/3}\Zbar$ at $d=64n$, averaged over the prior; bars are 95\% intervals. (a)~Lower-bound estimates of the bits per pair $\iota$ of $X^{(4)}+\tau G$ against $\tau/(M/n)$, with the learner's own $M$; shaded: between the two halves of Theorem~\ref{thm:transfer} at its own $M$, $V$ and $\Vres$; open: ridge ERM, as in Figure~\ref{fig:measurements}(a). (b)~Totals $n\iota$ at $\tau=n^{-2/3}$ for the releases $X^{(4)}$ and $X^{(2)}$ of the two learners on the same supersamples, with the bounds $3n^{1/3}$ on $X^{(2)}$ at every law and $n/8$ on $X^{(4)}$ at the zero-mean law of Proposition~\ref{prop:separation}(v).}
\label{fig:quartic}
\end{figure}
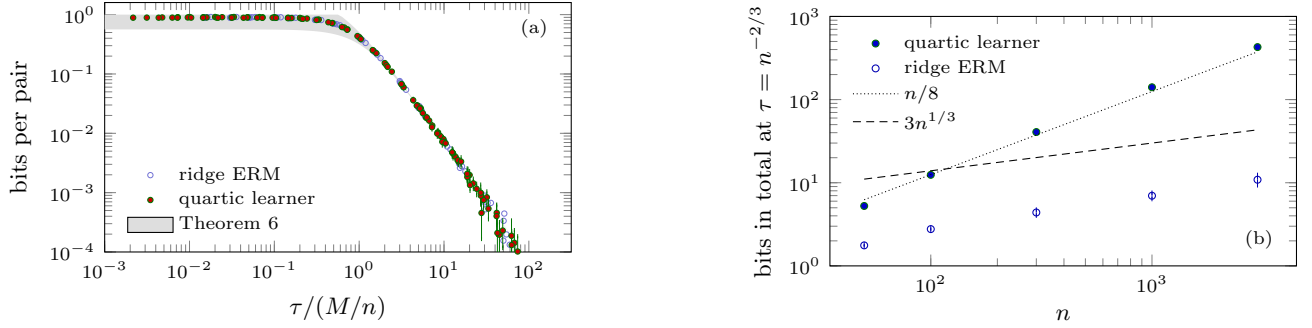

%% file: app_J_related.tex
\section{Related work in detail}\label{app:related}

This appendix states exactly the results closest to Theorem~\ref{thm:main}, says what each of them
implies for the noisy release, and credits the earlier proofs that already contain the moment bounds
of the affine case (Remark~\ref{rem:affine}).

% Macro \raggedcell (also used in app_K_lean.tex): a ragged-right cell of a p-column. The preamble does not
% load the array package, so a row whose last cell starts with \raggedcell ends with \tabularnewline.
\providecommand{\raggedcell}{\raggedright\hspace{0pt}}

\paragraph{Exact Statements.}
Table~\ref{tab:statements} gives the exact form of the results summarized in Table~\ref{tab:known},
together with the other results closest to Theorem~\ref{thm:main}, in the notation of this paper:
$\thh$ is the output of the learner, $S_n=(Z_1,\dots,Z_n)$ its training set, $U$ and $\Zt$ the
selection bits and the supersample, $\eps$ the accuracy, and $\xi$ the false-positive rate of a
tracer on fresh points. We rename two symbols of the cited works that have other meanings here:
$\delta'$ is the failure probability of a learner that is accurate with high probability, and
$n(\eps)$ is the sample complexity of the learner in \citet{livni2023}. As usual, $O$ and $\Omega$
hide constants, and $\tilde\Omega$ also logarithmic factors. For \citet{attias2024}, statement
numbers and conditions refer to arXiv:2402.09327v2. For
\citet{voitovych2025}, statement numbers refer to arXiv:2502.17384v2, and for \citet{steinke2020} to
the extended version, arXiv:2001.09122.

\begin{table}[tp]
\caption{Exact Statements of the Results Closest to Theorem~\ref{thm:main}}\label{tab:statements}
\centering\footnotesize
\defaultaddspace=2pt
\begin{tabular}{@{}p{0.14\linewidth}p{0.56\linewidth}p{0.25\linewidth}@{}}
\toprule
\raggedcell Result & \raggedcell Statement & \raggedcell Consequence for $X+\tau G$\tabularnewline
\midrule
\raggedcell \citet{livni2023}, Thm.~1
& \raggedcell For $0<\eps<1/54$, every learner with expected excess risk at most $\eps$ and sample
complexity $n(\eps)$ has, for some law on the scaled sign cube in dimension $d$ and a convex
Lipschitz loss, quadratic in the construction,
$I(\thh;S_n)\ge\sum_iI(\thh;Z_i)=\tilde\Omega\big(d/(\eps^5n(\eps)^6)\big)$. Lower bounds on the
conditional information are left open.
& \raggedcell None: $I(\thh;S_n)\ge I(\thh;U\mid\Zt)$, so not even the conditional model information
is bounded below.\tabularnewline
\addlinespace
\raggedcell \citet{attias2024}, Thms.~4.1--4.2, 4.5--4.6, Cors.~5.6--5.7
& \raggedcell Learners whose excess risk, at every law on the scaled sign cube, is at most $\eps$ with
probability at least $1-\delta'$, for the affine loss $-\ip\theta z$ on the unit ball with
$\delta'\le\eps$ (Thm.~4.1), or for the strongly convex loss $\norm\theta^2/2-\ip\theta z$ on $\R^d$
with $\delta'=O(n^{-2})$ (Thm.~4.2); every sample size at least the sample complexity of the learner;
$d\ge\Omega(n^2\log n)$. Then $I(\thh;U\mid\Zt)=\Omega(\eps^{-2})$ and $\Omega(\eps^{-1})$, respectively,
and the same orders hold for its individual-sample form (Cors.~5.6--5.7, the first for proper
learners). $\xi$-sound adversaries, which read $\thh$, the point and the data law and flag any of the
fresh points with probability at most $\xi$, certify recall of the same orders with probability
$\Omega(1)$ up to logarithmic factors (Thms.~4.5--4.6) if $d\ge\Omega(n^2\log(n/\xi))$ and, for the affine loss, $\delta'<\eps$
and $n$ equals the learner's sample complexity, assumed of order $\log(1/\delta')/\eps^2$.
& \raggedcell Upper bound only: $\Is(X+\tau G)\le I(\thh;U\mid\Zt)$ by
\eqref{eq:hierarchy}.\tabularnewline
\addlinespace
\raggedcell \citet{voitovych2025}, Thm.~2.7
& \raggedcell Affine loss, with parameters in the largest cube inscribed in a unit $\ell_s$ ball,
$s\ge2$, and sparse sign data. In the Euclidean case, for $\xi<1/e$ and
$1/(6\sqrt n)\le\eps\le\min\{c\,(d/(n^2\log(1/\xi)))^{1/2},1/6\}$, with a constant $c$, every learner with
expected excess risk at most $\eps$ has a tracer that reads $\thh$, the
point and the population mean, with false-positive rate at most $\xi$ and expected recall
$\Omega(\eps^{-2})$. At $\eps\asymp n^{-1/2}$ this requires $d\gtrsim n\log(1/\xi)$ and gives recall
$\Omega(n)$.
& \raggedcell The decisions are functions of $\thh$, the point and the mean, not of the release, and
involve no $\tau$.\tabularnewline
\addlinespace
\raggedcell \citet{voitovych2025}, Thm.~3.5
& \raggedcell Recall bounds derived from a lower bound on the trace value, the mean tracer score of
the training points, and a bound on the sum of squared scores (Lemma~3.4), through a counting
inequality of Paley--Zygmund type (Lemma~A.11).
& \raggedcell The two moments that Theorem~\ref{thm:transfer} turns into information
(Remark~\ref{rem:affine}).\tabularnewline
\addlinespace
\raggedcell \citet{haghifam2023}, Thm.~17
& \raggedcell Gradient descent with $n^2$ steps of size $n^{-3/2}$ on the loss $-\ip\theta z$ over
the unit ball, with data among the standard basis vectors of $\R^{2n^2}$: its evaluated-CMI bounds
do not vanish, while its generalization error is $O(n^{-1/2})$. They also show that analysing a
surrogate whose final iterate is corrupted by Gaussian noise cannot establish minimax rates for
gradient descent.
& \raggedcell One algorithm on one instance, at $\tau=0$, with no upper bound.\tabularnewline
\addlinespace
\raggedcell \citet{wangmao2023b}, Ex.~1 and App.~F
& \raggedcell On the algorithm and instance of the previous row (their Example~1), with constant
probability the sign of $X_i$ reveals $U_i$ (Appendix~F).
& \raggedcell As in the previous row.\tabularnewline
\addlinespace
\raggedcell \citet{wangmao2023a}, Sec.~3, Thms.~3.1--3.2
& \raggedcell They define the ld-CMI of pair $i$, $I(X_i;U_i\mid\Zt)$. For losses with values in $[0,1]$,
$|\gen_D|\le n^{-1}\sum_i\sqrt{2I(X_i;U_i\mid\Zt)}$, with information in nats, and a sharper form with
the square root inside the average over $\Zt$ (Thm.~3.1); the same bound holds with $I(X_i;U_i)$ (Thm.~3.2). After Thm.~3.1:
$I(X_i;U_i\mid\Zt)$ is at most the information in the two losses of pair $i$, which is at most that
in the two predictions, which is at most $I(\thh;U_i\mid\Zt)$.
& \raggedcell The data-processing direction of \eqref{eq:hierarchy}. By
Corollary~\ref{cor:consequences}(i), the range-scaled bounds stay above a constant at some law of
the family.\tabularnewline
\addlinespace
\raggedcell \citet{wangmao2024}, Sec.~6
& \raggedcell The tradeoff between accuracy and CMI, observed for the CMI of the model, is
anticipated to hold for ld-CMI; no proof is given.
& \raggedcell Established for the class of Assumption~\ref{ass:class} by Theorem~\ref{thm:main} at
$\tau=0$.\tabularnewline
\addlinespace
\raggedcell \citet{steinke2020}, Thm.~6.8
& \raggedcell For a deterministic algorithm with uniform stability $\gamma$ and a continuous data
distribution, the evaluated CMI of the losses, each perturbed by independent Gaussian noise of
variance $\tau^2$, is at most $\gamma^2n^2/(2\tau^2)$ nats.
& \raggedcell An upper bound of the same Gaussian-channel type; the upper half of
Theorem~\ref{thm:main} uses accuracy with matched curvature instead of stability.\tabularnewline
\addlinespace
\raggedcell Theorem~\ref{thm:main}
& \raggedcell Every proper learner that is $\eps$-accurate on the family, for a link and a
regularizer of Assumption~\ref{ass:class}, with $\eps\le\eps_0$, $n\asymp\neps$ and
$d\ge\Cdim\max\{n,\neps\}$, has
$\sup_p\Qc_{D_p}(X+\tau G)=\Theta\big(n/(1+(\tau/\eps)^2)\big)$ for each
$\Qc\in\{\Ii,\Ij,\Is\}$ and every $\tau\ge0$.
& \raggedcell Lower and upper bounds, at every noise level.\tabularnewline
\bottomrule
\end{tabular}
\end{table}

\paragraph{What These Results Imply for the Release.}
Given the candidates, the release is a function of the model, so data processing gives
$\Is(X+\tau G)\le I(X;U\mid\Zt)\le I(\thh;U\mid\Zt)$, as in \eqref{eq:hierarchy}, and
\citet{wangmao2023a} record the same chain for each pair. Lower bounds on model information
therefore do not pass to the release. The bound of \citet{livni2023} concerns $I(\thh;S_n)$, which is
at least $I(\thh;U\mid\Zt)$, because $I(\thh;U\mid\Zt)\le I(\thh;U,\Zt)=I(\thh;S_n)$, so it does not bound even the conditional model
information from below; \citet{attias2024} prove such a bound, answering the question he left open.
Tracer decisions are functions of the model, the point and the data law, which for
\citet{voitovych2025} enters through its mean, not of the release, and none of these conclusions involves observation noise.
Proposition~\ref{prop:separation} shows that conclusions which involve no noise level and are
invariant under reparametrizations of the model cannot determine the information of the noisy
release: two learners, sample-optimal for two different losses and related by a fixed bijection of
the ball, have the same model information and tracers with identical decisions, yet at
$\tau=n^{-2/3}$ the release of the first carries at most $3n^{1/3}$ bits at every product law and that
of the second at least $n/8$ bits at the zero-mean law. At $\tau=0$, when the link is strictly
decreasing, a tracer that thresholds an increasing function of the margin yields a decoder from the
signs of the loss differences, and Lemma~\ref{lem:sign} turns it into information once its expected
recall exceeds $n/2+n\xi$. The recall bounds of \citet{attias2024} and \citet{voitovych2025} are
order statements with unspecified constants and do not ensure this, so even at $\tau=0$ the lemma
does not turn them into lower bounds on the information of the loss differences. The loss-level lower bounds in the table
\citep{haghifam2023,wangmao2023b} concern gradient descent on one instance, at $\tau=0$, and come
without an upper bound. Theorem~\ref{thm:main} differs in each of these respects: it holds for every
proper learner that is accurate on the family at the optimal sample size, in every dimension
$d\ge\Cdim\max\{n,\neps\}$, at every noise level with a matching upper bound, and for nonlinear links
with matched curvature. At $\tau=0$ it establishes, for the class of Assumption~\ref{ass:class}, the
tradeoff that \citet{wangmao2024} anticipated. Through \eqref{eq:hierarchy} it also gives model
information of order $n$ in dimension $d\gtrsim n$; Appendix~\ref{app:converse} compares this with the
bounds of \citet{attias2024}.

\begin{remark}[The affine link]\label{rem:affine}
For the affine link $\phi(t)=-t$, the link remainder $u$ of Step~3 in Section~\ref{sec:moments}
vanishes and $\sigma_iX_i=\ip{\thh}{\Ztr_i-\Zgh_i}$. The ghost is independent of $\thh$ and has mean
$\mu_p$, so the signed mean is $\E_p\ip{\thh}{\sum_i(\Ztr_i-\mu_p)}$, and Lemma~\ref{lem:fp} turns
its prior average into $\kappa\,\E_\Pi\ip{\thh}{\mu_p}$. Earlier lower bounds for the model rest on
the same argument. The identity of Lemma~\ref{lem:fp} is the beta-prior fingerprinting lemma of
\citet{steinke2017}; it is Lemma~3.8 of \citet{voitovych2025} when the sparsity of their data equals
the dimension, and Lemma~6.3 of \citet{attias2024} rests on an identity of the same kind with a
uniform prior. Accuracy enters through the
alignment $\E\ip{\thh}{\mu_p}$. On the affine instance of \citet[Section~6.1.1]{attias2024}, which
has no regularizer, their Lemma~6.2 gives $\E\ip{\thh}{\mu_D}\ge\norm{\mu_D}-\eps-2\delta'$, where
$\mu_D$ is the mean of the data law $D$, for learners whose excess risk exceeds $\eps$ with
probability at most $\delta'$, and
\citet[Proposition~E.1]{voitovych2025} give the analogue for parameters in a cube. The energy bound
rests on the operator norm of the centred training matrix (Lemma~\ref{lem:matrix}), as does
Lemma~A.7 of \citet{attias2024}. With the ridge regularizer, the loss
$f_2(\theta,z)=\norm\theta^2/2-\ip\theta z$ is the strongly convex instance of
\citet[Section~7.1.1]{attias2024} restricted to the unit ball, and their Lemma~7.1 contains the
localization $\norm{\thh-\mu_D}^2\le2\eps$, which holds there with probability at least $1-\delta'$,
and their Lemma~7.3 the signed moment used here. For $f_2$, therefore, the moment bounds behind the lower
half of Theorem~\ref{thm:main} are of a known kind, and Theorem~\ref{thm:transfer} turns them into the
lower half at every noise level. The recall bounds of \citet[Theorem~3.5]{voitovych2025} rest on the
same two moments, the signed mean and the energy, which a counting inequality (their Lemma~A.11)
turns into recall. What Theorem~\ref{thm:main} adds lies elsewhere: the moment bounds for nonlinear
links and matched power curvature, where calibration through the upper growth and localization at
every law keep the train--ghost remainder \eqref{eq:taylor-main} below the signal, and
Proposition~\ref{prop:growth} shows that the upper growth cannot be dropped; the residual bound at
every law, which gives the upper half; and the two-sided dependence on $\tau$ that follows.
\end{remark}

\paragraph{Classical Tools.}
Besides fingerprinting, the proof uses classical tools: localization under uniform convexity, as in
uniformly convex stochastic optimization \citep{juditsky2014,daspremont2024}; the fair-bit inequality
\eqref{eq:fairbit}, which combines the lower bound of the Jensen--Shannon divergence by the
triangular discrimination \citep{topsoe2000}, in the normalization of \citet[Theorem~6]{luli2015},
with the Cauchy--Schwarz inequality; Gaussian maximum entropy; symmetrization and contraction
\citep{ledoux1991}, for the bound $4L/\sqrt n$; and the inequalities for power maps of
\citet[Chapter~12]{lindqvist2019}, for power regularizers.

\paragraph{Fingerprinting, Tracing and Privacy.}
Fingerprinting codes, as used by \citet{bun2014}, and the beta-prior fingerprinting lemma \citep{steinke2017},
which \citet{kamath2022} extend to exponential families, underlie lower bounds for private estimation
and selection \citep{bun2014,steinke2017,kamath2022}, lower bounds for private empirical risk
minimization \citep{bassily2014}, and the tracing lower bounds of \citet{attias2024} and
\citet{voitovych2025}. Corollary~\ref{cor:consequences}(ii) rests on the same mechanism. It has the
form of \citet[Theorem~3.6]{voitovych2025}, which bounds $e^{\epsdp}-1$ from below, for every
$(\epsdp,\delta)$-private learner of a prescribed accuracy, by a universal constant times the trace
value of the problem, less a term proportional to $\delta$; the trace value is defined through
tracers that read the model. In Corollary~\ref{cor:consequences}(ii), the test statistic
$\sum_i\sigma_iX_i$ is computed from the released loss differences and the selection bits alone, as
in one-run privacy auditing \citep{steinke2023auditing}, which includes each auditing example
independently with probability $1/2$, trains once and, in its input-space form, scores examples by their losses. For
$r=2$, the conclusion that $\epsdp=\Omega(1)$ unless $\delta\gtrsim\eps$ is consistent with the lower
bound of \citet[Theorem~5.5]{bassily2014} for private strongly convex empirical risk minimization
with $\delta=o(1/n)$ in dimension of order $n$; for convex Lipschitz losses, \citet{bassily2019} show
that, up to logarithmic factors, the optimal excess population loss of private algorithms is the
larger of the optimal non-private excess population loss and the optimal excess empirical loss of
private algorithms. For every $r\ge2$, the same
conclusion is consistent with the rates of \citet{asi2021} for private optimization of functions with
growth of power $r$ around the optimum, which they show to be optimal up to logarithmic factors: in
dimension of order $n$, these rates reach the non-private accuracy at the optimal sample size only when
$\epsdp=\Omega(1)$, up to such factors. Hypothesis-testing
formulations of privacy
\citep{dongrothsu2022,leemann2023} describe privacy by the tradeoff between the two error rates of a
test, which for \citet{leemann2023} is a test of membership.

\paragraph{Membership Inference.}
Information-theoretic and differential-privacy analyses of membership inference relate leakage to
the success of an attack: \citet{farokhi2020} measure membership leakage by a conditional mutual
information and relate its Kullback--Leibler form to the success probability of a hypothesis-testing
adversary, \citet{aubinais2025} identify the statistical quantity that governs the success of
membership inference attacks and bound it, and \citet{thudi2022} bound the precision of every
membership inference adversary against a differentially private learner. Among the attacks,
\citet{shokri2017} train shadow models and then an attack classifier on the confidence vectors that
the target model outputs; \citet{yeom2018} relate membership advantage to the generalization gap and
use an adversary that thresholds the loss; \citet{sablayrolles2019} show that, under
assumptions on the distribution of the parameters, Bayes-optimal membership inference depends only on
the loss; and \citet{carlini2022} use a per-example
likelihood-ratio test on loss-derived scores, calibrated with shadow models, which is loss-based but
not a single global threshold. Practical attacks are judged on natural data at low false-positive
rates \citep{carlini2022}. The balanced pairs of this paper are a different experiment, and its
results neither give nor contradict guarantees of that kind.

\paragraph{Memorization and Model-Level Limits.}
Other limits of information-based explanations concern what the model reveals about its sample. For
thresholds on a finite domain, every consistent proper learner reveals, for some realizable
distribution, a number of bits that grows doubly logarithmically with the size of the domain
\citep{bassily2018}; \citet{nachum2018} exhibit, for each VC dimension, a class with a bound of this
kind, of order the VC dimension times the double logarithm of the ratio of the domain size to
the VC dimension; and for every algorithm that learns thresholds on the line there are realizable
distributions on which its PAC-Bayes bound is arbitrarily large \citep{livni2020}.
\citet{feldman2025} prove, through strong data processing inequalities, that for several simple
binary classification problems an accurate learner must memorize information about its training
data, with a tradeoff between the number of samples and the amount of memorized information.
Theorem~\ref{thm:main} concerns a smaller observation, the loss differences, on the family of product
laws of Section~\ref{sec:setting}.

\paragraph{Other Information-Based Bounds.}
\citet{nadjahi2024} bound the generalization of algorithms that slice the parameter space, through
disintegrated and sliced mutual information, with a rate--distortion extension for compressible
models; \citet{wanglongmao2025} split the generalization error of federated learning into an
out-of-sample gap and a participation gap and bound both with a two-level CMI construction; and
\citet{sefidgaran2025} derive CMI bounds through stochastic projection and lossy compression, which
give guarantees of order $n^{-1/2}$ on instances where mutual-information and CMI bounds had failed. These bounds use
observations or sampling structures other than the release $X+\tau G$. Fast-rate
bounds \citep{grunwald2021,hellstrom2022} divide an information of the kinds bounded here by $n$, up to
constants of the loss, so at some law of the family they do not vanish either
(Section~\ref{sec:consequences}). Bounds that scale the information by a stability parameter
\citep{wangmao2023b} and bounds through the binary Jensen--Shannon divergence of a randomly binarized
loss \citep{dong2025} rest on other premises, and Theorem~\ref{thm:main} does not rule out useful
bounds of these kinds, nor those of \citet{nadjahi2024}, \citet{wanglongmao2025} and
\citet{sefidgaran2025}.

%% file: app_K_lean.tex
\section{Lean 4 formalization}\label{app:lean}

This appendix maps the results of the paper to the declarations of the Lean~4 formalization that
prove them, and states where the formal statements differ from the text.

% Ragged-right cell of a p-column; defined in app_J_related.tex, repeated here so that this file
% also compiles on its own. A row whose last cell starts with \raggedcell ends with \tabularnewline.
\providecommand{\raggedcell}{\raggedright\hspace{0pt}}
% Caption of the second part of Table K.1, typeset like \caption but with the same number.
\makeatletter
\newcommand{\continuedcaption}[1]{\@makecaption{\tablename~\thetable}{#1}}
\makeatother

\paragraph{The Project.}
The formalization is written in Lean~4 \citep{moura2021lean}, toolchain version 4.32.2, and uses
Mathlib \citep{mathlib2020} at commit \texttt{905b958}; both are released under the Apache
License~2.0. The Lean project is part of the supplementary archive; it also contains formal results that this paper
does not use, and Table~\ref{tab:lean} lists the declarations behind the results of the paper. Its script
\texttt{scripts/verify.sh} builds the project and then checks, for every declaration of the project,
the axioms on which it depends. Only the three standard axioms of Lean and Mathlib occur:
propositional extensionality (\texttt{propext}), the axiom of choice (\texttt{Classical.choice}) and
the soundness of quotients (\texttt{Quot.sound}). The same check rules out proofs left incomplete
with \texttt{sorry}.

\paragraph{Reading the Table.}
Table~\ref{tab:lean} lists the declarations that prove each result. Every name in the table is
declared exactly once in the project, so it identifies its declaration without a namespace. In
declaration names, the word \texttt{native} marks objects of the setting of
Section~\ref{sec:setting}: the problem class of Assumption~\ref{ass:class}, its learners and the
loss-difference release.

\paragraph{Conventions.}
Information is measured in bits: \leanname{mutualInfo} is the Kullback--Leibler divergence of the
joint law from the product of the marginals, divided by $\ln2$, with the convention that an infinite
divergence gives $0$, and \leanname{condMutualInfo} defines conditional information by the chain rule,
$I(A;B\mid C)=I(A;(B,C))-I(A;C)$. Every information in the formal statements is an information about
selection bits, which are uniform on a finite set; for these the divergence is finite
(\leanname{mutualInfoNats_ne_top_of_uniform}), so the convention never applies. The three
information measures of the release lie in $[0,n]$ at every law of the family, and the worst case
$\sup_p\Qc_{D_p}(X+\tau G)$ is their supremum over $p\in[-1,1]^d$ as a real number
(\leanname{nativeWorstInformation}). The class \leanname{nativeAccurateLearners} consists of the
jointly measurable proper learners that are $\eps$-accurate on the family, as in
Section~\ref{sec:setting}. The formal statements also carry the measurability and integrability
conditions that the text leaves implicit.

\paragraph{Where the Formal Statements Differ from the Text.}
Theorem~\ref{thm:main} asserts that suitable $\eps_0$, $c_1$ and $C_1$ exist. Its formal statements
provide explicit ones of their own: the threshold \leanname{mainThreshold}, a minimum of five
explicit quantities, and the constants \leanname{mainLowerConstant} and \leanname{mainUpperConstant};
the dimension multiplier is the $\Cdim$ of Theorem~\ref{thm:main}. The results that assume the
hypotheses of Theorem~\ref{thm:main}, namely the orders in Proposition~\ref{prop:finite},
Corollary~\ref{cor:consequences}, the order of the second-moment bound on prior average,
Remarks~\ref{rem:reflection} and~\ref{rem:strength} and the model-level converse, are formalized with
the same threshold and constants. The threshold $\eps_1$ and the constants of
Appendix~\ref{app:constants} are those of \leanname{native_loss_information_law_paper_threshold},
which takes the conditions on $\eps_1$ as hypotheses and uses the finite dimension condition
$d\ge32\max\{n,1/v\}$. The forms \leanname{native_loss_information_law_local},
\leanname{native_loss_information_prior_averaged_local} and \leanname{finite_native_budgets_local}
assume only Assumption~\ref{ass:class}. Where they need it, the other statements that involve the
link or the regularizer assume in addition that they are measurable on all of $\R$ and $\R^d$; only
the values of $\phi$ on $[-1,1]$ and of $g$ on $W$ enter the results, and the local forms are derived
from the measurable case by replacing $\phi$ and $g$ with measurable functions that agree with them
there.
Corollary~\ref{cor:consequences}(i) is formalized for the unconditional terms $I(X_i;U_i)$. In
Proposition~\ref{prop:separation}(v), both bounds are formalized for $\Is$. The upper bound then holds
for $\Ii$ and $\Ij$ by \eqref{eq:hierarchy}; the lower bound $n/8$ for $\Ii$, and hence for $\Ij$,
follows from the formalized part~(iv) at $\tau=n^{-2/3}$ by the arithmetic in the proof of part~(v).
The upper clause of Proposition~\ref{prop:samples} is formalized as the existence of a proper learner
with these properties; the learner constructed in its proof is epoch SGD. In the formal statement of
Proposition~\ref{prop:noreg}, the constant $c_+$ is called $N$, and the threshold
\leanname{unregErrorThreshold} is written as
$\min\{b/4,\,b^2/(16(H+1)),\,b^2/(32(H+1)(2\sqrt N/b+2))\}$; this equals the threshold of the
text, because the middle term always exceeds the last. The lower bound $n\ge b^2/(1024\eps^2)$ of
Appendix~\ref{app:noreg} is \leanname{unreg_endpoint_sample_necessity} applied with $N=b^2/1024$, and
\leanname{unregularized_native_sgd_target} asserts that a proper learner that is $\eps$-accurate on
every law on the unit ball exists when $n\eps^2\ge L^2$; the learner constructed in its proof is
projected SGD.

\paragraph{What Is Not Formalized.}
The measurements reported in Section~\ref{sec:measurements} and Appendix~\ref{app:measurements}, and
the orders in the example of Section~\ref{sec:result}; the conditional mean in
Lemma~\ref{lem:ridge-moments}(a) and the exact identities of its parts (c) and~(d), of which
Table~\ref{tab:lean} lists bounds; the ratios of constants, the values for ridge and logistic ridge, and the sample sizes at which $M$
becomes positive in Appendix~\ref{app:constants}, wherever they are quoted; the descriptions of prior work; three one-line steps, namely the bounds
$\Ii(Y)\le\sum_iI(Y_i;U_i\mid\Zt)\le\Is(Y)$ of Section~\ref{sec:setting}, the first inequality of
\eqref{eq:directions}, and the inequality $I(X_i;U_i\mid\Zt)\ge I(X_i;U_i)$ in the proof of
Corollary~\ref{cor:consequences}(i); the middle term of \eqref{eq:fairbit}; the vanishing of the three
information measures in Appendix~\ref{app:zero-slope}; in Appendix~\ref{app:converse}, the
individual-sample and prior-average forms of the model-level converse and the comparison with
\citet{attias2024}; the unnumbered remarks on noise at the scale of the range
(Section~\ref{sec:consequences} and Appendix~\ref{app:noisy-bound}) and on rounding without noise
(Section~\ref{sec:necessity} and Appendix~\ref{app:dimension}); the form of Proposition~\ref{prop:noreg} in terms of $n$
(Appendix~\ref{app:noreg}); the constant $0.007$ of
Proposition~\ref{prop:growth}, whose formal statement has the smaller constant $c_2/2$; and the
correspondence between the statements of the text and the formal statements, which rests on reading
them side by side.

\begin{table}[tp]
\caption{Results of the Paper and the Lean Declarations That Prove Them}\label{tab:lean}
\centering\small
\begin{tabular}{@{}p{0.22\linewidth}p{0.74\linewidth}@{}}
\toprule
\raggedcell Result & \raggedcell Declarations\tabularnewline
\midrule
\raggedcell Assumption~\ref{ass:class}
& \raggedcell \leanname{NativeLinkClass}, \leanname{NativeRegularizerClass}\tabularnewline
\raggedcell Section~\ref{sec:setting}
& \raggedcell \leanname{signProductLaw} (the law $D_p$), \leanname{superMeasure} (the supersample),
\leanname{nativeAccurateLearners} (the learners of Theorem~\ref{thm:main})\tabularnewline
\raggedcell Definition~\ref{def:release}
& \raggedcell \leanname{nativeLossRelease}; \leanname{releaseInfoI}, \leanname{releaseInfoJ},
\leanname{releaseInfoS}, built on \leanname{mutualInfo} and \leanname{condMutualInfo}\tabularnewline
\raggedcell Display~\eqref{eq:hierarchy}, Lemma~\ref{lem:directions}
& \raggedcell \leanname{information_directions_supersample},
\leanname{information_directions_tracer_supersample}\tabularnewline
\addlinespace[2pt]
\raggedcell Theorem~\ref{thm:main}
& \raggedcell \leanname{native_loss_information_law_local} (Assumption~\ref{ass:class} as stated),
\leanname{native_loss_information_law} ($\phi$ and $g$ measurable),
\leanname{native_loss_information_law_infimum} (infimum over a class),
\leanname{native_loss_information_prior_averaged_local} (the lower bound on average over the
prior); \leanname{nativeWorstInformation} (the worst case);
\leanname{mainDimensionConstant_eq_of_le_four} ($\Cdim=32$ when $b\le4$)\tabularnewline
\raggedcell Appendix~\ref{app:constants}
& \raggedcell Theorem~\ref{thm:main} at $\eps_1$:
\leanname{native_loss_information_law_paper_threshold}. The bound \eqref{eq:explicit-threshold} and
the positivity of $M$ on $(0,\eps_1]$: \leanname{explicit_threshold_bound},
\leanname{paperThresholdBracket_anti}, \leanname{native_signed_budget_threshold}. The bracket
\eqref{eq:explicit-bracket}: \leanname{explicit_bracket}. The limits \eqref{eq:explicit-limits} and
the bounds over the regime: \leanname{explicit_schedule_limits}, \leanname{explicit_regime_limits}.
Monotonicity in $d$: \leanname{native_budgets_mono_dim}. Quartic logistic learning, the limits and
the threshold $6.0\cdot10^{-4}$: \leanname{quartic_schedule_limits},
\leanname{quartic_bracket_root}\tabularnewline
\raggedcell Proposition~\ref{prop:samples}
& \raggedcell \leanname{sample_complexity_proposition}, \leanname{nativeAccurateLearners_nonempty}
(the class of Theorem~\ref{thm:main} is nonempty)\tabularnewline
\raggedcell Appendix~\ref{app:links}
& \raggedcell \leanname{logisticLoss_nativeLinkClass}, \leanname{exponentialLoss_nativeLinkClass},
\leanname{squaredMarginLoss_nativeLinkClass}, \leanname{boundaryLink_class} (the links of
Table~\ref{tab:links}); \leanname{balanced_label_native_loss} (the loss identity for balanced labels)\tabularnewline
\raggedcell Appendix~\ref{app:power}
& \raggedcell \leanname{anisotropicRegularizer_class}; \leanname{hessian_regularizer_class} (the
Hessian criterion, with the Hessian bounds assumed on all of $\R^d$)\tabularnewline
\raggedcell Remark~\ref{rem:reflection}
& \raggedcell \leanname{nativeLinkClass_reflect_of_pos_slope},
\leanname{reflectLearner_mem_nativeAccurateLearners}, \leanname{nativeNoisyInformation_reflect}
(parts (i)--(iii)); \leanname{native_loss_information_law_nonzero_slope},
\leanname{native_loss_information_law_infimum_reflected},
\leanname{main_information_prior_averaged_reflected}\tabularnewline
\raggedcell Remark~\ref{rem:strength}
& \raggedcell \leanname{NativeRegularizerClass.const_mul_uniform},
\leanname{native_loss_information_law_scaled}\tabularnewline
\raggedcell Proposition~\ref{prop:noreg}
& \raggedcell \leanname{unreg_native_noise_law} (the proposition, with the threshold
\leanname{unregErrorThreshold}), \leanname{unreg_information_upper_profile} (the upper bound at every
$p$, for every proper learner); \leanname{unreg_endpoint_sample_necessity} (the lower bound on the
sample size), \leanname{unregularized_native_sgd_target} (projected SGD)\tabularnewline
\addlinespace[2pt]
\raggedcell Corollary~\ref{cor:consequences}(i)
& \raggedcell \leanname{loss_range_certificate_nonvanishing};
\leanname{native_learner_expected_generalization_le},
\leanname{native_uniform_expected_generalization_regularized} (the bound
$4L/\sqrt n$)\tabularnewline
\raggedcell Second-moment bound \eqref{eq:second-moment}
& \raggedcell \leanname{abs_gen_le_perPair}, \leanname{abs_gen_le_ldCertificate};
\leanname{rms_certificate_tight_optimal_n} (its order on prior average)\tabularnewline
\raggedcell Corollary~\ref{cor:consequences}(ii)
& \raggedcell \leanname{dp_impossibility_lower_constant_two_L_sub}, \leanname{dp_order_two_L},
\leanname{dp_impossibility_normalized}, \leanname{dp_impossibility_normalized_two_L_div};
\leanname{link_oscillation_iSup_le} (the oscillation of $\phi$ on $[-1,1]$ is at most
$2L$)\tabularnewline
\raggedcell Lemma~\ref{lem:dp-forms}, Proposition~\ref{prop:dp-pair}, Corollary~\ref{cor:dp-audit}
& \raggedcell \leanname{integral_le_of_dpBound}, \leanname{integral_sub_le_of_dpBound},
\leanname{test_advantage_le_of_dpBound}; \leanname{dp_signed_release_le};
\leanname{dp_budget_combination}, \leanname{dp_normalized_budgets} (the normalized form, for the
budgets of Proposition~\ref{prop:finite})\tabularnewline
\bottomrule
\end{tabular}
\end{table}

\begin{table}[tp]
\continuedcaption{Results of the Paper and the Lean Declarations That Prove Them (Continued)}
\centering\small
\begin{tabular}{@{}p{0.22\linewidth}p{0.74\linewidth}@{}}
\toprule
\raggedcell Result & \raggedcell Declarations\tabularnewline
\midrule
\raggedcell Theorem~\ref{thm:transfer}
& \raggedcell \leanname{transfer_theorem_full}, \leanname{transfer_theorem_all_noise} (both parts,
with the decoder of Appendix~\ref{app:decoder}); \leanname{prior_averaged_information_lower_bound}
(part (a), for any independent centred noise of variance $\tau^2$);
\leanname{release_information_upper_bound} (part (b))\tabularnewline
\raggedcell Inequality~\eqref{eq:fairbit}
& \raggedcell \leanname{fair_bit_moment_bound} (variance form),
\leanname{fair_bit_moment_bound_second_moment} (second-moment form)\tabularnewline
\raggedcell Appendix~\ref{app:decoder}, sharpness
& \raggedcell \leanname{erasure_exact_moments}, \leanname{erasure_clipping_attains},
\leanname{erasure_joint_decoder_upper}\tabularnewline
\raggedcell Remark~\ref{rem:noisy-model}
& \raggedcell \leanname{noisy_model_information_le}, \leanname{noisy_model_information_order}
\tabularnewline
\raggedcell Proposition~\ref{prop:finite}
& \raggedcell \leanname{finite_native_budgets_local}, \leanname{finite_native_budgets_unconditional}
(the finite bounds); \leanname{main_recipe_orders} (their orders)\tabularnewline
\raggedcell Lemmas~\ref{lem:fp}, \ref{lem:matrix} and~\ref{lem:weighted}
& \raggedcell \leanname{beta_fingerprinting_randomized}; \leanname{matrix_energy_tail},
\leanname{matrixEnergyK_one_div_thirty_two_lt} ($K_0^*<2.031$);
\leanname{weighted_tail_bound}\tabularnewline
\addlinespace[2pt]
\raggedcell Proposition~\ref{prop:growth}
& \raggedcell \leanname{matched_growth_counterexample} (the curvature and growth of $g$),
\leanname{meanERM_accuracy} (the accuracy of the mean learner);
\leanname{counter_every_learner_lower} (every accurate learner, with the constant $c_2/2$
of Appendix~\ref{app:growth} in place of $0.007$), \leanname{counter_meanERM_upper} (the mean learner),
\leanname{counter_order_tight} (both); \leanname{oversamplingPowerLowerConstant_two_gt}
($c_2>2\cdot10^{-5}$)\tabularnewline
\raggedcell Lemma~\ref{lem:powerlinear}
& \raggedcell \leanname{oversampling_finite_information}\tabularnewline
\raggedcell Zero slope (Appendix~\ref{app:zero-slope})
& \raggedcell \leanname{zero_slope_population_optimal},
\leanname{zero_learner_loss_difference}\tabularnewline
\raggedcell Proposition~\ref{prop:dimension}
& \raggedcell \leanname{linear_dimension_information_le} (the bound),
\leanname{dimension_needed_of_worst_case} (the dimension it forces)\tabularnewline
\addlinespace[2pt]
\raggedcell Proposition~\ref{prop:separation}
& \raggedcell \leanname{boundary_proposition} (parts (i)--(v)), \leanname{boundary_main_text}
(summary)\tabularnewline
\raggedcell Lemma~\ref{lem:sign}
& \raggedcell \leanname{zero_noise_decoder_success} (the average success of the decoder),
\leanname{zero_noise_information_transfer}, \leanname{measureReal_accept_le_of_fresh_point} (the
ghost is a fresh point); \leanname{zero_noise_information_transfer_supersample} (the bound in the
supersample model)\tabularnewline
\raggedcell Lemmas~\ref{lem:power-accuracy} and~\ref{lem:full-rank}
& \raggedcell \leanname{powerLinearRisk_clippedPowerMap_sub_eq},
\leanname{powerLinearRisk_clippedPowerMap_sub_le} (Lemma~\ref{lem:power-accuracy}(b)),
\leanname{integral_powerLinearRisk_clippedPowerMap_average_sub_le}
(Lemma~\ref{lem:power-accuracy}(c)); \leanname{boundary_noiseless_mask_determined}
(Lemma~\ref{lem:full-rank}(b) for $r=2$ and $r=4$: outside an event of probability at most
$e^{-225n/32}$, the release determines $U$)\tabularnewline
\raggedcell Model-level converse (first display of Appendix~\ref{app:converse})
& \raggedcell \leanname{native_model_information_lower}\tabularnewline
\raggedcell Lemma~\ref{lem:ridge-moments}
& \raggedcell \leanname{ridge_release_eq} (the decomposition in part (a)), \leanname{ridge_signed_eq}
(part (b)); \leanname{ridge_residual_le}, \leanname{ridge_energy_gap_bounds} (bounds that follow from
parts (c) and (d))\tabularnewline
\bottomrule
\end{tabular}
\end{table}